\documentclass[twoside]{article}

\usepackage{hyperref}  
\usepackage[accepted]{aistats2019}
\usepackage{amsfonts}
\usepackage{paralist}
\ifdefined\usebigfont

\usepackage[fontsize=13pt]{scrextend}
\usepackage[round]{natbib}

\else
\fi
\usepackage{times}
\usepackage[round]{natbib}

\usepackage{amsthm}
\usepackage{graphicx}
\usepackage{amsmath}
\usepackage{amssymb}
\usepackage[dvipsnames]{xcolor} 
\usepackage{thmtools,thm-restate}
\usepackage{booktabs}
\usepackage{caption}
\usepackage{nicefrac}
\usepackage{enumitem}
\usepackage{float}
\usepackage{subfig}

\newcommand{\w}{\mathbf{w}}
\newcommand{\e}{\exp}
\newcommand{\R}{\mathbb{R}}
\newcommand{\vect}[1]{\mathbf{#1}}
\newcommand{\what}{\hat{\mathbf{w}}}
\newcommand{\wtilde}{\tilde{\mathbf{w}}}
\newcommand{\wcheck}{\bar{\mathbf{w}}}
\newcommand{\wcheckSec}{{\wcheck}_2}
\newcommand{\htilde}{\tilde{h}}
\newcommand{\ftilde}{\tilde{f}}
\newcommand{\wvec}{{\mathbf{w}}}
\newcommand{\rvec}{{\mathbf{r}}}
\newcommand{\bm}[1]{\ensuremath{\boldsymbol{#1}}}
\newcommand{\rhoVec}{\bm{\rho}(t)}
\newcommand{\whatNorm}{\Vert\what\Vert}
\newcommand{\wNorm}{\Vert \wvec(t)\Vert}
\newcommand{\x}{{\mathbf{x}}}

\newcommand{\innerprod}[2]{#1^{\top}#2}
\newcommand{\cL}{\mathcal{L}}
\renewcommand{\c}{\mathcal}
\newcommand{\norm}[1]{\left\lVert#1\right\rVert}
\newcommand{\abs}[1]{\left\lvert#1\right\rvert}
\renewcommand{\b}{\mathbb}

\newcommand{\genericConstVec}{{\bm{a}}}

\newcommand{\xn}{{\mathbf{x}_n}}
\newcommand{\xk}{{\mathbf{x}_k}}
\newcommand{\xnT}{{\mathbf{x}_n^\top }}
\newcommand{\xkT}{{\mathbf{x}_k^\top }}
\newcommand{\sumn}{\sum\limits_{n=1}^N}
\newcommand{\sumnsv}{\sum\limits_{n\in \mathcal{S}}}
\newcommand{\sumnnsv}{\sum\limits_{n \notin \mathcal{S}}}
\newcommand{\set}{\mathcal{S}}
\newcommand{\sumnnsvp}{\sum\limits_{ n \notin \mathcal{S} \atop  \xnT \rvec(t)\ge 0}}
\newcommand{\derL}{\nabla \mathcal{L}}
\newcommand{\lossPower}{\nu}
\DeclareMathOperator*{\argmin}{\arg\!\min}

\newcommand{\fsym}{\psi}
\newcommand{\fsymSec}{\zeta}
\newcommand{\fsymThird}{\tau}
\newcommand{\eqw}{\vect{w}} 
\newcommand{\np}{\mathcal{W}} 
\newcommand{\lw}[1]{\vect{W}_{#1}}
\newcommand{\lwinfty}[1]{\overline{\vect{W}}_{#1}^{\infty}}
\newcommand{\winfty}{\overline{\vect{W}}_{\infty}}
\newcommand{\bdelta}{\bm{\delta}}
\newcommand{\note}[1]{{}}
\newcommand{\ip}[2]{\left\langle#1,#2\right\rangle}
\renewcommand{\P}{\mathcal{P}}
\newcommand{\winf}{\bar{\wvec}_{\infty}}
\newcommand{\op}{\vect{Q}} 
\newcommand{\setinf}{\mathcal{S}^\prime}
\newcommand{\sumnsvinf}{\sum\limits_{n\in \setinf}}
\newcommand{\sumnnsvinf}{\sum\limits_{n \notin \setinf}}

\newtheorem{theorem}{Theorem}
\newtheorem{assumption}{Assumption}
\newtheorem{remark}{Remark}
\newtheorem{lemma}{Lemma}

\newtheorem{definition}{Definition}
\newtheorem{claim}{Claim}
\newtheorem{corollary}{Corollary}

\newcommand{\dnote}[1]{}
\newcommand{\mnote}[1]{}
\newcommand{\jnote}[1]{}
\newcommand{\snote}[1]{}

\newcommand{\remove}[1]{{}}

\begin{document}
\twocolumn[

\aistatstitle{Convergence of Gradient Descent on Separable Data}
\runningauthor{Mor Shpigel Nacson, Jason D. Lee, Suriya Gunasekar, Pedro H. P. Savarese, Nathan Srebro, Daniel Soudry}

\aistatsauthor{ Mor Shpigel Nacson\textsuperscript{1} \ \ \  Jason D. Lee\textsuperscript{2} \ \ \  Suriya Gunasekar\textsuperscript{3} \ \ \  Pedro H. P. Savarese\textsuperscript{3} \ \ \ Nathan Srebro\textsuperscript{3} \ \ \ Daniel Soudry\textsuperscript{1}}

\aistatsaddress{ \textsuperscript{1}Technion, Israel,\ \ \ \  \textsuperscript{2}USC Los Angeles, USA,\ \ \ \ \textsuperscript{3}TTI Chicago, USA} 
]

\begin{abstract}
We provide a detailed study on the implicit bias of gradient descent when optimizing loss functions with strictly monotone tails, such as the logistic loss, over separable datasets. We look at two basic questions: (a) what are the conditions on the tail of the loss function under which gradient descent converges in the direction of the $L_2$ maximum-margin separator? (b) how does the rate of margin convergence depend on the tail of the loss function and the choice of the step size? We  show that for a large family of super-polynomial tailed losses, gradient descent iterates on linear networks of any depth converge in the direction of $L_2$ maximum-margin solution, while this does not hold for losses with heavier tails. Within this family, for simple linear models we show that the optimal rates with fixed step size is indeed obtained for the commonly used exponentially tailed losses such as logistic loss. However, with a fixed step size the optimal convergence rate is extremely slow as $1/\log(t)$, as also proved in \citet{soudry2017implicit}. 
For linear models with exponential loss, we further prove that the convergence rate could be improved to $\log (t) /\sqrt{t}$  by using aggressive step sizes that compensates for the rapidly vanishing gradients. Numerical results suggest this method might be useful for deep networks.

\remove{
The implicit bias of gradient descent is not fully understood even in simple linear classification tasks (e.g., logistic regression). \cite{soudry2017implicit} studied this bias on separable data, where there are multiple solutions that correctly classify the data. It was found that, when optimizing monotonically decreasing loss functions with exponential tails using gradient descent, the linear classifier specified by the gradient descent iterates converge to the $L_2$ max margin separator. 
However, the convergence rate to the maximum-margin solution with fixed step size was found to be extremely slow: $1/\log(t)$. 

Here we examine how the convergence is influenced by (a) using  different loss functions and (b) by  using variable step sizes. 
First, we calculate the convergence rate for loss functions with poly-exponential tails near $\exp(-u^{\nu})$. We prove that $\nu=1$ yields the optimal convergence rate in the range $\nu>0.25$. Based on further analysis we conjecture that this remains the optimal rate for $\nu \leq 0.25$, and even for sub-poly-exponential tails --- until loss functions with polynomial tails no longer converge to the max margin. 
Second, we prove the convergence rate could be improved to $(\log t) /\sqrt{t}$ for the exponential loss, by using aggressive step sizes which compensate for the rapidly vanishing gradients. }
\end{abstract}

	\section{INTRODUCTION}
    In learning over-parameterized models, where the training objective has multiple global optima, each optimization algorithm can have a distinct implicit bias. Hence, different algorithms learn different models with  different generalization to the population loss. This effect of the implicit bias of the optimization algorithm on the learned model is particularly prominent in deep learning, where the generalization or the inductive bias is not sufficiently driven by explicit regularization or restrictions on the model capacity \citep{neyshabur2015search,zhang2017understanding,Hoffer2017}. Thus, in order to understand what is the true inductive bias in such high capacity models, it is important to rigorously understand how optimization affects the implicit bias. 
    
    Consider learning a homogeneous linear predictor $\x\to\innerprod{\w}{\x}$ using  unregularized logistic regression over separable data. For this problem, \citet{soudry2017implicit}  showed that the gradient descent iterates converge in direction to the maximum-margin separator with unit $L_2$ norm, and this implicit bias holds independently of initialization and step size (given the step size is small enough). This is exactly the solution of the homogeneous hard margin support vector machine (SVM) where the $L_2$ norm constraint on the parameters $\w$ is explicitly added. 
    More surprisingly, 
    \citet{soudry2017implicit} also showed that the rate of convergence to the maximum-margin solution is $O(1/\log{(t)})$. This is much slower compared to the rate of convergence of the loss function itself, which is shown to be $O(1/t)$. This implies that the classification boundary of logistic regression, and hence the generalization of the classifier, continues to change long after the  $0$-$1$ error on training examples has diminished to zero, or the logistic loss is very small. 
    In a follow up work, \citet{gunasekar2018conv} showed that for exponential loss, gradient descent on fully connected deep linear networks also has the same bias asymptotically. However, the convergence rates were not analyzed in this work on deep linear networks. 
    
    Despite this recent line of interesting results, the implicit bias of gradient descent is not entirely understood even in simple linear classification tasks. For example, the analysis of \citet{soudry2017implicit} and \citet{gunasekar2018conv} crucially relied on strict monotonicity of the loss function to get an initialization--independent  characterization of the bias of gradient descent. However, in these works the results are derived specifically for tight exponential tailed losses and exponential loss, respectively. While exponential tailed losses such as logistic and cross entropy losses are indeed the most widely used losses in training deep neural networks, we do not yet know: Do such losses with tight exponential tail have a special significance? Can a similar convergence to maximum-margin separator be achieved by other strictly monotonic losses? How is the rate of convergence to maximum-margin solution affected by the tail? Are there other ways to accelerate the convergence?
    
    \remove{Here we provide a more detailed study of this problem, focusing on the rate of convergence of gradient descent to the maximum-margin solution.  First, we ask whether this result can be extended to different loss functions, with different tails, beyond the tight exponential tail of logistic loss or exponential loss: do we still get convergence to the $L_2$ maximum-margin separator? Does a heavier or lighter tail gives a faster rate of convergence?  
    
    We show that convergence to the $L_2$ maximum-margin solution can be extended to various losses with faster than polynomial tails, but not to losses with polynomial tails. However, our analysis suggests that the (popular) exponential tail is optimal in terms of the rates. We then focus on the optimal case of the exponential loss and ask whether we can accelerate the convergence to the maximum-margin by using more aggressive and variable step sizes.  The answer is yes, and we show that using normalized gradient updates, i.e. step size proportional to the inverse gradient, we can get rates as fast as $O(\log{t}/\sqrt{t})$ instead of $1/\log{t}$. Preliminary numerical results suggests we might also be able to improve similarly the convergence rates for deep networks.   
    }
    
 Here we provide a detailed study of this problem, focusing on the rate of convergence of the margin:
\begin{compactenum}
\item \textit{What are the conditions on the tail of the loss function under which gradient descent converges to the $L_2$ maximum-margin separator?} We show that convergence to the $L_2$ maximum-margin solution can be extended to losses with super polynomial tails, but not to losses with (sub) polynomial tails.
\item \textit{Does a heavier or lighter tail gives a faster rate of convergence?} In our analysis, losses with exponential tails, which include the commonly used logistic loss, can indeed be shown to have the optimal rate of convergence of the margin. 
\item \textit{Extensions to deep linear networks.} We show that similar analysis and the same asymptotic rates hold more generally for linear networks with fully connected layers. Interestingly, the results suggest that increasing the number of layers (depth)  decreases the convergence rate only marginally, even in the limit of infinite depth.
\item \textit{For exponential loss, which obtains the optimal margin convergence rate, can we accelerate the convergence to the maximum-margin by using variable step sizes?} The answer is yes, and we show that using normalized gradient updates, i.e., step size proportional to the inverse gradient, we can get a much faster rate of $O(\log{t}/\sqrt{t})$ instead of $1/\log{t}$. Experimental results suggest this improvement in rate over standard gradient descent  might also extend for non-linear neural networks. 
 \end{compactenum}

    \remove{
        \section{Setup and review of previous results \label{sec: prev-results}} 
    	\note{
    	\[
    	    \eqw(t)=\P(\np(t))=\lw{1}(t)\cdot...\cdot \lw{L}(t)\in \R^d
        \]
    	\begin{equation} \label{eq: Loss for net parameters}
    	    \mathcal{L}_\P\left(\np(t)\right) = \sumn \ell\left(y_n\ip{\P\left(\np(t)\right)}{\xn}\right)
    	\end{equation}
    	\begin{equation} \label{eq: GD dynamic}
    	    \np\left(t+1\right) = \np\left(t\right) - \eta\nabla_{\np}\mathcal{L}_{\P}\left(\np(t)\right)
    	\end{equation}
    	\begin{equation}
    	    \mathcal{L}\left(\eqw(t)\right) = \sumn \ell\left(y_n\ip{\eqw(t)}{\xn}\right)
    	\end{equation}
    	}
    	Consider a dataset $\left\{ \mathbf{x}_{n},y_{n}\right\} _{n=1}^{N}$,
        with binary labels $y_{n}\in\left\{ -1,1\right\} $. We denote the data matrix $\mathbf{X}=[\x_1,\dots,\x_N]$ and  $\sigma_{\max}\left(\text{\ensuremath{\mathbf{X}} }\right)$ as the maximal singular value of $\mathbf{X}$. We assume, without loss of generality that $\sigma_{\max}\left(\text{\ensuremath{\mathbf{X}} }\right) \leq 1$ and $\forall n: ||\x_n||<1$, where $||\cdot ||$ denotes the $L_2$ norm.
        
        We analyze learning by minimizing an empirical loss of the form 
        \begin{equation}
     \mathcal{L}\left(\mathbf{w}\right)=\sum_{n=1}^{N}\ell\left(y_{n}\mathbf{w}^{\top}\mathbf{x}_{n}\right)\,,\label{eq: general loss functions}
        \end{equation}
        where $\mathbf{w}\in\mathbb{R}^{d}$ is the weight vector. A bias
        term could be added in the usual way, extending $\mathbf{x}_{n}$
        by an additional `1' component. To simplify notation, we assume that
        $\forall n:\,y_{n}=1$ \textemdash{} this is  without loss of
        generality, since we can always re-define $y_{n}\mathbf{x}_{n}$ as
        $\mathbf{x}_{n}$.
        
    	The gradient descent (GD) iteration with fixed step size $\eta$ is given by
    	\begin{equation} \label{GD}
    	\wvec(t+1) = \wvec(t) - \eta \nabla \mathcal{L}(\wvec(t)) = \wvec(t) - \eta \sumn \ell'(\wvec(t)^\top \xn)\xn.
    	\end{equation}
    	
        We look at the iterates of GD on linearly separable datasets  with monotonic loss functions. 
    	\begin{definition} \label{def: the data is separable} [Linear separability]
    		The dataset is linearly separable if there exists a separator $\wvec_*$ such that $\forall n:\ \wvec_*^\top\xn>0$.
    	\end{definition}
    	\begin{definition} \label{def: l(u) assumptions}[Strict Monotone Loss]
    		$\ell(u)$ is a differentiable monotonically decreasing function bounded from below. Without loss of generality, let $\forall u,\,\ell\left(u\right)>0,\ \ell^{\prime}\left(u\right)<0$
    		and $\lim_{u\rightarrow\infty}\ell\left(u\right) =\lim_{u\rightarrow\infty}\ell^{\prime}\left(u\right)=0$ and $\lim_{u\rightarrow-\infty}\ell^{\prime}\left(u\right)\neq0$. 
    	\end{definition}
        \remove{
    We restate Lemma 1 from \cite{soudry2017implicit}, which gives the basic asymptotic behavior in this setting:
    
        \begin{lemma} \label{Lemma: w(t)->infty}
        Let $\mathbf{w}\left(t\right)$
        be the iterates of gradient descent (eq.~\eqref{GD})
        on a $\beta$-smooth $\mathcal{L}$ and any starting point $\wvec(0)$. If the data is linearly separable (Definition \ref{def: the data is separable}), $\ell$ is a strict monotone loss (Definition \ref{def: l(u) assumptions}), and $\eta<2\beta^{-1}$ then we have: (1) $\lim_{t\rightarrow\infty}\mathcal{L}\left(\mathbf{w}\left(t\right)\right)=0$,
    	(2) $\lim_{t\rightarrow\infty}\left\Vert \mathbf{w}\left(t\right)\right\Vert =\infty$, and
    	(3) $\forall n:\,\lim_{t\rightarrow\infty}\mathbf{w}\left(t\right)^{\top}\mathbf{x}_{n}=\infty$.
        \end{lemma}}
         For strictly monotone losses over separable data, there are no finite global minima of the objective in eq.~\eqref{eq: general loss functions}, and gradient descent iterates will diverge to infinity. While the norm of the iterates $\norm{\w(t)}$ diverges, the classification boundary is entirely specified by the direction of $\w(t)/\norm{\w(t)}$. Can we say something interesting about which direction the iterates $\w(t)$ converge to? \cite{soudry2017implicit}  characterized this direction for loss function with exponential tails, defined below,
      
        	\begin{definition} \label{def: exponential tail} [Tight Exponential Tail]	A function $f(u)$ has a ``tight exponential tail", if there exist positive constants $\mu_+,\mu_-$, and $\bar{u}$ such that $\forall u>\bar{u}$:
    		$$(1-\e(-\mu_- u))e^{-u} \le f(u) \le (1+\e(-\mu_+ u))e^{-u} $$ 
            
    	\end{definition}
        
        \begin{definition} \label{def: beta smooth} [$\beta$-smooth function]	A function is $\beta$-smooth if its derivative is Lipschitz continuous with a Lipschitz constant $\beta$.
    	\end{definition}
    
        \begin{theorem}[Theorem 3 in \cite{soudry2017implicit}, rephrased]\label{theorem: ICLR theorem 3} For almost all datasets that are linearly separable (Definition \ref{def: the data is separable}), and any $\beta$-smooth $\mathcal{L}$ (Definition \ref{def: beta smooth}) with a strictly monotone loss function $\ell$ (Definition \ref{def: l(u) assumptions}), for which $-\ell^{\prime}$ has a tight exponential tail (Definition \ref{def: exponential tail}), the gradient descent iterates in eq.~\eqref{GD}  with any step size $\eta<2\beta^{-1}$ and any initialization $\w(0)$ will behave as:
        \begin{equation}
        \mathbf{w}\left(t\right)=\hat{\mathbf{w}}\log t+\boldsymbol{\rho}\left(t\right),\label{eq: asymptotic form}
        \end{equation}
        where the residual $\boldsymbol{\rho}\left(t\right)$ is bounded 
        and $\hat{\mathbf{w}}$ is the following $L_{2}$ maximum-margin separator:
        \begin{equation}\label{w_hat equation}
        \hat{\mathbf{w}}=\underset{\mathbf{\mathbf{w}}\in\mathbb{R}^{d}}{\mathrm{argmin}}\left\lVert \mathbf{w}\right\rVert^2 \,\,\mathrm{s.t.}\,\,\mathbf{w}^{\top}\mathbf{x}_{n}\geq1.
        \end{equation}
        \end{theorem}
        
        Theorem \ref{theorem: ICLR theorem 3} holds for common classification loss functions, including the logistic loss and the exponential loss\footnote{Note that for exp-loss, $\mathcal{L}(\wvec)$ does not have a global $\beta$ smoothness parameter. However, if we initialize with $\eta < 1/ \mathcal{L}(\w(0))$ then it is straightforward to show the gradient descent iterates maintain bounded local smoothness $\beta(t) \leq \mathcal{L}(\w(t)) \leq  \mathcal{L}(\w(0))$, so we will have $\eta < \beta^{-1}(t)$ for all iterates.} and for all datasets except a measure zero set (e.g., with probability 1, for any dataset sampled from an absolutely continuous distribution). \cite{soudry2018journal} generalized this Theorem to include also the measure zero cases. Theorem \ref{theorem: ICLR theorem 3} implies logarithmically slow convergence in direction to the $L_2$ max-margin separator
            \begin{equation} \label{eq: direction convergence}
            \left \Vert \frac{\mathbf{w}\left(t\right)}{\left\Vert \mathbf{w}\left(t\right)\right\Vert}-\frac{\hat{\mathbf{w}}}{\left\Vert \hat{\mathbf{w}}\right\Vert}\right\Vert = O\left(\frac{1}{\log t}\right)
            \end{equation} 
            and in the margin
            \begin{equation}\label{eq: margin convergence}
             \min_{n} \frac{\w(t)^{\top} \x_n }{\norm{\w (t)}} = \gamma - O\left(\frac{1}{\log t}\right), \,\,\, \mathrm{with} \,\,\, \gamma=\max_{\w} \min_{n} \frac{\w^{\top} \x_n }{\norm{\w}} =\frac{1}{\whatNorm}.
            \end{equation}
        
        \cite{gunasekar2018implicit} further generalized this characterization to steepest descent with respect to an arbitrary norm, establishing convergence to the maximum-margin predictor with respect to the chosen norm.  The proof technique used for this more general setting is different and it is based on bounding the decrease in loss and increase in norm, generalizing the analysis of \cite{telgarsky2013margins} which shows how Boosting converges to the max-$\ell_1$-margin predictor.  This analysis does not rely on the data being non-degenerate as in Theorem \ref{theorem: ICLR theorem 3} (i.e.~it applies for any data set, not only almost all data sets).  Although \cite{gunasekar2018implicit} do not state a rate of convergence, the technique can be used to establish that the margin converges at the rate of $O(1/\log{t})$ as summarized in the following theorem (specialized here only for gradient descent), which is proved in appendix \ref{sec:proofs-jason}:
        	\begin{theorem} 
        For any separable data set (Definition \ref{def: the data is separable}), any initial point $\w(0)$, consider gradient descent iterates with a fixed step size $\eta< \frac{1}{\cL(\w(0))}$  for linear classification with the exponential loss $\ell(u)=\exp(-u)$.\\ 
        Then the iterates $\w(t)$ satisfy:
        $$  \min_n\frac{\w(t)^{\top}\x_n}{\|\w(t)\|}= \gamma -O\Big( \frac{1}{\log t}\Big) \, ,$$ 
        where $\gamma=\max_{\w} \min_{n} \frac{\w^{\top} \x_n }{\norm{\w}} =\frac{1}{\whatNorm}$ is the maximum-margin . 
        \label{thm:sd-exp}
        \end{theorem}
        Note that Theorem \ref{thm:sd-exp} ensures the rate of convergence of the margin, but does not specify how quickly $\w(t)$ itself converges to the max-margin predictor $\hat{\mathbf{w}}$.
        \note{
        \begin{theorem} [Theorem 1 in \cite{gunasekar2018conv}, rephrased]\label{Theorem: convnet Theorem 1}
        	    For any depth $L$, almost all linearly separable datasets $\left\{\xn,y_n\right\}_{n=1}^{N}$, almost all initialization $\np(0)$, and any bounded sequence of step sizes $\eta_t$, consider the sequence gradient descent iterates $\np(t)$ in eq.~\eqref{eq: GD dynamic} for minimizing $\mathcal{L}_\P \left(\np(t)\right)$ in eq.~\eqref{eq: Loss for net parameters} with exponential loss over $L$-layer fully connected linear network.
        	    
        		If the gradient descent iterates $\np(t)$ minimizes the objective, i.e., $\mathcal{L}_\P \left(\np(t)\right)\to0$, and incremental updates $\np(t+1)-\np(t)$ converge in direction, then $\eqw(t)=\P\left(\np\left(t\right)\right)$ converge in direction to $\bar{\eqw}_\infty=\lim_{t\to\infty}\frac{\eqw(t)}{\norm{\eqw(t)}}$ given by
        		\begin{equation} \label{eq: normalized max margin equation}
        		    \bar{\eqw}_\infty = \frac{\what}{\norm{\what}},\text{ where } \hat{\mathbf{w}}=\underset{\eqw\in\mathbb{R}^{d}}{\mathrm{argmin}}\left\lVert \eqw\right\rVert^2 \,\,\mathrm{s.t.}\,\,\eqw^{\top}\mathbf{x}_{n}\geq1.
        		\end{equation}
        \end{theorem}
    }
    }
	\section{SETUP AND REVIEW OF PREVIOUS RESULTS \label{sec: prev-results}} 
	
	Consider a dataset $\left\{ \mathbf{x}_{n},y_{n}\right\} _{n=1}^{N}$,
    with features $\mathbf{x}_{n}\in\mathbb{R}^d$ and binary labels $y_{n}\in\left\{ -1,1\right\}$. All the results in the paper are stated for data $\left\{ \mathbf{x}_{n},y_{n}\right\} _{n=1}^{N}$ which is \textit{strictly linearly separable}, i.e., there exists a separator $\wvec_*$ such that $\forall n:\ y_n \wvec_*^\top\xn>0$. 
    
    We study learning homogenous linear predictors by minimizing unregularized empirical losses of the form 
    \vspace{-0.5em}
    \begin{equation}
    \vspace{-0.5em}
 \mathcal{L}\left(\mathbf{w}\right)=\sum_{n=1}^{N}\ell\left(y_{n}\mathbf{w}^{\top}\mathbf{x}_{n}\right)\,,\label{eq: general loss functions}
    \end{equation}
    where $\mathbf{w}\in\mathbb{R}^{d}$ is the weight vector or the linear predictor. 
    To simplify notation, we assume that
    $\forall n:\,y_{n}=1$ \textemdash{} this is  without loss of
    generality, since we can always re-define $y_{n}\mathbf{x}_{n}$ as
    $\mathbf{x}_{n}$. We  denote the data matrix by $\mathbf{X}=[\x_1,\dots,\x_N]\in\mathbb{R}^{d\times N}$ and $\norm{\cdot}$ denotes the $L_2$ norm. 
    
	The gradient descent (GD) updates for minimizing $\c{L}(\w)$ in eq. \eqref{eq: general loss functions} with step size sequence $\{\eta_t\}_{t=0}^\infty$ is given by
	\begin{align} \label{GD}
	\wvec(t+1) &= \wvec(t) - \eta_t \nabla \mathcal{L}(\wvec(t)) \nonumber\\
	&= \wvec(t) - \eta_t \sumn \ell'(\wvec(t)^\top \xn)\xn.
	\vspace{-0.5em}
	\end{align}
    We look at the iterates of GD on linearly separable datasets  with monotonic loss functions.  
	\begin{definition} \label{def: l(u) assumptions}[Strict Monotone Loss]
		$\ell(u)$ is a differentiable strictly monotonically decreasing function bounded from below, i.e. $\forall u,\,\ell\left(u\right)^{\prime}<0$ and, without loss of generality, $\forall u,\, \ \ell\left(u\right)>0$
		and $\lim_{u\rightarrow\infty}\ell\left(u\right) =\lim_{u\rightarrow\infty}\ell^{\prime}\left(u\right)=0$. Also, $\lim\sup_{u\rightarrow-\infty}\ell^{\prime}\left(u\right)\neq0$. 
	\end{definition}
	Examples of strict monotone losses include common classification losses such as logistic loss, exponential loss, and probit loss. A key property of interest with such losses is that the empirical risk in eq.~\eqref{eq: general loss functions} over separable data does not have any finite global minimizers. Thus,  whenever the gradient descent updates in eq.~\eqref{GD} minimize the empirical loss $\c{L}(\w)$, the iterates $\wvec(t)$ will necessarily diverge to infinity. Nevertheless, in this case, even though the norm of the iterates $\norm{\w(t)}$ diverge, the classification boundary is entirely specified by the direction of $\w(t)/\norm{\w(t)}$. Can we say something interesting about which direction the iterates $\w(t)$ converge to? 
	
	For monotone losses with $-\ell'(u)$ satisfying the specific \textit{tight exponential tail} property (defined below), \citet{soudry2017implicit} characterized this direction to be the maximum-margin separator,
	\begin{definition} \label{def: exponential tail} [Tight Exponential Tail]	A scalar function $h(u)$ has a tight exponential tail, if there exist positive constants $\mu_+,\mu_-$, and $\bar{u}$ such that $\forall u>\bar{u}$:
	\[(1-\e(-\mu_- u))e^{-u} \le h(u) \le (1+\e(-\mu_+ u))e^{-u}.\]
	\end{definition}

    \begin{theorem}[Theorem 3 in \citet{soudry2017implicit}, rephrased]\label{theorem: ICLR theorem 3} For almost all  linearly separable  datasets $\{\xn,y_n\}_{n=1}^N$, and any $\beta$-smooth $\mathcal{L}$  with a strictly monotone loss function $\ell$ (Definition \ref{def: l(u) assumptions}), for which $-\ell^{\prime}$ has a tight exponential tail (Definition \ref{def: exponential tail}), the gradient descent iterates $\w(t)$  in eq.~\eqref{GD}  with any fixed step size satisfying\footnote{Note that for exponential loss $\ell(u)=\e(-u)$, $\mathcal{L}(\wvec)$ does not have a global smoothness parameter $\beta$. However,  with $\eta < 1/ \mathcal{L}(\w(0))$ it is straightforward to show the gradient descent iterates maintain bounded local smoothness $\beta(t) \leq \mathcal{L}(\w(t)) \leq  \mathcal{L}(\w(0))$, so we will have $\eta < \beta(t)^{-1}$ for all iterates, which suffices for the result to extend to exponential loss.} $\eta<2\beta^{-1}$ and any initialization $\w(0)$, will behave as:
    \begin{equation}
    \mathbf{w}\left(t\right)=\hat{\mathbf{w}}\log t+\boldsymbol{\rho}\left(t\right),\label{eq: asymptotic form}
    \end{equation}
    where the residual $\boldsymbol{\rho}\left(t\right)$ is bounded 
    and $\hat{\mathbf{w}}$ is the following $L_{2}$ maximum-margin separator:
    \begin{equation}\label{w_hat equation}
    \hat{\mathbf{w}}=\underset{\mathbf{\mathbf{w}}\in\mathbb{R}^{d}}{\mathrm{argmin}}\left\lVert \mathbf{w}\right\rVert^2 \,\,\mathrm{s.t.}\,\,y_n\mathbf{w}^{\top}\mathbf{x}_{n}\geq1.
    \end{equation}
    \end{theorem}
    
In Theorem \ref{theorem: ICLR theorem 3} and in the remainder of the paper, \textit{almost all} datasets refers to  
 all datasets except a measure zero set of $\{\xn\}_n$, e.g., with probability 1,  any dataset sampled from an absolutely continuous distribution. 
\remove{    
\begin{equation} \label{eq: direction convergence}
    \left \Vert \frac{\mathbf{w}\left(t\right)}{\left\Vert \mathbf{w}\left(t\right)\right\Vert}-\frac{\hat{\mathbf{w}}}{\left\Vert \hat{\mathbf{w}}\right\Vert}\right\Vert = O\left(\frac{1}{\log t}\right)
    \end{equation} 
    and in the margin
    \begin{equation}\label{eq: margin convergence}
     \min_{n} \frac{\w(t)^{\top} \x_n }{\norm{\w (t)}} = \gamma - O\left(\frac{1}{\log t}\right), \,\,\, \mathrm{with} \,\,\, \gamma=\max_{\w} \min_{n} \frac{\w^{\top} \x_n }{\norm{\w}} =\frac{1}{\whatNorm}
    \end{equation}

}

 Interestingly,  and somewhat surprisingly, Theorem \ref{theorem: ICLR theorem 3} implies logarithmically slow convergence in direction to the $L_2$ maximum-margin separator. This slow convergence rate also applies to the margin. This, in contrast, is much slower compared to the rate of convergence of the loss $\c{L}(\wvec(t))$ itself, which can be shown to decay as $O(1/t)$ (see Lemma $1$ in \citet{soudry2018journal}).
\remove{ as summarized in the following theorem (specialized here only for gradient descent), which is proved in appendix \ref{sec:proofs-jason}:
	\begin{theorem} 
For any separable data set (Definition \ref{def: the data is separable}), any initial point $\w(0)$, consider gradient descent iterates with a fixed step size $\eta< \frac{1}{\cL(\w(0))}$  for linear classification with the exponential loss $\ell(u)=\exp(-u)$.\\ 
Then the iterates $\w(t)$ satisfy:
$$  \min_n\frac{\w(t)^{\top}\x_n}{\|\w(t)\|}= \gamma -O\Big( \frac{1}{\log t}\Big) \, ,$$ 
where $\gamma=\max_{\w} \min_{n} \frac{\w^{\top} \x_n }{\norm{\w}} =\frac{1}{\whatNorm}$ is the maximum-margin . 
\label{thm:sd-exp}
\end{theorem}
Note that Theorem \ref{thm:sd-exp} ensures the rate of convergence of the margin, but does not specify how quickly $\w(t)$ itself converges to the max-margin predictor $\hat{\mathbf{w}}$.
\mnote{I think theorem \ref{thm:sd-exp} should move to the faster rates section}
\snote{i agree, see note above in the text}
}

\paragraph{Multilayer linear networks.} In a  recent follow up work, \citet{gunasekar2018conv} extend such results to fully connected deep linear networks, where the objective is non-convex. 
A multi-layer linear network consists of nodes arranged in $L$ layers. We use the convention that for an $L$ layer network, the inputs features $\mathbf{x}$ form the source nodes in the zeroth layer $l = 0$ and the output is sink node in the final layer  $l = L$. Let $d_l$ for $l=0,1,\ldots,L$ denote the number of nodes in layer $l$. The network is parameterized by weight matrices $\mathcal{W}=\{\lw{l}\in\mathbb{R}^{d_{l-1}\times d_l}:l=1,2,\ldots, L\}$. Every such network represents a linear mapping given as follows:
\[\w=\P(\mathcal{W}):=\lw{1}\cdot\lw{2}\cdot...\cdot \lw{L}\in \R^d.\]
Unlike logistic regression, where the parameters of the linear model $\wvec\in\mathbb{R}^d$ are learned directly by minimizing the training loss, in training linear networks, the objective is instead minimized over the parameters of the network $\mathcal{W}=\{\lw{l}\in\mathbb{R}^{d_{l-1}\times d_l}:l=1,2,\ldots, L\}$. The empirical loss is given by:
\vspace{-0.5em}
\begin{equation} \label{eq: Loss for net parameters}
    \mathcal{L}_\P\left(\np\right) = \mathcal{L}\left(\P\left(\np\right)\right)=\sumn \ell\left(y_n\ip{\P\left(\np\right)}{\xn}\right)\,.
\end{equation}
Gradient descent iterates $\np\left(t\right)=\{\lw{l}(t)\}_{l=1}^L$ for the above objective are given by:
\begin{equation} \label{eq: GD dynamic}
    \forall l, \lw{l}\left(t+1\right) = \lw{l}\left(t\right) - \eta_t\nabla_{\lw{l}}\mathcal{L}_{\P}\left(\np(t)\right),  
\end{equation}
and the corresponding sequence of linear predictors along the gradient descent path is given by, 
\begin{equation} \label{eq: GD dynamic in w}
\eqw(t)=\P(\np(t))=\lw{1}(t)\cdot...\cdot \lw{L}(t)\in \R^d.
\end{equation}
For the special case of exponential loss, \citet{gunasekar2018conv} showed  that the linear separator $\w(t)$ in eq. \eqref{eq: GD dynamic} learned by gradient descent on fully connected network  (under additional conditions on convergence of the net parameters and gradients, and convergence of the loss) again converges in the direction of the $L_2$ maximum-margin separator (Theorem~$1$ in \citet{gunasekar2018conv}). This result, however, only applies to exponential loss and does not specify how quickly the margin  of $\eqw(t)$ converges to the maximum-margin  (in case of convergence).


\remove{\snote{May be just give a one line summary of the result here and defer the notation to later section. Also, I dont think we need to rewrite the theorem. It makes sense to have thm 1 as a reference since it also has rates etc. but restating this theorem might be repetitive}

\snote{Instead it will be good to add a list of contributions at this point, I am usually not a fan of bulleted list, but in this paper, it might make sense.}
\note{
	\[
	    \eqw(t)=\P(\np(t))=\lw{1}(t)\cdot...\cdot \lw{L}(t)\in \R^d
    \]
	\begin{equation} \label{eq: Loss for net parameters}
	    \mathcal{L}_\P\left(\np(t)\right) = \sumn \ell\left(y_n\ip{\P\left(\np(t)\right)}{\xn}\right)
	\end{equation}
	\begin{equation} \label{eq: GD dynamic}
	    \np\left(t+1\right) = \np\left(t\right) - \eta_t\nabla_{\np}\mathcal{L}_{\P}\left(\np(t)\right)
	\end{equation}
	\begin{equation}
	    \mathcal{L}\left(\eqw(t)\right) = \sumn \ell\left(y_n\ip{\eqw(t)}{\xn}\right)
	\end{equation}
	}
\note{
\begin{theorem} [Theorem 1 in \cite{gunasekar2018conv}, rephrased]\label{Theorem: convnet Theorem 1}
	    For any depth $L$, almost all linearly separable datasets $\left\{\xn,y_n\right\}_{n=1}^{N}$, almost all initialization $\np(0)$, and any bounded sequence of step sizes $\left\{\eta_t\right\}$, consider the sequence gradient descent iterates $\np(t)$ in eq.~\eqref{eq: GD dynamic} for minimizing $\mathcal{L}_\P \left(\np(t)\right)$ in eq.~\eqref{eq: Loss for net parameters} with exponential loss over $L$-layer fully connected linear network.
	    
		If the gradient descent iterates $\np(t)$ minimizes the objective, i.e., $\mathcal{L}_\P \left(\np(t)\right)\to0$, and incremental updates $\np(t+1)-\np(t)$ converge in direction, then $\eqw(t)=\P\left(\np\left(t\right)\right)$ converge in direction to $\bar{\eqw}_\infty=\lim_{t\to\infty}\frac{\eqw(t)}{\norm{\eqw(t)}}$ given by
		\begin{equation}
		    \bar{\eqw}_\infty = \frac{\what}{\norm{\what}},\text{ where } \hat{\mathbf{w}}=\underset{\eqw\in\mathbb{R}^{d}}{\mathrm{argmin}}\left\lVert \eqw\right\rVert^2 \,\,\mathrm{s.t.}\,\,\eqw^{\top}\mathbf{x}_{n}\geq1.
		\end{equation}
\end{theorem}
    Theorem \ref{Theorem: convnet Theorem 1} states that if the incremental updates $\np(t+1)-\np(t)$ converge in direction and the loss is minimized then $\eqw(t)$ converge in the direction of the maximum-margin separator. Note that Theorem \ref{Theorem: convnet Theorem 1} only applies to exponential loss and does not specify how quickly $\eqw(t)$ converge to the max margin separator (in case of convergence). 
}}
\section{MAIN RESULTS}	
In this section, we provide a detailed analysis of the implicit bias in linear models focusing on convergence and rate of convergence of margin under general tails and with variable step sizes. We use the following standard notation on asymptotic behaviour:
    \begin{inparaenum}[(a)]
    \item $f(u)=\omega(g(u))\Leftrightarrow\lim\limits_{u\to\infty} \abs{\frac{f(u)}{g(u)}}=\infty$,
    \item $f(u)=o(g(u))\Leftrightarrow \lim\limits_{u\to\infty}\frac{f(u)}{g(u)}= 0$, 
    \item $f(u)=O(g(u))\Leftrightarrow \lim\sup\limits_{u\to\infty}\frac{|f(u)|}{g(u)}<\infty$, 
    \item $f(u)=\Omega(g(u))\Leftrightarrow \lim\inf\limits_{u\to\infty}\frac{f(u)}{g(u)}>0$, and 
    \item $f(u)=\Theta(g(u)) \Leftrightarrow \Omega\left(g(u)\right)=f(u)=O\left(g(u)\right)$. 
    \end{inparaenum}

Previous results, summarized in Section \ref{sec: prev-results}, show that when minimizing exponentially tailed losses on separable datasets, gradient descent converges to the $L_2$ maximum-margin separator with a very slow rate of  $1/\log(t)$. While commonly used classification losses such as logistic loss, cross entropy loss, and exponential loss indeed have tight exponential tail, the significance of the exponential tail is not fully understood. What are the general conditions on the tail under which gradient descent converges to the maximum-margin solution? Can the rate of convergence be accelerated by choosing a heavier or lighter tail? 

\subsection{Linear networks with general tails}
We first show that for a large family of strictly monotone losses with super--polynomial tails specified (Assumption~\ref{def: general tail loss} below), gradient descent iterates converge to the maximum-margin solution. We will later also analyze the rate of convergence for this family of loss functions. 
    
\begin{assumption} \label{def: general tail loss} $\ell(u)$ is analytic and  satisfies the following:
\begin{compactenum}
\item \textbf{Strict monotonicity: } $\ell$ satisfies Definition~\ref{def: l(u) assumptions}. Since, $\forall u,\, \ell^{\prime}\left(u\right)<0$, let  $\ell'(u) = -\e(-f(u))$.
\item \textbf{Super-polynomial tail:} $\ell(u)$ has a ``super-polynomial tail" if $\forall M>0$, $\exists u_0$ such that $\forall u\ge u_0$, $-\ell'(u) \le u^{-M}$. This is equivalent to $f(u)=\omega(\log(u))$.  
\item \textbf{Asymptotically convex}:  $\exists u_0$ such that $\forall u>u_0$, $\ell{''}(u)>0$. For strictly monotone decreasing losses, this is equivalent to $\forall u>u_0,$ $f'(u)=\frac{\ell^{\prime\prime}(u)}{-\ell'(u)}>0$. 

\item \textbf{Non-oscillatory tail}: $\lim_{u\to\infty} uf'(u)$ exists. For losses with super-polynomial tails where $f(u)=\omega(\log(u))$, this condition implies $f'(u)=\omega\left(u^{-1}\right)$.
\end{compactenum}
\end{assumption}
\begin{remark}
Assumption \ref{def: general tail loss}  captures a large family of strictly monotone losses with  super-polynomial tails that are relevant for binary classification tasks, and the last condition is rather technical to avoid undesirable oscillatory  behaviour like $f(u)=u+\sin(u)$. In particular, the assumption includes the following special cases:
\begin{asparaitem}
\item Logistic loss $\ell(u)=\log\left( 1+e^{-u}\right)$, for which $f(u)=\log(1+e^{u})=\omega(\log(u))$ and $f'(u)=\frac{e^u}{1+e^{u}}=\omega\left(u^{-1}\right)$.
\item Other losses with tight exponential tail (Definition~\ref{def: exponential tail}), like the exponential loss $\ell(u)=\e(-u)$.
\item ``Poly-exponential" tailed losses given by $\ell'(u)=-\e(-u^\nu)$ for degree $\nu>0$, e.g., the probit loss. 
\item Sub-exponential super-polynomial tails like  $\ell'(u)=-u^{-\log^\mu(u)}$ for $\mu>0$.
\end{asparaitem}
\end{remark}

\remove{
\begin{remark}The condition $f'(u)=\omega\left(u^{-1}\right)$ implies $f(u)=\omega\left(\log(u)\right)$, which, for strictly monotone losses, corresponds to strictly monotone losses\dnote{last sentence did not make sense}. If the loss derivative tail decreases to zero slower than $u^{-1}$ (which is the rate for $f(u)=\log(u)$) we may not converge to the maximum-margin separator. For example, hinge loss (which does not satisfy the necessary requirement) does not converge to the maximum-margin separator \mnote{elaborate on this example}. More examples for non-convergence and for sub-poly exponential tails the do converge to the max margin separator are in appendix section \ref{sec: Examples}.
\end{remark}
}


\remove{
\mnote{I'm not sure if we have the space but maybe explain the motivation behind $\omega(t^{-1})$ along the lines of what we have in the analysis for generic tails (without it the nsv part is not negligible compared to the sv part, in the proof this assumption is needed in the proof of Lemma \ref{lemma: f(g(t1))-f(g(t2)) to infty})} }

For depth--$L$ linear networks, we first show that the implicit bias of gradient descent for exponential loss from \citet{gunasekar2018conv} can be extended  more broadly to  super-polynomial tailed losses specified in Assumption~\ref{def: l(u) assumptions}. 

\begin{restatable}{theorem}{theoremMaxMarginGeneralTail} \label{Theorem: convergence to max margin for general tail}
	    For any depth $L$, almost all linearly separable datasets, almost all initialization and any bounded sequence of step sizes $\left\{\eta_t\right\}$, consider the sequence $\np(t)=\{\lw{l}(t)\}_{l=1}^L$ of gradient descent updates in eq. \eqref{eq: GD dynamic} for minimizing the empirical loss $\c{L}_\P(\mathcal{W})$ (eq.~\eqref{eq: Loss for net parameters})  with a strictly monotone loss function $\ell$ satisfying Assumption  \ref{def: general tail loss}, i.e.:
		$\ell'(u) = -\e(-f(u))<0$, 
		where asymptotically $f'(u)>0$ and $ f'(u)=\omega\left(u^{-1}\right)$.
		
	    If  \begin{inparaenum}[(a)] \item $\np(t)$ minimizes the empirical loss, i.e. $\mathcal{L}_{\P}\left(\np(t)\right)\to0$, \item $\np(t)$, and consequently $\eqw(t)=\P\left(\eqw(t)\right)$, converge in direction to yield a separator with positive margin, and \item the gradients with respect to the linear predictors $\nabla_\eqw \mathcal{L}\left(\eqw(t)\right)$ converge in direction, \end{inparaenum} then the limit direction is given by,  
	    \vspace{-0.5em}
	    \[
	    \vspace{-0.5em}
	    \bar{\eqw}_\infty=\lim\limits_{t\to\infty}\frac{\eqw(t)}{\norm{\eqw(t)}} = \frac{\what}{\norm{\what}}\, ,
	    \] 
	    where
		\begin{equation} 		     \hat{\mathbf{w}}=\underset{\eqw\in\mathbb{R}^{d}}{\mathrm{argmin}}\left\lVert \eqw\right\rVert^2 \,\,\mathrm{s.t.}\,\,\eqw^{\top}\mathbf{x}_{n}\geq1.
		     \label{eq: normalized max margin equation}
		\end{equation}
\end{restatable}

This theorem  is proved in Appendix \ref{sec: convergence to max margin for general tail}, while the basic ideas are sketched in Appendix \ref{sec: Generic Tails} for $L=1$.

\begin{remark} Theorem~\ref{Theorem: convergence to max margin for general tail} covers a large family of super-polynomial tails specified under Assumption~\ref{def: l(u) assumptions}. Conversely, for (sub) polynomial tails, we may not converge to the maximum-margin separator. In Appendix \ref{sec: Examples}, we show that we do not converge to the maximum-margin if $\ell(u)$ has polynomial tail. Additionally, with the hinge loss (which it is neither differentiable or strictly monotonic) we generally do not converge to the maximum-margin without regularization, as then GD typically converges to a finite minimizer that depends on the initialization.
\end{remark} 

\begin{remark}
    \citet{rosset2004margin} also investigated the connection between the loss function choice and the maximum-margin solution. In this work, \citet{rosset2004margin} considered linear models with monotone loss functions and explicit norm regularization. We discuss the connections between \citet{rosset2004margin} results and ours in appendix \ref{sec: Additional Related Work}.
\end{remark}

\begin{remark} \citet{gunasekar2018conv} characterized the implicit bias of gradient descent for fully connected linear networks for the special case of exponential loss $\ell(u)=\e(-u)$. Theorem~\ref{Theorem: convergence to max margin for general tail} generalizes this characterization to a larger family of losses, which in particular includes the commonly used logistic loss. Logistic loss, despite having the same exponential tail as the exponential loss, was not explicitly analyzed in \citet{gunasekar2018conv}. 
\end{remark}

We now continue to characterizing the convergence rates.
\subsection{Rates of convergence}
To calculate the convergence rates we will make an additional assumption.
    \begin{restatable}{assumption}{assumptionfexpansion} \label{assumption: f can be expended}
	    $f(u)$ is real analytic on $\mathbb{R}_{++}$ and satisfies $\forall k\in \mathbb{N}:\ \abs{\dfrac{f^{(k+1)}(u)}{f^\prime(u)}}=O\left(u^{-k}\right)$. 
	\end{restatable}
    While the above assumption is not required to show asymptotic convergence of gradient descent to the maximum-margin separator (Theorem~\ref{Theorem: convergence to max margin for general tail}), we do require the additional assumption to calculate the rates. This assumption implies that the loss tail does not decay too fast. In particular, Assumption~\ref{assumption: f can be expended} is \textit{not} satisfied by super-poly-exponential tails like $\ell'(u)=\e{(-\e({u^{\nu}}))}$ for $\nu>0$ or $\ell'(u)=\e{(-\e{(\log^\mu(u))})}$ for $\mu>1$, and additionally avoids oscillatory functions like $\sin(u)$.
	
	Nevertheless, a large class of interesting monotone functions satisfy this assumption, including cases where $f(u)$ is polynomial and poly-logarithmic functions. Within this family, we look at the margin rate of convergence of the gradient descent iterates, for $L=1$ in two regimes:
	\begin{compactenum}
	\item $f'(u)=\omega(1)$, which  implies $-\ell'(u)=\omega(\exp(-u))$. This case includes loss functions with tails \textit{lighter} than the exponential tail, for example poly-exponential tail $\ell(u)=\e(-u^\nu)$ with s strictly greater than one exponent, $\nu>1$. 
	\item $f'(u)=\omega(u^{-1})$ and $f'(u)=o(1)$:  or $-\ell'(u)=o(\exp(-u))$. This case includes loss functions with tails \textit{heavier} than the exponential tail, such as $\ell(u)=\e(-u^\nu)$ for $\nu<1$ or $\ell(u)=\e(-\log^\mu(u))$ for $\mu>0$. 
	\end{compactenum}
	
	We first look at the rates for the special case of $L=1$ where the parameters $\w$ of the linear models are directly learned using gradient descent. This is the setting analyzed in \cite{soudry2017implicit} with tight exponential tailed losses.
	The following theorem is proved  in Appendix \ref{sec: proof of theorem about general tail rates}.
	\begin{restatable}{theorem}{theoremGeneralRatesSimple} \label{theorem: general convergence rates simplified}
	    
	    For almost all linearly separable datasets, almost all initialization, any bounded sequence of step sizes $\left\{\eta_t\right\}<2\beta^{-1}$, and a single layer $L=1$,  consider the sequence of gradient descent updates in eq.~\eqref{GD} for minimizing the empirical loss $\mathcal{L}(\eqw)$ (eq.~\eqref{eq: general loss functions}) with a strictly monotone $\beta$-smooth loss function $\ell$ satisfying
		$\ell'(u) = -\e(-f(u))<0$, 
		where asymptotically $f'(u)=\Omega\left(\frac{1}{u}\log^{1+\epsilon}(u)\right)$ for some $\epsilon>0$ and satisfies Assumption \ref{assumption: f can be expended}.
		
		If  \begin{inparaenum}[(a)] \item $\eqw(t)$ converges in direction to yield a separator with positive margin, and \item the gradients with respect to the linear predictors $\nabla_\eqw \mathcal{L}\left(\eqw(t)\right)$ converge in direction, \end{inparaenum} then the margin convergence of $\eqw(t)$ to the max margin $\gamma=\max_{\w} \min_{n} \frac{\w^{\top} \x_n }{\norm{\w}}$ satisfies:
		
		\begin{enumerate}
			\item If $f'(u)=\omega(1)$ (which implies $f(t)=\omega(t)$), then
			\[
			 \gamma - \min\limits_{n} \frac{\xnT \wvec (t)}{\Vert \wvec(t)\Vert}=O\left(\frac{1}{f^{-1}\left(\log(t)\right)}\right).
			\]
			
			\item If $f'(u)=o(1)$ and $f$ is strictly concave, then
			\[
				\gamma - \min\limits_{n} \frac{\xnT \wvec (t)}{\Vert \wvec(t)\Vert}=\Omega\left(\frac{1}{\log(t)}\right) 
			\]
			and the optimal rate is obtained for exponential loss.
		\end{enumerate}
	\end{restatable}
	From the proof of Theorem~\ref{theorem: general convergence rates simplified}, we can also calculate the rates of convergence for the normalized direction $\w(t)/\norm{\w(t)}$ to the maximum-margin separator $\what/\norm{\what}$, as well as the convergence of the angle between them. 
	\begin{corollary} We examine super-polynomial tailed losses satisfying the assumptions of the previous Theorem, when the loss tail does not decay too fast, i.e. $\abs{\frac{f'(u)}{f(u)}}=O(u^{-1})$. The optimal rate of convergence to the maximum-margin of GD with fixed step size is $1/\log(t)$. This optimal rate is attained by exponentially tailed losses, where $f(u)=\Theta(u)$ (or $f'(u)=\Theta(1)$). This includes the popular losses of logistic loss and exponential loss. 
	\end{corollary}
	\begin{proof} 
	For the case of $f'(u)=\omega(1)$,  $f(t)=\omega(t)\Rightarrow f^{-1}(t)=o(t)$ and thus, the rate for this case $O\left(\frac{1}{f^{-1}\left(\log(t)\right)}\right)$ is sub-optimal compared with the rate for exponential loss which is $1/\log(t)$ (from Theorem~\ref{theorem: ICLR theorem 3}). In appendix sections \ref{sec: f=w(u) rate negative example}, \ref{sec: f=w(u) rate positive example} we give a positive example that demonstrates that this upper bound is tight, i.e., it is obtained for some datasets, and a negative example which shows a case in which the upper bound is not obtained. In general as long as the loss tail does not decay too fast, i.e. $\abs{\frac{f'(u)}{f(u)}}=O(u^{-1})$, the rate in this case is $\Omega\left(\frac{1}{\log(t)}\right)$ (see appendix \ref{sec: proof that rate is Omega(1/log(t))}). Secondly, for the case of $f'(u)=o(1)$ the asymptotic rate is $\Omega(1/\log(t))$, so the optimal rate we can hope for with any tail is $O(1/\log(t))$. In appendix \ref{sec: proof of theorem about general tail rates} we show that the exponential tail obtains this optimal rate. Additionally, in Appendix~\ref{sec: appendix results on poly-exp tails}, we show that for the special case of poly-exponential losses  $\ell'(u)=-\e(-u^\nu)$ with $0.25<\nu\le 1$, the rate is indeed $O(1/\log(t))$ and the constants in the rates for $\nu<1$ are strictly worse than that of exponential tail with $\nu=1$. 
	\end{proof}
	\begin{remark}  Note that for $L=1$ 
	by Lemma $1$ in \cite{soudry2017implicit}, the assumption in Theorem~\ref{Theorem: convergence to max margin for general tail} that $\mathcal{L}_{P}\left(\np(t)\right)\to0$ is  satisfied for appropriate choices of step size. Moreover for the special case of poly-exponential tails with $\ell'(u)=-\e(-u^\nu)$ for $\nu>0.25$, the convergence to the maximum-margin separator and the convergence rates can be obtained without the assumptions that $\eqw(t)$ and $\nabla_\eqw \mathcal{L}\left(\eqw(t)\right)$ converge in direction (see Appendix \ref{sec: appendix results on poly-exp tails}). 
    \end{remark}
	
	Now we state the results for the general case of $L$--layer linear network. 
    \begin{restatable}{theorem}{theoremGeneralRatesSimpleOde} \label{theorem: general convergence rates simplified with ode}
		Under assumption \ref{assumption: f can be expended} and the conditions and notations of Theorem \ref{Theorem: convergence to max margin for general tail}, if the SVM support vectors span the data then for any depth $L$ the network equivalent linear predictor $\eqw(t)$ satisfies:
			\[
				\gamma - \min\limits_{n} \frac{\xnT \wvec (t)}{\Vert \wvec(t)\Vert}=\begin{cases} O\left(\frac{1}{g(t)}\right), & f'(u)=\omega(1)\\
				\Theta\left(\frac{1}{g(t)f'\left(g\left(t\right)\right)}\right), & \text{otherwise}
				\end{cases}
			\]
			where $g(t)$ is the asymptotic solution of
			\begin{equation} \label{eq: linear NN ode}
			    \frac{dg(t)}{dt}=-\ell^{\prime}\left(g\left(t\right)\right)\left(g(t)\right)^{2\left(1-L^{-1}\right)}.
			\end{equation}
	\end{restatable}

\begin{remark}  Importantly, from Assumption \ref{def: general tail loss}, $-\ell^{\prime}(u)$ has super-polynomial tail, which suggests the factor $\left(g(t)\right)^{2\left(1-L^{-1}\right)}$ only negligibly affects the asymptotic solution of eq.~\eqref{eq: linear NN ode}. This implies that $\forall L>1$, and even in the limit $L\rightarrow \infty$, the rate predicted by this Theorem \ref{theorem: general convergence rates simplified with ode} will only be slightly smaller than the $L=1$ case of Theorem \ref{theorem: general convergence rates simplified}. This difference will become negligible in the limit $t\rightarrow\infty$. For example, for the case of exponential loss, we prove in appendix \ref{Sect: g(t)=log(t)+o(log(t)) proof for $L>1$ and f(u)=u} that the ODE solution is $g(t)=\log(t)+o\left(\log(t)\right)$. Thus, in this case, the margin converges as $O(1/\log(t))$ for any depth.
\end{remark}

\subsection{Faster rates using variable  step sizes}
Our analysis so far suggests that exponential tails have an optimal convergence rate, and for exponential tail losses with a bounded step size, we have an extremely slow rate of convergence, $O(1/\log t)$. Therefore, the question is can we somehow accelerate this rate using variable unbounded step sizes. Fortunately, at least for linear models trained with exponential loss, the answer is yes and we can indeed show faster rate of convergence by aggressively increasing the step size to compensate for the vanishing gradient. Specially, we examine the following normalized GD algorithm:
\begin{align}
\w_{t+1} =\w_t - \eta_t \frac{\nabla\cL(\w(t))}{ \norm{\nabla\cL(\w(t))}}\,.
\label{eq:norm-gd1}
\end{align}
Recall that $\gamma=\max_{\w:\norm{\w}\le1}\min_n\innerprod{\w}{\xn}$ is the maximum-margin of the dataset with unit $L_2$ norm separators, and without loss of generality assume $\forall n:\norm{\xn}\le1$. 

By the triangle inequality, we have that $\norm{\nabla\cL(\w(t))}=\norm{\sum_n\e(-\w(t)^\top\xn)\xn}\le \c{L}(\w(t))$. We additionally have the following inequality for all $t$,
\begin{align*}
\norm{\nabla\cL(\w(t))}&=\max_{\w:\norm{\w}\le1}\sum_n\e(-\w(t)^\top\xn)\w^\top\x_n\\
&\ge \gamma \sum_n\e(-\w(t)^\top\xn)=\gamma\cL(\w(t))\,.
\end{align*}
Thus, for all $\w$, the two-sided bound \[\gamma \cL(\w) \le \|\nabla \cL(\w)\|\le \cL(\w)\] holds, and, up to a scaling of step--sizes, the normalized GD in eq. \eqref{eq:norm-gd1} can be alternatively expressed as the following 
\begin{align}
\w_{t+1} =\w_t - \eta_t \frac{\nabla\cL(\w(t))}{\cL(\w(t))}\,.
\label{eq:norm-gd}
\end{align}
We chose to state our results in terms of eq.~\eqref{eq:norm-gd} (normalizing GD by $\cL(\w(t))$) so that the stepsize choice $\eta_t$ does not depend on the optimal margin $\gamma$ which is unknown.
The following theorem proved in Appendix \ref{sec:proofs-jason} shows that using normalized GD can improve the rate of convergence of the margin of the separator to  $\log t/\sqrt{t}$ compared to $O(1/\log{t})$ for fixed step sizes. 
\begin{theorem}	
    For any separable data set and any initial point $\w(0)$, consider the normalized GD updates in eq.~\eqref{eq:norm-gd} with a variable step size $\eta_t=\frac{1}{\sqrt{t+1}}$ and exponential loss $\ell(u)=\exp(-u)$. 

Then the margin of the iterates $\w(t)$ converges to the max-margin $\gamma$ with rate $t^{-1/2} \log t$: 
\begin{align*} \frac{\w(t+1) ^T \x_n }{\norm{\w(t+1)}}\!&\!\ge\gamma\!-\!  \frac{1+\log (t+1)}{ \gamma (4 \sqrt{t+2}-4)}\!-\!\frac{\log \cL(\w(0))}{\gamma ( 2 \sqrt{t+2} -2)}. \end{align*} 
\label{thm:sd-fast}
\end{theorem}
\vspace{-1.5em}
In the appendix we prove a more general version of Theorem \ref{thm:sd-fast}, which obtains the same rate for any steepest descent algorithm. Also, note that normalized GD as in eq.~\eqref{eq:norm-gd1} was analyzed before, but for other purposes. For example, \citet{Levy2016} showed a stochastic version of it can better escape saddle points.
Here we study the effect of normalization on the implicit bias of the algorithm.

The observation that aggressive changes in the step size can improve convergence rate is applied in the AdaBoost literature \citep{schapire2012boosting}, where exact line-search is used. We use a slightly less aggressive strategy of decaying step sizes with normalized gradient descent, attaining a rate of $\log(t)/\sqrt{t}$. This rate almost matches $1/\sqrt{t}$, which is the optimal rate in terms of margin suboptimality for solving hard margin SVM. This rate is achieved by the best known methods.\footnote{The best known method in terms of margin suboptimality, and using vector operations (operations on all training examples), is the aggressive Perceptron, which achieves a rate of $\sqrt{N/t}$. \cite{clarkson2012sublinear} obtained an improved method which they showed is optimal, that does not use vector operations.  \cite{clarkson2012sublinear} method achieves a rate of $\sqrt{(N+d)/t}$, where now $t$ is the number of scalar operations.} This suggests that gradient descent with a more aggressive step size policy is quite efficient at margin maximization.

We emphasize our goal here is not to develop a faster SVM optimizer, but rather to understand and improve gradient descent and local search in a way that might be applicable also for deep neural networks, as indicated by the numerical results we present next.

\section{EXPERIMENTS WITH NORMALIZED GRADIENT DESCENT}

\begin{figure*}
\begin{centering}
\includegraphics[width=\textwidth] {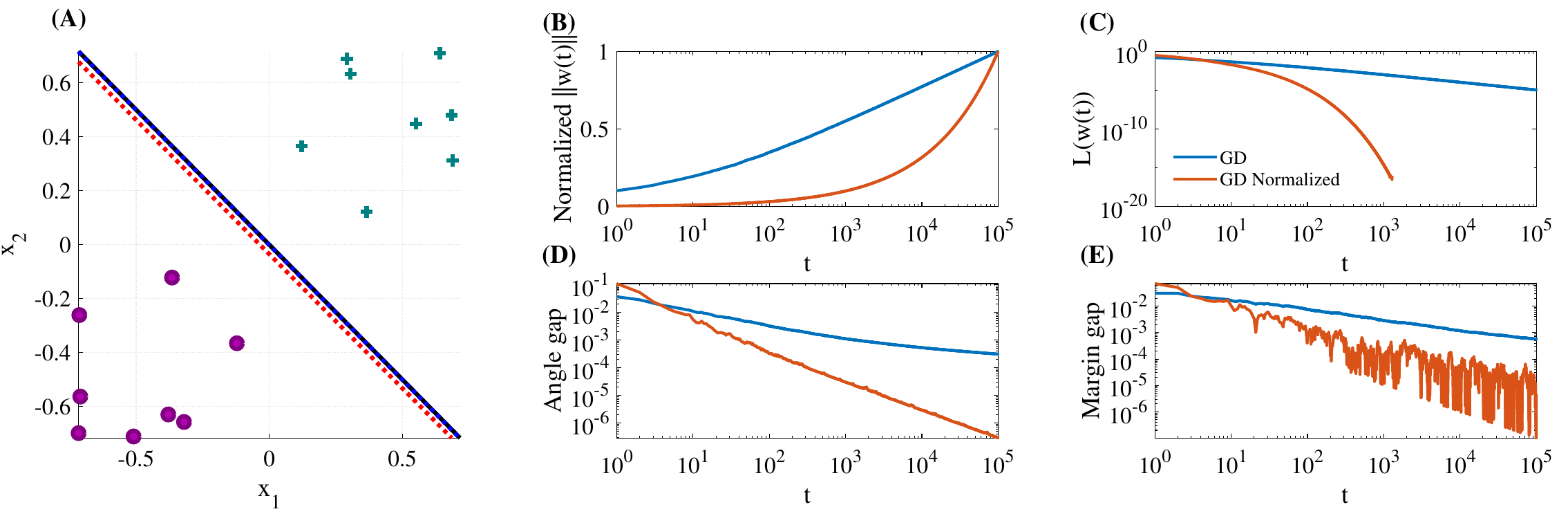}
\par\end{centering}
\vspace{-10pt}
\caption{ \label{Fig: NormGD} Visualization of the convergence of GD in comparison to normalized GD in a synthetic logistic regression dataset in which the $L_{2}$ maximum-margin vector $\hat{\mathbf{w}}$ is precisely known.
\textbf{(A)} The dataset (positive and negatives samples ($y=\pm1$)
are respectively denoted by $'+'$ and $'\circ'$), max margin separating
hyperplane (black line), and the solution of GD (dashed red) and
normalized GD (dashed blue) after $10^5$ iterations. For both GD and
Normalized GD, we show: \textbf{(B) }The norm of $\mathbf{w}\left(t\right)$,
normalized so it would equal to $1$ at the last iteration, to facilitate
comparison; \textbf{(C) }The training loss; and \textbf{(D\&E) }the
angle and margin gap of $\mathbf{w}\left(t\right)$ from $\hat{\mathbf{w}}$. As can be seen in panels \textbf{(C-E)}, normalized GD converges to the maximum-margin separator significantly faster, as expected from our results. More details are given in appendix \ref{sec:fig1 details}.
\label{fig:Synthetic-dataset}}
\vspace{-10pt}
\end{figure*}

\remove{
\begin{figure}%
    \centering
    \subfloat[2 layer Linear Net]
    {{\includegraphics[width=0.4\textwidth]{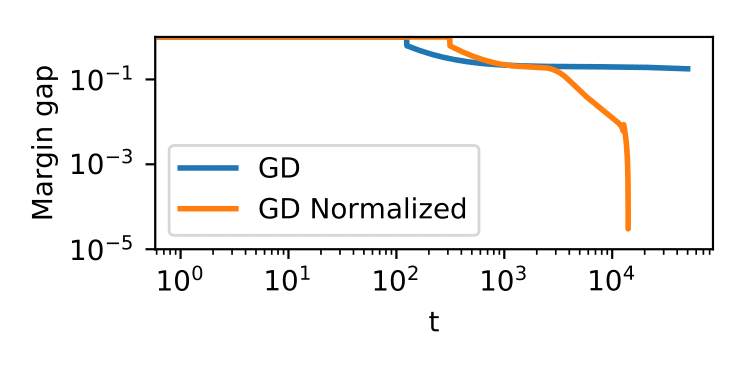} }
    }\par
    \subfloat[3 layer Linear Net]
    {{\includegraphics[width=0.4\textwidth]{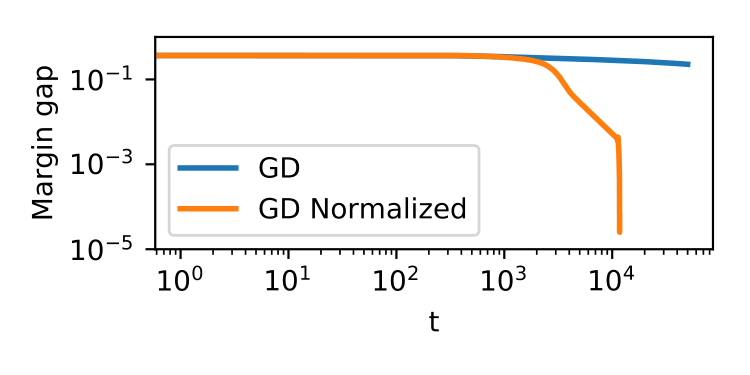} }
    }
    \caption{Margin convergence plots for 2 and 3-layered linear networks on synthetic clustered data, trained with GD and normalized GD --- the latter provides significantly faster convergence.}
    \label{fig:margin_convergence}%
\end{figure}
}

\remove{
\begin{figure*}%
    \centering
    {\includegraphics[width=0.32\textwidth]{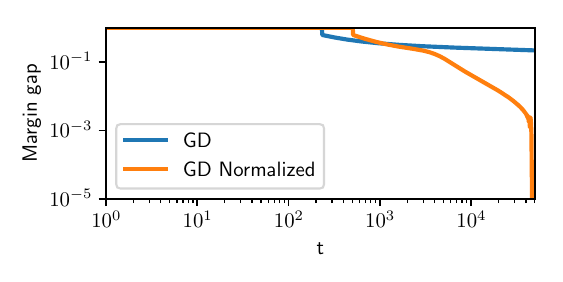}
    }
    {\includegraphics[width=0.32\textwidth]{figures/linear2gap}
    } 
    {\includegraphics[width=0.32\textwidth]{figures/linear3gap}
    }
    \vspace{-10pt}
    \caption{Margin convergence plots for 1, 2 and 3-layered linear networks on synthetic clustered data, trained with GD and normalized GD --- the latter provides significantly faster convergence.}
    \label{fig:margin_convergence}%
\end{figure*}
}

\begin{figure}%
    \centering
    {\includegraphics[width=0.4\textwidth]{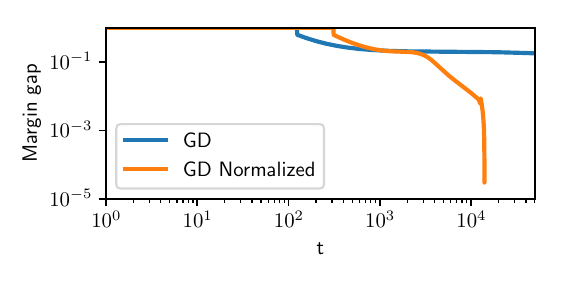}
    }\par
    \vspace{-10pt}
    {\includegraphics[width=0.4\textwidth]{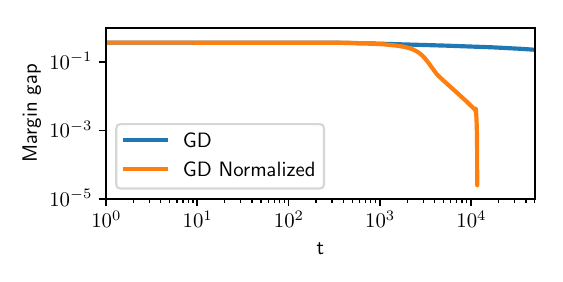}
    }
    \vspace{-15pt}
    \caption{Margin convergence plots for 2 (top) and 3 (bottom) layered linear networks on synthetic clustered data, trained with GD and normalized GD --- the latter provides significantly faster convergence.}
    \label{fig:margin_convergence}%
    \vspace{-10pt}
\end{figure}

\begin{figure}%
    \centering
    {\includegraphics[width=0.45\textwidth]{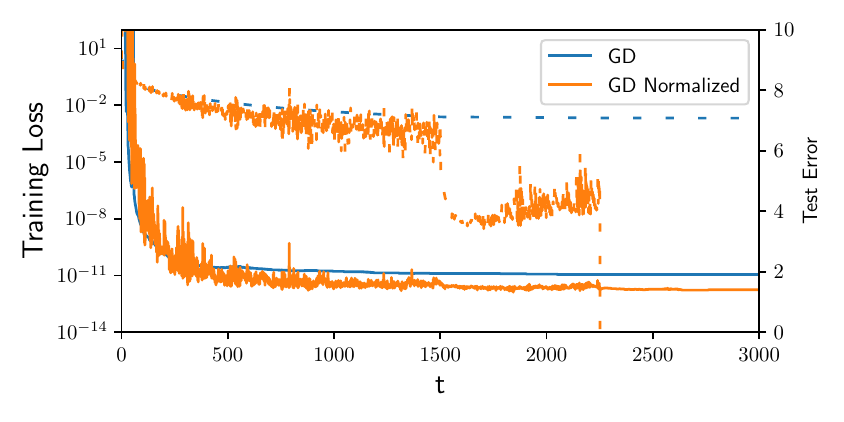} }%
    \vspace{-15pt}
    \caption{MNIST digit classification with a 2-layer feedforward neural network. Training loss (dashed lines) stagnates with GD once gradients become small, while normalized GD keeps making progress. Normalized GD also achieves lower test error (solid lines).}%
    \label{fig:mnist_both}%
    \vspace{-10pt}
\end{figure}

\begin{figure}%
    \centering
    {\includegraphics[width=0.45\textwidth]{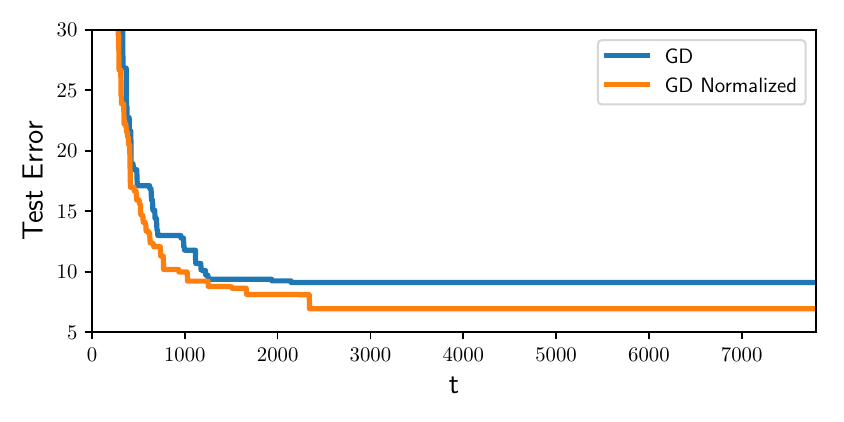} }%
    \vspace{-15pt}
    \caption{Test performance of a Wide ResNet 28-4 on CIFAR-10, with $\eta = 2.0$, where normalized GD outperforms GD by absolute $2.17\%$. We plot 'best yet' test error: the lowest error seen up to iteration $t$. Unlike curves reported in \cite{wideresnet}, progress stops early in training: there is no change in the 'best yet' test error after $t=2350$, even with the decays in learning rate. This suggests that regularization and/or momentum might be required to achieve better results.}%
    \label{fig:cifar22}%
    \vspace{-10pt}
\end{figure}

\remove{
\begin{figure*}%
    \centering
    \subfloat[1 layer Linear Net \\ (Logistic Regression)]
    {{\includegraphics[width=0.33\textwidth]{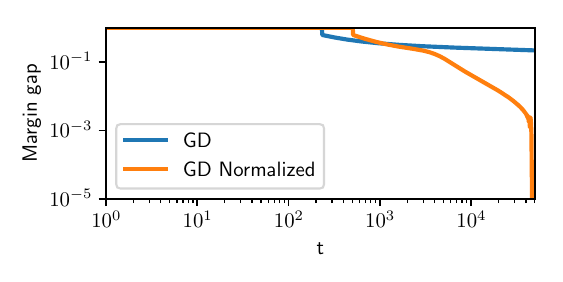} }
    }
    \subfloat[2 layer Linear Net]
    {{\includegraphics[width=0.33\textwidth]{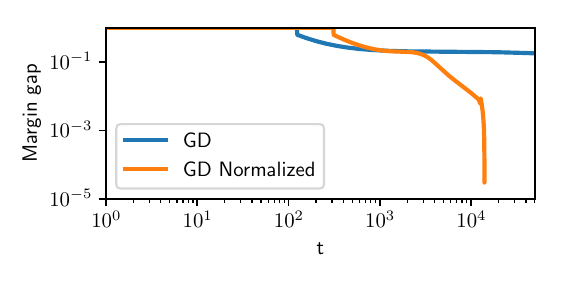} }
    }
    \subfloat[3 layer Linear Net] {{\includegraphics[width=0.33\textwidth]{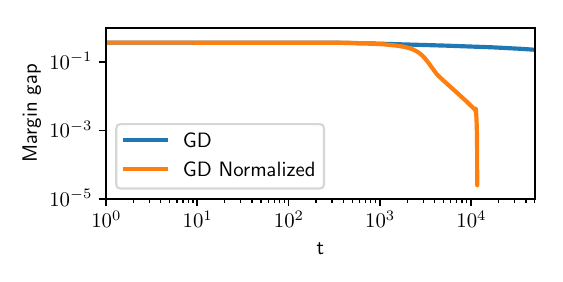} }
    }
    \caption{Margin convergence plots for 1, 2 and 3-layered linear networks on synthetic clustered data, trained with GD and normalized GD --- the latter provides significantly faster convergence.}
    \label{fig:margin_convergence}%
\end{figure*}
}

\remove{
\begin{figure}%
    \centering
    \subfloat[1 layer Linear Net \\ (Logistic Regression)]
    {{\includegraphics[width=0.4\textwidth]{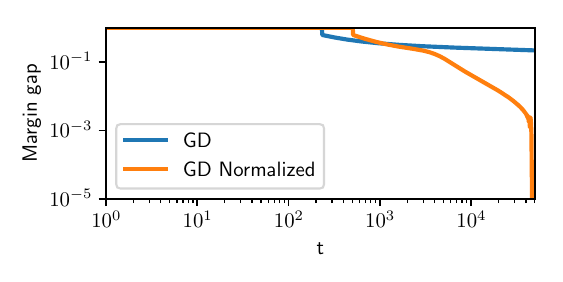} }%
    }\par
    \subfloat[2 layer Linear Net]
    {{\includegraphics[width=0.4\textwidth]{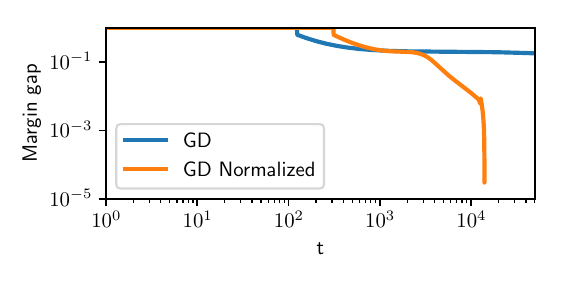} }%
    }\par
    \subfloat[3 layer Linear Net] {{\includegraphics[width=0.4\textwidth]{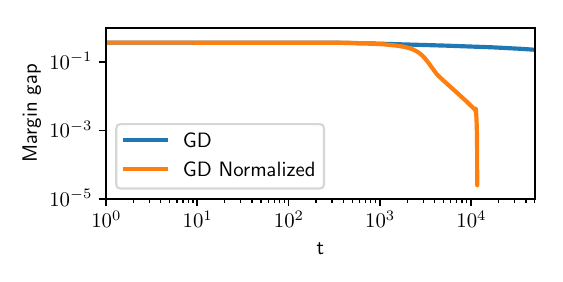} }%
    }
    \caption{Margin convergence plots for 1 (left), 2 (middle) and 3 (right) layered linear networks on synthetic clustered data, trained with GD and normalized GD --- the latter provides significantly faster convergence.}
    \label{fig:margin_convergence}%
\end{figure}
}

In the following experiments, we implement the normalized GD  in eq.~\eqref{eq:norm-gd1} with step sizes separately tuned for each experiment. 
\subsection{Linear Networks on Synthetic Data}

First, in Figure \ref{Fig: NormGD} we visualize the different rates for GD and normalized GD when training a plain logistic regression model on synthetic data. As expected from Theorem~\ref{thm:sd-fast}, we find that normalized GD converges significantly faster than unnormalized GD. 

Additionally, we evaluate experimentally the convergence rates of GD and normalized GD for multi-layer linear models. Networks with $L \in \{1,2,3\}$ layers and $10$ neurons per hidden layer are trained with GD and normalized GD on a synthetic binary classification dataset composed of $600$ points, sampled from two normal distributions (one for each class). 

We use a fixed learning rate $\eta=5 \times 10^{-3}$ chosen through grid-search, and train each network for $5 \times 10^4$ total iterations. Figure \ref{fig:margin_convergence} shows the margin gaps during training, with normalized GD providing faster convergence rates across models. Appendix \ref{sec-ngd_linear} provides details on data generation and training, along with results on ReLU networks.

\subsection{Image Classification on MNIST}
The MNIST dataset is composed of 70,000 grayscale images of 0-9 digits (10 classes total), each having $28 \times 28$ pixels. We use 10,000 images for testing and the rest for training and validation. Unlike harder datasets such as CIFAR-10 and CIFAR-100, MNIST provides a task where simple models can successfully separate the training examples. Hence, we train a 2-layer feedforward network with 5,000 hidden neurons and ReLU activations (ReLU$(x) = \max(0,x)$) with full-batch GD and normalized GD using the cross-entropy loss, for a total of 3,000 iterations. We decay the learning rate by a factor of $5$ at $50\%, 75\%$ and $87.25\%$ of the total number of iterations. 

We performed grid-search over initial learning rate values $\{0.1, 0.3, 0.5, 1.0, 2.5, 5.0\}$ using 5,000 images randomly chosen from the training set as validation, and $\eta = 1.0$ yielded better results for both GD and normalized GD. We use no regularization nor data augmentation, since our goal is to observe the contrast between GD and normalized GD as the training loss decreases and gradients become small. Figure \ref{fig:mnist_both} shows the training loss and test error at each iteration $t$: while the training loss stagnates early for GD, normalized GD keeps decreasing it. Normalized GD also reaches lower test error: $1.4\%$ compared to $1.91\%$.

\subsection{Image Classification on CIFAR-10}

The CIFAR-10 dataset \citep{krizhevsky2009} consists of 60,000 colored $32 \times 32$ images belonging to one of 10 possible classes, and is split into 50,000 training and 10,000 test points. The goal of this experiment is to evaluate whether normalized GD can provide advantages for training complex models on more realistic tasks, when using the standard cross-entropy loss.

For that, we train a Wide ResNet 28-4 \citep{wideresnet}, a 28-layer convolutional neural network with residual connections and a total of 5.8M parameters. This architecture is capable of reaching less than $4\%$ test error on CIFAR-10 given more features per convolutional layer, making Wide ResNets a strong model baseline to compare the benefits of normalized GD against GD. Following \cite{wideresnet}, we pre-process the dataset by performing channel-wise normalization on each image using statistics computed from the training set. Horizontal flips and random crops are used during training for data-augmentation. We also follow the same learning rate schedule, decaying it by a factor of 5 at $30\%$, $60\%$ and $80\%$ of the total iterations.

To select a learning rate for each method, we train the network for 3,000 iterations with $\eta \in \{1.0, 1.5, 2.0, 2.5, 3.0\}$. Both methods performed better on a validation set of 5000 images with $\eta = 2.0$. Figure \ref{fig:cifar22} shows the test performance when training the model for 7,800 iterations with $\eta=2.0$, where normalized GD achieves $6.93\%$ test error, while GD yields $9.90\%$.

Note that, while normalized GD outperformed GD in this full-batch setting, its performance is still subpar when compared to the standard optimization for Wide ResNets, which includes SGD with Nesterov momentum and weight decay. To confirm whether momentum and weight decay can have strong positive impacts in a model's performance, we also trained a Wide ResNet 28-4 using SGD, with and without momentum/weight decay. We observed that removing momentum and weight decay resulted in a test error increase from $4.45\%$ to $7.75\%$ (larger error than normalized GD). This suggests an importance in reconciling weight decay, momentum and gradient normalization.

\remove{
\subsection{Conjecture and comparison with exact results}
This analysis provides the following characterization of the asymptotic solution:
      \begin{restatable}{conjecture}{conjectureGenericTail} \label{conjecture: generic tail}
		For almost all datasets which are linearly separable (Definition \ref{def: the data is separable}) and given a $\beta$-smooth $\mathcal{L}$ (Definition \ref{def: beta smooth}), with a strictly monotone loss function $\ell$ (Definition \ref{def: l(u) assumptions}) such that:
		\begin{equation}
		\ell'(u) = -\e(-f(u)),
		\end{equation}
		where $f'(t)>0$ and $ f(t)=\omega\left(\log(t)\right)$,
		gradient descent (as in eq.~\eqref{GD}) with stepsize $\eta<2\beta^{-1}$ and any starting point $\wvec(0)$ will behave as:
		\begin{equation} 
		\wvec(t) = \what g(t)+\bm{\rho}(t)
		\end{equation}
		where
        \begin{equation} \label{eq: g' = exp(...)}
			\dot{g}(t) = \e\left(-f\left(g(t)\right)\right),\ ||\rhoVec|| = o(g(t))
		\end{equation}
		and
		\begin{equation}
			\bm{\rho}(t) = \begin{cases}
			\left(f'(g(t))\right)^{-1}\genericConstVec+o\left(\left(f'\left(g
			\left(t\right)\right)\right)^{-1}\right), & \text{if $\left(f'\left(g
			\left(t\right)\right)\right)^{-1}=\Omega(1)$}\\
			O(1), & \text{otherwise},
			\end{cases}
		\end{equation}	
		where $\genericConstVec$ is not dependent on $f(u),\ g(u)$ and so
		\begin{equation}
		\lim_{t\to\infty} \frac{\wvec(t)}{\wNorm}=\frac{\what}{\whatNorm} \, ,
		\end{equation}
		where $\what$ is the $L_2$ max margin vector (eq.~\eqref{w_hat equation})
	\end{restatable}
    In appendix \ref{sec: CALCULATION OF CONVERGENCE RATES} we show this conjecture  implies the convergence rates specified in table \ref{table: generic convergence rates}.
    \begin{table} [ht]
        \centering
			\normalsize	
			\begin{tabular}{c  c c} 
				\toprule
				\multicolumn{1}{c}{$-\ell'(u) \asymp e^{-u^\lossPower}$} & \multicolumn{2}{c}{$-\ell'(u) \asymp e^{-f(u)}$, $f(u)=\omega(\log(t))$}\\
				 $0<\lossPower\le \frac{1}{4}$ \vphantom{$\displaystyle \frac{1}{f'(g(t))}$} & $\frac{1}{f'(g(t))}=\Omega(1)$ & \text{otherwise}\\ 
				\midrule
				 \multicolumn{1}{c}{$C_1\lossPower^{-1}\log^{-1}(t) + o(\log^{-1}(t)) $} \vphantom{$\displaystyle O\left(\frac{1}{\log^{\frac{1}{\lossPower}(t)}}\right)$} & $C_1 \left(
				g(t)f'(g(t))\right)^{-1} + o\left(\left(
				g(t)f'(g(t))\right)^{-1}\right) $ &  $O\left(g^{-1}(t)\right)$\\ 
				 \multicolumn{1}{c}{$C_2\lossPower^{-2}\log^{-2}(t) + o(\log^{-2}(t)) $} \vphantom{$\displaystyle  \frac{1}{o\left(\left(
				g(t)f'(g(t))\right)^{-1}\right)}$} & $C_2 \left(
				g(t)f'(g(t))\right)^{-2} + o\left(\left(
				g(t)f'(g(t))\right)^{-2}\right) $ &  $O\left(g^{-2}(t)\right)$\\
				\bottomrule
			\end{tabular}
			\captionsetup{width=\textwidth}
			\caption{Summary of convergence rates derived from Conjecture \ref{conjecture: generic tail}, to complete our exact results from Table \ref{table: convergence rates}. The two lines are the same as in Table \ref{table: convergence rates}: the first line is the convergence rate for both the distance and the suboptimality of the margin (with $C_3$ instead of $C_1$). The second line is the angle convergence rate. The function $g(t)$ satisfies eq.~\eqref{eq: g' = exp(...)}. If $\log\left(f^{\prime}\left(t\right)\right)=o\left(f\left(t\right)\right)$, then we can approximate $g(t) \approx f^{-1}\left(\log(t)\right)$. } \label{table: generic convergence rates}
	\end{table}
	
To prove this conjecture in general we have to justify various assumptions (e.g., the existence of certain limits) and approximations (e.g., Taylor expansions) we made during the analysis above. It turns out that this becomes more and more difficult as the tail of the loss derivative becomes heavier. Therefore, our exact results on Poly-exponential tails (Theorem \ref{theorem: main}), which assumed asymptotically for $u>0$ that
    $$f(u)=u^\lossPower,\ f'(u)=\lossPower u^{\lossPower-1}, f'(g(t)) = \lossPower g^{\lossPower-1}(t),$$    
were only proved for $\lossPower>0.25$. These results are consistent with Conjecture \ref{conjecture: generic tail}: \\
When $\lossPower>1$, $\left(f'(g(t))\right)^{-1}=\lossPower^{-1}g^{1-\lossPower}(t)$ is not $\Omega(1)$ and therefore $\rhoVec=O(1)$.\\ 
    When $\lossPower<1$, $\left(f'(g(t))\right)^{-1}=\lossPower^{-1} g^{1-\lossPower}=\Omega(1)$ and therefore $$\rhoVec=\left(f'(g(t))\right)^{-1}\genericConstVec+o\left(\left(f'(g(t))\right)^{-1}\right) = \lossPower^{-1}g^{1-\lossPower}(t)\bm{a}+o(g^{1-\lossPower}(t)).$$
    
This suggests that Theorem \ref{theorem: main} holds even for $\lossPower\leq 0.25$, and that indeed the exponential tail ($\lossPower =1$) obtains the optimal convergence rate over all poly-exponential loss functions.

In appendix \ref{sec: Examples}, we demonstrate the conjecture by examining two examples. The first example proves that when $f(t)$ is not $\omega(\log(t))$ (e.g. the loss has a power-law tail) we might not convergence to the max-margin separator. The second example demonstrates convergence with sub-poly-exponential tails (that satisfies $f(t)=\omega(\log(t))$).
    }
        
 \section{DISCUSSION}
In this work, we have examined the behavior of gradient descent on separable data, in binary linear classification tasks. First, in Theorem \ref{Theorem: convergence to max margin for general tail} we proved the linear classifier resulting from a multilayer linear neural networks converges in direction to the $L_2$ maximum-margin on almost all linearly separable data --- for a wide family of monotone, convex loss functions with super-exponential tails and some technical conditions (Assumption \ref{def: general tail loss}). In contrast, polynomially tailed loss function do not lead to convergence to the maximum-margin. Intuitvely, the reason behind this is that for super-polynomial loss functions the datapoints with the largest margin (i.e., the support vectors) become dominant in the gradient, while for polynomial or heavier tails the contribution of non-support vectors is never negligible.  

Next, we examine the convergence rate for a linear classifier with loss within this wide family of loss functions. We prove in Theorem \ref{theorem: general convergence rates simplified} that the exponential tail has the optimal rate. This offers a possible explanation to the empirical preference of the exponentially-tailed loss functions over other losses (e.g. the probit loss): that the exponential loss leads to a faster convergence to the asymptotic (implicitly biased) solution, as we showed here. This result is somewhat surprising, and we do not have an intuitive explanation why this should be true. 

In Theorem \ref{theorem: general convergence rates simplified with ode}, we extend these results to multilayer linear neural networks, and show similar convergence rates, with only a negligible decrease in the rate with the depth --- even when the number of layers is infinite. Note that in this Theorem we already assume convergence of the loss to zero. However, if we do converge, it is somewhat surprising that this rate does not depend much on the depth, as one might expect to have convergence rate issues due to exploding or vanishing gradients. 

In Theorem \ref{thm:sd-fast} we showed that the convergence of GD  for an exponential loss function could be significantly accelerated by simply increasing the learning rate. In fact, GD can also approximate the regularization path in the following sense. 
Let $R= \norm{\w_t}$, and $\w_{R} =\arg\min_{\norm{\w} \le R} \cL( \w)$. Then
\vspace{-0.3em}
\begin{align}
\cL(\w(t)) -\cL(\w_R ) \le \cL(\w(0)) \exp(-c \gamma^2 t)\,.
\end{align}
As a simple implication of this, the normalized GD path starting at $\w_0  =0$  has $\cL(\w(0)) = n$, so after $t\ge \log(n/\epsilon)/\gamma^2$ steps the loss achieved by $\w_t$ is $\epsilon$ close to the best predictor of the same norm. This shows that GD is closely approximating the regularization path. 

Finally, we show numerically that normalized GD can significantly improve the convergence speed of GD on synthetic datasets for linear predictors (Figure \ref{fig:Synthetic-dataset}), linear multilayer networks (Figure \ref{fig:margin_convergence}), and even non-linear ReLU multilayer networks (Appendix \ref{sec-ngd_linear}). Additionally, we show normalized GD can improve the results of GD on standard datasets such as MNIST (by $0.5\%$) and CIFAR-10 (by $3\%$). However, a gap remains from achieving state of the art results. Our experiments indicate the origin of this gap is the use of weight decay and momentum (which are outside the scope of this paper). This suggests that reconciling regularization, momentum and gradient normalization might be of particular interest for future work, possibly reducing the gap between mini-batch and full-batch training.

Recent work explore extensions of the implicit bias result for linear models to non-strictly-separable datasets \citep{ji2018risk} and to stochastic gradient descent \citep{ji2018risk,nacson2018stochastic,xu2018convergence}. It remains to be seen if the results of this work could be also extended to such settings. Additionally, combining our results with the results of a parallel work, \cite{ji2018gradient}, might enable us to weaken some of the assumptions in this paper. We discuss \cite{ji2018gradient} work in appendix \ref{sec: Additional Related Work}.

\subsubsection*{Acknowledgements}

The authors are grateful to C. Zeno, and N. Merlis for
helpful comments on the manuscript. This research was supported by the Israel Science foundation (grant No. 31/1031), and by the Taub foundation. A Titan Xp used for this research was donated by the NVIDIA Corporation. PS, SG and NS were partially supported by NSF awards IIS-1302662 and IIS-1764032.

\remove{
Theorems \ref{theorem: main} and \ref{thm:sd-fast}, and their proof methods, both seem to have their own strengths and weaknesses. The analysis behind Theorem \ref{theorem: main} allows exact calculation of convergence rates to the max-margin separator (including constants in some cases). These rates are easy to calculate and understand intuitively. However, it is significantly harder to prove them rigorously. Such a proof seems to become harder as the loss tail becomes heavier since we have to consider additional terms in the asymptotic calculations (this is why Theorem  \ref{theorem: main} stops at $\nu>0.25$). Additionally, the current results for $\nu\neq 1$ are limited to ``almost every dataset" (e.g., any dataset sampled from an absolutely continuous distribution). However, we believe that it is possible to derive the corrections to the convergence rate resulting from zero measure cases and that these should be of lower order (e.g., $O(\log\log(t))$ for the exponential tail), as proved in \cite{soudry2018journal} for exponential tails. 

In contrast, the proof of Theorem \ref{thm:sd-fast} is significantly simpler, does not require any assumptions on the dataset beyond its linear separability, and it easily generalizes to steepest descent, and to variable step sizes. However, this approach has some weaknesses. First, the results are only stated for the exact exp-loss. In contrast, Theorem \ref{theorem: main} only requires \emph{the tail} of the loss to be poly-exponential. Extension of this result to losses with exponential tails seems possible, based on the methods of \cite{telgarsky2013margins}, but less so to other types of tails. Second, this theorem only provides a bound, not an exact asymptotic result (in contrast to Theorem \ref{theorem: main}). Furthermore, this bound is only on the margin, so we may have a different rate on the convergence to max-margin separator itself; this is the case exactly in the zero measure cases of Theorem \ref{theorem: main}, as shown in \cite{soudry2018journal}.
}


\bibliographystyle{plainnat}
	\newpage
	\onecolumn
	\appendix

	\part*{Appendix}
	\section{Additional Related Work} \label{sec: Additional Related Work}
	\begin{enumerate}
	\item  \citet{rosset2004margin} investigated the connection between the loss function choice and the maximum-margin solution. They considered linear models with monotone loss functions and explicit norm regularization. The authors examined the solutions of the regularized loss function
	\[
	     \mathbf{w}(\lambda) = \argmin\limits_{\wvec}\mathcal{L}\left(\mathbf{w}\right) + \lambda \norm{\mathbf{w}}_p^p= \argmin\limits_{\wvec} \sum_{n=1}^{N}\ell\left(y_{n}\mathbf{w}^{\top}\mathbf{x}_{n}\right) + \lambda \norm{\mathbf{w}}_p^p
	\]
	for some $p$ as the regularization vanishes, i.e., $\lambda\to0$. We focus here on the euclidean norm, meaning $p=2$. In this case, they proved that if $\exists T$ (possibly $T=\infty$) so that
    \begin{equation}\label{eq: rosset2004 assumption}
        \forall \epsilon>0: \ \lim\limits_{t\to T}\frac{\ell(t\cdot(1-\epsilon))}{\ell(t)}=\infty\,,
    \end{equation}
    then $\mathbf{w}(\lambda)$ converge into the direction of the maximum-margin solution as $\lambda\to0$. Note, that since in this paper we assume that the loss is \underline{strictly} decreasing this implies that the last equation can only be satisfied with $T\to\infty$.
    
    We examine the main differences and similarities between \citet{rosset2004margin} results and ours.
    
    \textbf{Main differences in settings:}
    In our work, we examine the convergence of GD iterates and its implicit bias, while \citet{rosset2004margin} focused on the explicit bias in the limit of the regularization path, as the regularization vanishes. Thus, in our results we take into account the optimization dynamics. In addition, we examine linear fully connected networks while \citet{rosset2004margin} only considered linear models.
    
    \textbf{Relation between \citet{rosset2004margin} and our results:}
    In eq.~\eqref{eq: rosset2004 assumption}, \citet{rosset2004margin} states a condition on the loss function that guarantees convergence to the maximum-margin separator. In our Theorem \ref{Theorem: convergence to max margin for general tail}, the key assumption on the loss function is that $\ell'(u)=-\exp(-f(u))<0$ and $f'(u)=\omega(u^{-1})$ (Assumption \ref{def: l(u) assumptions}). \citet{rosset2004margin} condition is weaker than our assumption because
    $f'(u)=\omega(u^{-1})$ implies that $\forall \epsilon>0: \lim\limits_{t\to \infty}\exp\left(f(u) - f((1-\epsilon)u)\right)=\infty$ (Lemma \ref{lemma: f(g(t1))-f(g(t2)) to infty}).
    This also implies that $\forall \epsilon>0: \ \lim\limits_{t\to \infty}\frac{\ell'(t\cdot(1-\epsilon))}{\ell'(t)}=\infty$. Using L'Hospital's rule we have that
    $\forall \epsilon\in(0,1): \ \lim\limits_{t\to \infty}\frac{\ell(t\cdot(1-\epsilon))}{\ell(t)}
    =\lim\limits_{t\to \infty}(1-\epsilon)\frac{\ell'(t\cdot(1-\epsilon))}{\ell'(t)}=\infty$.
    Thus, our assumption \ref{def: l(u) assumptions} implies the assumption in eq.~\eqref{eq: rosset2004 assumption}. It is still unclear if the opposite direction is also true, \textit{i.e.}, if eq.~\eqref{eq: rosset2004 assumption} implies Assumption \ref{def: l(u) assumptions}.
    
    Furthermore, in Theorem \ref{Theorem: convergence to max margin for general tail} we also make additional assumptions on the the convergence of GD iterate and its gradients. These additional assumption are required in our analyses since analyzing the optimization dynamics in opposed to examining the regularization path limit poses additional technical challenges.
    
    \item After the submission of this paper to AISTATS, another related work appeared \cite{ji2018gradient}. \cite{ji2018gradient} consider fully connected linear networks, separable dataset and strictly decreasing loss functions which are $\beta$-smooth and $G$-Lipschitz. They show that for GD with particular decreasing step sizes and mild assumptions on the initialization, the loss converges to zero. They connect this result to an alignment phenomenon between different layers. In addition, for the logistic loss and under the additional assumption that the SVM support vectors span the all space $\mathbb{R}^d$, they show that the network equivalent linear predictor converges in the direction of the maximum-margin separator. It remains to be seen if combining our results and \cite{ji2018gradient}, we can weaken the assumptions we required to prove convergence rates for linear neural nets.   \end{enumerate}
    \remove{
    \citet{rosset2004margin} requirement on the loss function $\forall \epsilon>0: \ \lim\limits_{t\to \infty}\frac{\ell(t\cdot(1-\epsilon))}{\ell(t)}=\infty$ is closely related to a requirement we had in our analysis on the loss function derivative: $\forall \epsilon>0: \ \lim\limits_{t\to \infty}\frac{\ell'(t\cdot(1-\epsilon))}{\ell'(t)}=\infty$ as we explain next.
    WLOG, we can write $\ell(u)=\exp(-h(u))$ for some general function $h(u)$. Thus, the condition $\forall \epsilon>0: \lim\limits_{t\to \infty}\frac{\ell(t\cdot(1-\epsilon))}{\ell(t)}=\infty$ is equivalent to
    $\forall \epsilon>0: \lim\limits_{t\to \infty}\exp\left(h(u) - h((1-\epsilon)u)\right)=\infty$. In our analysis, we had a the same requirement for $f(u)$ that is used to describe the loss function derivative, $\ell'(u)=-\exp(-f(u))$. Our analysis required that $\forall \epsilon>0: \lim\limits_{t\to \infty}\exp\left(f(u) - f((1-\epsilon)u)\right)=\infty$. We showed that if $f'(u)=\omega(u^{-1})$ (Assumption \ref{def: general tail loss}) then this condition is satisfied (Lemma \ref{lemma: f(g(t1))-f(g(t2)) to infty}). If we assume that $\ell(u)=\exp(-h(u))$ then we also have that $\ell'(u)=-h'(u)\exp(-h(u)) = - \exp(-h(u)+\log(h'(u)))$. Assuming that $\log(h'(u))=o(h(u))$ we get that the loss function tail has the same asymptotic behaviour as the loss function itself.
	}
	\section{Adaptive Learning Rate}\label{app:learning rate}
For learning linear models with exponential loss, 
\citet{gunasekar2018implicit} provide an alternative proof for convergence to max-margin solution when using gradient descent. This result also generalized the characterization of implicit bias for general steepest descent algorithm. While  \citet{gunasekar2018implicit}
do not state a rate of convergence, the technique can be used to establish that the margin converges at the rate of $O(1/\log{t})$ 
as summarized in the following theorem (specialized here only for gradient descent):
	\begin{theorem} 
For any separable data set, any initial point $\w(0)$, consider gradient descent iterates with a fixed step size $\eta< \frac{1}{\cL(\w(0))}$  for linear classification with the exponential loss $\ell(u)=\exp(-u)$.\\ 
Then the iterates $\w(t)$ satisfy:
$$  \min_n\frac{\w(t)^{\top}\x_n}{\|\w(t)\|}= \gamma -O\Big( \frac{1}{\log t}\Big) \, ,$$ 
where $\gamma=\max_{\w} \min_{n} \frac{\w^{\top} \x_n }{\norm{\w}} =\frac{1}{\whatNorm}$ is the maximum-margin . 
\label{thm:sd-exp}
\end{theorem}
Note that Theorem \ref{thm:sd-exp} ensures the rate of convergence of the margin, but does not specify how quickly $\w(t)$ itself converges to the max-margin predictor $\hat{\mathbf{w}}$.
    	\subsection{Proof for Theorems \ref{thm:sd-exp} and \ref{thm:sd-fast} \label{sec:proofs-jason}}
 In this section we prove extended versions of Theorems \ref{thm:sd-exp} and \ref{thm:sd-fast}. In this section only, the norm $\norm{\cdot}$ is a general norm (not the $L_2$ norm, like in the rest of the paper). First, we state definitions and auxiliary results.
 
The following lemma is a standard result in convex analysis.
\begin{lemma}[Fenchel Duality]
Let $\mathbf{A} \in \b{R}^{m \times n}$, and $f,g$ be two closed convex functions. Then
\begin{align}
\max_\w - f^* ( -\mathbf{A}\w) - g^* (\w) \le \min_\mathbf{r} f(\mathbf{r}) +g(\mathbf{A}^\top \mathbf{r}).
\end{align}
\label{lem:fenchel}
\end{lemma}
Let $\mathbf{X} \in \b{R}^{d \times N}$ be the data matrix and without loss of generality $\norm{\xn}_*\leq 1$. Define $\{\mathbf{e}_{n}\in\mathbb{R}^{N}\}_{n=1}^N$ denote the standard basis in $\mathbb{R}^d$ and the $\norm{\cdot}$--margin as $\gamma = \max_{\norm{\w}=1} \min_{n \in [N]} \mathbf{e}_n ^\top \mathbf{X}^\top \w$. 

In the first auxillary result, we wish to show that $\norm{\nabla \cL(\w)}_* \ge \gamma \cL(\w)$ for all $\w$, which is an analog of the Polyak condition. Define $r_n (\w) = \exp(-\w^T \x_n)$ and let $\mathbf{r}(w)=[r_n(\w)]_{n=1}^N\in\mathbb{R}^N$ denote the vector formed by stacking $r_n(\w)$.  By noting that $\cL(\w) = \norm{\mathbf{r}(\w)}_1$ and $\nabla \cL(\w) = -\mathbf{X} \mathbf{r}(\w) $, this can be restated as $\frac{\norm{\mathbf{X} \mathbf{r}(\w)}_ * }{\norm{\mathbf{r}(\w)}_1}\ge \gamma$. Since we require this for all $\w$, with $r_n(\w) \ge 0$, and norms are homogeneous, this condition follows from showing that 
\begin{align}
\min_{\mathbf{r} \in \Delta_N} \norm{\mathbf{X} \mathbf{r}}_* \ge \gamma,
\end{align}
where $\Delta_N$ is the $N$-dimensional probability simplex.

\begin{lemma}\label{lem:duality}
The following duality holds for all $\mathbf{X}$:
\begin{align}
\min_{\mathbf{r}  \in \Delta_N} \norm{ \mathbf{X} \mathbf{r}}_*\ge \max_{\norm{\w}=1} \min_{n \in [N]} \mathbf{e}_n ^\top \mathbf{X} ^\top \w= \gamma.
\end{align}
This in turn implies that for all $\w\in\mathbb{R}^d$, $\norm{\nabla \cL(\w)}_* \ge \gamma \cL(\w)$.
\end{lemma}
\begin{proof}
 Let $f(\mathbf{r}) = \mathbf{1}_{\mathbf{r} \in \Delta_N}$ and $g(\mathbf{z}) = \norm{\mathbf{z}}_*$.  Thus $\min_{\mathbf{r}  \in \Delta_N} \norm{ \mathbf{X} \mathbf{r}}_* = f(\mathbf{r}) +g(\mathbf{X} \mathbf{r})$. The conjugates are $f^*( \w) = \max_{\mathbf{z} \in \Delta_N} \w^\top \mathbf{z}  = \max_n \innerprod{\mathbf{e}_n}{\w} $, and $g^*(\w) = \mathbf{1}_{\norm{\w}\le 1}$. The LHS of Lemma \ref{lem:fenchel} is
\begin{align}
\max_\w\big( -f^* (-\mathbf{X}^\top \w ) -g^* (\w) \big)&=\max_\w\big( - \max_n -\mathbf{e}_n ^\top \mathbf{X}^\top \w - \mathbf{1}_{\norm{\w} \le 1}\big)\\
 &= \max_{\norm{\w} \le 1} -\max_i -\mathbf{e}_n ^\top \mathbf{X}^\top \w \\
 &= \max_{\norm{\w} \le 1} \min_i \mathbf{e}_n ^\top \mathbf{X}^\top \w = \gamma.
\end{align}
Thus the LHS is equal to $\gamma$, since it is precisely the optimization program of $\norm{\cdot}$ -SVM. By weak duality (Lemma~\ref{lem:fenchel}), we have shown that $\min_{\mathbf{r} \in \Delta_N} \norm{\mathbf{X} \mathbf{r}}_* \ge \gamma$.
\end{proof}

Using this lower bound we proceed with the optimization analysis which largely follows the standard arguments from the optimization literature on first-order methods and the analysis extends the proof by \citet{telgarsky2013margins}, where a similar result was derived for the case of $L_1$ margin using AdaBoost (coordinate descent). We prove the theorems for general steepest descent algorithms which includes gradient descent as a special case.

\subsection{Proof of Theorem \ref{thm:sd-exp}}
Consider the steepest descent algorithm described by the updates below:
\begin{equation}
\begin{split}
\w(t+1) &= \w(t) - \eta \gamma_t \Delta \w(t) \\
\Delta \w(t) ^T \nabla \cL(\w(t)) &= \norm{\nabla \cL(\w(t)) }_*\\
\norm{\Delta \w(t)} &=1\\
\gamma_t &\triangleq \norm{\nabla \cL(\w(t))}_*.
\end{split}\label{eq:steepest_descent}
\end{equation}

It is easy to check  that when $\norm{.}$ is the $L_2$ norm,  steepest descent  above is simply gradient descent.

Next, we prove the generalized version of Theorem \ref{thm:sd-exp}, which applies to steepest descent (instead of just gradient descent). For fixed step sizes, the proof is essentially same as in the proof of Theorem $5$ in \citet{gunasekar2018implicit} where in the end we add the computation of rates. 
\begin{theorem} {\emph{\textbf{[Generalized Theorem \ref{thm:sd-exp}]}}} For any separable data set  and any initial point $\w(0)$, consider steepest descent updates with a fixed step size $\eta$ and exponential loss $\ell(u)=\exp(-u)$. 

Let us assume that $\norm{\x_n}_* \le 1$. If the step size $\eta \le \frac{1}{\cL(\w(0))}$ , then the iterates $\w(t)$ satisfy: $$ \min_n\frac{\innerprod{\w(t)}{\x_n}}{\|\w(t)\|}= \max_{\w} \min_{n} \frac{\innerprod{\w}{\x_n}}{\norm{\w}} -O\Big( \frac{1}{\log t}\Big)$$ In particular, if there is a unique maximum-$\|.\|$ margin solution $\w^\star_{\|.\|}=\arg\max_{\w} \min_{n} \frac{\innerprod{\w}{\x_n}}{\|\w\|}$, then $\lim\limits_{t\to\infty}\frac{\w(t)}{\|\w(t)\|}=\w^\star _{\|.\|}$. \end{theorem}

By Lemma $12$ of \citet{gunasekar2018implicit}, $\cL(\w(t)-\eta \gamma_t\Delta \w(t)) \le \cL(\w(t))$ for any $\eta < \frac{1}{\cL(w(0))}$. 
Then, we note the following:
\[\mathbf{v}^\top\nabla ^2 \cL (\w) \mathbf{v} = \sum_{n} r_n (\x_n ^\top \mathbf{v})^2 \le \sum_n r_n \norm{\x_n}_* ^2 \norm{\mathbf{v}}^2 \le \cL(\w) \max_n\norm{\x_n}_* ^2 \norm{\mathbf{v}}^2\le\cL(\w) \norm{\mathbf{v}}^2,\]
where the last inequality follows since we assumed without loss of generality $\norm{\xn}\le1$. 

Now using Taylor's theorem along with the steepest descent updates in eq.  \eqref{eq:steepest_descent} gives the following
\begin{align}
\cL(\w(t+1)) &\le\cL(\w(t) - \eta \gamma_t  \norm{\nabla \cL(\w(t)}_* + \frac12 \eta^2 \gamma_t ^2 \max_{r\in(0,1) }  \Delta \w(t) ^\top \nabla ^2 \cL( \w(t) -r \eta \gamma_t \Delta \w(t)) \Delta \w(t) \\
&\le \cL(\w(t)) - \eta \gamma_t ^2 + \frac12 \eta^2 \gamma_t^2 \cL(\w(t))   \label{eq:inter-step}\\
& \le \cL(\w(t))\left( 1- \eta  \frac{\gamma_t ^2}{\cL(\w(t))} +\frac12 \eta^2 \gamma_ t^2 \right) \\
&\le \cL(\w(t)) \exp\left( - \eta \frac{\gamma_t ^2}{\cL(\w(t))} + \frac12 \eta^2 \gamma_t ^2\right).
\end{align}
Recursing gives
\begin{align}
\cL(\w(t+1) ) &\le \cL(\w(0)) \exp \left( - \eta \left(\sum_{i=0}^t \frac{\gamma_i ^2}{\cL(\w_i)} + \frac12 \eta \gamma_i ^2\right) \right).
\end{align}

We show convergence of margin the following steps

\begin{asparaenum}[Step I.]
\item \textit{Lower bound on the un-normalized margin.}
\begin{equation}
\max_{n \in[N]} \exp(-\w(t+1)^T x_n) \le \cL(\w(t+1))\le \cL(\w(0)) \exp \Big( - \eta \big(\sum_{i=0}^t \frac{\gamma_i ^2}{\cL(\w_i)} + \frac12 \eta \gamma_i ^2\big) \Big).
\end{equation}

By  applying $-\log$,
\begin{equation}
\min_{j \in [n]} \w(t+1) ^\top \x_n \ge \eta \sum_{i=0}^t \frac{\gamma_i ^2}{\cL(\w_{(i)})}  - \frac12 \eta^2 \sum_{i=0}^t \gamma_i ^2 - \log \cL(\w_0).
\end{equation}

\item \textit{Upper bound on norm $\norm{\w(t+1)}$.}
\begin{align}
\norm{\w(t+1)} = \norm{\sum_{i=0}^t \eta \gamma_i \Delta \w(i)} \le \eta\sum_{i=0}^t \gamma_i .\label{eq:norm-UB}
\end{align}

\item \textit{Convergence of margin. }
For every $n \in [N]$, from the above steps we have that
\begin{align}
\frac{\w(t+1) ^\top \x_n} {\norm{\w(t+1)}} = \frac{ \sum_{i=0}^t \frac{\gamma_i ^2}{\cL(\w_{(i)})}}{\sum_{i=0}^t \gamma_i } - \frac{\eta}{2} \frac{\sum_{i=0}^t \gamma_i ^2}{\sum_{i=0}^t \gamma_i}  -\frac{ \log \cL(\w_0)}{\eta  \sum_{i=0}^t \gamma_i}.
\end{align}
Use that $\gamma_i \ge \gamma \cL(\w_i)$ by the duality result in Lemma~\ref{lem:duality},
\begin{align}
\frac{\w(t+1) ^\top \x_n}{\norm{\w(t+1)}} &\ge \gamma-  \frac{\eta}{2} \frac{\sum_{i=0}^t \gamma_i ^2}{\sum_{i=0}^t \gamma_i}  -\frac{ \log \cL(\w_0)}{\eta  \sum_{i=0}^t \gamma_i} \label{eq:margin-lower-bound}
\end{align}

In order to prove the rest of the Theroem, we show that \begin{inparaenum}\item $\sum_{i=0}^t \gamma_i ^2 < \infty$, and \item $\sum_{i=0}^t \gamma_i  =\Omega(\log{t})$\end{inparaenum}
\begin{compactenum}[(a)]
\item Proof that $\sum_{i=0}^t \gamma_i ^2 < \infty$. From eq. \eqref{eq:inter-step}, 
\begin{align}
\cL(\w(t+1)) & \le \cL(\w(t)) - \eta \gamma_t ^2 + \frac12 \eta^2 \gamma_t^2 \cL(\w(0))\le\cL(\w(t)) - \frac \eta 2 \gamma_t ^2 ,
\end{align}
where the last inequality follows since $\eta\le \frac{1}{\cL(\w(0))}$. Now using telescoping sum gives the following for all $t>0$,
\begin{align}
&\cL(\w(t+1)) \le \cL(\w(0) ) - \frac{\eta}{2}\sum_{i=0} ^t \gamma_i ^2\\
\implies &\sum_{i=0}^t\gamma_i ^2 \le \frac{\cL(\w(0)) - \cL(\w(t+1))}{ \eta/2}\le \frac{\cL(\w(0))}{ \eta/2}< \infty.
\label{eq:gamma-sq-bound}
\end{align}

\item Next we show that $\eta \sum_{i=0}^t \gamma_i =\Omega(  \log t)$. From eq.  \eqref{eq:inter-step} again, 
\begin{align}
\cL(\w(t+1)) &\le \cL(\w(t)) - \eta \gamma_t ^2 +\frac12 \eta^2 \gamma_t ^2 \cL(\w(0)).
\end{align}
Since we chose $\eta< \frac{1}{\cL(\w(0))}$ and $\gamma_t\ge\gamma\mathcal{L}(\w(t))$ from Lemma~\ref{lem:duality}, we have
\begin{align}
\cL(\w(t+1)) &\le \cL(\w(t))-\frac{\eta}{2} \gamma_t ^2\le \cL(\w(t)) -\frac{\eta}{2} \gamma^2 \cL(\w(t))^2.
\end{align}

The following claim is proved at the end of this section. 
\begin{claim}[Solve Recursion] For a non-negative sequence $\{a_t\}_t$, the recursion $a_{t+1} \le a_t - c^2 a_t^2$ implies
\begin{align}
a_{t+1} \le \frac{1}{(t+1)c^2/(1-c^2 a_0) + 1/a_0}.
\end{align}
\label{claim-recursion}
\end{claim}
We use Claim \ref{claim-recursion}  with $c^2 = \frac{\eta}{2} \gamma^2 $ and $a_t=\cL(\w(t))$. 
Note that $c^2a_0\frac{\eta}{2} \gamma^2 \cL(\w(0)) \le \frac12$, since $\eta<\frac{1}{\cL(\w(0))}$ and  $\gamma=\max_{\norm{\w}\le1}\min_n\innerprod{\xn}{\w}\le\max_n\norm{\xn}_*\le1$, we have  $2c^2\ge \frac{c^2}{1-c^2 a_0} > 0$. Thus,
\begin{align}
\cL(\w(t+1)) &\le \frac{1}{\eta \gamma^2 (t+1) + 1/\cL(\w(0))}\le \frac{1}{\eta \gamma^2 (t+1)}\triangleq q(t+1).
\end{align}

We then lower bound $\norm{\w_{t+1}}$. Since $\norm{\x_n}_* \le 1$, then
\begin{align}
q(t) &\ge \cL(\w(t)) \ge \exp(-\w(t)^\top \x_n)\\
\implies \log \frac{1}{q(t)}&\le \w(t)^\top \x_n \le \norm{\w(t)}.
\end{align}
From eq. \eqref{eq:norm-UB},
\begin{align}
\eta \sum_{i=0}^{t}\gamma_i\ge \norm{\w(t+1)}\ge \log \frac{1}{q(t+1)}=\log (\eta \gamma^2 (t+1)).
\label{eq:gamma-bound}
\end{align}
\end{compactenum}
Putting together the inequalities from eqs.  \eqref{eq:margin-lower-bound}, \eqref{eq:gamma-sq-bound}, and \eqref{eq:gamma-bound}
\begin{align}
\frac{\w(t+1) ^\top \x_n}{\norm{\w(t+1)}} &\ge \gamma-  \frac{\eta}{2} \frac{\sum_{i=0}^t \gamma_i ^2}{\sum_{i=0}^t \gamma_i}  -\frac{ \log \cL(\w_0)}{\eta  \sum_{i=0}^t \gamma_i}\\
&\ge \gamma -\frac{\eta\cL(\w(0))+\log{\cL(\w(0))}}{\log{(\eta \gamma^2 (t+1))}}
\end{align}
\end{asparaenum}

This completes the proof of Theorem \ref{thm:sd-exp}, the proof of intermediate Claim~\ref{claim-recursion} is given below.

\begin{proof}[Proof of Claim \ref{claim-recursion}] For a non-negative decreasing sequence satisfying, $a_{t+1}\le a_t - c^2 a_t^2$, by inversion  we have
\begin{align}
\frac{1}{a_{t+1}} &\ge \frac{1}{a_t ( 1- c^2 a_t)}= \frac{1}{a_t} + \frac{c^2}{1-c^2 a_t}\ge \frac{1}{a_t} + \frac{c^2}{1-c^2a_0}.
\end{align}
Suming from $i=0,\ldots,t$,
\begin{align}
\frac{1}{a_{t+1} } \ge \frac{1}{a_0} + (t+1) \frac{c^2}{1-c^2a_0}
\implies a_{t+1} \le \frac{1}{1/a_0 +(t+1)c^2/(1-c^2a_0) }.
\end{align}
\end{proof}

\subsection{Proof of Theorem \ref{thm:sd-fast}}
We not look at the steepest descent with varying step sizes algorithm:
\begin{align}
\w(t+1) &= \w(t) - \eta_t \gamma_t \mathbf{p}_t \\
\mathbf{p}_t ^\top \nabla \cL(\w(t)) &= \norm{\nabla \cL(\w(t)) }_*\\
\norm{\mathbf{p}_t} &=1\\
\gamma_t &\triangleq \frac{\norm{\nabla \cL(\w(t))}_*}{\cL(\w(t))}.\label{alg:norm-sd-vary-step}
\end{align}

Note that for quadratic norm normalized steepest descent becomes normalized gradient descent. In this section we will prove the generalized version of Theorem \ref{thm:sd-fast}, which applies for normalized steepest descent (instead of just normalized gradient descent):
\begin{theorem} {\emph{\textbf{[Generalized Theorem \ref{thm:sd-fast}]}}}
. For any separable data set, any initial point $\w(0)$, consider the normalized steepest descent updates above with a variable step size $\eta_t=\frac{1}{\sqrt{t+1}}$ for linear classification with the exponential loss $\ell(u)=\exp(-u)$. 

The margin of the iterates $\w(t)$ converge to max margin $\gamma$ at rate 
$t^{-1/2} \log t$: 
\begin{align} \frac{\w(t+1) ^\top \x_n }{\norm{\w(t+1)}}& \ge\gamma- \frac12 \frac{1+\log (t+1)}{ \gamma (2 \sqrt{t+2}-2)}-\frac{\log \cL(\w(0))}{\gamma ( 2 \sqrt{t+2} -2)}. \end{align} 
\end{theorem}

\begin{proof}
Since $\eta_t / L(w_t) < 1 / L(w_t)$, this stepsize choice satisfies the conditions of Lemma 12 in  \citep{gunasekar2018implicit}, and so the objective function is decreasing.

 The progress in one step is
\begin{align}
\cL(\w(t+1)) \le \cL(\w(t)) \exp \big( - \eta_t  \gamma_t ^2 +\frac{\eta_t ^2}{2} \gamma_t ^2)\\
\le \cL(w_0) \exp\big( - \sum_{i=0}^t \eta_i \gamma_i^2 + \sum_{i=0}^t \frac{\eta_i ^2}{2} \gamma_i ^2 \big)
\end{align}
The margin bound is
\begin{align}
\max_{n \in [N]} \exp( -\w(t+1) ^\top \x_n) \le \cL(\w(t+1) ) \le \exp\big( - \sum_{i=0}^t \eta_i  \gamma_i^2 + \sum_{i=0}^t \frac{\eta_i ^2}{2} \gamma_i ^2 \big).
\end{align}
By applying $-\log$,
\begin{align}
\min_{n \in [N]} \w(t+1) ^\top \x_n \ge \sum_{i=0}^t \eta_i \gamma_i^2  -  \sum_{i=0}^t \frac{\eta_i ^2}{2} \gamma_i ^2 -\log \cL(w_0).
\end{align}
The norm growth is
\begin{align}
\norm{\w(t+1)} = \norm{ \sum_{i=0}^t \eta_i \gamma_i p_i} \le \sum_{i=0}^t \eta_i \gamma_i .
\end{align}
Thus the margin of every point $j$ satisfies
\begin{align}
\frac{\w(t+1) ^\top \x_n }{\norm{\w(t+1)}} \ge \frac{\sum_{i=0}^t \eta_i  \gamma_i^2}{\sum_{i=0}^t \eta_i \gamma_i}  -  \frac{\sum_{i=0}^t \frac{\eta_i ^2}{2} \gamma_i ^2}{\sum_{i=0}^t \eta_i \gamma_i} -\frac{\log \cL(w_0)}{\sum_{i=0}^t \eta_i \gamma_i}.
\end{align}
Choose $\eta_i = \frac{1}{\sqrt{i+1}}$ so that $\sum_{i=0}^t \eta_i \ge 2 \sqrt{t+2} -2$. Since $\gamma_i \ge \gamma$, then
\begin{align}
\frac{\w(t+1) ^\top \x_n }{\norm{\w(t+1)}}& \ge \gamma \frac{\sum_{i=0}^t \eta_i  \gamma_i}{\sum_{i=0}^t \eta_i  \gamma_i^2} - \frac{\sum_{i=0}^t \frac{\eta_i ^2}{2} \gamma_i ^2}{\sum_{i=0}^t \eta_i \gamma_i}-\frac{\log \cL(w_0)}{\gamma ( 2 \sqrt{t+2} -2)}\\
&=\gamma  - \frac{\sum_{i=0}^t \frac{\eta_i ^2}{2} \gamma_i ^2}{\sum_{i=0}^t \eta_i \gamma_i}-\frac{\log \cL(w_0)}{\gamma ( 2 \sqrt{t+2} -2)}
\end{align}
Assume that $\norm{\x_n}_* \le 1$, so that $\norm{\nabla \cL(\w(t))}_* \le \max_{j \in[n]} \norm{\x_n}_* \cL(\w(t))$. Thus $\gamma_i \le 1$. Using this
\begin{align}
\frac{\w(t+1) ^\top \x_n }{\norm{\w(t+1)}}& \ge \gamma- \frac12 \frac{\sum_{i=0}^t \eta_i ^2}{ \gamma \sum_{i=0}^t \eta_i}-\frac{\log \cL(w_0)}{\gamma ( 2 \sqrt{t+2} -2)}\\
&\ge \gamma- \frac12 \frac{1+\log (t+1)}{ \gamma (2 \sqrt{t+2}-2)}-\frac{\log \cL(w_0)}{\gamma ( 2 \sqrt{t+2} -2)}
\end{align}
    \end{proof}
\section{Tail Analysis -- Proof sketch} \label{sec: Generic Tails}

In this section we describe non-rigorously the main ideas for our proofs on the results on the affect of loss tail on the convergence rate. In later appendix sections we give the complete proofs.
Recall we consider strictly monotone losses (Definition~\ref{def: l(u) assumptions}) with a general tail, given by
$-\ell^\prime(u)=\exp(-f(u))$, such that $f(u)$ is a strictly increasing function of $u$. 
\subsection{Convergence to the max-margin separator}
    From Lemma 1 in \cite{soudry2017implicit} we know that for linearly separable datasets, and smooth strictly monotonic loss functions, the iterates of GD entail that $\wNorm \rightarrow \infty$ and $\mathcal{L}(\w(t)) \rightarrow 0$ as $t\rightarrow \infty$, if the learning rate is sufficiently small. Now, if $\lim_{t\to\infty} \wvec(t)/\wNorm$  exists, then we can write $\wvec(t)=\wvec_\infty g(t)+\rhoVec$ where $\lim_{t\to\infty}g(t)=\infty$, $\forall n:\ \xnT\wvec_\infty>0$  and $\lim_{t\to\infty} \norm{\rhoVec}/g(t)=0$. Using this result, the gradients can be written as:
    \begin{equation}
	-\derL(\wvec(t)) = \sumn \e\left(-f\left(\wvec(t)^\top\xn\right)\right)\xn
	= \sumn \e\left(-f\left(g(t)\wvec_\infty^\top\xn +\rhoVec^\top\xn\right)\right)\xn
	\end{equation}
    As $g(t) \to \infty$ the exponents become more negative, since $f(t)$ is an increasing function, $\forall n:\ \xnT\wvec_\infty>0$ and $\norm{\rhoVec}=o(g(t))$. Therefore, if $f$ is increasing sufficiently fast, only samples with minimal margin $\wvec_\infty^\top \xn$ contribute to the sum. Examining the gradient descent dynamics, this implies that $\wvec(t)$ and also its scaling $\what = \frac{\wvec(t)}{\min_n \wvec_\infty^\top \xn}$ are a linear non negative combination of support vectors:
    \begin{equation}
    	\what = \sumn \alpha_n \xn\,\ \forall n: \left( \alpha_n \ge 0 \text{ and } \what^\top \xn =1 \right) \text{ or } \left( \alpha_n = 0 \text{ and } \what^\top \xn >1 \right)
    \end{equation}
    these are exactly the KKT conditions for the SVM problem and we can conclude that $\wvec_\infty$ is proportional to $\what$.

\subsection{Calculation of rates and validity conditions}
	Next, we aim to find $g(t)$ and $||\rhoVec||$ so we can calculate the convergence rates. Also, we aim to find what are the conditions on $f$ so this calculation would break.
    To simplify our analysis we examine the continuous time version of GD, in which we take the limit $\eta \rightarrow 0$. In this limit
	\begin{equation}
	\dot{\wvec}(t) = -\derL(\wvec(t)) = \sum_{n=1}^N \e\left(-f\left(\xnT\wvec(t)\right)\right)\xn,
	\end{equation}
	We define $\mathcal{S}= \argmin_n{\what^\top \xn}$, i.e., the set of indices of support vectors, so $\forall n \in \set $ we have $\what^{\top}\x_n=1$. From our reasoning above, if $f$ increases fast enough, then we expect that the contribution of the non-support vectors to the gradient would be negligible, and therefore
	\begin{equation} \label{eq: GD-ODE}
	\dot{\wvec}(t) \approx \sumnsv \e\left(-f\left(\xnT\wvec(t)\right)\right)\xn,
	\end{equation}
    Additionally, if we assume that $\rhoVec$ converges to some direction $\genericConstVec$,  and $\mathbf{b}$ is some vector orthogonal to the support vectors (if such direction exists), then we expect that asymptotic solution to be of the form
	$$ \wvec(t) = g(t)\what+h(t)\genericConstVec + \mathbf{b},\ \mathrm{s.t.} \,\, h(t) = o(g(t)). $$
	In order for this to be a valid solution, it must satisfy eq.~\eqref{eq: GD-ODE}. We verify this by substitution and examining the leading orders
\begin{align*}
	\dot{g}(t)\what & \approx \sumnsv \exp\left(-f\left(g(t)\xnT\what+h(t)\xnT\genericConstVec\right)\right)\xn\\
	& \overset{(1)}{\approx} \sumnsv \exp\left(-\left(f\left(g(t)\xnT\what\right)+f'\left(g(t)\xnT\what\right)h(t)\xnT\genericConstVec\right)\right)\xn\\
	&\overset{(2)}{\approx} \e(-f(g(t)))\sumnsv \e(-f'\left(g(t)\right)h(t)\xnT\genericConstVec)\xn ,
	\end{align*}
    
   where in (1) we used a Taylor approximation and in (2) we used that $\what^{\top} \x_n = 1,\, \forall n \in \set $. For the last equation to to hold, we require
	\begin{equation} \label{eq: g definition}
	\dot{g}(t) = \e\left(-f\left(g(t)\right)\right),\ h(t) = \frac{1}{f'\left(g(t)\right)}
	\end{equation}
    and $\genericConstVec$ satisfies the equations:
	\begin{flalign} \label{eq: a def} 
	    &\forall n\in \mathcal{S}\ :\ \e(-\xnT\genericConstVec)=\alpha_n,\ \bar{\op} \genericConstVec = 0,
	\end{flalign}
	where we define $\op\in \mathbb{R}^d$ as the orthogonal projection matrix to the subspace spanned by the support vectors, and $\bar{\op}=I-\op$ as the complementary projection matrix. Equation \ref{eq: a def} has a unique solution for almost every dataset from Lemma 8 in \cite{soudry2017implicit}. Specifically, this equation does not have a solution when one of the $\alpha_n$ must be equal to zero (i.e., some support vectors exert ``zero force'' on the the margin --- and this happens only in measure zero cases).
    
    Since we assume that $h(t)=o(g(t))$ we must have $\lim_{t\to\infty} g(t)f'(g(t))=\lim_{u\to\infty} uf'(u)=\infty$ meaning $f'(t)=\omega(t^{-1})$ which implies $f(t)=\omega(\log(t))$. This condition must hold for this analysis to make sense. Moreover, the differential equation that defines $g(t)$ (eq.~\eqref{eq: g definition}) is generally intractable. However, if the condition $\log\left(f^{\prime}\left(t\right)\right)=o\left(f\left(t\right)\right)$
holds (which is true for many functions), then we can approximate 
\[
\dot{g}\left(t\right)\approx\exp\left(-f\left(g\left(t\right)\right)-\log\left(f^{\prime}\left(g\left(t\right)\right)\right)\right)
\]
which has a closed form solution 
\[
g\left(t\right)=f^{-1}\left(\log\left(t+C\right)\right).
\] 

    \section{Proof of Theorem \ref{Theorem: convergence to max margin for general tail}} \label{sec: convergence to max margin for general tail}
    \subsection{Auxillary lemma}
    In order to prove Theorem \ref{Theorem: convergence to max margin for general tail}, we need to expand Lemma 8 in \cite{gunasekar2018implicit} to general loss functions. To this end, we first prove the following lemma:
	\begin{lemma} \label{lemma: f(g(t1))-f(g(t2)) to infty}
	Let $f$ be a differentiable function so that $f'(u)=\omega\left( u^{-1}\right)$ and let $g$ be a function satisfying $\lim_{t\to\infty} g(t)=\infty$. Then, for any $a_1>a_2$: 
	\[
	    \lim_{t\to\infty} \left( f\left(a_1\cdot g\left(t\right)\right)-f\left(a_2 \cdot g\left(t\right)\right)\right)=\infty
	\]
	\end{lemma}
	\begin{proof}
	We denote $\tilde{f}(t)\triangleq f\left(a_1\cdot g\left(t\right)\right)-f\left(a_2 \cdot g\left(t\right)\right)$. Let $M>0$ be any number. We need to prove that $\exists \tilde{t}>0$ s.t. $\forall t>\tilde{t}$: $\tilde{f}(t)\ge M$. We show this in the following steps:
	\begin{compactitem}
	\item First, 
	$f'(u)=\omega\left( u^{-1}\right)$ implies that $\forall M',\exists u_1$ so that $\forall u>u_1: f'(u)\ge M' u^{-1}$. We choose $M'=\dfrac{M}{\log(a_1/a_2)}>0$ (since $a_1>a_2$). 
	\item Secondly, $\lim_{t\to\infty} g(t)=\infty$ implies that $\exists t_2>0$ so that $\forall t>t_2:\ a_2\cdot g\left(t\right)>u_1$ (the corresponding $u_1$ for the $M'$ we chose). 
    \end{compactitem}
We define $\tilde{t}=t_2$, so that  $\forall t>\tilde{t}$:
	\[
	    f\left(a_1\cdot g\left(t\right)\right)-f\left(a_2 \cdot g\left(t\right)\right)=\int_{a_2\cdot g\left(t\right)}^{a_1 \cdot g\left(t\right)} f'(\tau)d\tau \ge \int_{a_2\cdot g\left(t\right)}^{a_1 \cdot g\left(t\right)} M' \frac{1}{\tau} d\tau = M' \log\left(\frac{a_1\cdot g\left(t\right)}{a_2\cdot g\left(t\right)}\right)=M.
	\]
	
	\end{proof}
	Using this lemma, we can prove an extended version of Lemma 8 in \cite{gunasekar2018implicit}, which only applied to exponential tail.
	\begin{lemma} \label{lemma: every accumulation point of z/norm(z) is a linear combination of sv}
	    For almost all datasets $\left\{ \mathbf{x}_{n},y_{n}\right\}_{n=1}^N$ which are linearly separable, consider any sequence $\eqw(t)$ that minimizes the empirical loss in eq. \eqref{eq: general loss functions}, i.e., $\mathcal{L}\left(\eqw(t)\right)\to0$ with a strictly monotone loss function $\ell$ satisfying Assumption  \ref{def: l(u) assumptions}, i.e., 
		$\ell'(u) = -\e(-f(u))$,
		with $f'(u)>0$ and $ f'(u)=\omega\left(u^{-1}\right)$.
		
		If \begin{inparaenum}[(a)] \item $\bar{\eqw}_{\infty}\triangleq\lim\limits_{t\to\infty}\frac{\eqw(t)}{\norm{\eqw(t)}}$ exists and has a positive margin, and \item  $ \vect{z}_\infty\triangleq\lim\limits_{t\to\infty} \frac{-\nabla_{\eqw} \mathcal{L}\left(\eqw(t)\right)}{\norm{\nabla_{\eqw} \mathcal{L}\left(\eqw(t)\right)}}$ exists \end{inparaenum}, then 
        $\exists \left\{ \alpha_n\ge0\right\}_{n=1}^N$ s.t.
		\[
		\vect{z}_\infty=\lim_{t\to\infty} \frac{-\nabla_{\eqw} \mathcal{L}\left(\eqw(t)\right)}{\norm{\nabla_{\eqw} \mathcal{L}\left(\eqw(t)\right)}}=\lim_{t\to\infty}\frac{-\sumnsvinf \ell^{\prime}\left(\eqw(t)^\top\xn\right)\xn}{\norm{\sumnsvinf \ell^{\prime}\left(\eqw(t)^\top\xn\right)\xn}}
		=\sumnsvinf \alpha_n y_n\xn
		\]
		where $S'$ denotes the indices  of support vectors (data points with the smallest margin) of the limit direction $\bar{\eqw}_{\infty}$ given by $\displaystyle \setinf=\left\{ n: y_n\bar{\eqw}_{\infty}^\top\xn = \min_{\bar{n}} y_{\bar{n}} \bar{\eqw}_{\infty}^\top\vect{x}_{\bar{n}} \right\}$.
	\end{lemma}
	\begin{proof}
	$\left\{ \mathbf{x}_{n},y_{n}\right\} _{n=1}^{N}$ is a linearly separable dataset. We assume that
    $\forall n:\,y_{n}=1$ without loss of
    generality, since we can always re-define $y_{n}\mathbf{x}_{n}$ as
    $\mathbf{x}_{n}$. 

	From the assumption in the lemma, we have that $\mathcal{L}(\eqw(t)) \rightarrow 0$ as $t\rightarrow \infty$ and that the loss is a strictly monotone loss function (Definition \ref{def: l(u) assumptions}). This implies that $\norm{\eqw(t)} \rightarrow \infty$ as $t\rightarrow \infty$.
	Also, since $\dfrac{\eqw(t)}{\norm{\eqw(t)}}$ converges in direction to $\bar{\eqw}_{\infty}$ we can write $\eqw(t)=\bar{\eqw}_{\infty} g(t)+\bm{\rho}(t)$ where $\lim\limits_{t\to\infty} g(t) = \infty$ and $\lim\limits_{t\to\infty} \dfrac{\norm{\bm{\rho}(t)}}{g(t)} =0$. From these conditions we also have that $\bar{\eqw}_\infty^\top\vect{X}>0$, where recall that $\mathbf{X}$  denotes the data matrix $\mathbf{X}=[\vect{x}_1,\dots,\vect{x}_N]$. 
	
	We define the following additional notations:
	\begin{itemize}
	    \item $\gamma = \min\limits_n \bar{\eqw}_\infty^\top\xn>0$. This is the minimal margin of $\bar{\eqw}_\infty$.
	    \item $\displaystyle \setinf=\left\{ n: \bar{\eqw}_{\infty}^\top\xn = \gamma \right\}$.
	    \item $\bar{\gamma} = \min_{n\notin\set} \bar{\eqw}_\infty^\top\xn$. This is the second smallest margin of $\bar{\eqw}_\infty$.
	    \item $\bar{\gamma}_n= \bar{\eqw}_\infty^\top\xn$. This is the margin of the datapoint $\xn$
	    \item $B= \max_n \norm{\xn}$.
	\end{itemize}
	Since $\lim\limits_{t\to\infty} \dfrac{\norm{\bm{\rho}(t)}}{g(t)} =0$ we have that $\forall \epsilon_1,\epsilon_2>0, \exists t_{\epsilon}>0$ such that $\forall t>t_\epsilon$, the following holds 
	\begin{align}
	    & \max_{n\in\set} \rhoVec^\top\xn\le \norm{\rhoVec}B\le \epsilon_1 \gamma g(t), \label{eq: max <rho,xn> upper bound}\\
	    & \min_{n\notin\set}\rhoVec^\top\xn \ge -\norm{\rhoVec}B \ge -\epsilon_2 \bar{\gamma} g(t). \label{eq: min <rho,xn> lower bound}
	\end{align}
	For the general loss we defined in the lemma the gradients are given by
	\begin{align*}
	    -\nabla_{\eqw} \mathcal{L}\left( \eqw(t) \right)&=\sumn \exp\left(-f\left(g(t)\xnT\winf+\rhoVec^\top\xn\right)\right)\xn\\
	    &= \sumnsvinf \exp\left(-f\left(g(t)\gamma+\rhoVec^\top\xn\right)\right)\xn
	    +
	    \sumnnsvinf \exp\left(-f\left(g(t)\bar{\gamma}_n+\rhoVec^\top\xn\right)\right)\xn\\
	    &= I(t)+II(t),
	\end{align*}
	where \mbox{$I(t) = \sumnsvinf \exp\left(-f\left(g(t)\gamma+\rhoVec^\top\xn\right)\right)\xn$ and $II(t)=\sumnnsvinf \exp\left(-f\left(g(t)\bar{\gamma}_n+\rhoVec^\top\xn\right)\right)\xn$.} \\
	\textbf{Step 1:} Show that $\displaystyle\lim_{t\to\infty} \frac{\norm{II(t)}}{\norm{I(t)}}=0$.\\
	We define $\vect{v}(t)\in\R^{\abs{\setinf}}$ as $\forall k=1,...,\abs{\setinf}$: $\vect{v}_k(t)=\exp\left(-f\left(g(t)\gamma+\rhoVec^\top\vect{x}_{i_k}\right)\right)$ where $i_{1},...,i_{\abs{\setinf}}\in\setinf$. In addition, with these indices, we define $X_{\setinf}=\left(
  \begin{array}{c c c}
    \mid &  &  \mid \\
    \vect{x}_{i_1}  & \ldots & \vect{x}_{i_{\abs{\setinf}}}    \\
    \mid &  & \mid 
  \end{array}
\right)$. We denote $\sigma_{\min}\left(\vect{X}_{\setinf}\right)$ as the minimal singular value of $\vect{X}_{\setinf}$. For almost all datasets, $\sigma_{\min}\left(\vect{X}_{\setinf}\right)>0$ (from Claim 1 in \cite{gunasekar2018implicit}).

Using these notations, we can lower bound $\norm{I(t)}$. $\forall t>t_{\epsilon}$:
	\begin{align} \label{eq: I(t) lower bound}
	    \norm{I(t)}&=\norm{\sumnsvinf \exp\left(-f\left(g(t)\gamma+\rhoVec^\top\xn\right)\right)\xn}\overset{(1)}{\ge} \sigma_{\min}\left( X_\setinf \right) \min_{n\in\setinf} \exp\left(-f\left(g(t)\gamma+\rhoVec^\top\xn\right)\right) \nonumber\\
	    & \overset{(2)}{=}  \sigma_{\min}\left( X_\setinf \right)\exp\left(-f\left(g(t)\gamma+\max_{n\in\setinf} \rhoVec^\top\xn\right)\right)\overset{(3)}{\ge}
	    \sigma_{\min}\left( X_\setinf \right)\exp\left(-f\left(g(t)\gamma\left(1+\epsilon_1 \right)\right)\right),
	\end{align}
	where in (1) we used the fact that
	\[
	    \norm{\sumnsvinf \exp\left(-f\left(g(t)\gamma+\rhoVec^\top\xn\right)\right)\xn} = \norm{X_\setinf \vect{v}}\ge \sigma_{\min} \left( X_\setinf \right) \norm{\vect{v}}\ge \sigma_{\min} \left( X_\setinf \right) \min_{n\in\setinf} {\vect{v}_n},
	\]
	in (2) we used the fact that $f$ is strictly increasing (or $f'(u)>0$) and in (3) we used eq.~\eqref{eq: max <rho,xn> upper bound}.
	
	Next, we upper bound $\norm{II(t)}$. $\forall t>t_{\epsilon}$:
	\begin{align} \label{eq: II(t) upper bound}
	    \norm{II(t)}&=\norm{\sumnnsvinf \exp\left(-f\left(g(t)\bar{\gamma}_n+\rhoVec^\top\xn\right)\right)\xn} 
	     \overset{(1)} {\le} NB \max_{n\notin \setinf}  \exp\left(-f\left(g\left(t\right)\bar{\gamma}_n+\rhoVec^\top\xn\right)\right)\nonumber\\
	    & \overset{(2)}{=}  NB\exp\left(-f\left(\min_{n\notin\setinf}g(t)\bar{\gamma}_n+\min_{n\notin\setinf}\rhoVec^\top\xn\right)\right) \overset{(3)}{\le} NB\exp\left(-f\left(g(t)\bar{\gamma}\left(1-\epsilon_2 \right)\right)\right),
	\end{align}
	where in (1) we used the triangle inequality along with $\norm{\xn}\le B$, in (2) we used the fact that $f$ is  strictly increasing and in (3) we used eq.~\eqref{eq: min <rho,xn> lower bound} and $\forall n\notin\setinf: \bar{\gamma}_n\ge\bar{\gamma}$.
	
	Combining equations \ref{eq: I(t) lower bound}, \ref{eq: II(t) upper bound} we have that $\forall \epsilon_1,\epsilon_2>0,\exists t_{\epsilon}>0$ so that $\forall t>t_\epsilon$:
	\[
	    \frac{\norm{II(t)}}{\norm{I(t)}}\le \frac{NB}{\sigma_{\min}\left( X_\setinf \right)}\exp\left(-\left(f\left(g(t)\bar{\gamma}\left(1-\epsilon_2 \right)\right)\right)-f\left(g(t)\gamma\left(1+\epsilon_1 \right)\right)\right).
	\]
	For the choice $\epsilon_1 = \dfrac{\bar{\gamma}-\gamma}{4\gamma}, \ \epsilon_2 = \dfrac{\bar{\gamma}-\gamma}{4\bar{\gamma}}$ we have that $\bar{\gamma}\left(1-\epsilon_2 \right)>\gamma\left(1+\epsilon_1 \right)$ and thus, using Lemma \ref{lemma: f(g(t1))-f(g(t2)) to infty} and Squeeze theorem we have that $\displaystyle \lim_{t\to\infty} \frac{\norm{II(t)}}{\norm{I(t)}}=0$.
	
	\noindent\textbf{Step 2:} Using $\displaystyle\lim_{t\to\infty} \frac{\norm{II(t)}}{\norm{I(t)}}=0$ show that, $\displaystyle \lim_{t\to\infty}\frac{-\nabla_{\eqw} \mathcal{L}\left(\eqw(t)\right)}{\norm{\nabla_{\eqw} \mathcal{L}\left(\eqw(t)\right)}}
	= \sumnsv \alpha_n \xn$ for some $\alpha_n\ge0$ (if the limit exists).

	Since $\displaystyle\lim_{t\to\infty} \frac{\norm{II(t)}}{\norm{I(t)}}=0$, we have  $\frac{\norm{I(t)+II(t)}}{\norm{I(t)}}$ satisfying 
   $1- \frac{\norm{II(t)}}{\norm{I(t)}}\le\frac{\norm{I(t)+II(t)}}{\norm{I(t)}}\le 1+ \frac{\norm{II(t)}}{\norm{I(t)}}.$
   Using squeeze theorem, we get $\frac{\norm{I(t)+II(t)}}{\norm{I(t)}}\to 1$.
   
   Now consdier the limit direction of gradients, 
	\[
	\frac{-\nabla_{\eqw} \mathcal{L}\left(\eqw(t)\right)}{\norm{\nabla_{\eqw} \mathcal{L}\left(\eqw(t)\right)}} = \frac{I(t)}{\norm{I(t)+II(t)}}+\frac{II(t)}{\norm{I(t)+II(t)}}.
	\]
	\begin{compactitem}
	\item $\displaystyle \norm{\frac{II(t)}{\norm{I(t)+II(t)}}}=\frac{\norm{II(t)}}{\norm{I(t)}}\frac{\norm{I(t)}}{I(t)+II(t)}\xrightarrow{t\to\infty} 0$. 
	\item Similarly, $\lim_{t\to\infty}\frac{\norm{I(t)}}{\norm{I(t)+II(t)}}=\lim_{t\to\infty}\frac{{I(t)}}{\norm{I(t)}}\frac{\norm{I(t)}}{I(t)+II(t)}
	\lim_{t\to\infty}\frac{{I(t)}}{\norm{I(t)}}$.
	\item Finally, since $I(t) \propto \sumnsv v_n(t) \xn$ for $v_k(t)>0$, then every limit point of $\displaystyle \frac{-\nabla_{\eqw} \mathcal{L}\left(\eqw(t)\right)}{\norm{\nabla_{\eqw} \mathcal{L}\left(\eqw(t)\right)}}$ converges to $\sumnsv \alpha_n \xn$ for some $\alpha_n\ge0$.
	\end{compactitem}
	Summarizing, if $\displaystyle \lim_{t\to\infty}\frac{-\nabla_{\eqw} \mathcal{L}\left(\eqw(t)\right)}{\norm{\nabla_{\eqw} \mathcal{L}\left(\eqw(t)\right)}}$ exists, then  \[ \frac{-\nabla_{\eqw} \mathcal{L}\left(\eqw(t)\right)}{\norm{\nabla_{\eqw} \mathcal{L}\left(\eqw(t)\right)}}=\lim_{t\to\infty}\frac{-\sumnsvinf \ell^{\prime}\left(\eqw(t)^\top\xn\right)\xn}{\norm{\sumnsvinf \ell^{\prime}\left(\eqw(t)^\top\xn\right)\xn}}=\sumnsvinf \alpha_n \xn.\]
	This completes the proof for general tail.
	\end{proof}
	\remove{
	The following lemma is a paraphrasing of Lemma 8 in \cite{gunasekar2018conv}.
	\mnote{need to update lemma to fix bug Suriya found (there is an updated version in the convnet nips version)}
	\begin{lemma}\label{lem:uconv}
	 For a sequence $\np(t)$ that asymptotically minimizes \ref{eq: Loss for net parameters}, i.e. $\mathcal{L}_{P}\left(\np(t)\right)\to0$, if the incremental updates $\np(t+1)-\np(t)$ converges in direction, then the sequence of iterates $\np(t)$ and the corresponding sequence of linear predictors $\eqw(t)=\P\left(\np\left(t\right)\right)$ also converge in direction, and moreover $\lim\limits_{t\to\infty}\frac{\np(t)}{\norm{\np(t)}}=\lim\limits_{t\to\infty}\frac{\np(t+1)-\np(t)}{\norm{\np(t+1)-\np(t)}}.$
	\end{lemma}
	}
	\subsection{Theorem \ref{Theorem: convergence to max margin for general tail} Proof}
	\theoremMaxMarginGeneralTail*
    The proof for this theorem is very similar to Theorem 1 proof in \cite{gunasekar2018conv}. The main exception is that instead of using Lemma 8 in \cite{gunasekar2018conv}, which only applies to exponential loss, we use the extended lemma which was proved in the previous section (Lemma \ref{lemma: every accumulation point of z/norm(z) is a linear combination of sv}). 
	\begin{proof}
        Let $\np(t)=[\lw{l}(t)\in\R^{D_{l-1}\times D_l}]_{l=1}^L$ denote the iterates of individual matrices $\lw{l}(t)$ along the gradient descent path, and $\eqw(t)=\lw{1}(t)...\lw{l}(t)$ denote the corresponding sequence of linear predictors. \\
        
        We first introduce the following notation. 
\begin{compactenum}
\item Let $\winfty=\lim\limits_{t\to \infty}\frac{\np(t)}{\norm{\np(t)}}$ denote the limit direction of the parameters, with component matrices in each layer denoted as $\winfty=[\lwinfty{l}]$. We have that for some $\bdelta_{\lw{l}}(t)\to 0$ the following representation of $\lw{l}(t)$ holds.
\begin{align}
\lw{l}(t)&=\lwinfty{l} g(t)+\bdelta_{\lw{l}}(t)\,g(t), \label{eq:ufcn}
\end{align}
where $g(t)=\norm{\np(t)}$ and $\bdelta_{\lw{l}}(t)\to 0$.
\item For $0<l_1< l_2\le L$, denote $\lw{l_1:l_2}(t)=\lw{l_1}(t) \lw{l_1+1}(t)\ldots\lw{l_2}(t)$ and $\lwinfty{l_1:l_2}=\lwinfty{l_1} \lwinfty{l_1+1}\ldots\lwinfty{l_2}$. Using eq.~\eqref{eq:ufcn}, we can check by induction on $l_2-l_1$ that $\exists$ $\bdelta_{\lw{l_1:l_2}}(t)\to0$ such that the following holds,
        \begin{align}
        \lw{l_1:l_2}(t)&=\lwinfty{l_1:l_2}\, g(t)^{l_2-l_1+1}+\bdelta_{\lw{l_1:l_2}}(t)\,g(t)^{l_2-l_1+1}. 
        \label{eq:ulfcn}
        \end{align}

\item Denote the gradients with respect to linear predictors as $\vect{z}(t)=-\nabla_{\eqw}\mathcal{L}(\eqw(t))$. Since we assume that $\vect{z}(t)$ converges in direction, let $\bar{\vect{z}}^{\infty}=\lim\limits_{t\to\infty} \frac{\vect{z}(t)}{\norm{\vect{z}(t)}}$. Denoting $p(t)=\norm{\vect{z}(t)}$, for some $\bdelta_{\vect{z}}(t)\to0$, we can write $\vect{z}(t)$ as, 
\begin{equation}
\vect{z}(t)=\bar{\vect{z}}^\infty p(t)+\bdelta_{\vect{z}}(t)\,p(t).
\label{eq:fcn-z}
\end{equation}
\item From Lemma~\ref{lemma: every accumulation point of z/norm(z) is a linear combination of sv}, we have that $\exists \{\alpha_n\}_{n\in S_\infty}$ such that $\bar{\vect{z}}^\infty=\sum\limits_{n\in S_\infty}\alpha_n\,y_n\x_n$, where $S_\infty$ are support vectors of $\bar{\eqw}^\infty=\lim\limits_{t\to\infty}\frac{\eqw(t)}{\norm{\eqw(t)}}$.
\end{compactenum}

The proof of Theorem \ref{Theorem: convergence to max margin for general tail} is fairly straight forward from using Lemma \ref{lemma: every accumulation point of z/norm(z) is a linear combination of sv}. In the following arguments we show that a positive scaling $\tilde{\eqw}_\infty=\gamma \lim\limits_{t\to\infty} \frac{\P(\np(t))}{\norm{\P(\np(t))}}$ satisfies the following KKT conditions for the optimality of explicitly regularized convex problem in eq.~\eqref{eq: normalized max margin equation}:
        \begin{equation}
        \begin{split}
        \exists \{\alpha_n\}_{n=1}^N\quad\text{ s.t. }\quad& \forall n,\, \ip{\xn}{\eqw} \ge 1, \eqw = \sum_n\alpha_n\, \xn,\\
        & \forall n, \alpha_n\ge0 \text{ and } \alpha_n=0, \forall i\notin {S}:=\{i\in[N]: \ip{\xn}{\eqw}=1\}.
        \end{split}
        \label{eq:kkt-fcn}
        \end{equation}
Since $\bar{\eqw}_\infty\triangleq\lwinfty{1:L}$ has strictly positive margin, we can scale $\lwinfty{1:L}$ to get $\tilde{\eqw}^\infty=\gamma \bar{\eqw}_\infty$ with unit margin, i.e., $\forall n,\,\ip{\xn}{\tilde{\eqw}^\infty} \ge 1$.
For the dual variables, we again use a positive scaling of $\alpha_n$ from Lemma~\ref{lemma: every accumulation point of z/norm(z) is a linear combination of sv}, such that $\bar{\vect{z}}^\infty=\sum_{n\in S_\infty}\alpha_n\,\xn$. In order to prove the theorem, we need to show that $\tilde{\eqw}^\infty\propto \bar{\vect{z}}^\infty$ or equivalently $\lwinfty{1:L}\propto\bar{\vect{z}}^\infty$. 

 Computing the gradients descent updates for $\lw{1}(t)$, we have 
            \[
            \lw{1}(t+1)-\lw{1}(t)=-\eta_t\nabla_{\lw{1}}\mathcal{L}_{\P}(\np(t))=-\eta_t\nabla_{\lw{1}}\mathcal{L}(\lw{1:L})=-\eta_t\nabla_{\eqw}\mathcal{L}(\eqw(t)){\lw{2:L}(t)}^\top=\eta_t\vect{z}(t){\lw{2:L}(t)}^\top\,.
            \]
            Using eq.~\eqref{eq:fcn-z} we obtain
            \begin{align*}
                \Delta \lw{1}\triangleq\lw{1}(t+1)-\lw{1}(t) &= \eta_t \Big( \bar{\vect{z}}^\infty p(t)+\bdelta_{\vect{z}}(t)\,p(t) \Big) \left(\lwinfty{2:L}\, g(t)^{L-1}+\bdelta_{\lw{2:L}}(t)\,g(t)^{L-1}\right)^\top \\
                &\overset{(1)}{=}\left(\eta_t p(t)g(t)^{L-1}\right)\left(\bar{\vect{z}}^\infty \right(\lwinfty{2:L}\left)^\top  +\bdelta(t)\right)\,,
            \end{align*}
            where in $(1)$ $\bdelta(t) = \bdelta_{\vect{z}}(t)\left(\lwinfty{2:L}+\bdelta_{\lw{2:L}}(t)\right)^\top + \bdelta_{\vect{z}}(t) \bdelta_{\lw{2:L}}(t)^\top \to0$.
            This implies that $\lw{1}(t+1)-\lw{1}(t)$ converge in direction with positive margin (see Claim 1 in \cite{gunasekar2018conv}).

Summing the last equation over $t$ we obtain
    \begin{equation} \label{eq:last1}
        \lw{1}(t)-\lw{1}(0) = \bar{\vect{z}}^\infty (\lwinfty{2:L})^\top \sum_{u<t}\eta_u p(u)g(u)^{L-1} + \sum_{u<t}\eta_u p(u)g(u)^{L-1}\bdelta(u)\,.
    \end{equation}
From Claim 1 in \cite{gunasekar2018conv} we have that $\norm{\bar{\vect{z}}^\infty (\lwinfty{2:L})^\top}>0$ and $ \sum_{u<t}\eta_t p(t)g(t)^{L-1}\to\infty$.
This implies that the sequence $b_t=\sum_{u<t}\eta_u p(u) g(u)^{L-1}$ is monotonic increasing and diverging. Thus, for $a_t=\sum_{u<t}\bdelta(u)\eta_u p(u) g(u)^{L-1}$, using Stolz-Cesaro theorem~(Theorem 11 in \cite{gunasekar2018conv}), we have  
\begin{flalign}
\nonumber\lim_{t\to\infty}\frac{a_t}{b_t}=\lim_{t\to\infty}\frac{\sum_{u<t}\bdelta(u)\eta_u p(u) g(u)^{L-1}}{\sum_{u<t}\eta_u p(u) g(u)^{L-1}}=\lim_{t\to\infty}\frac{a_{t+1}-a_{t}}{b_{t+1}-b_{t}}=\lim_{t\to\infty}\bdelta(t)=0.\\
\implies \text{for }\bdelta(t)_2\to0,\; \text{we have }\sum_{u<t}\bdelta(u)\eta_u p(u) g(u)^{L-1}=\bdelta_2(t)\sum_{u<t}\eta_u p(u) g(u)^{L-1}.
\label{eq:last2}
\end{flalign}

Substituting eq. \eqref{eq:last2} in eq. \eqref{eq:last1}, we have 
\begin{equation} 
        \lw{1}(t) = \left[\bar{\vect{z}}^\infty (\lwinfty{2:L})^\top +\bdelta_3(t) \right]\sum_{u<t}\eta_u p(u)g(u)^{L-1}\,,
    \end{equation}
    where we defined $\bdelta_3(t) \triangleq \bdelta_2(t) + \lw{1}(0)/ \sum_{u<t}\eta_u p(u)g(u)^{L-1} \to 0$. Note that the last equation implies that $\lw{1}(t)$ and $\Delta \lw{1}$ converge in the same direction. Multiplying the last equation from the right with $\lwinfty{2:L}$ we get that
    \[
        \lwinfty{1:L} \propto \bar{\vect{z}}^\infty = \sum\limits_{n\in S_\infty}\alpha_n\,y_n\x_n,
    \]
    where $S_\infty$ are support vectors of $\bar{\eqw}^\infty=\lim\limits_{t\to\infty}\frac{\eqw(t)}{\norm{\eqw(t)}}$. This concludes the proof for Theorem \ref{Theorem: convergence to max margin for general tail}.
    
    We continue our derivation in order to obtain a useful result which will be used in following proofs. Since $\lw{1}(t)$ and $\Delta \lw{1}$ converge in the same direction and 
            \begin{equation}
                \lw{l}(t) = \lwinfty{l}g(t)+\delta_{\lw{l}}(t)g(t) \label{eq:ufcn}
            \end{equation}
             we have that
            \begin{align}
                \Delta\lw{1}(t)&=\lwinfty{1} h(t)+\bdelta_{\Delta\lw{1}}(t)\,h(t)\label{eq:dufcn}
                \end{align}
        where $h(t)=\norm{\Delta\np(t)}$, $g(t)=\sum_{u<t} \eta_u h(u)\to\infty$, and $\bdelta_{\Delta\lw{l}}(t),\bdelta_{\lw{l}}[t]\to 0$. 
            Consider the following arguments on $\Delta\lw{1}(t) \lw{2:L}(t)$,
            \begin{equation}
            \begin{split}
            &\Delta\lw{1}(t) \lw{2:L}(t)=\vect{z}(t)\norm{\lw{2:L}(t)}^2\\
            \overset{(a)}\implies &\left(\lwinfty{1} h(t)+\bdelta_{\Delta\lw{1}}(t)\,h(t)\right)\left(\lwinfty{2:L}\, g(t)^{L-1}+\bdelta_{\lw{2:L}}(t)\,g(t)^{L-1}\right)=\vect{z}(t)\norm{\lw{2:L}(t)}^2\\
            \overset{(b)}\implies &\frac{\vect{z}(t)\norm{\lw{2:L}(t)}^2}{h(t)\,g(t)^{L-1}}=\lwinfty{1:L}+\bdelta(t)=\bar{\eqw}_\infty+\tilde{\bdelta}(t),
            \end{split}
            \label{eq:fcn-slackness}
            \end{equation}
            where in $(a)$, we used eqs. \ref{eq:dufcn}-\ref{eq:ufcn}, and in $(b)$ we have $\tilde\bdelta(t)=\bdelta_{\Delta\lw{1}}(t)\bdelta_{\lw{2:L}}(t)+\bdelta_{\Delta\lw{1}}(t)\lwinfty{2:L}+\lwinfty{1}\bdelta_{\lw{2:L}}(t)\to0$. 
             Denote $s(t):=\frac{\norm{\vect{z}(t)}\norm{\lw{2:L}(t)}^2}{h(t)\,g(t)^{L-1}}$ From eq.~\eqref{eq:fcn-slackness} using triangle inequality we have that 
            \begin{equation}
            \norm{\bar{\eqw}_\infty}-\norm{\tilde\bdelta(t)}\le s(t)\le \norm{\bar{\eqw}_\infty}+\norm{\tilde\bdelta(t)}.
            \end{equation}
            Since $\tilde\bdelta(t)\to0$,  by squeeze theorem, we have that $\lim_{t\to\infty}s(t)=\norm{\bar{\eqw}_\infty}$. Using this in eq.~\eqref{eq:fcn-slackness}, we get the following:
            \begin{equation}
            \frac{\vect{z}(t)}{\norm{\vect{z}(t)}}=\frac{\bar{\eqw}_\infty}{s(t)}+\frac{\bdelta(t)}{s(t)}.
            \label{eq:fcn-slackness1}
            \end{equation}
            \remove{
\remove{
\paragraph{Showing KKT conditions for $\tilde{\w}^\infty\propto\P_{full}(\bar{\u}^\infty)$.}
Using our notation described above, we have $\bar{\u}^\infty_{1:L}=\P_{full}(\bar{\u}^\infty)$.
In the following arguments we show that a positive scaling $\tilde{\w}^\infty=\gamma \bar{\u}^\infty_{1:L}$ satisfies the following KKT conditions for the optimality of $\ell_2$ maximum-margin  problem in eq. \eqref{eq:gd-fcn}:
\begin{equation}
\begin{split}
\exists \{\alpha_n\}_{n=1}^N\quad\st\quad& \forall n,\,y_n\innerprod{\x_n}{\w} \ge 1, \w = \sum_n\alpha_n\,y_n\x_{n},\\
& \forall n, \alpha_n\ge0 \tand \alpha_n=0, \forall i\notin {S}:=\{i\in[N]:y_{n}\innerprod{\x_n}{\w}=1\}.
\end{split}
\label{eq:kkt-fcn}
\end{equation}

As we saw in proof of Theorem~\ref{thm:metathm}, since $\bar{\u}^\infty_{1:L}=\P_{full}(\bar{\u}^\infty)$ has strictly positive margin, using homogeneity of $\P_{full}$, we can scale $\bar{\u}^\infty_{1:L}$ to get $\tilde{\w}^\infty=\gamma \bar{\u}^\infty_{1:L}$ with unit margin, \ie $\forall n,\,y_n\innerprod{\x_n}{\tilde{\w}^\infty} \ge 1$.
For  dual variables, we again use a positive scaling of $\alpha_n$ from Lemma~\ref{lem:grad-conv}, such that $\bar{\z}^\infty=\sum_{n\in S_\infty}\alpha_n\,y_n\x_n$. In order to prove the theorem, we need to show that $\tilde{\w}^\infty\propto \bar{\z}^\infty$ or equivalently $\bar{\u}^\infty_{1:L}\propto\bar{\z}^\infty$.

Recall that in the proof of Theorem~\ref{thm:metathm}, we showed a version of stationarity in the parameter space in eq. \eqref{eq:stationarity-meta}, repeated below.
\begin{equation}
\bar\u^\infty\propto\nabla{\u}\P(\bar{\u}^\infty)\bar\z^\infty. 
\label{eq:p-stat-fcn}
\end{equation} 

This case in particular includes $\P_{full}$ which is homogeneous with $\nu=L$. We special case the result fully connected network. In particular, for the parameters of the first layer $\u_1$, we have $\P(\u)=\u_1\u_{2:L}$, where $\u_1\in\bR^{d\times d_{1}}$ and $\u_{2:L}\in\bR^{d_1\times 1}$. 
This implies, for any $\z$, $ \nabla_{\u_1}\P(\u)\z=\z\u_{2:L}^{\top}$. Using this along with eq. \eqref{eq:p-stat-fcn}, we  get the following expression for some positive scalar $\bar{\gamma}$
\begin{equation}
\bar\u_1^\infty=\bar{\gamma}\,\nabla_{\u_1}\P(\bar\u^\infty)\bar\z^\infty=\bar{\gamma}\,\bar{\z}^\infty\bar\u_{2:L}^{\infty^\top}\implies \bar\u_{1:L}^\infty=\bar\u_1^\infty\bar\u_{2:L}^\infty=\bar{\gamma}\,\bar{\z}^\infty\cdot \norm{\bar\u_{2:L}^\infty}^2\propto \bar{\z}^\infty.
\end{equation}

Since $\bar\u_{1:L}^\infty\propto\tilde\w^\infty$, we have shown that $\tilde\w^\infty\propto \bar{\z}^\infty$, which completes our proof of Theorem~\ref{thm:fcn}. 
        }
        \mnote{end of old remove}
        
        The proof of Theorem \ref{Theorem: convergence to max margin for general tail} is fairly straight forward from using Lemma \ref{lem:uconv} and \ref{lemma: every accumulation point of z/norm(z) is a linear combination of sv}. In the rest of this section, $\norm{.}$ denotes the  Euclidean norm. In the following arguments we show that a positive scaling $\tilde{\eqw}_\infty=\gamma \lim\limits_{t\to\infty} \frac{\P(\np(t))}{\norm{\P(\np(t))}}$ satisfies the following KKT conditions for the optimality of explicitly regularized convex problem in eq.~\eqref{eq: normalized max margin equation}:
        \begin{equation}
        \begin{split}
        \exists \{\alpha_n\}_{n=1}^N\quad\text{ s.t. }\quad& \forall n,\, \ip{\xn}{\eqw} \ge 1, \eqw = \sum_n\alpha_n\, \xn,\\
        & \forall n, \alpha_n\ge0 \text{ and } \alpha_n=0, \forall i\notin {S}:=\{i\in[N]: \ip{\xn}{\eqw}=1\}.
        \end{split}
        \label{eq:kkt-fcn}
        \end{equation}
        
        \begin{enumerate}[leftmargin=*]
        \item For iterates $\np(t)=\left[\lw{l}\right]_{l=1}^L$, denote the incremental updates using $\Delta\lw{l}(t)=\eta_t^{-1}(\lw{l}(t+1)-\lw{l}(t))$ and $\Delta\np(t)=\left[\Delta\lw{l}(t)\right]_{l=1}^L$. 
        From Lemma \ref{lem:uconv}, we have that if the incremental updates $\np(t+1)-\np(t)$ converge in direction, then $\winfty=\lim\limits_{t\to \infty}\frac{\np(t)}{\norm{\np(t)}}=\lim\limits_{t\to \infty}\frac{\Delta\np(t)}{\norm{\Delta\np(t)}}$. 
        
        Let $\winfty=[\lwinfty{l}]$, where moved the $\infty$ to the super-script to avoid cluttering of notation for individual matrices $\lwinfty{l}$. For fully connected networks, we have:
        \begin{align}
        \Delta\lw{l}(t)&=\lwinfty{l} h(t)+\bdelta_{\Delta\lw{l}}(t)\,h(t)\label{eq:dufcn}\\
        \lw{l}(t)&=\lwinfty{l} g(t)+\bdelta_{\lw{l}}(t)\,g(t), \label{eq:ufcn}
        \end{align}
        where $h(t)=\norm{\Delta\np(t)}$, $g(t)=\sum_{u<t} \eta_u h(u)\to\infty$, and $\bdelta_{\Delta\lw{l}}(t),\bdelta_{\lw{l}}[t]\to 0$. 
        \item Let  $\bar{\eqw}_\infty=\lwinfty{1}\lwinfty{2}\ldots\lwinfty{L}$. Also, for $0<l_1<l_2<L$, denote $\lw{l_1:l_2}(t)=\lw{l_1}(t) \lw{l_1+1}(t)\ldots\lw{l_2}(t)$ and $\lwinfty{l_1:l_2}=\lwinfty{l_1} \lwinfty{l_1+1}\ldots\lwinfty{l_2}$. Using eq.~\eqref{eq:ufcn}, we can check by induction on $l_2-l_1$ that $\exists$ $\bdelta_{\lw{l_1:l_2}}(t)\to0$ such that the following holds,
        \begin{align}
        \lw{l_1:l_2}(t)&=\lwinfty{l_1:l_2}\, g(t)^{l_2-l_1+1}+\bdelta_{\lw{l_1:l_2}}(t)\,g(t)^{l_2-l_1+1}. 
        \label{eq:ulfcn}
        \end{align}
        \item Denote the gradients with respect to linear predictors as $\vect{z}(t)=-\nabla_{\eqw}\mathcal{L}(\eqw(t))$. Computing the gradients descent updates for $\lw{1}(t)$, we have 
        \[
        \Delta\lw{1}(t)=-\nabla_{\lw{1}}\mathcal{L}_{\P}(\np(t))=-\nabla_{\lw{1}}\mathcal{L}(\lw{1:L})=-\nabla_{\eqw}\mathcal{L}(\eqw(t)){\lw{2:L}(t)}^\top=\vect{z}(t){\lw{2:L}(t)}^\top
        \]
        Consider the following arguments on $\Delta\lw{1}(t) \lw{2:L}(t)$,
        \begin{equation}
        \begin{split}
        &\Delta\lw{1}(t) \lw{2:L}(t)=\vect{z}(t)\norm{\lw{2:L}(t)}^2\\
        \overset{(a)}\implies &\left(\lwinfty{1} h(t)+\bdelta_{\Delta\lw{1}}(t)\,h(t)\right)\left(\lwinfty{2:L}\, g(t)^{L-1}+\bdelta_{\lw{2:L}}(t)\,g(t)^{L-1}\right)=\vect{z}(t)\norm{\lw{2:L}(t)}^2\\
        \overset{(b)}\implies &\frac{\vect{z}(t)\norm{\lw{2:L}(t)}^2}{h(t)\,g(t)^{L-1}}=\lwinfty{1:L}+\bdelta(t)=\bar{\eqw}_\infty+\tilde{\bdelta}(t),
        \end{split}
        \label{eq:fcn-slackness}
        \end{equation}
        where in $(a)$, we used eqs. \ref{eq:dufcn}-\ref{eq:ufcn}, and in $(b)$ we have $\tilde\bdelta(t)=\bdelta_{\Delta\lw{1}}(t)\bdelta_{\lw{2:L}}(t)+\bdelta_{\Delta\lw{1}}(t)\lwinfty{2:L}+\lwinfty{1}\bdelta_{\lw{2:L}}(t)\to0$. 
        
         Denote $s(t):=\frac{\norm{\vect{z}(t)}\norm{\lw{2:L}(t)}^2}{h(t)\,g(t)^{L-1}}$ From eq.~\eqref{eq:fcn-slackness} using triangle inequality we have that 
        \begin{equation}
        \norm{\bar{\eqw}_\infty}-\norm{\tilde\bdelta(t)}\le s(t)\le \norm{\bar{\eqw}_\infty}+\norm{\tilde\bdelta(t)}.
        \end{equation}
        Since $\tilde\bdelta(t)\to0$,  by squeeze theorem, we have that $\lim_{t\to\infty}s(t)=\norm{\bar{\eqw}_\infty}$. Using this in eq.~\eqref{eq:fcn-slackness}, we get the following:
        \begin{equation}
        \frac{\vect{z}(t)}{\norm{\vect{z}(t)}}=\frac{\bar{\eqw}_\infty}{s(t)}+\frac{\bdelta(t)}{s(t)}.
        \label{eq:fcn-slackness1}
        \end{equation}
        Note that the limit as $t\to\infty$ of RHS in eq.~\eqref{eq:fcn-slackness1}  is $\frac{\bar{\eqw}_\infty}{\norm{\bar{\eqw}_\infty}}$. This implies, $\lim_{t\to\infty}\frac{\vect{z}(t)}{\norm{\vect{z}(t)}}$ exists. 
        
        \item Let $\bar{\vect{z}}_\infty=\lim_{t\to\infty}\frac{\vect{z}(t)}{\norm{\vect{z}(t)}}$ and let $\set_\infty=\{n:\ip{\bar{\eqw}_\infty}{\xn}=\min_{n}\ip{\bar{\eqw}_\infty}{\xn}\}$. From Lemma \ref{lemma: every accumulation point of z/norm(z) is a linear combination of sv}, $\exists\{\alpha_n\ge 0\}_{n\in \set_\infty}$ such that $\bar{\vect{z}}_\infty=\sum_{n\in \set_\infty} \alpha_n\,\xn$. 
        
        \item Finally, consider a positive scaling of $\bar{\eqw}_\infty$ such that $\tilde{\eqw}_\infty=\gamma \bar{\eqw}_\infty$ has unit margin, i.e. $\min_n \ip{\xn}{\tilde{\eqw}_\infty}= 1$. Taking limit of $t\to\infty$ in  eq.~\eqref{eq:fcn-slackness1} and using $\bar{\vect{z}}_\infty=\lim_{t\to\infty}\frac{\vect{z}(t)}{\norm{\vect{z}(t)}}=\sum_{n\in \set_\infty} \alpha_n\,\xn$, we have 
        \[\tilde{\eqw}_\infty=\gamma \norm{\bar{\eqw}_\infty}\bar{\vect{z}}_\infty=\sum_{n\in \set_\infty}(\gamma\norm{\bar{\eqw}_\infty}\alpha_n)\,\xn.\]
        \end{enumerate}
        Using $\tilde{\alpha}_n=\gamma\norm{\bar{\eqw}_\infty}\alpha_n\ge0$, we have shown that $\tilde{\np}_\infty$ satisfies the KKT conditions for \ref{eq: normalized max margin equation} specified in eq.~\eqref{eq:kkt-fcn}. This concludes the proof for Theorem \ref{Theorem: convergence to max margin for general tail}. }
	\end{proof}
	\remove{
	We add the proof here for completeness.
	\mnote{need to update}
	\begin{proof}
        Let $\np(t)=[\lw{l}(t)\in\R^{D_{l-1}\times D_l}]_{l=1}^L$ denote the iterates of individual matrices $\lw{l}(t)$ along the gradient descent path, and $\eqw(t)=\lw{1}(t)...\lw{l}(t)$ denote the corresponding sequence of linear predictors. \\
        
        \remove{
        We first introduce the following notation. 
\begin{compactenum}
\item Let $\bar\u^\infty=\lim\limits_{t\to \infty}\frac{\u[t]}{\norm{\u[t]}}$ denote the limit direction of the parameters, with component matrices in each layer denoted as $\bar\u^\infty=[\bar\u_l^\infty]$. Specializing \eqref{eq:u} for fully connected networks, we have:
\begin{align}
\u[t]_l&=\bar{\u}_l^\infty g(t)+\bdelta[t]_{\u_l}\,g(t), \label{eq:ufcn}
\end{align}
where $g(t)=\norm{\u[t]}$ and $\bdelta[t]_{\u_l}\to 0$. 
\item For $0<l_1< l_2\le L$, denote $\u[t]_{l_1:l_2}=\u[t]_{l_1} \u[t]_{l_1+1}\ldots\u[t]_{l_2}$ and $\bar\u_{l_1:l_2}^\infty=\bar\u_{l_1}^\infty \bar\u_{l_1+1}^\infty\ldots\bar\u_{l_2}^\infty$. Using eq. \eqref{eq:ufcn}, we can check by induction on $l_2-l_1$ that $\lim\limits_{t\to\infty}\frac{\u[t]_{l_1:l_2}}{g(t)^{l_2-l_1+1}}=\bar\u^\infty_{l_1:l_2}$, and hence $\exists\bdelta[t]_{\u_{l_1:l_2}}\to0$ such that the following holds,
\begin{align}
\u[t]_{l_1:l_2}&=\bar\u_{l_1:l_2}^\infty\, g(t)^{l_2-l_1+1}+\bdelta[t]_{\u_{l_1:l_2}}\,g(t)^{l_2-l_1+1}. 
\label{eq:ulfcn}
\end{align}
\item Let $\z[t]=-\nabla_{\w}\c{L}(\w[t])$.  Again repeating eq. \eqref{eq:z} for fully connected networks, we have for some $\bdelta[t]_{\z}\to0$ and $p(t)=\norm{\z[t]}$, 
\begin{equation}
\z[t]={\bar\z}^\infty p(t)+\bdelta[t]_{\z}\,p(t).
\label{eq:fcn-z}
\end{equation}
\item From Lemma~\ref{lem:grad-conv}, we have that $\exists \{\alpha_n\}_{n\in S_\infty}$ such that $\bar{\z}^\infty=\sum_{n\in S_\infty}\alpha_n\,y_n\x_n$, where $S_\infty$ are support vectors of $\bar{\w}^\infty=\lim\limits_{t\to\infty}\frac{\w[t]}{\norm{\w[t]}}\propto\P_{full}(\bar{\u}^\infty)$.
\end{compactenum}

The proof of Theorem~\ref{thm:fcn} is fairly straight forward from using Lemma \ref{lem:grad-conv} and the intermediate results in the proof of Theorem~\ref{thm:metathm}. 

\paragraph{Showing KKT conditions for $\tilde{\w}^\infty\propto\P_{full}(\bar{\u}^\infty)$.}
Using our notation described above, we have $\bar{\u}^\infty_{1:L}=\P_{full}(\bar{\u}^\infty)$.
In the following arguments we show that a positive scaling $\tilde{\w}^\infty=\gamma \bar{\u}^\infty_{1:L}$ satisfies the following KKT conditions for the optimality of $\ell_2$ maximum-margin  problem in eq. \eqref{eq:gd-fcn}:
\begin{equation}
\begin{split}
\exists \{\alpha_n\}_{n=1}^N\quad\st\quad& \forall n,\,y_n\innerprod{\x_n}{\w} \ge 1, \w = \sum_n\alpha_n\,y_n\x_{n},\\
& \forall n, \alpha_n\ge0 \tand \alpha_n=0, \forall i\notin {S}:=\{i\in[N]:y_{n}\innerprod{\x_n}{\w}=1\}.
\end{split}
\label{eq:kkt-fcn}
\end{equation}

As we saw in proof of Theorem~\ref{thm:metathm}, since $\bar{\u}^\infty_{1:L}=\P_{full}(\bar{\u}^\infty)$ has strictly positive margin, using homogeneity of $\P_{full}$, we can scale $\bar{\u}^\infty_{1:L}$ to get $\tilde{\w}^\infty=\gamma \bar{\u}^\infty_{1:L}$ with unit margin, \ie $\forall n,\,y_n\innerprod{\x_n}{\tilde{\w}^\infty} \ge 1$.
For  dual variables, we again use a positive scaling of $\alpha_n$ from Lemma~\ref{lem:grad-conv}, such that $\bar{\z}^\infty=\sum_{n\in S_\infty}\alpha_n\,y_n\x_n$. In order to prove the theorem, we need to show that $\tilde{\w}^\infty\propto \bar{\z}^\infty$ or equivalently $\bar{\u}^\infty_{1:L}\propto\bar{\z}^\infty$.

Recall that in the proof of Theorem~\ref{thm:metathm}, we showed a version of stationarity in the parameter space in eq. \eqref{eq:stationarity-meta}, repeated below.
\begin{equation}
\bar\u^\infty\propto\nabla{\u}\P(\bar{\u}^\infty)\bar\z^\infty. 
\label{eq:p-stat-fcn}
\end{equation} 

This case in particular includes $\P_{full}$ which is homogeneous with $\nu=L$. We special case the result fully connected network. In particular, for the parameters of the first layer $\u_1$, we have $\P(\u)=\u_1\u_{2:L}$, where $\u_1\in\bR^{d\times d_{1}}$ and $\u_{2:L}\in\bR^{d_1\times 1}$. 
This implies, for any $\z$, $ \nabla_{\u_1}\P(\u)\z=\z\u_{2:L}^{\top}$. Using this along with eq. \eqref{eq:p-stat-fcn}, we  get the following expression for some positive scalar $\bar{\gamma}$
\begin{equation}
\bar\u_1^\infty=\bar{\gamma}\,\nabla_{\u_1}\P(\bar\u^\infty)\bar\z^\infty=\bar{\gamma}\,\bar{\z}^\infty\bar\u_{2:L}^{\infty^\top}\implies \bar\u_{1:L}^\infty=\bar\u_1^\infty\bar\u_{2:L}^\infty=\bar{\gamma}\,\bar{\z}^\infty\cdot \norm{\bar\u_{2:L}^\infty}^2\propto \bar{\z}^\infty.
\end{equation}

Since $\bar\u_{1:L}^\infty\propto\tilde\w^\infty$, we have shown that $\tilde\w^\infty\propto \bar{\z}^\infty$, which completes our proof of Theorem~\ref{thm:fcn}. 
        }
        
        The proof of Theorem \ref{Theorem: convergence to max margin for general tail} is fairly straight forward from using Lemma \ref{lem:uconv} and \ref{lemma: every accumulation point of z/norm(z) is a linear combination of sv}. In the rest of this section, $\norm{.}$ denotes the  Euclidean norm. In the following arguments we show that a positive scaling $\tilde{\eqw}_\infty=\gamma \lim\limits_{t\to\infty} \frac{\P(\np(t))}{\norm{\P(\np(t))}}$ satisfies the following KKT conditions for the optimality of explicitly regularized convex problem in eq.~\eqref{eq: normalized max margin equation}:
        \begin{equation}
        \begin{split}
        \exists \{\alpha_n\}_{n=1}^N\quad\text{ s.t. }\quad& \forall n,\, \ip{\xn}{\eqw} \ge 1, \eqw = \sum_n\alpha_n\, \xn,\\
        & \forall n, \alpha_n\ge0 \text{ and } \alpha_n=0, \forall i\notin {S}:=\{i\in[N]: \ip{\xn}{\eqw}=1\}.
        \end{split}
        \label{eq:kkt-fcn}
        \end{equation}
        
        \begin{enumerate}[leftmargin=*]
        \item For iterates $\np(t)=\left[\lw{l}\right]_{l=1}^L$, denote the incremental updates using $\Delta\lw{l}(t)=\eta_t^{-1}(\lw{l}(t+1)-\lw{l}(t))$ and $\Delta\np(t)=\left[\Delta\lw{l}(t)\right]_{l=1}^L$. 
        From Lemma \ref{lem:uconv}, we have that if the incremental updates $\np(t+1)-\np(t)$ converge in direction, then $\winfty=\lim\limits_{t\to \infty}\frac{\np(t)}{\norm{\np(t)}}=\lim\limits_{t\to \infty}\frac{\Delta\np(t)}{\norm{\Delta\np(t)}}$. 
        
        Let $\winfty=[\lwinfty{l}]$, where moved the $\infty$ to the super-script to avoid cluttering of notation for individual matrices $\lwinfty{l}$. For fully connected networks, we have:
        \begin{align}
        \Delta\lw{l}(t)&=\lwinfty{l} h(t)+\bdelta_{\Delta\lw{l}}(t)\,h(t)\label{eq:dufcn}\\
        \lw{l}(t)&=\lwinfty{l} g(t)+\bdelta_{\lw{l}}(t)\,g(t), \label{eq:ufcn}
        \end{align}
        where $h(t)=\norm{\Delta\np(t)}$, $g(t)=\sum_{u<t} \eta_u h(u)\to\infty$, and $\bdelta_{\Delta\lw{l}}(t),\bdelta_{\lw{l}}[t]\to 0$. 
        \item Let  $\bar{\eqw}_\infty=\lwinfty{1}\lwinfty{2}\ldots\lwinfty{L}$. Also, for $0<l_1<l_2<L$, denote $\lw{l_1:l_2}(t)=\lw{l_1}(t) \lw{l_1+1}(t)\ldots\lw{l_2}(t)$ and $\lwinfty{l_1:l_2}=\lwinfty{l_1} \lwinfty{l_1+1}\ldots\lwinfty{l_2}$. Using eq.~\eqref{eq:ufcn}, we can check by induction on $l_2-l_1$ that $\exists$ $\bdelta_{\lw{l_1:l_2}}(t)\to0$ such that the following holds,
        \begin{align}
        \lw{l_1:l_2}(t)&=\lwinfty{l_1:l_2}\, g(t)^{l_2-l_1+1}+\bdelta_{\lw{l_1:l_2}}(t)\,g(t)^{l_2-l_1+1}. 
        \label{eq:ulfcn}
        \end{align}
        \item Denote the gradients with respect to linear predictors as $\vect{z}(t)=-\nabla_{\eqw}\mathcal{L}(\eqw(t))$. Computing the gradients descent updates for $\lw{1}(t)$, we have 
        \[
        \Delta\lw{1}(t)=-\nabla_{\lw{1}}\mathcal{L}_{\P}(\np(t))=-\nabla_{\lw{1}}\mathcal{L}(\lw{1:L})=-\nabla_{\eqw}\mathcal{L}(\eqw(t)){\lw{2:L}(t)}^\top=\vect{z}(t){\lw{2:L}(t)}^\top
        \]
        Consider the following arguments on $\Delta\lw{1}(t) \lw{2:L}(t)$,
        \begin{equation}
        \begin{split}
        &\Delta\lw{1}(t) \lw{2:L}(t)=\vect{z}(t)\norm{\lw{2:L}(t)}^2\\
        \overset{(a)}\implies &\left(\lwinfty{1} h(t)+\bdelta_{\Delta\lw{1}}(t)\,h(t)\right)\left(\lwinfty{2:L}\, g(t)^{L-1}+\bdelta_{\lw{2:L}}(t)\,g(t)^{L-1}\right)=\vect{z}(t)\norm{\lw{2:L}(t)}^2\\
        \overset{(b)}\implies &\frac{\vect{z}(t)\norm{\lw{2:L}(t)}^2}{h(t)\,g(t)^{L-1}}=\lwinfty{1:L}+\bdelta(t)=\bar{\eqw}_\infty+\tilde{\bdelta}(t),
        \end{split}
        \label{eq:fcn-slackness}
        \end{equation}
        where in $(a)$, we used eqs. \ref{eq:dufcn}-\ref{eq:ufcn}, and in $(b)$ we have $\tilde\bdelta(t)=\bdelta_{\Delta\lw{1}}(t)\bdelta_{\lw{2:L}}(t)+\bdelta_{\Delta\lw{1}}(t)\lwinfty{2:L}+\lwinfty{1}\bdelta_{\lw{2:L}}(t)\to0$. 
        
         Denote $s(t):=\frac{\norm{\vect{z}(t)}\norm{\lw{2:L}(t)}^2}{h(t)\,g(t)^{L-1}}$ From eq.~\eqref{eq:fcn-slackness} using triangle inequality we have that 
        \begin{equation}
        \norm{\bar{\eqw}_\infty}-\norm{\tilde\bdelta(t)}\le s(t)\le \norm{\bar{\eqw}_\infty}+\norm{\tilde\bdelta(t)}.
        \end{equation}
        Since $\tilde\bdelta(t)\to0$,  by squeeze theorem, we have that $\lim_{t\to\infty}s(t)=\norm{\bar{\eqw}_\infty}$. Using this in eq.~\eqref{eq:fcn-slackness}, we get the following:
        \begin{equation}
        \frac{\vect{z}(t)}{\norm{\vect{z}(t)}}=\frac{\bar{\eqw}_\infty}{s(t)}+\frac{\bdelta(t)}{s(t)}.
        \label{eq:fcn-slackness1}
        \end{equation}
        Note that the limit as $t\to\infty$ of RHS in eq.~\eqref{eq:fcn-slackness1}  is $\frac{\bar{\eqw}_\infty}{\norm{\bar{\eqw}_\infty}}$. This implies, $\lim_{t\to\infty}\frac{\vect{z}(t)}{\norm{\vect{z}(t)}}$ exists. 
        
        \item Let $\bar{\vect{z}}_\infty=\lim_{t\to\infty}\frac{\vect{z}(t)}{\norm{\vect{z}(t)}}$ and let $\set_\infty=\{n:\ip{\bar{\eqw}_\infty}{\xn}=\min_{n}\ip{\bar{\eqw}_\infty}{\xn}\}$. From Lemma \ref{lemma: every accumulation point of z/norm(z) is a linear combination of sv}, $\exists\{\alpha_n\ge 0\}_{n\in \set_\infty}$ such that $\bar{\vect{z}}_\infty=\sum_{n\in \set_\infty} \alpha_n\,\xn$. 
        
        \item Finally, consider a positive scaling of $\bar{\eqw}_\infty$ such that $\tilde{\eqw}_\infty=\gamma \bar{\eqw}_\infty$ has unit margin, i.e. $\min_n \ip{\xn}{\tilde{\eqw}_\infty}= 1$. Taking limit of $t\to\infty$ in  eq.~\eqref{eq:fcn-slackness1} and using $\bar{\vect{z}}_\infty=\lim_{t\to\infty}\frac{\vect{z}(t)}{\norm{\vect{z}(t)}}=\sum_{n\in \set_\infty} \alpha_n\,\xn$, we have 
        \[\tilde{\eqw}_\infty=\gamma \norm{\bar{\eqw}_\infty}\bar{\vect{z}}_\infty=\sum_{n\in \set_\infty}(\gamma\norm{\bar{\eqw}_\infty}\alpha_n)\,\xn.\]
        \end{enumerate}
        Using $\tilde{\alpha}_n=\gamma\norm{\bar{\eqw}_\infty}\alpha_n\ge0$, we have shown that $\tilde{\np}_\infty$ satisfies the KKT conditions for \ref{eq: normalized max margin equation} specified in eq.~\eqref{eq:kkt-fcn}. This concludes the proof for Theorem \ref{Theorem: convergence to max margin for general tail}. 
	\end{proof}}
	
	\section{Proof of Theorems \ref{theorem: general convergence rates simplified} and \ref{theorem: general convergence rates simplified with ode}}
	
	In Theorem~\ref{Theorem: convergence to max margin for general tail}, we showed that gradient descent on separable dataset converge in direction to the $L_2$ maximum-margin separator for a large family of super polynomial tailed loss functions specified by Assumption~\ref{def: general tail loss}. 
	Theorems \ref{theorem: general convergence rates simplified} and \ref{theorem: general convergence rates simplified with ode} show rate of convergence at which gradient descent converges to the maximum-margin separator. 
	
	We prove these theorems in the following steps:
	\begin{enumerate}
	    \item We first give a general result that specifies the weights model using an ordinary differential equation. This is stated in Theorem~\ref{theorem: general convergence rates}, the proof of which is provided in appendix section \ref{sec: proof of theorem about general tail rates}.
	    \item With the results from Theorems~\ref{theorem: general convergence rates}, we explicitly calculate the rates for general tails for deep linear networks. The calculation is in appendix section \ref{section: Asymptotic rates for depth L linear networks}. This completes the proof of Theorem~\ref{theorem: general convergence rates simplified with ode}.
	    \item We give an additional result characterizing $\rhoVec$ component that is not in the support vectors span ($\norm{\bar{\op}\rhoVec}$), for the special case of $L=1$. This result is stated in Theorem \ref{Lemma: non-support vectors direction converge for L=1} and proved in section \ref{section: Proof that non-support vectors direction converge for L=1}.
	    \item Finally, we special case Theorem~\ref{theorem: general convergence rates simplified with ode} and use Theorem \ref{Lemma: non-support vectors direction converge for L=1} to get simplified results for the case of $L=1$, thus proving Theorem~\ref{theorem: general convergence rates simplified}. 
	\end{enumerate}
	
	\begin{restatable}{theorem}{theoremGeneralRates} \label{theorem: general convergence rates}
	    Under Assumption \ref{assumption: f can be expended} and the conditions and notations of Theorem \ref{Theorem: convergence to max margin for general tail},  the  equivalent linear predictor of a depth $L$ linear network will behave as:
		\begin{equation} 
		\eqw(t) = \tilde{g}(t)\what +\bm{\rho}(t)
		\end{equation}
		where $\bm{\rho}(t)=o\left( \tilde{g}(t) \right)$, $\rho(t)^\top \what=0$, and $\what$ is the  $L_{2}$ max margin separator 
		\[\hat{\mathbf{w}}=\underset{\mathbf{\mathbf{w}}\in\mathbb{R}^{d}}{\mathrm{argmin}}\left\lVert \mathbf{w}\right\rVert^2 \,\,\mathrm{s.t.}\,\,\mathbf{w}^{\top}\mathbf{x}_{n}\geq1.\] 
            
		Further,  $\tilde{g}(t)$ and $\rhoVec$ are the asymptotic solution of the following, 
		\[
	         \lim_{t\to\infty}\frac{L\eta_t\e\left( -f\left(\tilde{g}\left(t\right)\right)\right)\left(\tilde{g}(t)\right)^{2\left(1-\frac{1}{L}\right)}\phi_1(t)}{ \gamma^{1-\nicefrac{2}{L}} \frac{d}{dt}\tilde{g}(t)}=1, \text{and}
	    \]
	    \[\norm{\op\rhoVec}=\phi_2(t)\left(f'\left(\tilde{g}(t)\right)\right)^{-1}+o\left(\left(f'\left(\tilde{g}(t)\right)\right)^{-1}\right).\]
	    
	    where $\phi_1(t)=\Theta(1)$ and $\phi_2(t)=\Theta(1)$ are positive functions that depend only on the data set and not on the loss function $\ell$;  $\gamma=\min_n\frac{\hat{\w}^\top{\xn}}{\norm{\hat{\w}}}$ is the maximum-margin attainable for the dataset with unit $L_2$ norm separator;  $\op\in\R^{d\times d}$ is the orthogonal projection matrix to the subspace spanned by the support vectors.\\
	    If, in addition, the support vector span the dataset then
	    \[
	        \norm{\bar{\op}\rhoVec}=O(1).
	    \]
	\end{restatable}
	For the case $L=1$, if we assume that the loss is $\beta$-smooth, $\eta_t<2\beta^{-1}$ and $f'(t)=\Omega\left(\frac{\log^{1+\epsilon}(t)}{t}\right)$ we can omit the requirement that the loss is minimized (this is guaranteed in this case) and that the support vector span the dataset and prove the following Theorem.
	\begin{restatable}{theorem}{lemmaOrthogonalProj} \label{Lemma: non-support vectors direction converge for L=1}
	For $L=1$, $\beta$-smooth loss and $\eta<2\beta^{-1}$, if $f'(t)=\Omega\left(\frac{\log^{1+\epsilon}(t)}{t}\right)$ for some $\epsilon>0$ then
	    \[
	        \norm{\bar{\op}\rhoVec}=O(1),
	    \]
	    where $\op\in\R^{d\times d}$ is the orthogonal projection matrix to the subspace spanned by the support vectors and $\bar{\op}=I-\op$ is the complementary projection.
	
	\end{restatable}

	\subsection{Asymptotic rates for depth $L$ linear networks} \label{section: Asymptotic rates for depth L linear networks}
	From Theorem \ref{theorem: general convergence rates}, we can write $\wvec(t) = \what g(t) + \rhoVec$ where $\rhoVec = o(g(t))$ and $\rho(t)^\top \what=0$.	
	We can use this to calculate the normalized weight vector:
	\begin{flalign} \label{Calculation of the normalized weight vector}
	&\frac{\wvec(t)}{\Vert \wvec(t)\Vert} =\frac{g(t)\what +\rhoVec}{\sqrt{g(t)^2\what^\top\what + \rhoVec^\top\rhoVec}} =\frac{\what + g^{-1}(t)\rhoVec}{\whatNorm\sqrt{1 + \frac{\Vert\rhoVec\Vert^2}{g^2(t)\whatNorm^2}}}\nonumber\\
	& \overset{(1)}= \frac{\what +g^{-1}(t)\rhoVec}{\whatNorm}\left[1-O\left(\frac{\Vert\rhoVec\Vert^2}{\whatNorm^2}\frac{1}{g^2(t)}\right)\right]\nonumber\\
	& = \frac{\what}{\whatNorm}+\frac{\rhoVec}{g(t)\whatNorm}-O\left(\frac{\Vert\rhoVec\Vert^2}{\whatNorm^2}\frac{1}{g^2(t)}\right)\frac{\what}{\norm{\what}}
	\end{flalign}
	where in $(1)$ we used $\frac{1}{\sqrt{1+x}}=1-\frac{1}{2}x+\frac{3}{4}x^2+O(x^3)$.\\
	
	Calculation of the margin:
	\begin{flalign} \label{eq: Calculation of the margin}
		&\min_{n} \frac{\xnT \wvec(t)}{\Vert \wvec(t)\Vert}\overset{(1)}{=}\min_{n\in\set} \frac{\xnT \wvec(t)}{\Vert \wvec(t)\Vert}\nonumber\\
		& =\min_{n\in\set} \xnT\left[\frac{\what}{\whatNorm}+\frac{\rhoVec}{g(t)\whatNorm}-\frac{\what}{\whatNorm}O\left(\frac{\Vert\rhoVec\Vert^2}{\whatNorm^2}\frac{1}{g^2(t)}\right)\right]\nonumber\\
		& =\frac{1}{\whatNorm}+\frac{\min_{n\in\set}\xnT\rhoVec}{g(t)\whatNorm}+O\left(\frac{\Vert\rhoVec\Vert^2}{\whatNorm^2}\frac{1}{g^2(t)}\right)\frac{1}{\norm{\what}},
	\end{flalign}
	where in (1) we used the fact that $\frac{\wvec(t)}{\Vert \wvec(t)\Vert}$ converge to the maximum-margin separator and thus the minimal value is obtained on the support vectors.\\
    From Theorem \ref{theorem: general convergence rates}, we can also characterize $\rhoVec$: 
	\begin{equation} \label{eq: rho def generic case appendix}
	        \norm{\op\rhoVec}=\begin{cases} \fsym_2(t)\left(f'\left(g(t)\right)\right)^{-1}+o\left(\left(f'\left(g(t)\right)\right)^{-1}\right), & \text{if } \left(f'\left(g(t)\right)\right)^{-1}=\Omega(1)\\
	        O\left(1\right), & \text{otherwise}
	        \end{cases}, \ \norm{\bar{\op}\rhoVec}=O(1)
	\end{equation}
	where $\fsym_2(t)=\Theta(1)$, $\op\in\R^{d\times d}$ is the orthogonal projection matrix to the subspace spanned by the support vectors and $\bar{\op}=I-\op$ is the complementary projection.\\
      Substituting eq.~\eqref{eq: rho def generic case appendix} into eqs. \ref{Calculation of the normalized weight vector}, \ref{eq: Calculation of the margin} we get
      \[
        \abs{\gamma - \min\limits_{n} \frac{\xnT \wvec (t)}{\Vert \wvec(t)\Vert}}=\begin{cases} O\left(\frac{1}{g(t)}\right), & f'(u)=\omega(1)\\
		C_1\frac{1}{g(t) f'\left(g\left(t\right)\right)}+O\left(\frac{1}{g(t)}\right), & \text{otherwise}
		\end{cases}
      \]
      where $C_1$ is a constant independent of $f(t)$.
     \subsection{Asymptotic rates for $L=1$} \label{sec: proving corollary L=1 optimal rate is obtained for exp loss}
	    In this section we want to show that, in the special case of $L=1$, the optimal margin convergence rate is obtained for exponential loss.
	    
	    Using Theorem \ref{theorem: general convergence rates} with general tail and $L=1$ and without assuming the support vectors span the data we have that $\tilde{g}(t)$ is the asymptotic solution of
	\[
	    \e\left(-f\left(g(t)\right)\right)=\frac{1}{\eta_t\fsym_1(t)}\frac{d}{dt}g(t)
	\]
	and
	\[
	        \norm{\op\rhoVec}=\begin{cases} \fsym_2(t)\left(f'\left(g(t)\right)\right)^{-1}+o\left(\left(f'\left(g(t)\right)\right)^{-1}\right), & \text{if } \left(f'\left(g(t)\right)\right)^{-1}=\Omega(1)\\
	        O\left(1\right), & \text{otherwise}
	        \end{cases}
    \]
    for positive functions $\fsym_1(t)=\Theta(1),\ \fsym_2(t)=\Theta(1)$ independent of $f$. 
    
    Additionally, From Theorem \ref{Lemma: non-support vectors direction converge for L=1} we have that
    \[
        \norm{\bar{\op}\rhoVec}=O(1).
    \]
	We note that under Theorem \ref{Lemma: non-support vectors direction converge for L=1} assumptions ($L=1$, $\beta$-smooth loss and $\eta<2\beta^{-1}$) we have $\lim_{t\to\infty} \mathcal{L}(\eqw(t))=0$ from Lemma \ref{Lemma: w(t)->infty}, so we can use Theorem \ref{theorem: general convergence rates} without this assumption.
	
	    We denote $\tilde{\fsym}_1(t)=\eta_t\fsym_1=\Theta(1)$. We define $u(t)=\int\limits_{0}^{t} \tilde{\fsym}_1(x)dx =H(t)\Rightarrow t=H^{-1}(u)$ (this is well defined since $H(t)$ is monotonic increasing) and $\hat{g}(u)=g\left(H^{-1}(u)\right)=g(t)$. 
	    Using these definition we have
        \begin{equation} \label{eq: ode for g_hat}
            \frac{d}{du} \hat{g}(u)= \frac{1}{\tilde{\fsym}_1(t)}\frac{d}{dt}g(t)=\e\left(-f\left(\hat{g}(u)\right)\right).
        \end{equation}
        Since $\tilde{\fsym}_1(t)=\Theta(1)$ we know that exists positive constants $C_L,C_U,t_1$ so that $\forall t>t_1:$ $C_L\le\tilde{\fsym}_1(t)\le C_U\Rightarrow C_L t\le H(t)\le C_U t$. This implies $\hat{g}\left(C_L t\right)\le g(t)=\hat{g}\left(H(t)\right)\le \hat{g}\left(C_U t\right)$ (since $\hat{g}(u)$ is an increasing function). This will enable us to characterize $\tilde{g}(t)$ asymptotic behaviour using $\hat{g}(u)$.
        
        For functions with tight exponential tail ($f(t)=\Theta(t)$) the margin convergence rate is $O(\frac{1}{\log(t)})$ and we know that this bound is tight (this result was proved in \cite{soudry2017implicit}).
        
        For $f(t)=\omega(t)$ (the tail goes to zero faster than exponential tail) the margin convergence rates are proportional to $1/g(t)$ (from the calculation in the previous section and Theorems \ref{theorem: general convergence rates} and \ref{Lemma: non-support vectors direction converge for L=1} results). Additionally, For functions with $f(t)=\omega(t)$, the asymptotic solution for eq.~\eqref{eq: ode for g_hat} is $\hat{g}(u)=f^{-1}\left(\log(u)\right)$, i.e. \mbox{$\displaystyle \lim_{u\to\infty} \frac{\hat{g}(u)}{f^{-1}\left(\log(u)\right)}=1$} (this result is proved in \cite{odeSolMathOverFlow}). This implies slower convergence rates than the rates obtained with exponential tail ($1/\log(t)$) since in this case $f^{-1}\left(\log(t)\right)=o\left(\log(t)\right)$.
        
	    For $f(t)=o(t)$ we first prove the following claim.
	    \begin{restatable}{claim}{claimidealrate}\label{claim: relation for f(t)=o(t)}
	    For a strictly concave function $f$ that satisfies $f'(t)>0,\ f'(t)=o(1)$ and $f'(t)=\Omega\left(t^{-1}\log^{1+\epsilon}\left(t\right)\right)$, $\exists x'$ so that $\forall x>x'$:
	    \begin{equation}
	            \frac{1}{f^{-1}\left(x\right)f'\left(f^{-1}\left(x\right)\right)}> \frac{1}{x}.
        \end{equation}
	    \end{restatable}
	\begin{proof}
	        We denote $h(x)=f^{-1}(x)$. $h(x)$ is strictly convex since $f$ is strictly increasing and strictly concave. 
	        Substituting $h(x)$ and $h'(x)=\frac{1}{f'\left(f^{-1}\left(x\right)\right)}$ into the equation, we need to show that $\exists x_1$ so that $\forall x>x_1$
	        \begin{equation}
	            \frac{h'(x)}{h\left(x\right)}> \frac{1}{x}.
	        \end{equation}
	        From the gradient inequality, $\forall x>x'>0$:
	        \begin{align*}
	            h'(x)\left(x-x'\right)&> h(x)-h(x')\\
	            h'(x)&> \frac{h(x)-h(x')}{x-x'}.
	        \end{align*}
	        Additionally, since $h(t)=\omega(t)$ (from definition and $f(t)=o(t)$) $\exists x''$ so that $\forall x>x''$:
	        \[
	            h(x)> \frac{h(x')}{x'} x \Leftrightarrow  -xh(x')> -x'h(x) \Leftrightarrow \frac{h(x)-h(x')}{x-x'}> \frac{h(x)}{x} .
 	        \]
 	        Thus, for $x>\max(x',x'')$
 	        \[
 	            h'(x)> \frac{h(x)}{x} \Leftrightarrow \frac{h'(x)}{h(x)}> \frac{1}{x}.
 	        \]
	    \end{proof}
	    
	    For $f(t)=o(t)$ we have
	    \[
	        \gamma - \min_n \frac{\xnT\eqw(t)}{\norm{\eqw(t)}}=\frac{C_1}{g(t)f'(g(t))}+o\left(\frac{1}{g(t)f'(g(t))}\right),
	    \]
	    where $C_1$ is a constant independent of $f$. In order to show that the optimal rate is obtained for exponential loss we need to show that, asymptotically,
	    \begin{equation} \label{eq: exp tail optimal for f(t)=o(t) requierment}
	        \frac{1}{g_1(t)f'\left(g_1\left(t\right)\right)}> \frac{1}{g_2(t)}.
	    \end{equation}
	    where $g_1(t)$ is the solution of the following equation
	    \[
	    \e\left(-f\left(g_1(t)\right)\right)=\frac{1}{\eta_t\fsym_1(t)}\frac{d}{dt}g_1(t)
	    \]
	    for $f(t)=o(t)$ and $g_2(t)$ is the solution of this equation for exp tail $f(u)=u$ (asymptotically), i.e.
	    \[
	         \e\left(-g_2(t)\right)=\frac{1}{\eta_t\fsym_1(t)}\frac{d}{dt}g_2(t).
	    \]
	    Substituting $t=H^{-1}(u)$ (time rescaling - as explained above) to eq.~\eqref{eq: exp tail optimal for f(t)=o(t) requierment} we obtain the equivalent equation
	    \begin{equation}
	        \frac{1}{\hat{g}_1(t)f'\left(\hat{g}_1\left(t\right)\right)}> \frac{1}{\hat{g}_2(t)},
	    \end{equation}
	    where $\hat{g}_1(u)=g_1(H^{-1}(u))$ and $\hat{g}_2(u)=g_2(H^{-1}(u))$.
	    The obtained ODE (as explained in eq.~\eqref{eq: ode for g_hat}) for $\hat{g}_1(t)$ and $\hat{g}_2(t)$ are
	    \[
	    \e\left(-f\left(\hat{g}_1(t)\right)\right)=\frac{d}{dt}\hat{g}_1(t)\,;\,\e\left(-\hat{g}_2(t)\right)=\frac{d}{dt}\hat{g}_2(t)\Rightarrow \hat{g}_2(t)=\log(t+C).
	    \]
	    Thus, we need to show that
	    \[
	        \frac{1}{\hat{g}_1(t)f'\left(\hat{g}_1\left(t\right)\right)}> \frac{1}{\log(t)}
	    \]
	    (the constant $C$ only contributes an $o(1/log(t))$ term and we are interested in the leading term for characterizing the rates).
	    From claim \ref{claim: g=finv_log for concave f} we have that $\hat{g}_1(t)=f^{-1}\left(\log\left(t\right)\right)+o\left( f^{-1}\left(\log\left(t\right)\right) \right)$. Thus, since we are only interested in the leading term, we need to show that
	    \remove{
	    From assumption \ref{assumption: f can be expended} we have that $\abs{\frac{f''(x)}{f'(x)}}=o\left(x^{-1}\right)$. This implies that for sufficiently large $x$, $h(x)=xf'(x)$ is an increasing function ($h'(x)=f'(x)+xf''(x)>0$). Additionally,
	    assuming $f'(t)=o(1)$ we have
	    \begin{align*}
	        \frac{d}{dt}{\hat{g}_1}(t)=\e\left(-f\left(\hat{g}_1\left(t\right)\right)\right)&\leq \frac{\e\left(-f\left(\hat{g}_1\left(t\right)\right)\right)}{f'\left(\hat{g}_1\left(t\right)\right)}\Leftrightarrow\\
	         \e\left(f\left(\hat{g}_1\left(t\right)\right)\right)f'\left(\hat{g}_1\left(t\right)\right)\frac{d}{dt}\hat{g}_1(t)&\le 1\Leftrightarrow\\
	        \frac{d}{dt}\e\left(f\left(\hat{g}_1\left(t\right)\right)\right)&\le 1 \Rightarrow\\
	        \e\left(f\left(\hat{g}_1\left(t\right)\right)\right) &\le t+C \Leftrightarrow\\
	        \hat{g}_1(t)&\le f^{-1}\left(\log(t+C)\right)
	    \end{align*}
	    where in the last transition we used the fact that $f$ is an increasing function (and so does $f^{-1}$).
	    Using the last two results we have
	    
	    \[
	        \frac{1}{g(t)f'\left(g\left(t\right)\right)}\ge \frac{1}{f^{-1}\left(\log(t)\right)f'\left(f^{-1}\left(\log(t)\right)\right)}
	    \]
	    
	    and thus it is sufficient to show that} $\exists t'$ so that $\forall t>t'$:
	    \begin{equation}
	            \frac{1}{f^{-1}\left(\log(t)\right)f'\left(f^{-1}\left(\log(t)\right)\right)}> \frac{1}{\log(t)}.
        \end{equation}
        Using claim \ref{claim: relation for f(t)=o(t)} with $x=\log(t)$ we get the desired result.
        
        \subsection{$g(t)=f^{-1}\left(\log(t)\right)+o\left(f^{-1}\left(\log(t)\right)\right)$ for $L=1$}
        \begin{claim} \label{claim: g=finv_log for concave f}
            If \[\frac{d}{dt} g(t)=\e\left(-f\left(g(t)\right)\right)\,,\] where $f$ is concave, $f'(t)=o(1)$, $f'(t)=\Omega\left(t^{-1}\log^{1+\epsilon}\left(t\right)\right)$ then \[g(t)=f^{-1}\left(\log(t)\right)+o\left(f^{-1}\left(\log(t)\right)\right)\,.\]
        \end{claim}
        \begin{proof}
            Note that $g\left(t\right)$ is increasing since $\forall t:\dot{g}(t)=\exp\left(-f\left(g\left(t\right)\right)\right)>0$
and unbounded (as we will prove next) and therefore, $\lim_{t\to\infty}g(t)=\infty$.
In order to prove that $g(t)$ is unbounded we assume in contradiction
that $\exists M,t_{0}$ such that $\forall t>t_{0}:$ $g\left(t\right)\le M$.
Thus, $\forall t>t_{0}:$ $\dot{g}(t)=\exp\left(-f\left(g\left(t\right)\right)\right)\ge\exp\left(-f\left(M\right)\right)$
(since $f$ is increasing) which implies $g(t)\ge\exp\left(-f\left(M\right)\right)t+c\to\infty$
in contradiction to our assumption.\\
We want to show that $g(t)=f^{-1}\left(\log(t)\right)+o\left(f^{-1}\left(\log(t)\right)\right)$.

\textbf{First step: proving that $g(t)\le f^{-1}\left(\log\left(t+C_{1}\right)\right)$}

\[
\dot{g}(t)=\exp\left(-f\left(g\left(t\right)\right)\right)\overset{(1)}{\le}\exp\left(-f\left(g\left(t\right)\right)\right)\frac{1}{f'\left(g\left(t\right)\right)}\ ,
\]
where in $(1)$ we used $f'(t)=o(1)$ and $\lim_{t\to\infty}g(t)=\infty$.
Thus, we have

\begin{align*}
\exp\left(f\left(g\left(t\right)\right)\right)f'\left(g\left(t\right)\right)\dot{g}(t) & \le1\Rightarrow\\
\exp\left(f\left(g\left(t\right)\right)\right) & \le t+C_{1}\Leftrightarrow\\
g\left(t\right) & \le f^{-1}\left(\log\left(t+C_{1}\right)\right)\ ,
\end{align*}
where in the last transition we used the fact that $f$ is increasing
(and so does $f^{-1}$).

\textbf{Second step: proving that $g(t)\ge f^{-1}\left(\log\left(t+C_{2}\right)\right)+o\left(f^{-1}\left(\log\left(t+C_{2}\right)\right)\right)$ }

\[
\dot{g}(t)=\exp\left(-f\left(g\left(t\right)\right)\right)\overset{(1)}{\ge}\exp\left(-f\left(g\left(t\right)\right)\right)\frac{1}{g(t)f'\left(g\left(t\right)\right)}\ ,
\]
where in $(1)$ we used $f'(t)=\omega\left(t^{-1}\right)$. We have

\begin{equation}
\exp\left(f\left(g\left(t\right)\right)\right)f'\left(g\left(t\right)\right)\dot{g}(t)g(t)\ge1\label{eq:lower bound}
\end{equation}
In addition

\begin{align}
\int\left[\exp\left(f\left(g\left(t\right)\right)\right)f'\left(g\left(t\right)\right)\dot{g}(t)g(t)\right]dt & \overset{(1)}{=}\int u\exp\left(f\left(u\right)\right)f'\left(u\right)du\nonumber \\
 & =\int u\frac{d}{du}\left[\exp\left(f\left(u\right)\right)\right]du\nonumber \\
 & =u\exp\left(f\left(u\right)\right)-\int\exp\left(f\left(u\right)\right)du\nonumber \\
 & \le u\exp\left(f\left(u\right)\right)\ ,\label{eq:upper bound}
\end{align}
where in $(1)$ we defined $u=g(t)$ and used $du=g'(t)dt$.\\
Combining the last two equations we obtain

\begin{align*}
g(t)\exp\left(f\left(g(t)\right)\right) & \ge t+C_{2}\\
f\left(g(t)\right)+\log\left(g\left(t\right)\right) & \ge\log\left(t+C_{2}\right)\\
f\left(g(t)\right) & \ge\log\left(t+C_{2}\right)-\log\left(g\left(t\right)\right)\\
 & \ge\log\left(t+C_{2}\right)-\log\left(f^{-1}\left(\log\left(t+C_{1}\right)\right)\right)\\
 & \ge\log\left(t+C_{2}\right)-h\left(t\right),
\end{align*}
where we defined $h\left(t\right)=\log\left(f^{-1}\left(\log\left(t+C_{1}\right)\right)\right)$.\\
Thus,

\[
g\left(t\right)\ge f^{-1}\left(\log\left(t+C_{2}\right)-h\left(t\right)\right)=f^{-1}\left(\log\left(t+C_{2}\right)\right)+o\left(f^{-1}\left(\log\left(t+C_{2}\right)\right)\right)
\]
since

\[
\lim_{t\to\infty}\frac{f^{-1}\left(\log\left(t+C_{2}\right)-h\left(t\right)\right)}{f^{-1}\left(\log\left(t+C_{2}\right)\right)}=1.
\]
This is true from the Squeeze Theorem since we have

\[
f^{-1}\left(\log\left(t+C_{2}\right)-h\left(t\right)\right)\le f^{-1}\left(\log\left(t+C_{2}\right)\right)
\]
because $f$ is increasing and $h(t)\ge0$ (for sufficiently large
$t$), and also, from the gradient inequality (we recall that $f^{-1}$
is convex since $f$ is concave and increasing)

\begin{align*}
f^{-1}\left(\log\left(t+C_{2}\right)-h\left(t\right)\right) & \ge f^{-1}\left(\log\left(t+C_{2}\right)\right)-\frac{1}{f'\left(f^{-1}\left(\log\left(t+C_{2}\right)\right)\right)}h\left(t\right)\\
 & =f^{-1}\left(\log\left(t+C_{2}\right)\right)-o\left(f^{-1}\left(\log\left(t+C_{2}\right)\right)\right)\ ,
\end{align*}
where in the last equality we used

\begin{align*}
 & \lim_{t\to\infty}\frac{h\left(t\right)}{f^{-1}\left(\log\left(t+C_{2}\right)\right)f'\left(f^{-1}\left(\log\left(t+C_{2}\right)\right)\right)}\\
 & =\lim_{t\to\infty}\frac{\log\left(f^{-1}\left(\log\left(t+C_{1}\right)\right)\right)}{f^{-1}\left(\log\left(t+C_{2}\right)\right)f'\left(f^{-1}\left(\log\left(t+C_{2}\right)\right)\right)}\\
 & =\lim_{t\to\infty}\underset{\to0}{\underbrace{\frac{\log\left(f^{-1}\left(\log\left(t+C_{2}\right)\right)\right)}{\log^{1+\epsilon}\left(f^{-1}\left(\log\left(t+C_{2}\right)\right)\right)}}}\cdot\underset{\to1}{\underbrace{\frac{\log\left(f^{-1}\left(\log\left(t+C_{1}\right)\right)\right)}{\log\left(f^{-1}\left(\log\left(t+C_{2}\right)\right)\right)}}}\cdot\underset{O\left(1\right)}{\underbrace{\frac{\log^{1+\epsilon}\left(f^{-1}\left(\log\left(t+C_{2}\right)\right)\right)}{f^{-1}\left(\log\left(t+C_{2}\right)\right)f'\left(f^{-1}\left(\log\left(t+C_{2}\right)\right)\right)}}}\\
 & =0\,,
\end{align*}
 where in the last transition we used $\frac{\log\left(f^{-1}\left(\log\left(t+C_{1}\right)\right)\right)}{\log\left(f^{-1}\left(\log\left(t+C_{2}\right)\right)\right)}\to1$ from the next claim.
\begin{claim}
$\forall C_{1},C_{2}$: $\lim_{t\to\infty}\frac{\log\left(f^{-1}\left(\log\left(t+C_{1}\right)\right)\right)}{\log\left(f^{-1}\left(\log\left(t+C_{2}\right)\right)\right)}=1$
\end{claim}

\begin{proof}
We assume WLOG $C_{2}>C_{1}$. Thus,

\[
\log\left(f^{-1}\left(\log\left(t+C_{1}\right)\right)\right)=\log\left(f^{-1}\left(\log\left(t+C_{2}\right)+\log\left(\frac{t+C_{1}}{t+C_{2}}\right)\right)\right)\le\log\left(f^{-1}\left(\log\left(t+C_{2}\right)\right)\right)
\]
and
\begin{align*}
& \log\left(f^{-1}\left(\log\left(t+C_{2}\right)+\log\left(\frac{t+C_{1}}{t+C_{2}}\right)\right)\right) \\
& \ge\log\left(f^{-1}\left(\log\left(t+C_{2}\right)\right)+\frac{\log\left(\frac{t+C_{1}}{t+C_{2}}\right)}{f'\left(f^{-1}\left(\log\left(t+C_{2}\right)\right)\right)}\right)\\
 & =\log\left(f^{-1}\left(\log\left(t+C_{2}\right)\right)\right)+\log\left(1+\frac{\log\left(\frac{t+C_{1}}{t+C_{2}}\right)}{f^{-1}\left(\log\left(t+C_{2}\right)\right)f'\left(f^{-1}\left(\log\left(t+C_{2}\right)\right)\right)}\right)\\
 & =\log\left(f^{-1}\left(\log\left(t+C_{2}\right)\right)\right)+o\left(1\right)
\end{align*}
\end{proof}
\textbf{Third step: proving that $g(t)=f^{-1}\left(\log(t)\right)+o\left(f^{-1}\left(\log(t)\right)\right)$}

We have
\[
f^{-1}\left(\log\left(t+C_{2}\right)\right)+o\left(f^{-1}\left(\log\left(t+C_{2}\right)\right)\right)\le g\left(t\right)\le f^{-1}\left(\log\left(t+C_{1}\right)\right).
\]
Using Claim \ref{claim:: f_inv(log(t)) ratio limit}, this eq. also
implies

\[
f^{-1}\left(\log\left(t+C_{2}\right)\right)+o\left(f^{-1}\left(\log\left(t+C_{2}\right)\right)\right)\le g\left(t\right)\le f^{-1}\left(\log\left(t+C_{2}\right)\right)+o\left(f^{-1}\left(\log\left(t+C_{2}\right)\right)\right).
\]
Therefore, we have
\[
\lim_{t\to\infty}\frac{g\left(t\right)}{f^{-1}\left(\log\left(t\right)\right)}=\frac{g\left(t\right)}{f^{-1}\left(\log\left(t+C_{2}\right)\right)}=1.
\]

\begin{claim}
\label{claim:: f_inv(log(t)) ratio limit}$\forall C_{1},C_{2}$:
$\lim_{t\to\infty}\frac{f^{-1}\left(\log\left(t+C_{1}\right)\right)}{f^{-1}\left(\log\left(t+C_{2}\right)\right)}=1$
\end{claim}

\begin{proof}
We assume WLOG $C_{2}>C_{1}$. Thus,

\[
f^{-1}\left(\log\left(t+C_{1}\right)\right)=f^{-1}\left(\log\left(t+C_{2}\right)+\log\left(\frac{t+C_{1}}{t+C_{2}}\right)\right)\le f^{-1}\left(\log\left(t+C_{2}\right)\right)
\]
and

\begin{align*}
f^{-1}\left(\log\left(t+C_{2}\right)+\log\left(\frac{t+C_{1}}{t+C_{2}}\right)\right) & \ge f^{-1}\left(\log\left(t+C_{2}\right)\right)+\frac{\log\left(\frac{t+C_{1}}{t+C_{2}}\right)}{f'\left(f^{-1}\left(\log\left(t+C_{2}\right)\right)\right)}\\
 & =f^{-1}\left(\log\left(t+C_{2}\right)\right)+o\left(f^{-1}\left(\log\left(t+C_{2}\right)\right)\right)
\end{align*}
\end{proof}
        \end{proof}
        \subsection{$g(t)=\log(t)+o(\log(t))$ proof for $L>1$ and f(u)=u}\label{Sect: g(t)=log(t)+o(log(t)) proof for $L>1$ and f(u)=u}

            We have that
            \begin{equation}
            \frac{dg\left(t\right)}{dt}=\exp\left(-g\left(t\right)\right)g^{b}(t),\ b\in[1,2]\,.\label{eq: dg/dt-1-1}
            \end{equation}
            We can write 
            
            \begin{equation}
            g(t)=\log(t)+b\log(g(t))+h(t).\label{eq: g(t) guess}
            \end{equation}
            First step: we want to show that $h(t)\le\log(\log(t)+C_{2})$. Substituting
            eq.~\eqref{eq: g(t) guess} into eq.~\eqref{eq: dg/dt-1-1} we get:
            
            \begin{align*}
            t^{-1}+b\frac{\dot{g}(t)}{g(t)}+h'(t) & =t^{-1}\exp(-h(t))\\
            t^{-1}+h'(t) & =t^{-1}\exp(-h(t))\left(1-\frac{b}{g(t)}\right).
            \end{align*}
            Since $\frac{b}{g(t)}<0$ and $t^{-1}>0$ we get
            \[
            h'(t)\le t^{-1}\exp(-h(t)).
            \]
            Integrating both sides we get:
            
            \begin{align*}
            \exp(h(t)) & \le\log(t)+C\\
            h(t) & \le\log\left(\log(t)+C\right)\\
             & =\log\left(\left(1+\frac{C}{\log(t)}\right)\log(t)\right)\\
             & =\log\left(\log(t)\right)+\log\left(1+\frac{C}{\log(t)}\right)
            \end{align*}
            and thus, $\exists t_{2},C_{2}>1$ so that $\forall t>t_{2}:$
            
            \[
            h(t)\le\log\left(\log(t)+C_{2}\right).
            \]
            Step 2: Showing that $g(t)\le C_{4}\log(t)$
            
            \[
            g(t)=log(t)+b\log(g(t))+h(t)\le log(t)+b\log(g(t))+\log\left(\log(t)+C_{2}\right).
            \]
            Since $g(t)-b\log(g(t))=\Theta(g(t))$, $\exists t_{3},C_{3},C_{4}$
            so that $\forall t>t_{3}$:
            
            \[
            g(t)\le C_{3}\left(\log(t)+\log\left(\log(t)+C_{2}\right)\right)\le C_{4}\log(t).
            \]
            Step 3: Showing that $g(t)\ge\log(t)$\\
            We define $s(t)=\exp(g(t))\Rightarrow\dot{s}(t)=\exp(g(t))\dot{g}(t)=g^{b}(t)=\left[\log(s(t))\right]^{b}$.
            Note that $g(t)\to\infty$ implies $s(t)\to\infty$.\\
            We have
            
            \[
            \lim_{t\to\infty}\frac{t}{s(t)}=\lim_{t\to\infty}\frac{1}{\left[\log(s(t))\right]^{b}}=0
            \]
            and therefore $s(t)=\omega(t)$. This implies that $\exists t_{4}$
            so that $\forall t>t_{4}$:
            
            \begin{align*}
            s(t) & \ge t\\
            \exp(g(t)) & \ge t\\
            g(t) & \ge\log(t).
            \end{align*}
            Combining the results from steps 2 and 3 we obtain $g(t)=\theta\left(\log(t)\right)$.
            
            \begin{align*}
            1 & \le\frac{g(t)}{\log(t)}=\frac{g(t)-b\log\left(g(t)\right)}{\log(t)}+\frac{b\log\left(g(t)\right)}{log(t)}\le\frac{\log(t)+\log\left(\log(t)+C_{2}\right)}{\log(t)}+\frac{b\log\left(C_{3}\left(\log(t)+\log\left(\log(t)+C_{2}\right)\right)\right)}{\log(t)}\to1
            \end{align*}
            From the Squeeze Theorem $\lim_{t\to\infty}\frac{g(t)}{log(t)}=1$.
            Thus, $g(t)=\log(t)+o\left(\log(t)\right)$.
        
    \remove{
	From Conjecture \ref{conjecture: generic tail}, we can write $\wvec(t) = \what g(t) + \rhoVec$ where $\rhoVec = o(g(t))$.	
	We can use this to calculate the normalized weight vector:
	\begin{flalign} \label{Calculation of the normalized weight vector}
	&\frac{\wvec(t)}{\Vert \wvec(t)\Vert}&\nonumber\\
	& =\frac{g(t)\what +\rhoVec}{\sqrt{g(t)^2\what^\top\what + \rhoVec^\top\rhoVec+2g(t)\what^\top\rhoVec}}\nonumber\\
	& =\frac{\what + g^{-1}(t)\rhoVec}{\whatNorm\sqrt{1 +2\frac{\what^\top\rhoVec}{g(t)\whatNorm^2}+ \frac{\Vert\rhoVec\Vert^2}{g^2(t)\whatNorm^2}}}\nonumber\\
	& = \frac{1}{\whatNorm}\left(\what +g^{-1}(t)\rhoVec\right)\left[1-\frac{\what^\top\rhoVec}{g(t)\whatNorm^2}+\left[\frac{3}{4}\left(2\frac{\what^\top\rhoVec}{\whatNorm^2}\right)^2-\frac{\Vert\rhoVec\Vert^2}{2\whatNorm^2}\right]\frac{1}{g^2(t)}+O\left(\left(\frac{\what^\top\rhoVec}{g(t)}\right)^3\right)\right]\nonumber\\
	& = \frac{\what}{\whatNorm}+\left(\frac{\rhoVec}{\whatNorm}-\frac{\what}{\whatNorm}\frac{\what^\top\rhoVec}{\whatNorm^2}\right)\frac{1}{g(t)}+O\left(\left(\frac{\what^\top\rhoVec}{g(t)}\right)^2\right)\nonumber\\
	& = \frac{\what}{\whatNorm}+\left(I-\frac{\what\what^\top}{\whatNorm^2}\right)\frac{1}{\whatNorm}\frac{\rhoVec}{g(t)}+O\left(\left(\frac{\what^\top\rhoVec}{g(t)}\right)^2\right)
	\end{flalign}
	we used $\frac{1}{\sqrt{1+x}}=1-\frac{1}{2}x+\frac{3}{4}x^2+O(x^3)$.\\
	We use eq.~\eqref{Calculation of the normalized weight vector} to calculate the angle:
	\begin{align} \label{eq: Calculation of the angle}
	&\frac{\wvec(t)^\top\what}{\wNorm\whatNorm}&\nonumber\\
	&=\frac{\what^\top}{\whatNorm^2}\left(\what +g^{-1}(t)\rhoVec\right)\left[1-\frac{\what^\top\rhoVec}{g(t)\whatNorm^2}+\left[\frac{3}{4}\left(2\frac{\what^\top\rhoVec}{\whatNorm^2}\right)^2-\frac{\Vert\rhoVec\Vert^2}{2\whatNorm^2}\right]\frac{1}{g^2(t)}+O\left(\left(\frac{\what^\top\rhoVec}{g(t)}\right)^3\right)\right]\nonumber\\
	& = 1 + \frac{2}{\whatNorm^2}\left[\left(\frac{\rhoVec^\top\what}{\whatNorm \Vert\rhoVec\Vert}\right)^2 - \frac{1}{4}\right]\frac{\Vert\rhoVec\Vert^2}{g^2(t)}+O\left(\left(\frac{\what^\top\rhoVec}{g(t)}\right)^3\right)
	\end{align}
	Calculation of the margin:
	\begin{flalign} \label{eq: Calculation of the margin}
		&\min_{n} \frac{\xnT \wvec(t)}{\Vert \wvec(t)\Vert}\nonumber\\
		& =\min_{n} \xnT\left[\frac{\what}{\whatNorm}+\left(\frac{\rhoVec}{\whatNorm}-\frac{\what}{\whatNorm}\frac{\what^\top\rhoVec}{\whatNorm^2}\right)\frac{1}{g(t)}+O\left(\left(\frac{\what^\top\rhoVec}{g(t)}\right)^2\right)\right]\nonumber\\
		& =\frac{1}{\whatNorm}+\frac{1}{\whatNorm}\left(\min_{n}\xnT\rhoVec-\frac{\what^\top\rhoVec}{\whatNorm^2}\right)\frac{1}{g(t)}
		+O\left(\left(\frac{\what^\top\rhoVec}{g(t)}\right)^2\right)
	\end{flalign}
    From Conjecture \ref{conjecture: generic tail}, we can also characterize $\rhoVec$: 
	\begin{equation} \label{eq: rho def generic case appendix}
			\bm{\rho}(t) = \begin{cases}
\left(f'(g(t))\right)^{-1}\genericConstVec+o\left(\left(f'(g(t))\right)^{-1}\right), & \text{if $\frac{1}{f'(g(t))}=\Omega(1)$}\\
			O(1), & \text{otherwise}
			\end{cases}
		\end{equation}	
      Substituting eq.~\eqref{eq: rho def generic case appendix} into eqs. \ref{Calculation of the normalized weight vector}, \ref{eq: Calculation of the angle}, \ref{eq: Calculation of the margin} we get:
	\begin{flalign} \label{eq: generic norm results}
		&\left \Vert \frac{\wvec(t)}{\Vert \wvec(t)\Vert} - \frac{\what}{\whatNorm} \right \Vert =\begin{cases}
		\left\Vert\left(I-\dfrac{\what\what^\top}{\whatNorm^2}\right)\genericConstVec\right\Vert\dfrac{1}{\whatNorm}\dfrac{1}{g(t)f'(g(t))}, & \text{if $\dfrac{1}{f'(g(t))}=\Omega(1)$}\\
			O(g^{-1}(t))\vphantom{\dfrac{0}{0}}, & \text{otherwise}
		\end{cases}&
	\end{flalign}
    \begin{equation} \label{eq: generic angle results}
	 1 - \frac{\wvec(t)^\top\what}{\wNorm\whatNorm}  = \begin{cases}
		\left(\dfrac{1}{4}-\left(\dfrac{\genericConstVec^\top\what}{\whatNorm \Vert\genericConstVec\Vert}\right)^2\right)\dfrac{2\Vert\genericConstVec\Vert^2}{\whatNorm^2}\dfrac{1}{\left(g(t)f'(g(t))\right)^2}, & \text{if $\dfrac{1}{f'(g(t))}=\Omega(1)$}\\
			O(g^{-2}(t))\vphantom{\dfrac{0}{0}}, & \text{otherwise}
		\end{cases}
	\end{equation}
	\begin{equation} \label{eq: generic margin results}
	\frac{1}{\whatNorm} - \min_{n} \frac{\xnT \wvec(t)}{\Vert \wvec(t)\Vert}  = \begin{cases}
		\dfrac{1}{\whatNorm}\left(\dfrac{\what^\top\genericConstVec}{\whatNorm^2} - \min_{n}\xnT\genericConstVec\right)\dfrac{1}{g(t)f'(g(t))}, & \text{if $\dfrac{1}{f'(g(t))}=\Omega(1)$}\\
			O(g^{-1}(t))\vphantom{\dfrac{0}{0}}, & \text{otherwise}
		\end{cases}
	\end{equation}
   }

	\remove{\mnote{To prove these Theorems we need:
	\begin{enumerate}
	    \item Use Theorem 10 which is stated below and proved in appendix section \ref{sec: proof of theorem about general tail rates} to get the model and ODE. This part is done.
	    \item With this model we can calculate $\gamma - \min\limits_{n} \frac{\xnT \wvec (t)}{\Vert \wvec(t)\Vert}$ to get that the rate is $C\cdot \frac{\norm{\op\rhoVec}}{g(t)}$. The calculation is in appendix section  \ref{sec: generic tail rates calculation}. This completes the proof of Theorem \ref{theorem: general convergence rates simplified with ode}. Notice the part about time rescaling which needs to be addressed (there is a note about this in the Theorem).
	    \item To complete the proof for Theorem \ref{theorem: general convergence rates simplified} we focus on the case $L=1$ and need to explain why the rate $C\cdot \frac{\norm{\op\rhoVec}}{g(t)}$ is optimal for exponential tail. I started doing this in appendix section \ref{sec: proving corollary L=1 optimal rate is obtained for exp loss} (which needs to be moved and renamed). The part for $f(t)=\omega(t)$ is done. The part for $f(t)=o(t)$ is almost done. I added a note there about what is missing.
	\end{enumerate}
	}
	
	For a specific family of functions this results can be proved without assuming the loss is minimized and the incremental update converge in direction. results on poly-exp tail are on appendix section \ref{sec: appendix results on poly-exp tails}.
	Using this Theorem we can calculate the convergence rates and show two important results:
	\begin{itemize}
	    \item In the case of a single layer, $L=1$, and general tail loss function the optimal rate is achieved for exponential loss (\mnote{We need to say that in the case $f(t)=o(t)$ we also need to assume $f'(t)=o(1)$ and concave to show this result}). This is the result in Theorem \ref{theorem: general convergence rates simplified}.
	    \item For any linear network with exponential loss the convergence rates are $1/g(t)$ where $g(t)$ is the asymptotic solution of the ode in Theorem \ref{theorem: general convergence rates simplified with ode}.
	\end{itemize}}
	\remove{
	\subsection{Implication: general tail with $L=1$}
	Using theorem \ref{theorem: general convergence rates} with general tail and $L=1$ we have that $\tilde{g}(t)$ is the asymptotic solution of
	\[
	    \e\left(-f\left(\tilde{g}(t)\right)\right)=\frac{1}{\eta_t\fsym_1(t)}\frac{d}{dt}\tilde{g}(t)
	\]
	and
	\[
	        \norm{\op_1\rhoVec}=\begin{cases} \fsym_2(t)\left(f'\left(\tilde{g}(t)\right)\right)^{-1}+o\left(\left(f'\left(\tilde{g}(t)\right)\right)^{-1}\right), & \text{if } \left(f'\left(\tilde{g}(t)\right)\right)^{-1}=\Omega(1)\\
	        O\left(1\right), & \text{otherwise}
	        \end{cases}
    \]
    for positive functions $\fsym_1(t)=\Theta(1),\ \fsym_2(t)=\Theta(1)$ independent of $f$. This implies the convergence rates in table \ref{table: L=1 generic convergence rates} (the calculation is in appendix section \ref{sec: CALCULATION OF CONVERGENCE RATES}).
    \begin{table} [ht]
        \centering
			\normalsize	
			\begin{tabular}{c c c} 
				\toprule
				 & $\left(f'\left(g\left(t\right)\right)\right)^{-1}=\Omega(1)$ & \text{otherwise}\\ 
				\midrule
				$ \left \Vert \frac{\wvec(t)}{\Vert \wvec(t)\Vert} - \frac{\what}{\whatNorm} \right \Vert$ \text{ or } $ \gamma - \min\limits_{n} \frac{\xnT \wvec (t)}{\Vert \wvec(t)\Vert}$
				&$  \displaystyle \frac{C_1\fsym_2(t)}{g(t)f'(g(t))} + o\left(\frac{1}{g(t)f'(g(t))}\right) $ &  $O\left(g^{-1}(t)\right)$\\ 
				$1 - \dfrac{\eqw(t)^\top\what}{\norm{\eqw}\norm{\what}}$ &$ \displaystyle \frac{C_2\fsym_2^2(t)}{\left(g(t)f'(g(t))\right)^2} + o\left(\frac{1}{\left(
				g\left(t\right)f'\left(g\left(t\right)\right)\right)^{2}}\right) $ &  $O\left(g^{-2}(t)\right)$\\
				\bottomrule
			\end{tabular}
			\captionsetup{width=\textwidth}
			\caption{Summary of convergence rates derived from Theorem \ref{theorem: general convergence rates} for $L=1$ and general tail loss function. The first line is the convergence rate for both the distance and the suboptimality of the margin. The second line is the angle convergence rate. } \label{table: L=1 generic convergence rates}
	\end{table}
	\begin{restatable}{corollary}{corollarySingleLayyerOptimalLoss} \label{corollary: L=1 optimal rate is obtained for exp loss}
	    For $L=1$, the optimal rate is obtained for exponential loss (\mnote{We need to say that in the case $f(t)=o(t)$ we also need to assume $f'(t)=o(1)$ and concave to show this result}).
	\end{restatable}
	We show this Corollary in appendix section \ref{sec: prving corollary L=1 optimal rate is obtained for exp loss}.
	For a specific family of functions this results can be proved without assuming the loss is minimized and the incremental update converge in direction. results on poly-exp tail are on appendix section \ref{sec: appendix results on poly-exp tails}.
	The rates in this case are consistent with the rates in table \ref{table: L=1 generic convergence rates}. 
	
	\subsection{Implication: For any linear network with exponential loss}
	Using theorem \ref{theorem: general convergence rates} with exponential loss we get that $\tilde{g}(t)$ is the asymptotic solution of
	\[
	    \frac{\e\left(-\tilde{g}(t)\right)\left(\tilde{g}(t)\right)^{2\left(1-\frac{1}{L}\right)}\eta_t\fsym_1(t)}{\displaystyle \frac{d}{dt}\tilde{g}(t)}=1,
	\]
	where $\fsym_1(t)=\Theta(1)$.\\
	From numerical approximation $\tilde{g}(t)=\Theta\left(\log(t)\right)$. This result implies the convergence rates specified in table \ref{table: linear network convergence rates} for the network equivalent linear predictor (the calculation is in appendix section \ref{sec: CALCULATION OF CONVERGENCE RATES}).
	\begin{table} [h]
		\centering
			\normalsize	
			\begin{tabular}{c c} 
				\toprule				
				$\left \Vert \dfrac{\eqw(t)}{\Vert \eqw(t)\Vert} - \dfrac{\what}{\norm{\what}} \right \Vert$ \text{ or } $\displaystyle \gamma - \min_{n} \dfrac{\xnT \eqw (t)}{\Vert \eqw(t)\Vert}$ & $O\left(\dfrac{1}{\log(t)}\right)$ \vphantom{$\displaystyle O\left(\frac{1}{\log^{\frac{1}{3}}(t)}\right)$} \\ 
				 $1 - \dfrac{\eqw(t)^\top\what}{\norm{\eqw}\norm{\what}}$ &$O\left(\dfrac{1}{\log^2(t)}\right)$ \vphantom{$\displaystyle O\left(\frac{1}{\log^{\frac{1}{3}}(t)}\right)$} \\
				\bottomrule
			\end{tabular}
			\captionsetup{width=\textwidth}
			\caption{Summary of convergence rates for the linear network equivalent predictor for exponential loss function. The first line is the convergence rate for both the distance and the suboptimality of the margin. The second line is the angle convergence rate.} \label{table: linear network convergence rates}
	\end{table}
	}
	\subsection{Proof that $\gamma - \min\limits_{n} \frac{\xnT \wvec (t)}{\Vert \wvec(t)\Vert}=\Omega\left(\frac{1}{\log(t)}\right)$} \label{sec: proof that rate is Omega(1/log(t))}
	In this section, we need to prove that if $f'(u)=\omega(1)$ and $\abs{\frac{f'(u)}{f(u)}}=O(u^{-1})$ then $\gamma - \min\limits_{n} \frac{\xnT \wvec (t)}{\Vert \wvec(t)\Vert}=\Omega\left(\frac{1}{\log(t)}\right)$. From the calculation in appendix sections \ref{section: Asymptotic rates for depth L linear networks}, \ref{sec: proving corollary L=1 optimal rate is obtained for exp loss}, we have that exists $C>0$ so that $\gamma - \min\limits_{n} \frac{\xnT \wvec (t)}{\Vert \wvec(t)\Vert} = C\frac{1}{g(t)f'(g(t))}$ where 
	$g(t) = f^{-1}(\log(t)) + o(f^{-1}(\log(t)))$ from \cite{odeSolMathOverFlow}. From $\abs{\frac{f'(u)}{f(u)}}=O(u^{-1})$ we have that
	\[
	    f^{-1}(\log(t))f'(f^{-1}(\log(t))) = O\left( f^{-1}(\log(t)) \frac{f\left(f^{-1}(\log(t)\right)}{f^{-1}(\log(t)} \right) = O(\log(t))\,.
	\]
	Thus, combining this result with $g(t) = f^{-1}(\log(t)) + o(f^{-1}(\log(t)))$ and $\gamma - \min\limits_{n} \frac{\xnT \wvec (t)}{\Vert \wvec(t)\Vert} = C\frac{1}{g(t)f'(g(t))}$ we get that
	\[
	    \gamma - \min\limits_{n} \frac{\xnT \wvec (t)}{\Vert \wvec(t)\Vert}=\Omega\left(\frac{1}{\log(t)}\right)
	\]
	as required.
    \section{Proof of Theorem \ref{theorem: general convergence rates}} \label{sec: proof of theorem about general tail rates}
    \subsection{Preliminaries and Auxiliary Lemma}
    Recall Assumption~\ref{assumption: f can be expended}
    \assumptionfexpansion*
    \begin{restatable}{claim}{claimfexpansion} \label{claim: f(g+h)=f(g)+f'(g)h}
	    For any function $f(u)$ that satisfies assumption \ref{assumption: f can be expended}, and any two functions $g(t),\ h(t)$ such that $h(t)=o(g(t))$ and $\lim_{t\to\infty}g(t)=\infty$, $\exists t_1$ so that $\forall t>t_1$:
	    \[
	        f\left( g(t)+h(t)\right)=f\left(g(t)\right)+f^\prime \left(g(t)\right) h(t) + R(t),
	    \]
	    where $R(t)=o\left( f^\prime \left(g(t)\right)h(t)\right)$.
	\end{restatable}
    \begin{proof}
	    Since $h(t)=o(g(t))$ and $\lim_{t\to\infty}g(t)=\infty$ we have that $\exists t_1$ so that $\forall t>t_1: g(t)+h(t)>0$. From our assumption that $f$ is real analytic on $\mathbb{R}_{++}$ we get that $\forall t>t_1$:
	    \[
	        f\left(g(t)+h(t)\right) = f\left(g(t)\right) + f'\left(g(t)\right) h(t) + \sum_{k=2}^{\infty} \frac{1}{k!}f^{(k)}(g(t))h^k(t).
	    \]
	    We denote $R(t) = \sum_{k=2}^{\infty} \frac{1}{k!}f^{(k)}(g(t))h^k(t)$.
	    We need to show that $\lim\limits_{t\to\infty} \frac{R(t)}{f^\prime \left(g(t)\right) h(t)}=0.$
	    
	    Since $\forall k\in \mathbb{N}:\ \abs{\dfrac{f^{(k+1)}(t)}{f^\prime(t)}}=O\left(t^{-k}\right)$ we have that $\exists t_2>t_1$ and positive constants $\left\{ C_k\right\}_{k=2}^{\infty}$, so that $\forall t>t_2$:
	    \begin{align}
	        \nonumber&0\le \abs{\frac{R(t)}{f^\prime \left(g(t)\right) h(t)}}\le 
	        \abs{\frac{1}{f^\prime \left(g(t)\right) h(t)}\sum_{k=2}^{\infty} C_k f^{\prime}(g(t))\frac{h^k(t)}{\left(g(t)\right)^{k-1}}} \\
	        &= \abs{\sum_{k=2}^{\infty} C_k \left( \frac{h(t)}{g(t)} \right)^{k-1}}
	        \le \max_k C_k \abs{\frac{\frac{h(t)}{g(t)}}{1-\frac{h(t)}{g(t)}}}\xrightarrow{t\to\infty} 0
	    \end{align}
	    where in the last transition we used $h(t)=o(g(t))$. Thus, by the squeeze theorem $\displaystyle \lim_{t\to\infty} \frac{R(t)}{f^\prime \left(g(t)\right) h(t)}=0$.
	\end{proof}
	The following lemma will be useful in characterizing $\eqw(t)$ asymptotic behaviour.
    \begin{lemma} \label{eq: lemma about vector limit}
	    If $\displaystyle \lim_{t\to\infty} \sumn f_n(t)\xn=\sumn \alpha_n \xn$ where $\xn$ are linearly independent vectors then $\forall n\in\{1,...,N\}:\ \lim_{t\to\infty} f_n(t)=\alpha_n$.
	\end{lemma}
	\begin{proof}
	    $\displaystyle \lim_{t\to\infty} \sumn f_n(t)\xn=\sumn \alpha_n \xn$ implies that $\forall \epsilon>0$, $\exists t'$ so that $\forall t>t':$ 
	    \[
	        \norm{\sumn f_n(t)\xn-\sumn \alpha_n \xn}<\epsilon.
	    \]
	    Since $\vect{x}_1,...,\xn$ are linearly independent, $\forall k$ $\exists \vect{u}$ such that $\forall n\neq k: \ip{\xn}{\vect{u}}=0$ and $\ip{\xk}{\vect{u}}=1$. Using Cauchy-Schwarz inequality we have that, $\forall k=1,...,n$ and $\forall t>t'$:
	    \[
	        \abs{f_k(t)-\alpha_k}=\abs{\ip{\sumn f_n(t)\xn-\sumn \alpha_n \xn}{\vect{u}}}\le \norm{\sumn f_n(t)\xn-\sumn \alpha_n \xn}\norm{\vect{u}}< \norm{\vect{u}} \epsilon .
	    \]
	\end{proof}
	\noindent Using Theorem \ref{Theorem: convergence to max margin for general tail}, Lemma \ref{lemma: every accumulation point of z/norm(z) is a linear combination of sv}, Lemma \ref{eq: lemma about vector limit} and the claim \ref{claim: f(g+h)=f(g)+f'(g)h} we will prove Theorem \ref{theorem: general convergence rates} that characterizes $\eqw(t)$ asymptotic behavior.
    \subsection{Proof of Theorem \ref{theorem: general convergence rates}}
    \theoremGeneralRates*
    \begin{proof}
	    From Theorem \ref{Theorem: convergence to max margin for general tail} we have that $\displaystyle \lim_{t\to\infty} \frac{\eqw(t)}{\norm{\eqw(t)}}=\frac{\what}{\norm{\what}}$ where $\what$ is the maximum-margin separator. In addition, from lemma \ref{lemma: every accumulation point of z/norm(z) is a linear combination of sv} we have that $\lim_{t\to\infty} \norm{\eqw(t)}=\infty$. Combining these two results we can write
	    \begin{equation} \label{eq: wvec model}
	        \eqw(t)=\tilde{g}(t)\what + \rhoVec,
	    \end{equation}
	    where $\lim_{t\to\infty}\tilde{g}(t)=\infty$, $\norm{\rhoVec}=o\left(\tilde{g}(t)\right)$ and $\rhoVec^\top\what=0$.

	    Using Claim \ref{claim: f(g+h)=f(g)+f'(g)h} we have,
	    \begin{equation} \label{eq: claim f(g+h)=f(g)+f'(g)h result}
	        \forall n\in \set :\ f\left(\tilde{g}(t) + \rhoVec^\top\xn\right) = f\left(\tilde{g}(t)\right)+f^{\prime}\left(\tilde{g}(t) \right)\rhoVec^\top\xn+R_n(t),
	    \end{equation}
	    where $R_n(t)=o\left(f^{\prime}\left(\tilde{g}(t) \right)\rhoVec^\top\xn \right)$. 
	    We denote $\set=\left\{ n: \what^\top\xn = 1 \right\}$ (the indices of the support vectors) and recall that the maximum-margin separator can be expresses as a linear combination of the support vectors $\what=\sumnsv \alpha_n \xn$.
	    In order to calculate the convergence rates we need to characterize the asymptotic behavior of $\norm{\rhoVec}$ and $g(t)$.
	    
	    \paragraph{Step 1: }Proving that $\forall n\in\set:\ \fsym_{n}(t)\triangleq f^{\prime}\left(\tilde{g}(t) \right)\rhoVec^\top\xn+R_n(t)=O(1)$.

	    We want to prove that $\forall n\in\set$ $\fsym_{n}(t)$ is asymptotically bounded, i.e. $\exists t_1>0,m,M$ so that $\forall t>t_1:$ $m\le\fsym_{n}(t)\le M$. 
	    
	    We assume in contradiction that $\exists k\in\set$ so that $\fsym_k(t)$ is not asymptotically bounded from above. Using Bolzano-Weierstrass theorem we have that $\exists \left\{ \bar{t}_i\right\}_{i=1}^{\infty}$ so that the sequence $\fsym_k(\bar{t}_i)\xrightarrow{i\to\infty}\infty$.
	    
	    From Theorem \ref{Theorem: convergence to max margin for general tail} and Lemma \ref{lemma: every accumulation point of z/norm(z) is a linear combination of sv} $\exists \left\{ \tilde{\alpha}_n\ge0\ :\ n\in\set\right\}$ which are a positive scaling of $\alpha_n$ so that
	    \begin{equation} \label{eq: z/norm(z) limit}
	         \frac{\sumnsv \e\left( -f\left(\eqw(t)^\top\xn\right)\right)\xn}{\norm{\sumnsv \e\left( -f\left(\eqw(t)^\top\xn\right)\right)\xn}}
	        \to \sumnsv \tilde{\alpha}_n \xn.
	    \end{equation}
	     Substituting eq.~\eqref{eq: wvec model} into eq.~\eqref{eq: z/norm(z) limit} we get that
	     \begin{align}
	        \nonumber &\frac{\sumnsv \e\left( -f\left(\eqw(t)^\top\xn\right)\right)\xn}{\norm{\sumnsv \e\left( -f\left(\eqw(t)^\top\xn\right)\right)\xn}}
	        \overset{(1)}{=} \frac{\sumnsv \e\left( -f\left(\tilde{g}(t)\what^\top\xn+\rhoVec^\top\xn\right)\right)\xn}{\norm{\sumnsv \e\left( -f\left(\tilde{g}(t)\what^\top\xn+\rhoVec^\top\xn\right)\right)\xn}}\\
	        &\overset{(2)}{=}  \frac{\sumnsv \e\left( -f\left(\tilde{g}(t)\right)-\fsym_n(t)\right)\xn}{\norm{\sumnsv \e\left( -f\left(\tilde{g}(t)\right)-\fsym_n(t)\right)\xn}}
	        =\frac{\sumnsv \e\left( -\fsym_n(t)\right)\xn}{\norm{\sumnsv \e\left( -\fsym_n(t)\right)\xn}}
	        \to \sumnsv \tilde{\alpha}_n \xn\ ,
	     \end{align}
	     where in (1) we used eq.~\eqref{eq: wvec model}, in (2) we used $\forall n\in\set: \what^\top\xn=1$ and eq.~\eqref{eq: claim f(g+h)=f(g)+f'(g)h result}.
	     From the last equation we also have that
	     \begin{equation}
	        \frac{\sumnsv \e\left( -\fsym_n(\bar{t}_i)\right)\xn}{\norm{\sumnsv \e\left( -\fsym_n(\bar{t}_i)\right)\xn}}
	        \xrightarrow{i\to\infty} \sumnsv \tilde{\alpha}_n \xn.
	     \end{equation}
	     Combining the last equation with the facts that for almost every dataset $\forall n\in\set: \tilde{\alpha}_n>0$ (as a positive scaling of $\alpha_n$) and $\xn$ are linearly independent (from Lemma 8 in \cite{soudry2017implicit}) and Lemma \ref{eq: lemma about vector limit} we have that $\forall n\in\set:\ \e\left( -\fsym_n(\bar{t}_i)\right)=\Theta\left(\norm{\sum_{m=1}^{N} \e\left( -\fsym_{m}(\bar{t}_i)\right)\vect{x}_m}\right)$.

	     Thus, if for some $k\in \set$, $\fsym_k(\bar{t}_i)\xrightarrow{i\to\infty}\infty$, then this implies that $\forall n\in\set:\ \fsym_n(\bar{t}_i)\xrightarrow{i\to\infty}\infty$. In addition, since $\fsym_{n}(t)\triangleq f^{\prime}\left(\tilde{g}(t) \right)\rhoVec^\top\xn+R_n(t)$ where $R_n(t)=o\left(f^{\prime}\left(\tilde{g}(t) \right)\rhoVec^\top\xn \right)$ we get that $\forall n\in\set: f^{\prime}\left(\tilde{g}(\bar{t}_i) \right)\bm{\rho}(\bar{t}_i)^\top\xn\to\infty$. However, this implies
	     \[
	        0\overset{(1)}{=}f^{\prime}\left(\tilde{g}(\bar{t}_i) \right) \what^\top\bm{\rho}(\bar{t}_i) \overset{(2)}{=} \sumnsv \alpha_n f^{\prime}\left(\tilde{g}(\bar{t}_i) \right) \xn^\top\bm{\rho}(\bar{t}_i) \xrightarrow[i\to\infty]{(3)}\infty
	     \]
	     where in $(1)$ we used $\forall t:\ \what^\top\rhoVec=0$, in $(2)$ we used $\what=\sumnsv \alpha_n \xn$ and in $(3)$ we used that $\alpha_n>0$. We got a contradiction and thus our contradiction assumption must be false $\Rightarrow$ $\forall k\in\set: \fsym_k(t)$ is bounded from above. Similarly, $\forall k\in\set: \fsym_k(t)$ is bounded from below since $\tilde{\alpha}_n>0$ for all $n\in\set$. Combining these results we have that $\forall n\in\set: \fsym_n(t)=O(1)$.
	     
	     \paragraph{Step 2:} Characterizing $\rhoVec$ and $\tilde{g}(t)$ asymptotic behavior.
	     
	     From the previous step we have that $\forall n\in\set: \fsym_n(t)=O(1)$. We recall $\fsym_n$ definition $\fsym_{n}(t)\triangleq f^{\prime}\left(\tilde{g}(t) \right)\rhoVec^\top\xn+R_n(t)$ where $R_n(t)=o\left(f^{\prime}\left(\tilde{g}(t) \right)\rhoVec^\top\xn \right)$. This implies that $\forall n\in\set:  f^{\prime}\left(\tilde{g}(t) \right)\rhoVec^\top\xn=O(1)$. 
	     
	     Since this is true $\forall n\in\set$ we have that $\rhoVec$ components that are in subspace spanned by the support vectors are bounded, i.e. $f^{\prime}\left(\tilde{g}(t) \right)\norm{\op_1\rhoVec}=O(1) \Rightarrow \norm{\op_1\rhoVec}=O\left(\left(f^{\prime}\left(\tilde{g}(t) \right)\right)^{-1}\right)$. In the next steps we will further characterize $\norm{\op_1\rhoVec}$ behaviour.
	     We denote $\vect{z}(t)\triangleq -\nabla_{\eqw} \mathcal{L}\left( \eqw(t)\right)=-\sumn \ell^{\prime}\left(\eqw(t)^\top\xn\right)\xn$.\\
	     \textbf{Step 2.a:} Showing that $\displaystyle \lim_{t\to\infty} \frac{\vect{z}(t)}{\norm{\vect{z}(t)}}=\lim_{t\to\infty} \frac{1}{\norm{\vect{z}(t)}}\e\left( -f\left(\tilde{g}\left(t\right)\right)\right)\sumnsv \e\left(-f'\left(\tilde{g}(t)\right)\rhoVec^\top\xn \right)\xn$\\
	    From Theorem \ref{Theorem: convergence to max margin for general tail} we have that $\displaystyle \lim_{t\to\infty} \frac{\vect{z}(t)}{\norm{\vect{z}(t)}}$ exists (and finite). In addition,
	    \begin{align*}
    	    \lim_{t\to\infty} \frac{\vect{z}(t)}{\norm{\vect{z}(t)}}&\overset{(1)}{=}\lim_{t\to\infty}
    		\frac{-\sumnsv \ell^{\prime}\left(\eqw(t)^\top\xn\right)\xn}{\norm{\vect{z}(t)}}\\
    		&\overset{(2)}{=} \lim_{t\to\infty} \frac{\sumnsv \e\left( -f\left(\tilde{g}(t)\what^\top\xn+\rhoVec^\top\xn\right)\right)\xn}{\norm{\vect{z}(t)}}\\
            &\overset{(3)}{=}  \lim_{t\to\infty} \frac{\e\left( -f\left(\tilde{g}(t)\right)\right)\sumnsv \e\left( - f^{\prime}\left(\tilde{g}(t) \right)\rhoVec^\top\xn-R_n(t)\right)\xn}{\norm{\vect{z}(t)}}\\
            &\overset{(4)}{=}\lim_{t\to\infty} \frac{\e\left( -f\left(\tilde{g}(t)\right)\right)\sumnsv \e\left(-f'\left(\tilde{g}(t)\right)\rhoVec^\top\xn \right)\xn}{\norm{\vect{z}(t)}}\ ,
		\end{align*}
		where in (1) we used $\vect{z}(t)$ definition and Lemma \ref{lemma: every accumulation point of z/norm(z) is a linear combination of sv}, in (2) we used $-\ell'(u)=-\e\left(-f\left(u\right)\right)$ and $\eqw(t)=\tilde{g}(t)\what + \rhoVec$ (eq.~\eqref{eq: wvec model}), in 3 we used eq.~\eqref{eq: claim f(g+h)=f(g)+f'(g)h result} and in 4 we used the fact that $\forall n\in\set:\ \fsym_{n}(t)\triangleq f^{\prime}\left(\tilde{g}(t) \right)\rhoVec^\top\xn+R_n(t)=O(1)$ and $R_n(t)=o\left(f^{\prime}\left(\tilde{g}(t) \right)\rhoVec^\top\xn \right)$ which implies $R_n(t)\to0$.\\
		\textbf{Step 2.b:}  Showing that for a constant $C$: $\displaystyle \lim_{t\to\infty} \frac{\vect{z}(t)}{\norm{\vect{z}(t)}}=\lim_{t\to\infty} C\frac{\displaystyle  \frac{d}{dt}\tilde{g}(t)}{\norm{\vect{z}(t)}\left(\tilde{g}(t)\right)^{\frac{2}{L}\left(L-1\right)}}\sumnsv \alpha_n\xn$\\
		\remove{
		\mnote{new part}
		    Recall we denoted $\vect{z}(t)=-\nabla_{\eqw}\mathcal{L}(\eqw(t))$.
		    Computing the gradients descent updates for $\lw{1}(t)$, we have 
            \[
            \lw{1}(t+1)-\lw{1}(t)=-\eta_t\nabla_{\lw{1}}\mathcal{L}_{\P}(\np(t))=-\eta_t\nabla_{\lw{1}}\mathcal{L}(\lw{1:L})=-\eta_t\nabla_{\eqw}\mathcal{L}(\eqw(t)){\lw{2:L}(t)}^\top=\eta_t\vect{z}(t){\lw{2:L}(t)}^\top
            \]
             From the assumption that $\vect{z}(t)$ converge in direction we have that
            \begin{align}
            \vect{z}(t)=\bar{\vect{z}}^\infty p(t)+\bdelta_{\vect{z}}(t)\,p(t),
            \label{eq:z}
            \end{align}
            and thus
            \begin{align*}
                \Delta \lw{1}\triangleq\lw{1}(t+1)-\lw{1}(t) &= \eta_t \Big( \bar{\vect{z}}^\infty p(t)+\bdelta_{\vect{z}}(t)\,p(t) \Big) \left(\lwinfty{2:L}\, g(t)^{L-1}+\bdelta_{\lw{2:L}}(t)\,g(t)^{L-1}\right)^\top \\
                &\overset{(1)}{=}\left(\eta_t p(t)g(t)^{L-1}\right)\left(\bar{\vect{z}}^\infty \right(\lwinfty{2:L}\left)^\top  +\bdelta(t)\right)\,,
            \end{align*}
            where in $(1)$ $\bdelta(t) = \bdelta_{\vect{z}}(t)\left(\lwinfty{2:L}+\bdelta_{\lw{2:L}}(t)\right)^\top + \bdelta_{\vect{z}}(t) \bdelta_{\lw{2:L}}(t)^\top \to0$.
            This implies that $\lw{1}(t+1)-\lw{1}(t)$ converge in direction with positive margin (see Claim 1 in \cite{gunasekar2018conv}). In addition, combining this result with eq. 25 in \cite{gunasekar2018conv}, we have that $\lim\limits_{t\to\infty} \frac{\lw{1}}{\norm{\lw{1}}}=\lim\limits_{t\to\infty}\frac{\Delta \lw{1}}{\norm{\Delta \lw{1}}} = \bar{\vect{z}}^{\infty} \left(\lwinfty{2:L}\right)^\top\triangleq \lwinfty{1}$. 
            Using eq. \mnote{TBD}:
            \begin{equation}
                \lw{l}(t) = \lwinfty{l}g(t)+\delta_{\lw{l}}(t)g(t) \label{eq:ufcn}
            \end{equation}
            we obtain that
            \begin{align}
                \Delta\lw{l}(t)&=\lwinfty{l} h(t)+\bdelta_{\Delta\lw{l}}(t)\,h(t)\label{eq:dufcn}
                \end{align}
        where $h(t)=\norm{\Delta\np(t)}$, $g(t)=\sum_{u<t} \eta_u h(u)\to\infty$, and $\bdelta_{\Delta\lw{l}}(t),\bdelta_{\lw{l}}[t]\to 0$. 
            Consider the following arguments on $\Delta\lw{1}(t) \lw{2:L}(t)$,
            \begin{equation}
            \begin{split}
            &\Delta\lw{1}(t) \lw{2:L}(t)=\vect{z}(t)\norm{\lw{2:L}(t)}^2\\
            \overset{(a)}\implies &\left(\lwinfty{1} h(t)+\bdelta_{\Delta\lw{1}}(t)\,h(t)\right)\left(\lwinfty{2:L}\, g(t)^{L-1}+\bdelta_{\lw{2:L}}(t)\,g(t)^{L-1}\right)=\vect{z}(t)\norm{\lw{2:L}(t)}^2\\
            \overset{(b)}\implies &\frac{\vect{z}(t)\norm{\lw{2:L}(t)}^2}{h(t)\,g(t)^{L-1}}=\lwinfty{1:L}+\bdelta(t)=\bar{\eqw}_\infty+\tilde{\bdelta}(t),
            \end{split}
            \label{eq:fcn-slackness}
            \end{equation}
            where in $(a)$, we used eqs. \ref{eq:dufcn}-\ref{eq:ufcn}, and in $(b)$ we have $\tilde\bdelta(t)=\bdelta_{\Delta\lw{1}}(t)\bdelta_{\lw{2:L}}(t)+\bdelta_{\Delta\lw{1}}(t)\lwinfty{2:L}+\lwinfty{1}\bdelta_{\lw{2:L}}(t)\to0$. 
             Denote $s(t):=\frac{\norm{\vect{z}(t)}\norm{\lw{2:L}(t)}^2}{h(t)\,g(t)^{L-1}}$ From eq.~\eqref{eq:fcn-slackness} using triangle inequality we have that 
            \begin{equation}
            \norm{\bar{\eqw}_\infty}-\norm{\tilde\bdelta(t)}\le s(t)\le \norm{\bar{\eqw}_\infty}+\norm{\tilde\bdelta(t)}.
            \end{equation}
            Since $\tilde\bdelta(t)\to0$,  by squeeze theorem, we have that $\lim_{t\to\infty}s(t)=\norm{\bar{\eqw}_\infty}$. Using this in eq.~\eqref{eq:fcn-slackness}, we get the following:
            \begin{equation}
            \frac{\vect{z}(t)}{\norm{\vect{z}(t)}}=\frac{\bar{\eqw}_\infty}{s(t)}+\frac{\bdelta(t)}{s(t)}.
            \label{eq:fcn-slackness1}
            \end{equation}
		\mnote{end of new proof}}
		
		From eq.~\eqref{eq:fcn-slackness1} we have that
		\begin{equation} \label{eq: equation from theorem 1 in convNet paper}
    	    \lim_{t\to\infty} \frac{\vect{z}(t)}{\norm{\vect{z}(t)}}=\lim_{t\to\infty} \frac{\bar{\eqw}_\infty h(t)g^{L-1}(t)}{\norm{\vect{z}(t)}\norm{\lw{2:L}(t)}^2}.
		\end{equation}
	     In this equation, $\lw{2:L}^{(t)}\triangleq\lw{2}(t)\lw{3}(t)...\lw{L}(t)$ and $h(t)$ and $g(t)$ were used to model each layer weights:
	     \begin{align*}
	        \Delta \lw{l}(t)&=\lwinfty{l}h(t)+\delta_{\Delta\lw{l}}(t)h(t),\\
	         \lw{l}(t) &= \lwinfty{l}g(t)+\delta_{\lw{l}}(t)g(t),
	     \end{align*}
	     where $g(t)=\sum_{u<t} \eta_u h(u)$, $\Delta \lw{l}(t)\triangleq\eta_t^{-1}\left(\lw{l}(t+1)-\lw{l}(t)\right)$, $\winfty=\lim_{t\to\infty} \frac{\np(t)}{\norm{\np(t)}}$ and $\winfty=\left[\lwinfty{l}\right]$. From definition we have $h(t)=\eta_t^{-1}\left(g(t)-g(t-1)\right)$. Additionally, we can define $\bar{g}(u)=\bar{g}\left(\frac{t}{c}\right)=g(t),\ u=\frac{t}{c}$ (time rescaling) to show that $\lim\limits_{t\to\infty}\dfrac{g(t)-g(t-1)}{g'(t)}=1$ since:
	     \begin{align*}
            &\lim_{c\to\infty} \frac{\bar{g}(u+\frac{1}{c})-\bar{g}(u)}{\frac{1}{c}}=\bar{g}'(u)\implies  \lim_{c\to\infty} \frac{\bar{g}(u+\frac{1}{c})-\bar{g}(u)}{\frac{1}{c}\bar{g}'(u)}=1.\\
            &\text{Also, }\frac{d}{dt}g(t)=\frac{1}{c}\bar{g}'(u)\implies 1 = \lim_{c\to\infty} \frac{\bar{g}(u+\frac{1}{c})-\bar{g}(u)}{\frac{1}{c}\bar{g}'(u)}=
            \lim_{t\to\infty} \frac{g(t+1)-g(t)}{g'(t)}
	    \end{align*}
	     where in the last transition we used $\bar{g}(u)$ definition and $t=cu$. Combining this result with eq.~\eqref{eq: equation from theorem 1 in convNet paper} (using the fact that both limits exist and finite) we have:
	     \begin{equation} \label{eq: equation from theorem 1 in convNet paper with g'(t)}
    	    \lim_{t\to\infty} \frac{\vect{z}(t)}{\norm{\vect{z}(t)}}=\lim_{t\to\infty} \frac{\bar{\eqw}_\infty \eta_t^{-1}g'(t)g^{L-1}(t)}{\norm{\vect{z}(t)}\norm{\lw{2:L}(t)}^2}.
		\end{equation}
	     In addition, since we defined $\eqw(t)=\tilde{g}(t)\what + \rhoVec$ we have that $\tilde{g}(t)=\gamma \left(g(t)\right)^L$ where $\bar{\eqw}_\infty = \lwinfty{1}...\lwinfty{L}$ and $\gamma\triangleq \min_n \bar{\eqw}_\infty^\top \xn$. Using $\tilde{g}(t)$ eq.~\eqref{eq: equation from theorem 1 in convNet paper with g'(t)} can be written as
	     \begin{equation} \label{eq: equation from theorem 1 in convNet paper with tilde_g}
    	    \lim_{t\to\infty} \frac{\vect{z}(t)}{\norm{\vect{z}(t)}}=\lim_{t\to\infty} \frac{\displaystyle \frac{\bar{\eqw}_\infty}{\gamma L \eta_t}\cdot \frac{d}{dt}\tilde{g}(t)}{\norm{\vect{z}(t)}\left(\gamma^{-1}\tilde{g}(t)\right)^{\frac{2}{L}\left(L-1\right)}}=\lim_{t\to\infty} \frac{\displaystyle \gamma^{1-\nicefrac{2}{L}}}{L\eta_t}\cdot\frac{\displaystyle \frac{d}{dt}\tilde{g}(t)\sumnsv\alpha_n\xn}{\norm{\vect{z}(t)}\left(\tilde{g}(t)\right)^{\frac{2}{L}\left(L-1\right)}},
		\end{equation}
		where in the last step we used $\frac{\bar{\eqw}_\infty}{\gamma}=\what=\sumnsv \alpha_n\xn$.\\
	    \textbf{Step 2.c:} \\
	    Combining the result from the previous steps (using the fact that both limits exist and finite, $\xn$ are linearly independent and Lemma \ref{eq: lemma about vector limit}), we get, $\forall n\in\set$:
	    \[
	        \lim_{t\to\infty} \frac{\eta_tL\e\left( -f\left(\tilde{g}\left(t\right)\right)\right)\left(\tilde{g}(t)\right)^{\frac{2}{L}\left(L-1\right)}}{ \gamma^{1-\nicefrac{2}{L}}\frac{d}{dt}\tilde{g}(t)}\cdot\frac{ \e\left(-f'\left(\tilde{g}(t)\right)\rhoVec^\top\xn \right)}{\alpha_n}= 1.
	    \]
	    We denote $\displaystyle \fsymSec_n(t) \triangleq \frac{1}{\alpha_n} \e\left(-f'\left(\tilde{g}(t)\right)\rhoVec^\top\xn \right)>0$. $\fsymSec_n(t)=\Theta(1)$ since $f'\left(\tilde{g}(t)\right)\rhoVec^\top\xn=O(1)$.
	    Substituting this into the limit we get $\forall n\in\set$
	    \[
	        \lim_{t\to\infty} \frac{\e\left( -f\left(\tilde{g}\left(t\right)\right)\right)\left(\tilde{g}(t)\right)^{\frac{2}{L}\left(L-1\right)}}{ \frac{d}{dt}\tilde{g}(t)}\cdot\frac{L\eta_t\fsymSec_n(t)}{\gamma^{1-\nicefrac{2}{L}}}= 1,
	    \]
	    which implies $\forall n_1,n_2\in\set: \frac{\fsymSec_{n_1}(t)}{\fsymSec_{n_2}(t)}\to1$. Therefore, we can write
	    \[
	         \forall n\in\set:\ 
	        \e\left(-f'\left(\tilde{g}(t)\right)\rhoVec^\top\xn \right) = \alpha_n\left(\fsymSec(t)+\fsymSec(t)\tilde{\delta}_n(t)\right),
	    \]
	    where $\fsymSec(t)\triangleq\fsymSec_1(t)$ and $\forall n\in\set:\ \tilde{\delta}_n(t)\to0$.
	    \begin{claim}
	        For a given depth $L$, $\fsymSec(t)$ asymptotic behaviour is independent of $f$, i.e., $\displaystyle \lim_{t\to\infty}\frac{\fsymSec^{(1)}(t)}{\fsymSec^{(2)}(t)}=1$ where $\fsymSec^{(1)}(t)$ and $\fsymSec^{(2)}(t)$ correspond to two different function $f_1,f_2$.
	    \end{claim}
	    \begin{proof}
	        Let $f_1(t),f_2(t)$ be different functions with $\fsymSec^{(1)}(t),\fsymSec^{(2)}(t)>0$ so that $ \forall n\in\set:$
	        \begin{align}
	            \e\left(-f_1'\left(\tilde{g}_1(t)\right)\bm{\rho}_1(t)^\top\xn \right) &= \alpha_n\left(\fsymSec^{(1)}(t)+\fsymSec^{(1)}(t)\tilde{\delta}_n^{(1)}(t)\right)\\ \e\left(-f_2'\left(\tilde{g}_2(t)\right)\bm{\rho}_2(t)^\top\xn \right) &=\alpha_n\left(\fsymSec^{(2)}(t)+\fsymSec^{(2)}(t)\tilde{\delta}_n^{(2)}(t)\right).
	        \end{align}
	        We need to show that $\displaystyle \lim_{t\to\infty}\frac{\fsymSec^{(1)}(t)}{\fsymSec^{(2)}(t)}=1$. Since we know $\fsymSec(t)>0$, $\exists t_2$ such that $\forall t>t_2:\ \fsymSec^{(2)}(t)+\fsymSec^{(2)}(t)\tilde{\delta}_n^{(2)}(t)>0$ and $\fsymSec^{(1)}(t)+\fsymSec^{(1)}(t)\tilde{\delta}_n^{(1)}(t)>0$.
	        We define $\displaystyle \forall n\in\set: \tilde{\fsymSec}_n(t)=\log\left( \frac{\fsymSec^{(2)}(t)+\fsymSec^{(2)}(t)\tilde{\delta}_n^{(2)}(t)}{\fsymSec^{(1)}(t)+\fsymSec^{(1)}(t)\tilde{\delta}_n^{(1)}(t)}\right)$ (this is well defined $\forall t>t_2$). From the last two equations we have:
	        \[
	            \forall n\in\set:\ \left(f_1'\left(\tilde{g}_1(t)\right)\bm{\rho}_1(t)-f_2'\left(\tilde{g}_2(t)\right)\bm{\rho}_2(t)\right)^\top\xn = \tilde{\fsymSec}_n(t).
	        \]
	        Additionally, since $\bm{\rho}_1(t)^\top\what=\bm{\rho}_2(t)^\top\what=0$ (from definition) we get $\forall t>t_2$:
	        \begin{align*}
	            &0 = \left(f_1'\left(\tilde{g}_1(t)\right)\bm{\rho}_1(t)-f_2'\left(\tilde{g}_2(t)\right)\bm{\rho}_2(t)\right)^\top\what = \sumnsv \alpha_n \left(f_1'\left(\tilde{g}_1(t)\right)\bm{\rho}_1(t)-f_2'\left(\tilde{g}_2(t)\right)\bm{\rho}_2(t)\right)^\top\xn \\
	            &= \sumnsv \alpha_n \tilde{\fsymSec}_n(t)=\sumnsv \alpha_n \log\left( \frac{\fsymSec^{(2)}(t)+\fsymSec^{(2)}(t)\tilde{\delta}_n^{(2)}(t)}{\fsymSec^{(1)}(t)+\fsymSec^{(1)}(t)\tilde{\delta}_n^{(1)}(t)}\right)=\sumnsv \alpha_n \log\left(\frac{\fsymSec^{(2)}(t)}{\fsymSec^{(2)}(t)}\cdot \frac{1+\tilde{\delta}_n^{(2)}(t)}{1+\tilde{\delta}_n^{(1)}(t)}\right)\\
	            &=\log\left(\frac{\fsymSec^{(2)}(t)}{\fsymSec^{(2)}(t)}\right) \sumnsv \alpha_n+\sumnsv \alpha_n \log\left( \frac{1+\tilde{\delta}_n^{(2)}(t)}{1+\tilde{\delta}_n^{(1)}(t)}\right).
	        \end{align*}
	        Additionally, $\lim\limits_{t\to\infty}\sumnsv \alpha_n \log\left( \frac{1+\tilde{\delta}_n^{(2)}(t)}{1+\tilde{\delta}_n^{(1)}(t)}\right)=0$ since $\forall n\in\set:\ \tilde{\delta}_n^{(1)}(t)\to0,\tilde{\delta}_n^{(2)}(t)\to0$. Combining both results we get $\log\left(\frac{\fsymSec^{(2)}(t)}{\fsymSec^{(2)}(t)}\right) \sumnsv \alpha_n\to0$. Since for almost every dataset $\forall n\in\set: \alpha_n>0$, this implies $\log\left(\frac{\fsymSec^{(2)}(t)}{\fsymSec^{(2)}(t)}\right)\to0
	        \Rightarrow \displaystyle \lim_{t\to\infty}\frac{\fsymSec^{(1)}(t)}{\fsymSec^{(2)}(t)}=1$.
	    \end{proof}
	    \noindent Summarizing, we have  that $\rhoVec$ and $\tilde{g}(t)$ are the asymptotic solutions of:
	    \begin{align}
	         \frac{\e\left( -f\left(\tilde{g}\left(t\right)\right)\right)\left(\tilde{g}(t)\right)^{\frac{2}{L}\left(L-1\right)}}{ \frac{d}{dt}\tilde{g}(t)}\cdot\frac{L\eta_t\fsymSec(t)}{\gamma^{1-\nicefrac{2}{L}}}=1, \text{ and}\\
	         \forall n\in\set:\ 
	        \e\left(-f'\left(\tilde{g}(t)\right)\rhoVec^\top\xn \right) = \alpha_n\left(\fsymSec(t)+\fsymSec(t)\tilde{\delta}_n(t)\right). 
	    \end{align}
	    where $\fsymSec(t)=\Theta(1)$ is independent of $f$ and $\tilde{\delta}_n(t)\to0$. Thus, for almost every dataset\footnote{Note that this excludes the degenerate case in which $\forall n\in\set: \alpha_n=1$. We can show this similarly to Lemma 12 in \cite{soudry2018journal}.}, $\exists \fsymThird(t)=\Theta(1)$, that is only dependent on the data set and $L$ (specifically, it is not dependent on the loss function) such that $\displaystyle \norm{\op\rhoVec}= {\fsymThird(t)}{\left(f'\left(\tilde{g}(t)\right)\right)^{-1}}+o\left(\left({f'\left(\tilde{g}(t)\right)}\right)^{-1}\right)$ where $\op\in\R^{d\times d}$ is the  projection matrix to the subspace spanned by the support vectors.
	    This completes our proof with $\phi_1(t)=\fsymSec(t)$ and $\phi_2(t)=\tau(t)$.
	\end{proof}
	
	\section{Proof that non-support vectors direction converge for $L=1$} \label{section: Proof that non-support vectors direction converge for L=1}
	\lemmaOrthogonalProj*
	\begin{proof}
	From Theorem \ref{theorem: general convergence rates} we have that
	\[
	    \rvec(t)=\eqw(t)-\what g(t)
	\]
	where $\rvec(t)=o(g(t))$.
	We aim to show that $\norm{\bar{\op}\rvec(t)}$ is bounded. We have
	\begin{equation} \label{eq: Q_bar r 'ODE'}
	    \norm{\bar{\op} \rvec(t+1)}^2=\norm{\bar{\op}\rvec(t+1)-\bar{\op}\rvec(t)}^2+2\left(\bar{\op}\rvec(t+1)-\bar{\op}\rvec(t) \right)^\top\bar{\op}\rvec(t)+\norm{\bar{\op} \rvec(t)}^2.
	\end{equation}
	\begin{enumerate}
	    \item 
	    \begin{equation} \label{eq: Q_bar r(t) diff equation}
	        \norm{\bar{\op}\rvec(t+1)-\bar{\op}\rvec(t)}^2=\norm{-\eta \bar{\op} \nabla \mathcal{L}\left(\wvec(t) \right)}^2\le \eta^2 \norm{\nabla \mathcal{L}\left(\wvec(t) \right)}^2
	    \end{equation}
	    Additionally, From Lemma \ref{Lemma: L converge} we have
	    \begin{equation} \label{eq: norm (L) is integrable}
	        \sum_{u=1}^\infty \norm{\nabla \mathcal{L}\left(\wvec(t) \right)}^2<\infty.
	    \end{equation}
	    \item 
	    \begin{align} \label{eq: Q_bar(r(t+1)-r(t))r(t) bound}
	        &\left(\bar{\op}\rvec(t+1)-\bar{\op}\rvec(t) \right)^\top\bar{\op}\rvec(t)\nonumber\\
	        &\overset{(1)}{=}\eta \sumnnsv \e\left(-f\left(\eqw^\top\xn\right)\right)\xn^\top\bar{\op}\rvec(t)\nonumber\\
	        &= \eta \sum\limits_{n\notin\set \atop \xnT\bar{\op}\rvec(t)>0} \e\left(-f\left(g(t)\what^\top\xn+\rvec(t)^\top\xn\right)\right)\xn^\top\bar{\op}\rvec(t)\nonumber\\
	        & \overset{(2)}{\le} \eta \sum\limits_{n\notin\set \atop \xnT\bar{\op}\rvec(t)>0} \e\left(-f\left(g(t)\theta+\rvec(t)^\top\xn\right)\right)\xn^\top\bar{\op}\rvec(t)\nonumber\\
	        &\overset{(3)}{=} \eta \sum\limits_{n\notin\set \atop \xnT\bar{\op}\rvec(t)>0} \e\left(-f\left(g(t)\theta\right)-f'\left(\theta g(t)\right)\rvec(t)^\top\xn\left(1+\delta_n(t)\right)\right)\xn^\top\bar{\op}\rvec(t)\nonumber\\
	        &= \eta \sum\limits_{n\notin\set \atop \xnT\bar{\op}\rvec(t)>0} \e\left(-f\left(g(t)\theta\right)-\left(1+\delta_n(t)\right)f'\left(\theta g(t)\right)\op\rvec(t)^\top\xn-\left(1+\delta_n(t)\right)f'\left(\theta g(t)\right)\bar{\op}\rvec(t)^\top\xn\right)\xn^\top\bar{\op}\rvec(t)\nonumber\\
	        &\overset{(4)}{\le} \eta \sum\limits_{n\notin\set \atop \xnT\bar{\op}\rvec(t)>0} \e\left(-f\left(g(t)\theta\right)+\left(1+\delta_n(t)\right)f'\left(\theta g(t)\right)\norm{\op\rvec(t)}\norm{\xn}\right)\norm{\xn^\top\bar{\op}}\norm{\rvec(t)}\nonumber\\
	        &\overset{(5)}{\le} \eta \sum\limits_{n\notin\set \atop \xnT\bar{\op}\rvec(t)>0} \e\left(-f\left(g(t)\theta\right)+\left(1+\delta_n(t)\right)f'\left(\theta g(t)\right)\norm{\op\rvec(t)}\norm{\xn}\right)g(t)\nonumber\\
	        & \overset{(6)}{\le} \eta \sum\limits_{n\notin\set \atop \xnT\bar{\op}\rvec(t)>0} C\e\left(-f\left(\theta g(t) \right) \right) g(t) = \eta \sum\limits_{n\notin\set \atop \xnT\bar{\op}\rvec(t)>0} C\e\left(-f\left(\theta g(t) \right)+\log\left(g(t)\right) \right)\nonumber\\
	        & \le \eta N C\e\left(-f\left(\theta g(t) \right)+\log\left(g(t)\right) \right), \forall t>t_2,
	    \end{align}
	
	where in 1 we used eq.~\eqref{GD} (gradient descent dynamic) and the fact that $\forall n\in\set: \bar{\op}\xn=0$, in 2 we used the fact that $f(t)$ is monotonically decreasing and $\theta=\min_{n\notin\set}\what^\top\xn>1$. In 3 we define $\delta_n\to0$ and used Claim \ref{claim: f(g+h)=f(g)+f'(g)h} and the fact that $\abs{\rvec(t)^\top\xn}=o(g(t))$. In 4 we used Cauchy-Schwarz and the fact that $\exists t_1$ so that $\forall t>t_1$: $\left(1+\delta_n(t)\right)f'\left(\theta g(t)\right)\bar{\op}\rvec(t)^\top\xn\ge0$. In 5 we used $\norm{\rvec(t)}=o\left(g\left(t\right)\right)$. In 6, we used the fact that $\left(1+\delta_n(t)\right)f'\left(\theta g(t)\right)\norm{\op\rvec(t)}\norm{\xn}=\Theta(1)$ since $\norm{\op\rvec(t)}=\Theta\left(\frac{1}{f'\left(g\left(t\right)\right)}\right)$ and $\frac{f'(\theta g(t)}{f'(g(t))}=\Theta(1)$ ($f'(\theta g(t))=f'(g(t))+(\theta-1)f''(g(t))g(t)+o\left(f''(g(t))g(t)\right)=O\left(f'(g(t))\right)$ from assumption \ref{assumption: f can be expended}). This implies that
	$\exists C_1,t_2$ so that $\forall t>t_2>t_1$: $\left(1+\delta_n(t)\right)f'\left(\theta g(t)\right)\norm{\op\rvec(t)}\norm{\xn}\le C_1$ and we define $C = \exp(C_1)$.
	\item
	\begin{align*}
	    \norm{\bar{\op} \rvec(t)}^2 - \norm{\bar{\op} \rvec(t_1)}^2 &= \sum_{u=t_1}^{t-1} \left[ \norm{\bar{\op} \rvec(u+1)}^2 - \norm{\bar{\op} \rvec(u)}^2
	 \right]\\
	 &\overset{(1)}{=} \sum_{u=t_1}^{t-1} \left[ \norm{\bar{\op}\rvec(u+1)-\bar{\op}\rvec(u)}^2+2\left(\bar{\op}\rvec(u+1)-\bar{\op}\rvec(u) \right)^\top\bar{\op}\rvec(t)
	 \right]\\
	 &\le \eta^2 \sum_{u=t_1}^{t-1} \norm{\nabla \mathcal{L}\left(\wvec(u) \right)}^2
	 + 2 \eta N C\sum_{u=t_1}^{t-1} \e\left(-f\left(\theta g(u) \right)+\log\left(g(u)\right) \right)
	 \overset{(3)}{<}\infty
	\end{align*}
	 where in (1) we used eq.~\eqref{eq: Q_bar r 'ODE'}, in (2) we used eqs. \ref{eq: Q_bar r(t) diff equation} and \ref{eq: Q_bar(r(t+1)-r(t))r(t) bound} and in (3) we used 
	 the last transition we used eq.~\eqref{eq: norm (L) is integrable} and Claim \ref{claim: L=1 Integrability claim}.
	 \end{enumerate}
	 \end{proof}
	\subsection{Integrability proof}
    \begin{claim} \label{claim: L=1 Integrability claim}
        If $\exists \epsilon>0$ so that $f'\left(u\right)=\Omega\left(\frac{\log^{1+\epsilon}\left(u\right)}{u}\right)$ and $g(t)$ satisfies the following equation
        \begin{equation}
            \frac{dg\left(t\right)}{dt}=\exp\left(-f\left(g\left(t\right)\right)\right)\,.\label{eq: dg/dt}
        \end{equation}
        then 
        \[
            \int_{0}^{\infty}\exp\left(-f\left(g\left(t\right)c\right)+\log(g(t)\right)dt<\infty,
        \]
        where $c>1$.
    \end{claim}
    \begin{proof}
    We have
    \begin{align*}
    &\int_{0}^{\infty} \left[\exp\left(-f\left(g\left(t\right)c\right)+\log(g(t)\right)\right]dt \\ 
     & \overset{\left(1\right)}{=}\int_{g(0)}^{\infty}\exp\left(-\left(f\left(gc\right)-f\left(g\right)\right)+\log\left(g\right)\right)dg\\
     & \overset{\left(2\right)}{\leq}
      C
     +\int_{t_{0}}^{\infty}\exp\left(-\frac{C_{1}}{1+\epsilon_{1}}\left(\log^{2+\epsilon}\left(gc\right)-\log^{2+\epsilon}\left(g\right)\right)+\log(g)\right)dg\\
     & \overset{\left(3\right)}{\leq}C+\int_{t_{0}}^{\infty}\exp\left(-\frac{C_{1}}{1+\epsilon_{1}}\left(2+\epsilon\right)\log\left(c\right)\log^{1+\epsilon}\left(g\right)+\log(g)\right)dg\\
     & =C+\int_{t_{0}}^{\infty}\exp\left(-\log(g)\left(\frac{C_{1}}{1+\epsilon_{1}}\left(2+\epsilon\right)\log^{\epsilon}(t)-1\right)\right)dg\\
     & \overset{\left(4\right)}{\leq}C+\int_{t_{0}}^{\infty}\exp\left(-C'\log(g)\right)dg<\infty
    \end{align*}
    where in $\left(1\right)$ we used variable change, $g(t)\to\infty$, eq.~\eqref{eq: dg/dt} and $\frac{d}{dx}f^{-1}\left(x\right)=\left[f^{\prime}\left(f^{-1}\left(x\right)\right)\right]^{-1}$,
    in $\left(2\right)$ we used $g(t)\to\infty$ and $f'\left(u\right)=\Omega\left(\frac{\log^{1+\epsilon}\left(u\right)}{u}\right)$
    for some $\epsilon>0$ and therefore $\exists t_1$ so that $\forall t>t_1$:
    \begin{align*}
     & f\left(gc\right)-f\left(g\right)\\
     & =\int_{g(t)}^{cg(t)}f'(u)du\\
     & \ge\int_{g(t)}^{cg(t)}C_{1}\frac{\log^{1+\epsilon}(u)}{u}du\\
     & =\frac{C_{1}}{1+\epsilon}\left(\log^{2+\epsilon}(cg(t))-\log^{2+\epsilon}(g(t))\right).
    \end{align*}
    Additionally, we defined $C=\int_{0}^{t_{0}}\exp\left(-\left(f\left(gc\right)-f\left(g\right)\right)\right)dg$.
    In $\left(3\right)$ we used the fact that $\forall a\geq1$:
    \[
    \left(\log\left(gc\right)\right)^{a}=\left(\log\left(c\right)+\log\left(g\right)\right)^{a}=\log^{a}\left(g\right)\left(1+\frac{\log\left(c\right)}{\log\left(g\right)}\right)^{a}\ge\log^{a}\left(g\right)+a\log\left(c\right)\log^{a-1}\left(g\right)
    \]
    from Bernoulli's inequality (since $\frac{\log\left(c\right)}{\log\left(g\right)}\ge-1$
    for sufficently large t). In (4) we used the fact that $\exists C'>1$
    since for sufficently large t $\frac{C_{1}}{1+\epsilon_{1}}\left(2+\epsilon\right)\log^{\epsilon}(t)-1>1$.
    \end{proof}
	
\remove{    \section{Calculation of convergence rates\label{sec: CALCULATION OF CONVERGENCE RATES}} 
    
        \subsection{Generic tails} \label{sec: generic tail rates calculation}
        From Theorem \ref{theorem: general convergence rates}, we can write $\wvec(t) = \what g(t) + \rhoVec$ where $\rhoVec = o(g(t))$.	
	We can use this to calculate the normalized weight vector:
	\begin{flalign} \label{Calculation of the normalized weight vector}
	&\frac{\wvec(t)}{\Vert \wvec(t)\Vert} =\frac{g(t)\what +\rhoVec}{\sqrt{g(t)^2\what^\top\what + \rhoVec^\top\rhoVec+2g(t)\what^\top\rhoVec}} =\frac{\what + g^{-1}(t)\rhoVec}{\whatNorm\sqrt{1 +2\frac{\what^\top\rhoVec}{g(t)\whatNorm^2}+ \frac{\Vert\rhoVec\Vert^2}{g^2(t)\whatNorm^2}}}\nonumber\\
	& \overset{(1)}= \frac{\what +g^{-1}(t)\rhoVec}{\whatNorm}\left[1-\frac{\what^\top\rhoVec}{g(t)\whatNorm^2}+\left[\frac{3}{4}\left(2\frac{\what^\top\rhoVec}{\whatNorm^2}\right)^2-\frac{\Vert\rhoVec\Vert^2}{2\whatNorm^2}\right]\frac{1}{g^2(t)}+O\left(\left(\frac{\what^\top\rhoVec}{g(t)}\right)^3\right)\right]\nonumber\\
	& = \frac{\what}{\whatNorm}+\left(\frac{\rhoVec}{\whatNorm}-\frac{\what}{\whatNorm}\frac{\what^\top\rhoVec}{\whatNorm^2}\right)\frac{1}{g(t)}
	+O\left(\left(\frac{\what^\top\rhoVec}{g(t)}\right)^2\right)
	\frac{\hat{w}}{\norm{\hat{w}}}\nonumber\\
	& = \frac{\what}{\whatNorm}+\left(I-\frac{\what\what^\top}{\whatNorm^2}\right)\frac{1}{\whatNorm}\frac{\rhoVec}{g(t)}+O\left(\left(\frac{\what^\top\rhoVec}{g(t)}\right)^2\right)
	\end{flalign}
	where in $(1)$ we used $\frac{1}{\sqrt{1+x}}=1-\frac{1}{2}x+\frac{3}{4}x^2+O(x^3)$.\\
	We use eq.~\eqref{Calculation of the normalized weight vector} to calculate the angle:
	\begin{align} \label{eq: Calculation of the angle}
	&\frac{\wvec(t)^\top\what}{\wNorm\whatNorm}&\nonumber\\
	&=\frac{\what^\top}{\whatNorm^2}\left(\what +g^{-1}(t)\rhoVec\right)\left[1-\frac{\what^\top\rhoVec}{g(t)\whatNorm^2}+\left[\frac{3}{4}\left(2\frac{\what^\top\rhoVec}{\whatNorm^2}\right)^2-\frac{\Vert\rhoVec\Vert^2}{2\whatNorm^2}\right]\frac{1}{g^2(t)}+O\left(\left(\frac{\what^\top\rhoVec}{g(t)}\right)^3\right)\right]\nonumber\\
	& = 1 + \frac{2}{\whatNorm^2}\left[\left(\frac{\rhoVec^\top\what}{\whatNorm \Vert\rhoVec\Vert}\right)^2 - \frac{1}{4}\right]\frac{\Vert\rhoVec\Vert^2}{g^2(t)}+O\left(\left(\frac{\what^\top\rhoVec}{g(t)}\right)^3\right)
	\end{align}
	Calculation of the margin:
	\begin{flalign} \label{eq: Calculation of the margin}
		&\min_{n} \frac{\xnT \wvec(t)}{\Vert \wvec(t)\Vert}\overset{(1)}{=}\min_{n\in\set} \frac{\xnT \wvec(t)}{\Vert \wvec(t)\Vert}\nonumber\\
		& =\min_{n\in\set} \xnT\left[\frac{\what}{\whatNorm}+\left(\frac{\rhoVec}{\whatNorm}-\frac{\what}{\whatNorm}\frac{\what^\top\rhoVec}{\whatNorm^2}\right)\frac{1}{g(t)}+O\left(\left(\frac{\what^\top\rhoVec}{g(t)}\right)^2\right)\right]\nonumber\\
		& =\frac{1}{\whatNorm}+\frac{1}{\whatNorm}\left(\min_{n\in\set}\xnT\rhoVec-\frac{\what^\top\rhoVec}{\whatNorm^2}\right)\frac{1}{g(t)}
		+O\left(\left(\frac{\what^\top\rhoVec}{g(t)}\right)^2\right),
	\end{flalign}
	where in (1) we used the fact that $\frac{\wvec(t)}{\Vert \wvec(t)\Vert}$ converge to the maximum-margin separator and thus the minimal value is obtained on the support vectors.\\
    From Theorem \ref{theorem: general convergence rates}, we can also characterize $\rhoVec$: 
	\begin{equation} \label{eq: rho def generic case appendix}
	        \norm{\op_1\rhoVec}=\begin{cases} \fsym_2(t)\left(f'\left(\tilde{g}(t)\right)\right)^{-1}+o\left(\left(f'\left(\tilde{g}(t)\right)\right)^{-1}\right), & \text{if } \left(f'\left(\tilde{g}(t)\right)\right)^{-1}=\Omega(1)\\
	        O\left(1\right), & \text{otherwise}
	        \end{cases}
	\end{equation}	
	where $\fsym_2(t)=\Theta(1)$.
      Substituting eq.~\eqref{eq: rho def generic case appendix} into eqs. \ref{Calculation of the normalized weight vector}, \ref{eq: Calculation of the angle}, \ref{eq: Calculation of the margin} we get, for some constant vector $\vect{a}$ independent of $f$:
      \begin{equation} \label{eq: generic margin results}
      \frac{1}{\whatNorm} - \min_{n} \frac{\xnT \wvec(t)}{\Vert \wvec(t)\Vert}  = \begin{cases}
      \dfrac{1}{\whatNorm}\left(\dfrac{\what^\top\genericConstVec}{\whatNorm^2} - \min_{n\in\set}\xnT\genericConstVec\right)\dfrac{\fsym_2(t)}{g(t)f'(g(t))}, & \text{if ${f'(g(t))}=\Omega(1)$}\\
      O(g^{-1}(t))\vphantom{\dfrac{0}{0}}, & \text{otherwise}
      \end{cases}
      \end{equation}

    \remove{
	From Conjecture \ref{conjecture: generic tail}, we can write $\wvec(t) = \what g(t) + \rhoVec$ where $\rhoVec = o(g(t))$.	
	We can use this to calculate the normalized weight vector:
	\begin{flalign} \label{Calculation of the normalized weight vector}
	&\frac{\wvec(t)}{\Vert \wvec(t)\Vert}&\nonumber\\
	& =\frac{g(t)\what +\rhoVec}{\sqrt{g(t)^2\what^\top\what + \rhoVec^\top\rhoVec+2g(t)\what^\top\rhoVec}}\nonumber\\
	& =\frac{\what + g^{-1}(t)\rhoVec}{\whatNorm\sqrt{1 +2\frac{\what^\top\rhoVec}{g(t)\whatNorm^2}+ \frac{\Vert\rhoVec\Vert^2}{g^2(t)\whatNorm^2}}}\nonumber\\
	& = \frac{1}{\whatNorm}\left(\what +g^{-1}(t)\rhoVec\right)\left[1-\frac{\what^\top\rhoVec}{g(t)\whatNorm^2}+\left[\frac{3}{4}\left(2\frac{\what^\top\rhoVec}{\whatNorm^2}\right)^2-\frac{\Vert\rhoVec\Vert^2}{2\whatNorm^2}\right]\frac{1}{g^2(t)}+O\left(\left(\frac{\what^\top\rhoVec}{g(t)}\right)^3\right)\right]\nonumber\\
	& = \frac{\what}{\whatNorm}+\left(\frac{\rhoVec}{\whatNorm}-\frac{\what}{\whatNorm}\frac{\what^\top\rhoVec}{\whatNorm^2}\right)\frac{1}{g(t)}+O\left(\left(\frac{\what^\top\rhoVec}{g(t)}\right)^2\right)\nonumber\\
	& = \frac{\what}{\whatNorm}+\left(I-\frac{\what\what^\top}{\whatNorm^2}\right)\frac{1}{\whatNorm}\frac{\rhoVec}{g(t)}+O\left(\left(\frac{\what^\top\rhoVec}{g(t)}\right)^2\right)
	\end{flalign}
	we used $\frac{1}{\sqrt{1+x}}=1-\frac{1}{2}x+\frac{3}{4}x^2+O(x^3)$.\\
	We use eq.~\eqref{Calculation of the normalized weight vector} to calculate the angle:
	\begin{align} \label{eq: Calculation of the angle}
	&\frac{\wvec(t)^\top\what}{\wNorm\whatNorm}&\nonumber\\
	&=\frac{\what^\top}{\whatNorm^2}\left(\what +g^{-1}(t)\rhoVec\right)\left[1-\frac{\what^\top\rhoVec}{g(t)\whatNorm^2}+\left[\frac{3}{4}\left(2\frac{\what^\top\rhoVec}{\whatNorm^2}\right)^2-\frac{\Vert\rhoVec\Vert^2}{2\whatNorm^2}\right]\frac{1}{g^2(t)}+O\left(\left(\frac{\what^\top\rhoVec}{g(t)}\right)^3\right)\right]\nonumber\\
	& = 1 + \frac{2}{\whatNorm^2}\left[\left(\frac{\rhoVec^\top\what}{\whatNorm \Vert\rhoVec\Vert}\right)^2 - \frac{1}{4}\right]\frac{\Vert\rhoVec\Vert^2}{g^2(t)}+O\left(\left(\frac{\what^\top\rhoVec}{g(t)}\right)^3\right)
	\end{align}
	Calculation of the margin:
	\begin{flalign} \label{eq: Calculation of the margin}
		&\min_{n} \frac{\xnT \wvec(t)}{\Vert \wvec(t)\Vert}\nonumber\\
		& =\min_{n} \xnT\left[\frac{\what}{\whatNorm}+\left(\frac{\rhoVec}{\whatNorm}-\frac{\what}{\whatNorm}\frac{\what^\top\rhoVec}{\whatNorm^2}\right)\frac{1}{g(t)}+O\left(\left(\frac{\what^\top\rhoVec}{g(t)}\right)^2\right)\right]\nonumber\\
		& =\frac{1}{\whatNorm}+\frac{1}{\whatNorm}\left(\min_{n}\xnT\rhoVec-\frac{\what^\top\rhoVec}{\whatNorm^2}\right)\frac{1}{g(t)}
		+O\left(\left(\frac{\what^\top\rhoVec}{g(t)}\right)^2\right)
	\end{flalign}
    From Conjecture \ref{conjecture: generic tail}, we can also characterize $\rhoVec$: 
	\begin{equation} \label{eq: rho def generic case appendix}
			\bm{\rho}(t) = \begin{cases}
\left(f'(g(t))\right)^{-1}\genericConstVec+o\left(\left(f'(g(t))\right)^{-1}\right), & \text{if $\frac{1}{f'(g(t))}=\Omega(1)$}\\
			O(1), & \text{otherwise}
			\end{cases}
		\end{equation}	
      Substituting eq.~\eqref{eq: rho def generic case appendix} into eqs. \ref{Calculation of the normalized weight vector}, \ref{eq: Calculation of the angle}, \ref{eq: Calculation of the margin} we get:
	\begin{flalign} \label{eq: generic norm results}
		&\left \Vert \frac{\wvec(t)}{\Vert \wvec(t)\Vert} - \frac{\what}{\whatNorm} \right \Vert =\begin{cases}
		\left\Vert\left(I-\dfrac{\what\what^\top}{\whatNorm^2}\right)\genericConstVec\right\Vert\dfrac{1}{\whatNorm}\dfrac{1}{g(t)f'(g(t))}, & \text{if $\dfrac{1}{f'(g(t))}=\Omega(1)$}\\
			O(g^{-1}(t))\vphantom{\dfrac{0}{0}}, & \text{otherwise}
		\end{cases}&
	\end{flalign}
    \begin{equation} \label{eq: generic angle results}
	 1 - \frac{\wvec(t)^\top\what}{\wNorm\whatNorm}  = \begin{cases}
		\left(\dfrac{1}{4}-\left(\dfrac{\genericConstVec^\top\what}{\whatNorm \Vert\genericConstVec\Vert}\right)^2\right)\dfrac{2\Vert\genericConstVec\Vert^2}{\whatNorm^2}\dfrac{1}{\left(g(t)f'(g(t))\right)^2}, & \text{if $\dfrac{1}{f'(g(t))}=\Omega(1)$}\\
			O(g^{-2}(t))\vphantom{\dfrac{0}{0}}, & \text{otherwise}
		\end{cases}
	\end{equation}
	\begin{equation} \label{eq: generic margin results}
	\frac{1}{\whatNorm} - \min_{n} \frac{\xnT \wvec(t)}{\Vert \wvec(t)\Vert}  = \begin{cases}
		\dfrac{1}{\whatNorm}\left(\dfrac{\what^\top\genericConstVec}{\whatNorm^2} - \min_{n}\xnT\genericConstVec\right)\dfrac{1}{g(t)f'(g(t))}, & \text{if $\dfrac{1}{f'(g(t))}=\Omega(1)$}\\
			O(g^{-1}(t))\vphantom{\dfrac{0}{0}}, & \text{otherwise}
		\end{cases}
	\end{equation}
   }
    
	    \subsection{Corollary \ref{corollary: L=1 optimal rate is obtained for exp loss}} \label{sec: proving corollary L=1 optimal rate is obtained for exp loss}
	    In this section we want to show that the optimal margin convergence rate is obtained for exponential loss.
	    
	    Using theorem \ref{theorem: general convergence rates} with general tail and $L=1$ we have that $\tilde{g}(t)$ is the asymptotic solution of
	\[
	    \e\left(-f\left(\tilde{g}(t)\right)\right)=\frac{1}{\eta_t\fsym_1(t)}\frac{d}{dt}\tilde{g}(t)
	\]
	and
	\[
	        \norm{\op_1\rhoVec}=\begin{cases} \fsym_2(t)\left(f'\left(\tilde{g}(t)\right)\right)^{-1}+o\left(\left(f'\left(\tilde{g}(t)\right)\right)^{-1}\right), & \text{if } \left(f'\left(\tilde{g}(t)\right)\right)^{-1}=\Omega(1)\\
	        O\left(1\right), & \text{otherwise}
	        \end{cases}
    \]
    for positive functions $\fsym_1(t)=\Theta(1),\ \fsym_2(t)=\Theta(1)$ independent of $f$. This implies the convergence rates in table \ref{table: L=1 generic convergence rates}.
	    We denote $\tilde{\fsym}_1(t)=\eta_t\fsym_1=\Theta(1)$. We define $u(t)=\int\limits_{0}^{t} \tilde{\fsym}_1(x)dx =H(t)\Rightarrow t=H^{-1}(u)$ (this is well defined since $H(t)$ is monotonic increasing) and $\hat{g}(u)=\tilde{g}\left(H^{-1}(u)\right)=\tilde{g}(t)$. Using these definition we have
        \[
            \frac{d}{du} \hat{g}(u)= \frac{1}{\tilde{\fsym}_1(t)}\frac{d}{dt}\tilde{g}(t)=\e\left(-f\left(\hat{g}(u)\right)\right).
        \]
        The asymptotic solution for this equation is $\hat{g}(u)=f^{-1}\left(\log(u)\right)$, i.e. \mbox{$\displaystyle \lim_{u\to\infty} \frac{\hat{g}(u)}{f^{-1}\left(\log(u)\right)}=1$} (this result is proved in \cite{odeSolMathOverFlow}).
        Since $\tilde{\fsym}_1(t)=\Theta(1)$ we know that exists positive constants $C_L,C_U,t_1$ so that $\forall t>t_1:$ $C_L\le\tilde{\fsym}_1(t)\le C_U\Rightarrow C_L t\le H(t)\le C_U t$. Using the previous result we have that asymptotically $f^{-1}\left( \log\left(C_L t\right)\right)\le\tilde{g}(t)=\hat{g}\left(H(t)\right)\le f^{-1}\left( \log\left(C_U t\right)\right)$ (since $f^{-1}(t)$ is an increasing function).
	    
	    For functions with tight exponential tail ($f(t)=\Theta(t)$) the margin convergence rate is $O(\frac{1}{\log(t)})$ and we know that this bound is tight (this result was proved in \cite{soudry2017implicit}).\\ 
	    For $f(t)=\omega(t)$ (the tail goes to zero faster than exponential tail) the rates are proportional to $1/\tilde{g}(t)$. This implies slower convergence rates than the rates obtained with exponential tail ($1/\log(t)$) since in this case $f^{-1}\left(\log(t)\right)=o\left(\log(t)\right)$.
	    
	    For $f(t)=o(t)$ we first prove the following claim.
	    \begin{restatable}{claim}{claimidealrate}\label{claim: relation for f(t)=o(t)}
	    For a concave function $f$ that satisfies $f'(t)>0,\ f(t)=o(t)$ $\exists x'$ so that $\forall x>x'$:
	    \begin{equation}
	            \frac{1}{f^{-1}\left(x\right)f'\left(f^{-1}\left(x\right)\right)}\ge \frac{1}{x}.
        \end{equation}
	    \end{restatable}
	\begin{proof}
	        We denote $h(x)=f^{-1}(x)$. $h(x)$ is convex since $f$ is strictly increasing and concave. 
	        Substituting $h(x)$ and $h'(x)=\frac{1}{f'\left(f^{-1}\left(x\right)\right)}$ into the equation, we need to show that $\exists x_1$ so that $\forall x>x_1$
	        \begin{equation}
	            \frac{h'(x)}{h\left(x\right)}\ge \frac{1}{x}.
	        \end{equation}
	        From the gradient inequality, $\forall x>x'>0$:
	        \begin{align*}
	            h'(x)\left(x-x'\right)&\ge h(x)-h(x')\\
	            h'(x)&\ge \frac{h(x)-h(x')}{x-x'}.
	        \end{align*}
	        Additionally, since $g(t)=\omega(t)$ (from definition and $f(t)=o(t)$) $\exists x''$ so that $\forall x>x''$:
	        \[
	            h(x)\ge \frac{h(x')}{x'} x \Leftrightarrow  -xh(x')\ge -x'h(x) \Leftrightarrow \frac{h(x)-h(x')}{x-x'}\ge \frac{h(x)}{x} .
 	        \]
 	        Thus, for $x>\max(x',x'')$
 	        \[
 	            h'(x)\ge \frac{h(x)}{x} \Leftrightarrow \frac{h'(x)}{h(x)}\ge \frac{1}{x}.
 	        \]
	    \end{proof}
	    
	    For $f(t)=o(t)$ we have
	    \[
	        \gamma - \min_n \frac{\xnT\eqw(t)}{\norm{\eqw(t)}}=\frac{C_1\fsymSec_2(t)}{g(t)f'(g(t))}+o\left(\frac{1}{g(t)f'(g(t))}\right),
	    \]
	    where $C_1$ is a known constant and $\fsymSec_2(t)=\Theta(1)$ is a positive function independent of f. In order to show that the best rate is obtaind for exponential loss we need to show that, asymptotically,
	    \begin{equation}
	        \frac{1}{g(t)f'\left(g\left(t\right)\right)}\ge \frac{1}{\log(t)}.
	    \end{equation}
	    From assumption \ref{assumption: f can be expended} we have that $\abs{\frac{f''(x)}{f'(x)}}=o\left(x^{-1}\right)$. This implies that for sufficiently large $x$, $h(x)=xf'(x)$ is an increasing function ($h'(x)=f'(x)+xf''(x)>0$). Additionally,
	    assuming $f'(t)=o(1)$ we have
	    \begin{align*}
	        \dot{g}(t)=\e\left(-f\left(g\left(t\right)\right)\right)&\leq \frac{\e\left(-f\left(g\left(t\right)\right)\right)}{f'\left(g\left(t\right)\right)}\Leftrightarrow\\
	         \e\left(f\left(g\left(t\right)\right)\right)f'\left(g\left(t\right)\right)\dot{g}(t)&\le 1\Leftrightarrow\\
	        \frac{d}{dt}\e\left(f\left(g\left(t\right)\right)\right)&\le 1 \Rightarrow\\
	        \e\left(f\left(g\left(t\right)\right)\right) &\le t+C \Leftrightarrow\\
	        g(t)&\le f^{-1}\left(\log(t+C)\right)
	    \end{align*}
	    where in the last transition we used the fact that $f$ is an increasing function (and so does $f^{-1}$). \mnote{This set of equations is true for $\bar{g}(u)$ and not $g(t)$. We need to use time rescaling (as we did above) to show that this doesn't matter for our result}
	    Using the last two results we have
	    \[
	        \frac{1}{g(t)f'\left(g\left(t\right)\right)}\ge \frac{1}{f^{-1}\left(\log(t)\right)f'\left(f^{-1}\left(\log(t)\right)\right)}
	    \]
	    and thus it is sufficient to show that $\exists t'$ so that $\forall t>t'$:
	    \begin{equation}
	            \frac{1}{f^{-1}\left(\log(t)\right)f'\left(f^{-1}\left(\log(t)\right)\right)}\ge \frac{1}{\log(t)}.
        \end{equation}
        Using claim \ref{claim: relation for f(t)=o(t)} with $x=\log(t)$ we get the desired result.
        }
\section{Examples\label{sec: Examples}} 
\remove{
We recall Conjecture  \ref{conjecture: generic tail}, regarding general tails:
\conjectureGenericTail*
In this section we demonstrate this conjecture using two examples.
}
We recall Theorem  \ref{Theorem: convergence to max margin for general tail}, regarding general tails:
\theoremMaxMarginGeneralTail*
In this section we demonstrate this Theorem using two examples.
\subsection{Example: non-convergence to the max-margin separator}

\remove{Conjecture \ref{conjecture: generic tail} assumes that $f(t)=\omega(\log(t))$, and therefore, if $f(t)=O(\log(t))$, we may not converge to the max margin separator,}
Theorem  \ref{Theorem: convergence to max margin for general tail} assumes that $f'(t)=\omega\left(t^{-1}\right)$ which implies $f(t)=\omega(\log(t))$, and therefore, if $f(t)=O(\log(t))$, we may not converge to the max margin separator,
i.e. $\lim_{t \rightarrow \infty} \wvec (t) / \wNorm \neq \what / \whatNorm$. Next, we give an example for such a case.\\
Consider optimization with a power-law tailed loss $$\ell\left(u\right)=\begin{cases}
u^{-1} & ,u>1\\
2-u & ,u\leq1
\end{cases},$$ with two data points $\mathbf{x}_{1}=\left(1,0\right)$ and $\mathbf{x}_{2}=\left(0,2\right)$.
In this case $\hat{\mathbf{w}}=\left(1,0.5\right)$ and $\left\Vert \hat{\mathbf{w}}\right\Vert =\sqrt{5}/2$.
We take the limit $\eta\rightarrow0$, and obtain the continuous time
version of GD: 
\begin{align*}
\dot{w}_{1}\left(t\right) & =\frac{1}{w_{1}^2\left(t\right)}\,\,;\,\,\dot{w}_{2}\left(t\right)=\frac{0.5}{w_{2}^2\left(t\right)}.
\end{align*}
We can analytically integrate these equations to obtain 
\begin{align*}
\frac{1}{3}w_{1}^{3}\left(t\right) & =t+C\,\,;\,\,\frac{1}{3}w_{2}^{3}\left(t\right)=0.5t+C.
\end{align*}
so
\begin{align*}
w_{1}\left(t\right) & =\sqrt[3]{3t+w_{1}^3\left(0\right)}\,\,;\,\,w_{2}\left(t\right)=\sqrt[3]{1.5t+w_{2}^3\left(0\right)}.
\end{align*}
therefore, as $t\rightarrow\infty$
\begin{align*}
w_{1}\left(t\right) & /w_{2}\left(t\right)\rightarrow \sqrt[3]{2} \neq2=\hat{w}_{1}/\hat{w}_{2}.
\end{align*}
In this case, asymptotically we have that $f(t) = 2\log(t)$  which is not $\omega(\log(t))$, and therefore in this case Theorem \ref{Theorem: convergence to max margin for general tail} should not apply. Thus, we expect that $h(t)$ will not be $o(g(t))$, as we assumed, and this will break the analysis in appendix section \ref{sec: Generic Tails} (specifically,  $g(t)\asymp t^{\frac{1}{3}}$ and $h(t)=1/f'(g(t)) \asymp t^{\frac{1}{3}}$). Using this example, it is easy to verify that we do not converge to the max-margin separator whenever $-\ell^{\prime}(u)$ is polynomial.

In contrast, it is straightforward to verify, that similar analysis on the same example, only with poly-exponential tails, does yield convergence to the max-margin, as expected.
For example, with exponential loss we obtain $-\ell'\left(u\right)=e^{-u}$ 
\begin{align*}
\lim_{t\rightarrow \infty } w_{1}\left(t\right) & /w_{2}\left(t\right) = 2 =\hat{w}_{1}/\hat{w}_{2}
\end{align*}
In this case, $f(u)=u$, $g(t) \asymp \log(t)$, and $h(t)=f'(g(t))=1$, and so these results are consistent with Theorem \ref{theorem: ICLR theorem 3}.\\

\subsection{Example: sub-poly-exponential tails that converge to the max margin separator}

Theorem \ref{Theorem: convergence to max margin for general tail} implies that if $-\ell'\left(u\right)$ has a tail that decays faster than any polynomial tail, we will still converge to the max margin. To demonstrate this we analyze the same example as before, only with  $$-\ell'(u) = \begin{cases}
\e\left(-\log^\epsilon(u)-\log\left(\epsilon \log^{\epsilon-1}(u)\right)+\log(u)\right) & ,u>2\\
\e\left(-\log^\epsilon(2)-\log\left(0.5\epsilon \log^{\epsilon-1}(2)\right)\right) & ,u\leq 2
\end{cases} $$ for constant $\epsilon>1$. In this case $f(t)=\Theta(\log^\epsilon(t))=\omega(\log(t))$. We get:
	\begin{align*}
	\dot{w}_{1}\left(t\right) & =\e\left(-\log^\epsilon(w_1(t))-\log\left(\epsilon \log^{\epsilon-1}(w_1(t))\right)+\log(w_1(t))\right)\\
	\dot{w}_{2}\left(t\right) & =2\e\left(-\log^\epsilon(2w_2(t))-\log\left(\epsilon \log^{\epsilon-1}(2w_2(t))\right)+\log(2w_2(t))\right)
	\end{align*}
	We can analytically integrate these equations to obtain 
	\begin{align*}
	\e\left(\log^\epsilon w_1(t)\right) & =t+C\,\,;\,\,\e\left(\log^\epsilon(2w_2(t))\right)= 4t+C.
	\end{align*}
	so
	\begin{align*}
	w_{1}\left(t\right) & = \e\left(\log^{\epsilon^{-1}}(t+\tilde{C}_1)\right)\,\,;\,\,w_{2}\left(t\right)=\frac{1}{2}\e\left(\log^{\epsilon^{-1}}(4t+\tilde{C}_2)\right),
	\end{align*}
    where
    \begin{align*}
	\tilde{C}_1 & = \e\left(\log^{\epsilon}(w_1(0))\right)\,\,;\,\,\tilde{C}_2=\e\left(\log^{\epsilon}(2w_2(0))\right),
	\end{align*}
	therefore, as $t\rightarrow\infty$
	\begin{align*}
	w_{1}\left(t\right) & /w_{2}\left(t\right)\rightarrow 2=\hat{w}_{1}/\hat{w}_{2}
	\end{align*}
	However, we note that for $\epsilon < 1$ $w_{1}\left(t\right)  /w_{2}\left(t\right)\rightarrow 0$ and for $\epsilon=1$  $w_{1}\left(t\right)  /w_{2}\left(t\right)\rightarrow 0.5$, meaning that for $\epsilon \le 1$ we do not converge to the max margin separator, which is consistent with the conjecture, since then  $f(t) = O(\log(t))$.

\subsection{Example: Demonstrating that the upper bound in Theorem \ref{theorem: general convergence rates simplified}.1 is not always obtained} \label{sec: f=w(u) rate negative example}
    We analyze the same example as before, only with $$\ell'\left(u\right)=-\frac{1}{\nu}\exp\left( -u^\nu - (\nu-1)\log\left(u\right)\right),$$ for some $\nu>1$. In this case we get:
    \begin{align*}
	w_{1}\left(t\right) & = \log^{\frac{1}{\nu}}\left(t+C_1\right)\,\,;\,\,w_{2}\left(t\right)=\frac{1}{2}\log^{\frac{1}{\nu}}\left(2t+C_2\right),
	\end{align*}
	where $C_1=\exp\left(w_{1}(0)^\nu\right)$ and $C_2=\exp\left(\left(2w_{2}(0)\right)^\nu\right)$.
    Therefore, as $t\rightarrow\infty$
\begin{align*}
w_{1}\left(t\right) & /w_{2}\left(t\right)\rightarrow 2=\hat{w}_{1}/\hat{w}_{2}.
\end{align*}
Recall that the max-margin solution for this case is $\hat{\mathbf{w}}=\left(1,0.5\right)$ and $\left\Vert \hat{\mathbf{w}}\right\Vert =\sqrt{5}/2$.
We can write
\[
    \wvec(t) = \begin{bmatrix} w_1(t) \\ w_2(t)\end{bmatrix} = \left( \log^{\frac{1}{\nu}}\left(t+C_1\right) + \frac{1}{4}\log^{\frac{1}{\nu}}\left(2t+C_2\right)  \right)\cdot\frac{4}{5} \what
    + \left( \log^{\frac{1}{\nu}}\left(t+C_1\right) - \log^{\frac{1}{\nu}}\left(2t+C_2\right)  \right)\cdot\frac{2}{5} \what^\perp\,,
\]
where $\what^\perp = \begin{bmatrix} \frac{1}{2} \\ -1\end{bmatrix}$, $\ip{\what}{\what^\perp}=0$. Therefore, in this example $\wvec(t) = g(t)\what + \rhoVec$ where $g(t) = \left( \log^{\frac{1}{\nu}}\left(t+C_1\right) + \frac{1}{4}\log^{\frac{1}{\nu}}\left(2t+C_2\right)  \right)\cdot\frac{4}{5}$ and $\rhoVec = \left( \log^{\frac{1}{\nu}}\left(t+C_1\right) - \log^{\frac{1}{\nu}}\left(2t+C_2\right)  \right)\cdot\frac{2}{5} \what^\perp \approx C\frac{1}{\nu} \log^{\frac{1}{\nu}-1}(t)\to0$ for some constant $C>0$ independent of $\nu$. This implies that the margin convergence rate is proportional to $\frac{1}{\nu\log(t)}$, i.e., we obtain the same asymptotic rate as exponential loss, only with better constants.
\subsection{Example: Demonstrating that the upper bound in Theorem \ref{theorem: general convergence rates simplified}.1 is tight} \label{sec: f=w(u) rate positive example}
    Next, we give an example to show that the rate upper bound $O\left(\dfrac{1}{f^{-1}\left(\log(t)\right)}\right)$ for $f'(u)=\omega(1)$ is tight. Consider optimization with a loss that satisfies $$-\ell'\left(u\right)=\exp\left(-f(u)\right)$$ for some function $f'(u)=\omega(1)$ with one data point $\mathbf{x}_{1}=\left(1,0\right)$.
In this case $\hat{\mathbf{w}}=\left(1,0\right)$.
We take the limit $\eta\rightarrow0$, and obtain the continuous time
version of GD:
    \begin{align*}
\dot{w}_{1}\left(t\right) & =\exp\left(-f(w_{1}\left(t\right))\right)\,\,;\,\,\dot{w}_{2}\left(t\right)=0.
\end{align*}
    From \cite{odeSolMathOverFlow} we have that ${w}_{1}(t)=\Theta\left(f^{-1}(\log(t)\right)$ and from integrating the right equation we obtain $w_2(t) = w_2(0)$. Thus, using this example with $w_{2}\left(0\right)>0$, we see that the above upper bound is tight.
    
\section{Numerical results: additional details \label{sec:fig details}} 

\subsection{Implementation details of Figure \ref{Fig: NormGD} \label{sec:fig1 details}} 

The original dataset included four support vectors: $\mathbf{x}_{1}=\left(0.5,1.5\right),\mathbf{x}_{2}=\left(1.5,0.5\right)$
with $y_{1}=y_{2}=1$, and $\mathbf{x}_{3}=-\mathbf{x}_{1}$, $\mathbf{x}_{4}=-\mathbf{x}_{2}$
with $y_{3}=y_{4}=-1$. The $L_{2}$ normalized max margin vector
in this case was $\hat{\mathbf{w}}=\frac{1}{2}\left(1,1\right)$ with
margin equal to $\sqrt{2}$. Additional $6$ random data points were added from each class. These additional points are sufficiently far from the origin so they are not support vectors. Lastly, we re-scaled all datapoints so that $\max_n{||\x_n||}<1$, according to our assumption.

For training, we initialized $\w(0) \sim \mathcal{N}(0,\mathbf{I}_d)$, and used the optimal $\eta=1/\beta$ for GD, and the same as initial step size for normalized GD. 

Note that, in panel C, the training error $\mathcal{L}(\mathbf{w}(t))$ of normalized GD converges to zero (much faster than GD) --- until it disappears when reaching the lowest numerical precision level. Also, the margin gap figure for normalized GD appears less stable for Normalized GD. We suspect that this is because the index of the datapoint with the smallest margin rapidly switched due to the aggressive learning rate used.

\subsection{Neural Networks on a Toy Dataset} \label{sec-ngd_linear}

In what follows we compare GD to Normalized GD on linear and non-linear neural networks. For this purpose, we generate a 2-dimensional synthetic dataset composed of 600 data points, where positive and negative samples are generated from $\mathcal N(\boldsymbol \mu^+, \boldsymbol \Sigma^+)$ and $\mathcal N(\boldsymbol \mu^-, \boldsymbol \Sigma^-)$, respectively, with $\boldsymbol \mu^+ = (-5,2)$, $\boldsymbol \Sigma^+ = \bigl[ \begin{smallmatrix} 3 & -1\\ -1 & 3 \end{smallmatrix} \bigl]$, $\boldsymbol \mu^- = (5,-2)$ and $\boldsymbol \Sigma^- = \bigl[ \begin{smallmatrix} 2 & 3\\ 3 & 9 \end{smallmatrix} \bigl]$. Once the dataset was generated, the same points were used for all the following experiments.

We use a learning rate $\eta$ of $0.005$, which was empirically chosen so that optimization is stable but not slow. Larger learning rates would often result in both GD and Normalized GD presenting convergence issues, as in difficulty to reach (or stay at) a solution that separates the data. The weights were initialized from $\mathcal N(0, 0.1)$, and gradients were normalized together: $\sqrt{ \sum_{i=1}^d \lVert \nabla_{\bold W_i} \mathcal L \rVert_F^2 }$ was used to normalize each parameter's gradient, where $\bold W_i$ denotes the weight matrix of the $i$'th layer and $d$ the total number of layers of the network. Finally, each hidden layer contains $10$ hidden neurons.

\begin{figure}%
    \centering
    \subfloat[Dataset]{{\includegraphics[width=0.35\textwidth]{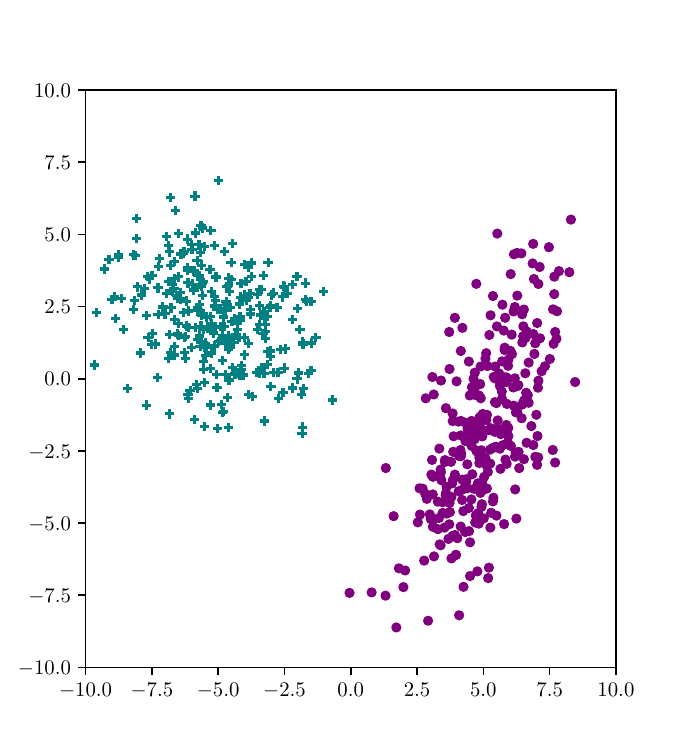} }}%
    \subfloat[Logistic Regression (1 layer Linear Network)]{{\includegraphics[width=0.65\textwidth]{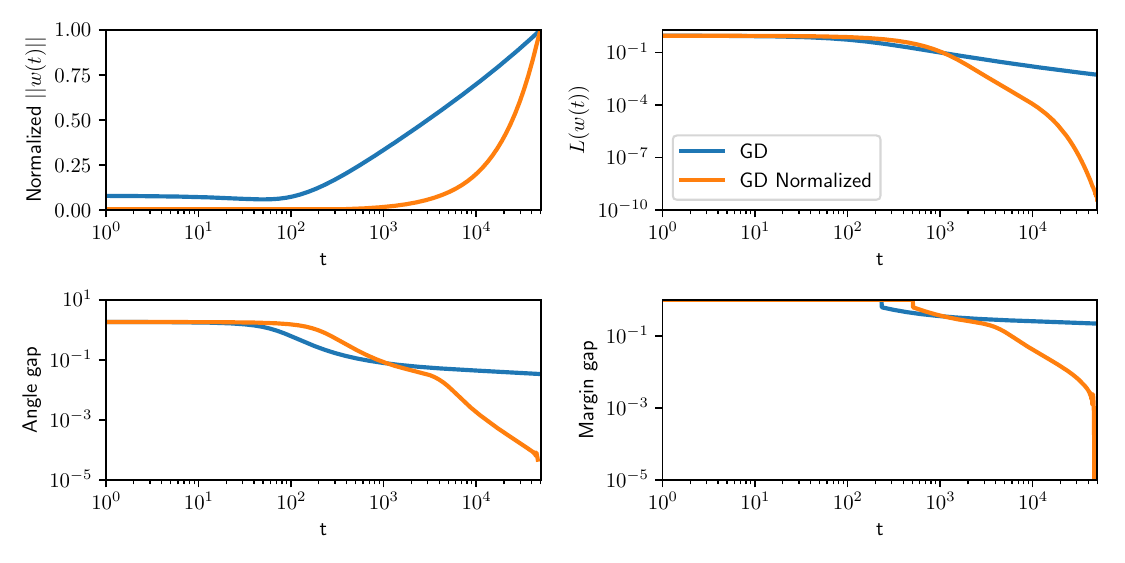} }}%
    \caption{a) Visualization of the synthetic dataset composed of 600 points: 300 labeled positive and 300 negative, again respectively denoted by $'+'$ and $'\circ'$. b) Convergence plots for a logistic regression trained with GD and Normalized GD for $5 \times 10^4$ epochs. Similarly to what is observed in Figure \ref{fig:Synthetic-dataset}, Normalized GD converges significantly faster to the max-margin solution.}%
    \label{fig:logistic-regression}%
\end{figure}

Figure \ref{fig:logistic-regression} shows the dataset and the convergence of GD and Normalized GD on logistic regression. We can see that Normalized GD converges significantly faster, similarly to Figure 1. To compute angle and margin gaps, we obtain the $L_2$ max margin vector $\hat {\bold w}$ from a SVM solver, along with the max margin itself.

\begin{figure}%
    \centering
    \subfloat[2 layer Linear Network]{{\includegraphics[width=0.5\textwidth]{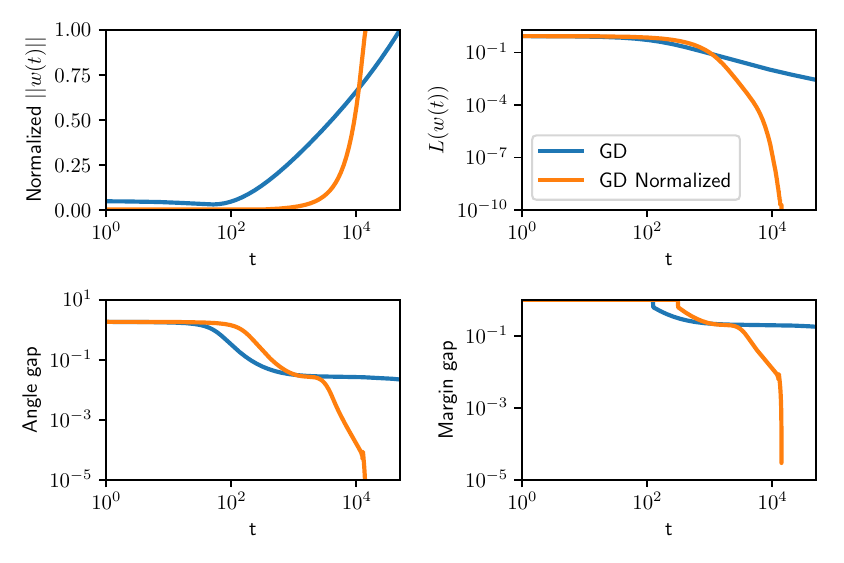} }}%
    \subfloat[2 layer ReLU Network]{{\includegraphics[width=0.5\textwidth]{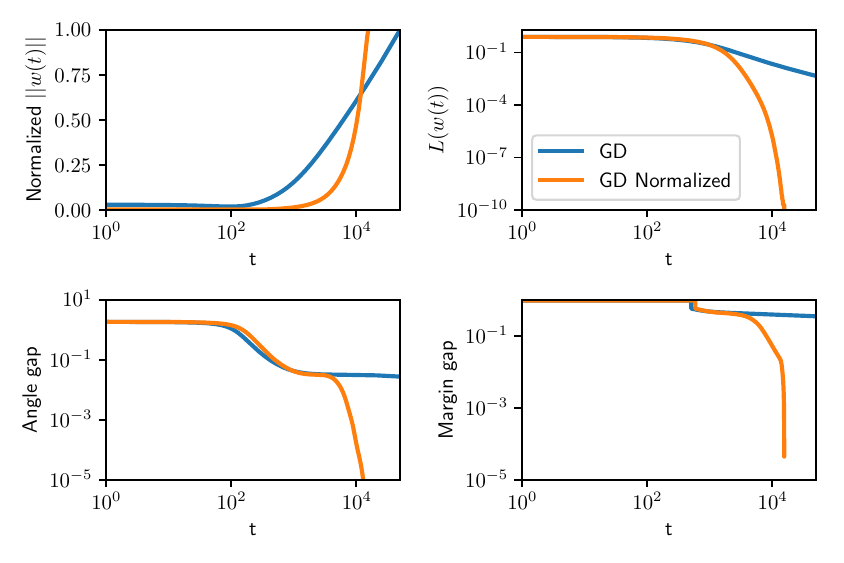} }}%
    \caption{Convergence plots for 2-layered neural networks with a $2 \times 10 \times 1$ architecture, trained for $5 \times 10^4$ epochs with GD and Normalized GD. (a,b): networks with linear / ReLU activations, respectively. We can observe that the plots for linear and ReLU networks look similar, and for both models Normalized GD still converged noticeably faster to the max margin solution. Additionally, we can see that Normalized GD converged faster in the 2-layer setting when compared to Figure \ref{fig:logistic-regression}, achieving $0$ numerical loss in roughly $10^4$ epochs.}%
    \label{fig:nn-2-layers}%
\end{figure}

In Figure \ref{fig:nn-2-layers} we see the convergence of GD and Normalized GD for 2-layer neural networks, with and without a ReLU non-linearity. We can observe that there is little difference between all plots, suggesting that our results might translate to more complex models, at least in well-behaved settings such as when the data is linearly separable. Note that for the non-linear network, the angle and margin were computed using $\bold w = \bold W_1 \bold W_2 \dots \bold W_d$, as if the model was a linear network. The same observation can be drawn from Figure \ref{fig:nn-3-layers}, which depicts convergence for 3-layered networks.

\begin{figure}%
    \centering
    \subfloat[3 layer Linear Network]{{\includegraphics[width=0.5\textwidth]{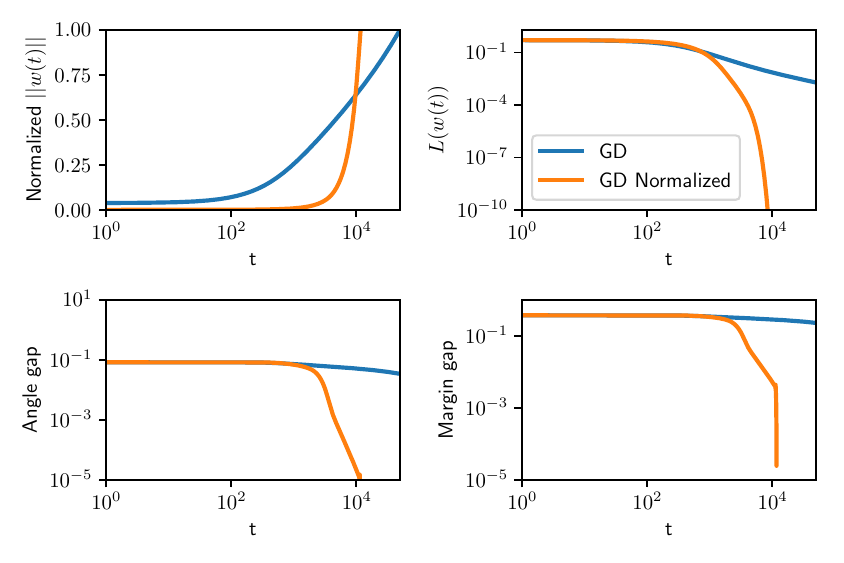} }}%
    \subfloat[3 layer ReLU Network]{{\includegraphics[width=0.5\textwidth]{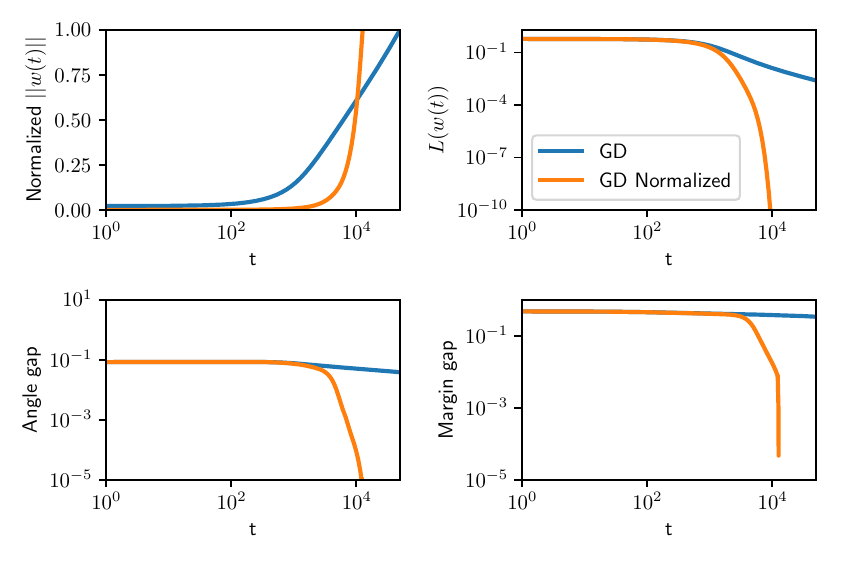} }}%
    \caption{Convergence plots for 3-layered neural networks with a $2 \times 10 \times 10 \times 1$ architecture, trained for $5 \times 10^4$ epochs with GD and Normalized GD. (a,b): networks with linear / ReLU activations, respectively. As in Figure \ref{fig:nn-2-layers}, we can observe that the plots of linear and ReLU networks are similar.}
    \label{fig:nn-3-layers}%
\end{figure}

\remove{
\subsection{Classification on MNIST with a ReLU network} 

In order to better evaluate whether Normalized GD can be useful for neural network optimization, we trained a 2-layer feedforward ReLU network on the MNIST digit classification dataset. It is composed of 70,000 grayscale images of 0-9 digits (10 classes total), each having $28 \times 28$ pixels. We used 10,000 images for testing and the rest for training and validation.

We trained a network with one hidden layer containing 5,000 hidden neurons and a ReLU activation $ReLU(x) = \max(0,x)$ with full-batch GD and Normalized GD, without any form of regularization, momentum or data augmentation. The network was trained for a total of 3,000 epochs, and the learning rate was divided by 5 at epochs 1,500, 2,250 and 2,625.

Since the optimal learning rate $\eta$ for GD and Normalized GD might differ, for each method we performed grid-search over learning rate values $\{0.1, 0.3, 0.5, 1.0, 2.5, 5.0\}$ using 5,000 images randomly chosen from the training set as validation. Both GD and Normalized GD performed the best with $\eta = 1.0$, and presented convergence issues with larger values. During validation we noticed that Normalized GD makes less progress than GD in the first epochs, when the gradient norms are typically large. To remedy this, we only normalized the gradients if the norm was lesser than $1$, which happened from the $11$'th epoch onwards.

 Figure \ref{fig:mnist} shows training loss and test error at each epoch $t$. The final test errors were $1.91\%$ and $1.4\%$ for GD and Normalized GD, respectively, which suggests that Normalized GD might also provide an advantage when training more complex networks on non-linearly separable data.

\begin{figure}%
    \centering
    \subfloat[Training loss]{{\includegraphics[width=0.5\textwidth]{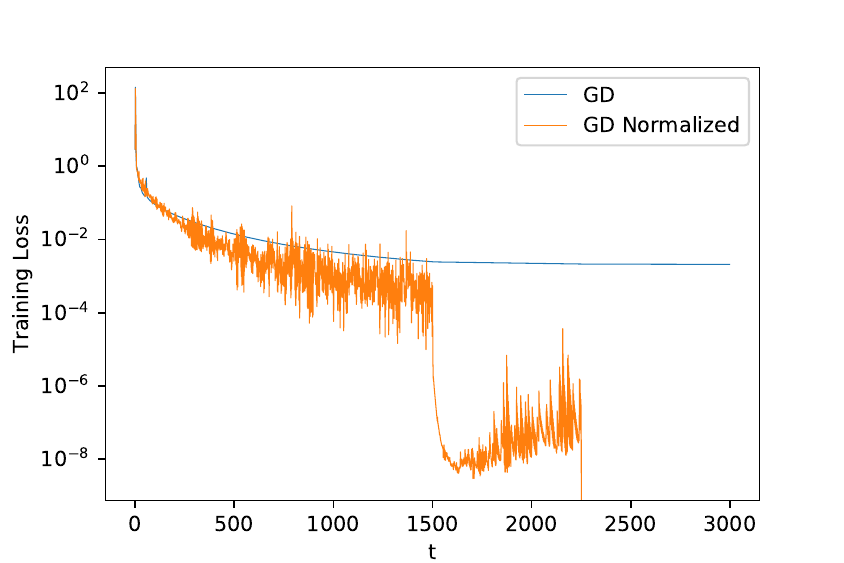} }}%
    \subfloat[Test error]{{\includegraphics[width=0.5\textwidth]{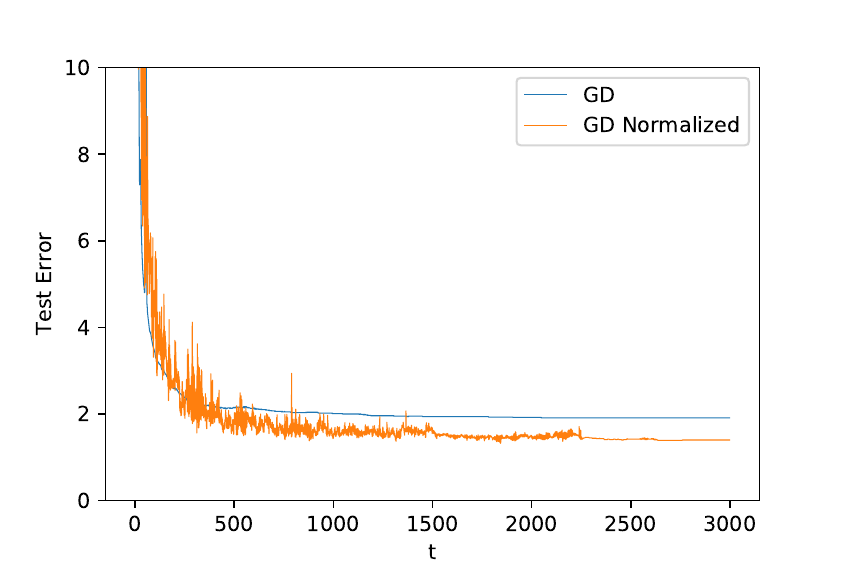} }}%
    \caption{MNIST digit classification with a 2-layered feedforward neural network with architecture $784 \times 5000 \times 10$ and a ReLU non-linearity in the hidden layer, trained using full-batches with GD and Normalized GD. a) Training loss per epoch $t$: we can see that Normalized GD converges to significantly lower values for the training loss, while GD stagnates after a fraction of the training. b) Test error ($\%$) per epoch $t$: GD achieves a final error of $1.91\%$, while Normalized GD achieves $1.4\%$.}%
    \label{fig:mnist}%
\end{figure}
}

\remove{
\subsection{Classification on CIFAR-10 with Residual Networks} \label{sec-ngd_cifar}

To check whether Normalized GD also provides advantages over GD in more realistic tasks, we train a Wide ResNet \citep{wideresnet} on the CIFAR-10 dataset \citep{krizhevsky2009} and compare the two methods. This model family is capable of reaching less than $4\%$ test error on CIFAR-10, a result close to the current state-of-the-art, making Wide ResNets a strong model baseline. For this experiment, we train a Wide ResNet 28-4: a 28-layered convolutional neural network with residual connections and a total of 5.8M parameters.

The CIFAR-10 dataset consists of 60,000 colored $32 \times 32$ images belonging to one of 10 possible classes, and is split into 50,000 training and 10,000 test points. We use standard data augmentation (horizontal flips and random crops), along with channel-wise normalization as pre-processing. Networks are trained with GD and Normalized GD, with no weight decay or momentum as to follow the presented theory.

We follow a similar schedule as in \cite{wideresnet}, decaying the learning rate by a factor of $5$ at $30\%$, $60\%$ and $80\%$ of the total iterations, along with a warm-up phase consisting of using a decayed rate during the first 150 iterations. To choose the initial learning rate for each method, we train the network for 3,000 iterations and use the learning rate that yields the lowest validation accuracy (using a set of 5,000 randomly chosen images). Two advantages of Normalized GD are observed in these initial experiments: first, it is less sensitive to different learning rates; second, it avoids sharp drops in performance which occur when using GD with large learning rates. Figure \ref{fig:cifar1} shows the training loss and test error for GD and normalized GD for initial learning rates $\eta=1.0$ and $\eta=2.5$.

Both GD and Normalized GD presented better results for $\eta = 2.0$, which was chosen as the initial learning rate for the final run. Additionally, we changed the warm-up training phase: instead of starting with a learning rate decayed by 5 and increasing it at epoch 150, we slowly increase it back to its original value, with constant increases at epochs 50, 100, 150 and 200. The curves are presented in Figure \ref{fig:cifar2}: GD reaches a test error of $9.1\%$, while Normalized GD yields $6.93\%$.

Note that, while Normalized GD outperformed GD in this full-batch setting, its performance is still subpar when compared to the standard optimization for Wide ResNets, which include SGD with Nesterov momentum and weight decay. To confirm whether momentum and weight decay can have strong negative impacts in a model's performance, we also trained a Wide ResNet 28-4 using SGD, with and without momentum/weight decay. We observed that removing momentum and weight decay resulted in a test error increase from $4.45\%$ to $7.75\%$ (larger error than Normalized GD). This suggests an importance in reconciling weight decay, momentum and gradient normalization.

\begin{figure}%
    \centering
    \subfloat[Training loss]
    {{\includegraphics[width=0.5\textwidth]{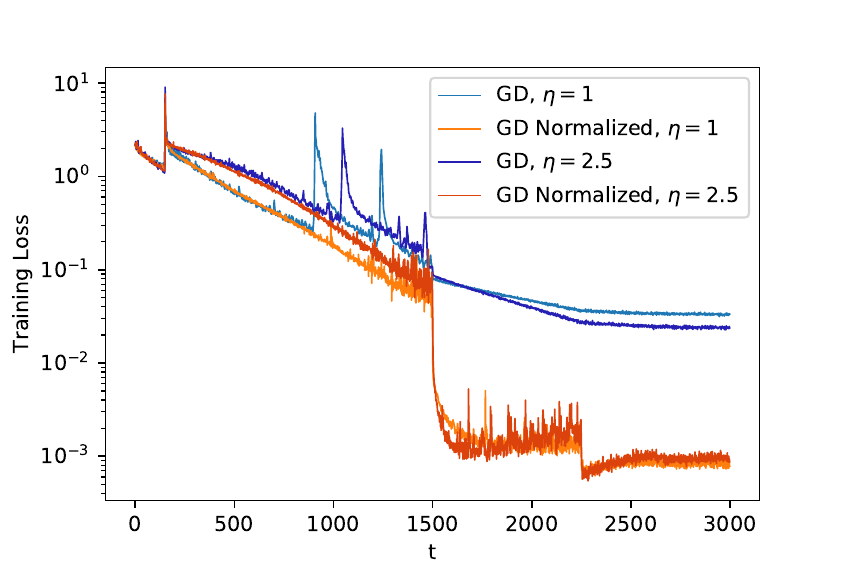} }}%
    \subfloat[Validation error]
    {{\includegraphics[width=0.5\textwidth]{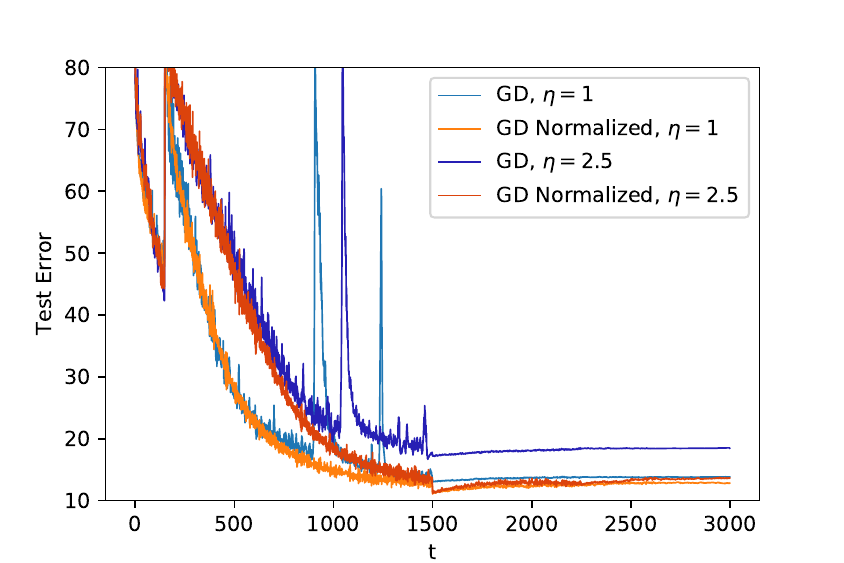} }}%
    \caption{Training a Wide ResNet 28-4 on CIFAR-10 with learning rates $\eta=1.0$ and $\eta=2.5$ for 3,000 epochs, as part of hyperparameter tuning. Normalized GD converges to similar training loss and test error for both learning rates, while $\eta = 2.5$ for GD yields a $4.66\%$ larger test error than $\eta = 1.0$. Additionally, sudden performance drops are observed for GD (between iterations 1,000 and 1,500), but not for Normalized GD (the spike at $t=150$ is due to the increase in learning rate).}%
    \label{fig:cifar1}%
\end{figure}

\begin{figure}%
    \centering
    \subfloat['Best yet' Training loss]
    {{\includegraphics[width=0.5\textwidth]{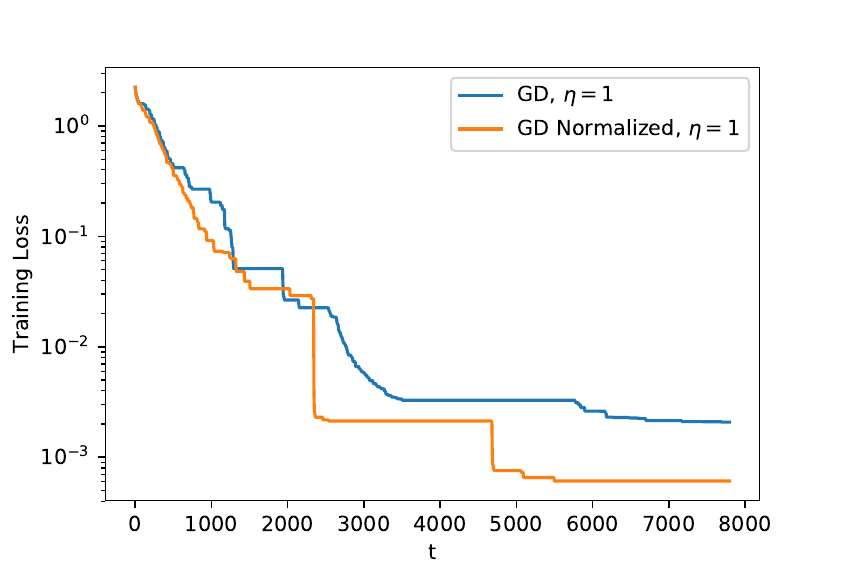} }}%
    \subfloat['Best yet' Test error]
    {{\includegraphics[width=0.5\textwidth]{figures/cifar_testerr.pdf} }}%
    \caption{Final curves while training a Wide ResNet 28-4 on CIFAR-10, with a learning rate $\eta = 2.0$ for a total of 7,800 iterations. For clarity, the curves are 'best yet' training loss and test error: the loss/error plotted at iteration $t$ is the minimum loss/error up to $t$. Normalized GD outperforms GD by absolute $2.17\%$, and dominates GD's test error during the whole training. However, unlike in curves reported in (cite wrn), progress stops early in training: for both GD and Normalized GD, there is no change in the 'best yet' test error after $t=2350$, even with the decays in learning rate. This suggests that regularization and/or momentum might be required to achieve better results.}%
    \label{fig:cifar2}%
\end{figure}
}

\section{Losses with poly-exponential tails} \label{sec: appendix results on poly-exp tails}
    In Theorem \ref{Theorem: convergence to max margin for general tail} we assume that the gradient descent iterates $\np(t)$ minimizes the objective, i.e., $\mathcal{L}_\P \left(\np(t)\right)\to0$, and that the incremental updates $\np(t+1)-\np(t)$ converge in direction. In this section we show that in the case of a single layer, $L=1$, and for a specific type of loss function these assumptions can be omitted.
    \begin{definition} \label{def: poly-exponential tail}
		A function $f(u)$ has a ``tight poly-exponential tail", if there exist positive constants $\mu_+,\mu_-,\lossPower$, and $\bar{u}$ such that $\forall u>\bar{u}$:
		$$  (1-\e(-\mu_- u^\lossPower))e^{- u^\lossPower} \le f(u) \le (1+\e(-\mu_+ u^\lossPower))e^{- u^\lossPower} $$
	\end{definition}
	\begin{theorem} \label{theorem: main}
		For almost all datasets that are linearly separable  and any $\beta$-smooth $\mathcal{L}$, with strictly monotone loss function $\ell $ (Definition \ref{def: l(u) assumptions}) for which $-\ell'(u)$ has a tight poly-exponential tail (Definition \ref{def: poly-exponential tail}) with $\lossPower>0.25$, given step size $\eta<2\beta^{-1}$ and any initialization $\wvec(0)$, the iterates of gradient descent in eq.~\eqref{GD} will behave as:
		\begin{equation} \label{define wVec}
			\wvec(t) = \what g(t)+\bm{\rho}(t),
		\end{equation}
        where $\hat{\mathbf{w}}$ is the following $L_{2}$ max margin separator:
    \begin{equation}
    \hat{\mathbf{w}}=\underset{\mathbf{\mathbf{w}}\in\mathbb{R}^{d}}{\mathrm{argmin}}\left\lVert \mathbf{w}\right\rVert^2 \,\,\mathrm{s.t.}\,\,\mathbf{w}^{\top}\mathbf{x}_{n}\geq1,
    \end{equation}
		and for a constant $\bm{a}$ independent of $\lossPower$,
		\begin{flalign}
			& g(t) = \log^{\frac{1}{\lossPower}}(t) + \frac{1}{\lossPower} \log(\lossPower\log^{1-\frac{1}{\lossPower}}(t))\log^{\frac{1}{\lossPower}-1}(t)
		\end{flalign}
		\begin{equation}
			\Vert\bm{\rho}(t)\Vert = \begin{cases}
			O(1), & \text{if $\lossPower>1$}.\\
			\frac{1}{\lossPower}g^{1-\lossPower}(t)\Vert\bm{a}\Vert+o(g^{1-\lossPower}(t)), & \text{if $\frac{1}{4}<\lossPower\le1$},
			\end{cases}
		\end{equation}	
	\end{theorem}
    
    	\begin{table} [t]
		\centering
			\normalsize	
			\begin{tabular}{c c c} 
				\toprule				
				& $\lossPower\ge 1$\vphantom{$\displaystyle\int$} & $\frac{1}{4}<\lossPower\le 1$ \\ 
				\midrule
				$\left \Vert \frac{\wvec(t)}{\Vert \wvec(t)\Vert} - \frac{\what}{\whatNorm} \right \Vert$ \text{ or } $\gamma - \min_{n} \frac{\xnT \wvec (t)}{\Vert \wvec(t)\Vert}$ & $O\left(\log^{-\frac{1}{\lossPower}}(t)\right)$ \vphantom{$\displaystyle O\left(\frac{1}{\log^{\frac{1}{\lossPower}(t)}}\right)$} & $\frac{C_1}{\lossPower}\log^{-1}(t) + o(\log^{-1}(t)) $ \\ 
				 $1 - \frac{\wvec(t)^\top\what}{\wNorm\whatNorm}$ &$O\left(\log^{-\frac{2}{\lossPower}}(t)\right)$ \vphantom{$\displaystyle O\left(\frac{1}{\log^{\frac{1}{\lossPower}}(t)}\right)$} & $\frac{C_2}{\lossPower^2}\log^{-2}(t) + o(\log^{-2}(t)) $\\
				\bottomrule
			\end{tabular}
			\captionsetup{width=\textwidth}
			\caption{Summary of convergence rates for Theorem \ref{theorem: main} for loss functions with exponential tail, when $-\ell^{\prime}(u) \asymp \exp(-u^{\lossPower})$. The first line is the convergence rate for both the distance and the suboptimality of the margin (with $C_3$ instead of $C_1$). The second line is the angle convergence rate. The constants are:\\ $C_1 = \left\Vert\left(I-\frac{\what\what^\top}{\whatNorm^2}\right)\frac{\genericConstVec}{\whatNorm}\right\Vert,\;
            C_2 = \left(\frac{1}{4}-\left(\frac{\genericConstVec^\top\what}			{\whatNorm \Vert\genericConstVec\Vert}\right)^2\right)\frac{2\Vert\genericConstVec\Vert^2}{\whatNorm^2},\;
				C_3 = \frac{1}{\whatNorm}\left(\frac{\what^\top\bm{a}}{\whatNorm^2} - \min_{n}\xnT\bm{a} \right)$.} \label{table: convergence rates}
	\end{table}

As we show in the next section Theorem~\ref{theorem: main} implies that $\wvec(t)/\Vert\wvec(t)\Vert$ converges to the normalized max margin separator $\what/\whatNorm$ for poly-exponential tails with $\lossPower>0.25$, but with a different rate than exponential loss. In Appendix section \ref{sec: CALCULATION OF CONVERGENCE RATES} we show that Theorem \ref{theorem: main} implies the convergence rates specified in Table \ref{table: convergence rates}. From this table, we can see that the optimal convergence rate for poly-exponential tails is achieved at $\lossPower=1$. Moreover, this rate becomes slower as $|\lossPower - 1|$ increases, at least in the range $\lossPower>0.25$. 

    Theorem \ref{theorem: main} is proved in appendix section \ref{sec:main-proof}. This Theorem is a generalization of Theorem \ref{theorem: ICLR theorem 3}, and therefore builds the ideas of \cite{soudry2017implicit}, as described non-rigorously in appendix section \ref{sec: Generic Tails}. The main proof is rather long, as we calculate exact asymptotic behavior, including constants in some cases, and do not assume the existence of limits.
    
    \section{Calculation of convergence rates\label{sec: CALCULATION OF CONVERGENCE RATES} for poly-exponential tails} 
    
        From Theorem \ref{theorem: main}, we can write $\wvec(t) = \what g(t) + \rhoVec$ where $\rhoVec = o(g(t))$.	
    	We can use this to calculate the normalized weight vector:
    	\begin{flalign} \label{Calculation of the normalized weight vector}
    	&\frac{\wvec(t)}{\Vert \wvec(t)\Vert} =\frac{g(t)\what +\rhoVec}{\sqrt{g(t)^2\what^\top\what + \rhoVec^\top\rhoVec+2g(t)\what^\top\rhoVec}} =\frac{\what + g^{-1}(t)\rhoVec}{\whatNorm\sqrt{1 +2\frac{\what^\top\rhoVec}{g(t)\whatNorm^2}+ \frac{\Vert\rhoVec\Vert^2}{g^2(t)\whatNorm^2}}}\nonumber\\
    	& \overset{(1)}= \frac{\what +g^{-1}(t)\rhoVec}{\whatNorm}\left[1-\frac{\what^\top\rhoVec}{g(t)\whatNorm^2}+\left[\frac{3}{4}\left(2\frac{\what^\top\rhoVec}{\whatNorm^2}\right)^2-\frac{\Vert\rhoVec\Vert^2}{2\whatNorm^2}\right]\frac{1}{g^2(t)}+O\left(\left(\frac{\what^\top\rhoVec}{g(t)}\right)^3\right)\right]\nonumber\\
    	& = \frac{\what}{\whatNorm}+\left(\frac{\rhoVec}{\whatNorm}-\frac{\what}{\whatNorm}\frac{\what^\top\rhoVec}{\whatNorm^2}\right)\frac{1}{g(t)}
    	+O\left(\left(\frac{\what^\top\rhoVec}{g(t)}\right)^2\right)
    	\frac{\hat{w}}{\norm{\hat{w}}}\nonumber\\
    	& = \frac{\what}{\whatNorm}+\left(I-\frac{\what\what^\top}{\whatNorm^2}\right)\frac{1}{\whatNorm}\frac{\rhoVec}{g(t)}+O\left(\left(\frac{\what^\top\rhoVec}{g(t)}\right)^2\right)
    	\end{flalign}
    	where in $(1)$ we used $\frac{1}{\sqrt{1+x}}=1-\frac{1}{2}x+\frac{3}{4}x^2+O(x^3)$.\\
    	We use eq.~\eqref{Calculation of the normalized weight vector} to calculate the angle:
    	\begin{align} \label{eq: Calculation of the angle}
    	&\frac{\wvec(t)^\top\what}{\wNorm\whatNorm}&\nonumber\\
    	&=\frac{\what^\top}{\whatNorm^2}\left(\what +g^{-1}(t)\rhoVec\right)\left[1-\frac{\what^\top\rhoVec}{g(t)\whatNorm^2}+\left[\frac{3}{4}\left(2\frac{\what^\top\rhoVec}{\whatNorm^2}\right)^2-\frac{\Vert\rhoVec\Vert^2}{2\whatNorm^2}\right]\frac{1}{g^2(t)}+O\left(\left(\frac{\what^\top\rhoVec}{g(t)}\right)^3\right)\right]\nonumber\\
    	& = 1 + \frac{2}{\whatNorm^2}\left[\left(\frac{\rhoVec^\top\what}{\whatNorm \Vert\rhoVec\Vert}\right)^2 - \frac{1}{4}\right]\frac{\Vert\rhoVec\Vert^2}{g^2(t)}+O\left(\left(\frac{\what^\top\rhoVec}{g(t)}\right)^3\right)
    	\end{align}
    	Calculation of the margin:
    	\begin{flalign} \label{eq: Calculation of the margin}
    		&\min_{n} \frac{\xnT \wvec(t)}{\Vert \wvec(t)\Vert}\overset{(1)}{=}\min_{n\in\set} \frac{\xnT \wvec(t)}{\Vert \wvec(t)\Vert}\nonumber\\
    		& =\min_{n\in\set} \xnT\left[\frac{\what}{\whatNorm}+\left(\frac{\rhoVec}{\whatNorm}-\frac{\what}{\whatNorm}\frac{\what^\top\rhoVec}{\whatNorm^2}\right)\frac{1}{g(t)}+O\left(\left(\frac{\what^\top\rhoVec}{g(t)}\right)^2\right)\right]\nonumber\\
    		& =\frac{1}{\whatNorm}+\frac{1}{\whatNorm}\left(\min_{n\in\set}\xnT\rhoVec-\frac{\what^\top\rhoVec}{\whatNorm^2}\right)\frac{1}{g(t)}
    		+O\left(\left(\frac{\what^\top\rhoVec}{g(t)}\right)^2\right),
    	\end{flalign}
    	where in (1) we used the fact that $\frac{\wvec(t)}{\Vert \wvec(t)\Vert}$ converge to the maximum-margin separator and thus the minimal value is obtained on the support vectors.
    	\remove{
        From Theorem \ref{theorem: general convergence rates}, we can also characterize $\rhoVec$: 
    	\begin{equation} \label{eq: rho def generic case appendix}
    	        \norm{\op_1\rhoVec}=\begin{cases} \fsym_2(t)\left(f'\left(\tilde{g}(t)\right)\right)^{-1}+o\left(\left(f'\left(\tilde{g}(t)\right)\right)^{-1}\right), & \text{if } \left(f'\left(\tilde{g}(t)\right)\right)^{-1}=\Omega(1)\\
    	        O\left(1\right), & \text{otherwise}
    	        \end{cases}
    	\end{equation}	
	where $\fsym_2(t)=\Theta(1)$.
      Substituting eq.~\eqref{eq: rho def generic case appendix} into eqs. \ref{Calculation of the normalized weight vector}, \ref{eq: Calculation of the angle}, \ref{eq: Calculation of the margin} we get, for some constant vector $\vect{a}$ independent of $f$:
      \begin{equation} \label{eq: generic margin results}
      \frac{1}{\whatNorm} - \min_{n} \frac{\xnT \wvec(t)}{\Vert \wvec(t)\Vert}  = \begin{cases}
      \dfrac{1}{\whatNorm}\left(\dfrac{\what^\top\genericConstVec}{\whatNorm^2} - \min_{n\in\set}\xnT\genericConstVec\right)\dfrac{\fsym_2(t)}{g(t)f'(g(t))}, & \text{if ${f'(g(t))}=\Omega(1)$}\\
      O(g^{-1}(t))\vphantom{\dfrac{0}{0}}, & \text{otherwise}
      \end{cases}
      \end{equation}
      
	\subsection{Poly-exponential tails } 
    }
      From equations \ref{Calculation of the normalized weight vector}, \ref{eq: Calculation of the angle}  we have:
	\begin{flalign} \label{eq: generic norm results}
		&\left \Vert \frac{\wvec(t)}{\Vert \wvec(t)\Vert} - \frac{\what}{\whatNorm} \right \Vert =\begin{cases}
		\left\Vert\left(I-\dfrac{\what\what^\top}{\whatNorm^2}\right)\genericConstVec\right\Vert\dfrac{1}{\whatNorm}\dfrac{\fsym_2(t)}{g(t)f'(g(t))}, & \text{if $\dfrac{1}{f'(g(t))}=\Omega(1)$}\\
			O(g^{-1}(t))\vphantom{\dfrac{0}{0}}, & \text{otherwise}
		\end{cases}&
	\end{flalign}
    \begin{equation} \label{eq: generic angle results}
	 1 - \frac{\wvec(t)^\top\what}{\wNorm\whatNorm}  = \begin{cases}
		\left(\dfrac{1}{4}-\left(\dfrac{\genericConstVec^\top\what}{\whatNorm \Vert\genericConstVec\Vert}\right)^2\right)\dfrac{2\Vert\genericConstVec\Vert^2}{\whatNorm^2}\dfrac{\fsym_2^2(t)}{\left(g(t)f'(g(t))\right)^2}, & \text{if $\dfrac{1}{f'(g(t))}=\Omega(1)$}\\
			O(g^{-2}(t))\vphantom{\dfrac{0}{0}}, & \text{otherwise}
		\end{cases}
	\end{equation}

	Additionally, from Theorem \ref{theorem: main}, we can write $\wvec(t) = \what g(t) + \rhoVec$, where:
	\begin{flalign} \label{eq: g def in sec about rates}
	& g(t) = \log^{\frac{1}{\lossPower}}(t) + \frac{1}{\lossPower} \log(\lossPower\log^{1-\frac{1}{\lossPower}}(t))\log^{\frac{1}{\lossPower}-1}(t)
	\end{flalign}
	\begin{equation} \label{eq: rho def in sec about rates}
	\bm{\rho}(t) = \begin{cases}
	O(1), & \text{if $\lossPower>1$}.\\
	\frac{1}{\lossPower}g^{1-\lossPower}(t)\bm{a}+o(g^{1-\lossPower}(t)), & \text{if $\frac{1}{4}<\lossPower<1$},
	\end{cases}
	\end{equation}
    and $\bm{a}$ is not dependent on $\lossPower$.\\
We can obtain the normalized weight vector convergence to normalized max margin vector in $L_2$ norm from substituting eqs. \ref{eq: g def in sec about rates}, \ref{eq: rho def in sec about rates} into eq.~\eqref{eq: generic norm results} :
    \begin{equation} 
		\left \Vert \frac{\wvec(t)}{\Vert \wvec(t)\Vert} - \frac{\what}{\whatNorm} \right \Vert =\begin{cases}
		\left\Vert\left(I-\dfrac{\what\what^\top}{\whatNorm^2}\right)\dfrac{\genericConstVec}{\whatNorm}\right\Vert\dfrac{1}{\lossPower \log(t)}+o(\log^{-1}(t)), & \frac{1}{4}<\lossPower\le 1\\
			O(\log^{-\frac{1}{\lossPower}}(t)), & \lossPower\ge1
		\end{cases}
	\end{equation}
	We can also obtain the angle convergence from substituting eqs. \ref{eq: g def in sec about rates}, \ref{eq: rho def in sec about rates} into eq.~\eqref{eq: generic angle results}:
    \begin{equation}
	 1 - \frac{\wvec(t)^\top\what}{\wNorm\whatNorm}  = \begin{cases}
		\left(\dfrac{1}{4}-\left(\dfrac{\genericConstVec^\top\what}{\whatNorm \Vert\genericConstVec\Vert}\right)^2\right)\dfrac{2\Vert\genericConstVec\Vert^2}{\whatNorm^2}\dfrac{1}{\lossPower^2\log^2(t)}, & \frac{1}{4}<\lossPower\le 1\\
			O(\log^{-\frac{2}{\lossPower}}(t)), & \lossPower\ge1
		\end{cases}
	\end{equation}
    We obtain the margin convergence from substituting eqs. \ref{eq: g def in sec about rates}, \ref{eq: rho def in sec about rates} into eq.~\eqref{eq: Calculation of the margin}:
	\begin{equation}
	\frac{1}{\whatNorm} - \min_{n} \frac{\xnT \wvec(t)}{\Vert \wvec(t)\Vert}  = \begin{cases}
		\dfrac{1}{\whatNorm}\left(\dfrac{\what^\top\genericConstVec}{\whatNorm^2} - \min_{n}\xnT\genericConstVec\right)\dfrac{1}{\lossPower \log(t)}, & \frac{1}{4}<\lossPower\le 1\\
			O(\log^{-\frac{1}{\lossPower}}(t)), & \lossPower\ge1
		\end{cases}
	\end{equation}
We can see that in the case of $\lossPower<1$ the rates are smaller for larger $\lossPower$ and that the optimal rates are achieved for $\lossPower=1$.
	
	\subsection{Proof of Theorem \ref{theorem: main} \label{sec:main-proof}}
	In the following proofs, we define $\what$ as the $L_2$ max margin vector, which satisfies eq.~\eqref{w_hat equation}:
	\begin{equation*}
		\what = \argmin_{\wvec\in \mathbb{R}^d} \norm{\wvec}^2 \text{ s.t. } \wvec^\top\xn\ge1
	\end{equation*}
    Let $\mathcal{S}=\{n:\what^\top \xn=1\}$ denote indices of support vectors of $\what$. From the KKT optimality conditions, we have for some $\alpha_n\ge 0$, 
    \begin{equation} \label{eq: what def with alpha}
    	\what = \sumnsv \alpha_n \xn
    \end{equation}
	Let  $\wtilde$ be a vector which satisfies the equations:
	\begin{flalign} \label{eq: wtilde def}
	    &\forall n\in \mathcal{S}\ :\ \eta\e(-\lossPower\xnT\wtilde)=\alpha_n,\ \bar{\mathbf{P}}_1 \wtilde = 0,
	\end{flalign}
	where we recall that we defined $\mathbf{P}_1\in \mathbb{R}^d$ as the orthogonal projection matrix to the subspace spanned by the support vectors, and $\bar{\mathbf{P}}_1=I-\mathbf{P}_1$ as the complementary projection matrix. Equation \ref{eq: wtilde def} has a unique solution for almost every dataset from Lemma 8 in \cite{soudry2017implicit}. 
Furthermore, let $C_i,\epsilon_i,t_i$ ($i\in \mathbb{N}$) be various positive constants which are independent of $t$, and denote, 
	\begin{equation} \label{theta definition}
		\theta = \min_{n\notin \mathcal{S}} \xnT\what > 1
	\end{equation}
    
    
    The following lemmata were proved in \cite{soudry2017implicit} [Lemma~1 and Lemma 5].
	\begin{lemma} \label{Lemma: w(t)->infty}
		Let $\mathbf{w}\left(t\right)$
    be the iterates of gradient descent (eq.~\eqref{GD})
    on a $\beta$-smooth $\mathcal{L}$ and any starting point $\wvec(0)$. If the data is linearly separable, $\ell$ is a strict monotone loss (Definition \ref{def: l(u) assumptions}), and $\eta<2\beta^{-1}$ then we have: (1) $\lim_{t\rightarrow\infty}\mathcal{L}\left(\mathbf{w}\left(t\right)\right)=0$,
	(2) $\lim_{t\rightarrow\infty}\left\Vert \mathbf{w}\left(t\right)\right\Vert =\infty$, and
	(3) $\forall n:\,\lim_{t\rightarrow\infty}\mathbf{w}\left(t\right)^{\top}\mathbf{x}_{n}=\infty$.
	\end{lemma}
    \begin{lemma} \label{Lemma: L converge}
		Let $\mathcal{L}(\wvec)$ be a $\beta$-smooth non-negative objective. If $\eta<2\beta^{-1}$, then for any $\wvec(0)$, with the GD sequence
		\begin{equation}
		\wvec(t+1) = \wvec(t) -\eta \derL(\wvec (t))
		\end{equation}
		we have that $\sum_{u=0}^{\infty}\Vert\derL(\wvec(u))\Vert^2<\infty$ and therefore $\lim_{u\to\infty}\Vert\derL(\wvec(u))\Vert^2=0$.
	\end{lemma}
	\noindent From Lemma \ref{Lemma: w(t)->infty}, $\forall n\ :\ \lim_{t\to\infty}\wvec(t)^\top\xn=\infty$. In addition, we assume that the negative loss derivative $-\ell'(u)$ has a poly-exponential tail $e^{-{u^\lossPower}}$. Combining both facts, we have positive constants $\mu_+,\mu_-$ and $\bar{t}$ such that $\forall n, \forall t> \bar{t}$:
	\begin{equation} \label{-l' upper bound}
	\ -\ell'(\wvec(t)^\top \xn) \le (1+\e(-\mu_+(\wvec(t)^\top \xn)^\lossPower))\e(-(\wvec(t)^\top \xn)^\lossPower)
	\end{equation}
	\begin{equation} \label{-l' lower bound}
	\ -\ell'(\wvec(t)^\top \xn) \ge (1-\e(-\mu_- (\wvec(t)^\top \xn)^\lossPower))\e(-(\wvec(t)^\top \xn)^\lossPower)
	\end{equation}
	
	\subsection{Case: $\lossPower>1$}  
    \subsubsection{Definitions and auxiliary calculations \label{sec: Definitions And Auxiliary Lemma}}
	In the following proofs, for any solution $\wvec(t)$, we define 
	\begin{equation} \label{define r(t), beta>1}
		\rvec(t) = \wvec(t)-g(t)\what  - g^{1-\lossPower}(t)\wtilde
	\end{equation}
and the following functions:
	\begin{align}
	g(t) &= \log^{\frac{1}{\lossPower}}(t) + \frac{1}{\lossPower} \log(\lossPower\log^{1-\frac{1}{\lossPower}}(t))\log^{\frac{1}{\lossPower}-1}(t),\\
	h(t) &= \left(1-\frac{1}{\lossPower}\right)\log^{-1}(t)\left(1-\log\left(\lossPower \log^{1-\frac{1}{\lossPower}}(t)\right)\right).
	\end{align}
	We notice that $\exists t_h$ such that $\forall t>t_h: h(t)<0$.\\
	Additionally, $g(t)$ has the following properties, as can be shown using basic analysis:\\
	\begin{equation} \label{diff g o(t^-1), beta>1}
	g(t+1) - g(t) = \Theta(t^{-1}\log^{\frac{1}{\lossPower}-1}(t))
	\end{equation}
	\begin{equation}\label{diff g-beta Theta(...), beta>1}
	g^{1-\lossPower}(t) - g^{1-\lossPower}(t+1) = \log^{\frac{1}{\lossPower}-1}(t) - \log^{\frac{1}{\lossPower}-1}(t+1) = \Theta(t^{-1}\log^{\frac{1}{\lossPower}-2}(t))
	\end{equation}
	\begin{equation} \label{g''(t)}
	\frac{1}{\lossPower t} \log^{\frac{1}{\lossPower}-1}(t)\left[1+h(t)\right] - (g(t+1)-g(t)) = o(t^{-2})
	\end{equation}
	We denote $\bar{C_1} = \frac{1-\lossPower}{\lossPower}$. $\exists m_1(t)=o(\log^{\epsilon}(t)),m_2(t)=o(\log^{\epsilon}(t))$ such that $\forall \epsilon>0$:
	\begin{equation} \label{eq: g(1-beta) first derivative, beta>1}
	    \left( g^{1-\lossPower}(t)\right)' = (1-\lossPower)g^{-\lossPower}(t)g'(t) = \bar{C}_1\frac{1}{t} \log^{\frac{1}{\lossPower}-2}(t) + \frac{1}{t} \log^{\frac{1}{\lossPower}-3}(t)m_1(t)
	\end{equation}
	\begin{equation} \label{eq: g(1-beta) second derivative, beta>1}
	    \left( g^{1-\lossPower}(t)\right)'' = \frac{1}{t^2}\log^{\frac{1}{\lossPower}-2}(t)m_2(t)
	\end{equation}
	In addition:
	\begin{flalign}
	&g^\lossPower(t) = \left(\log^{\frac{1}{\lossPower}}(t) + \frac{1}{\lossPower} \log(\lossPower\log^{1-\frac{1}{\lossPower}}(t))\log^{\frac{1}{\lossPower}-1}(t)\right)^\lossPower &\nonumber\\
	&  =\left(\log^{\frac{1}{\lossPower}}(t) \left(1 + \frac{1}{\lossPower} \log(\lossPower\log^{1-\frac{1}{\lossPower}}(t))\log^{-1}(t)\right)\right)^\lossPower \nonumber
	\ge \log(t) + \log(\lossPower\log^{1-\frac{1}{\lossPower}}(t)),\ \forall t>t_B \,,
	\end{flalign}
	where in the last transition we used $\lossPower \ge 1$, $ \exists t_B $ such that $$ \forall t>t_B:\ \frac{1}{\lossPower} \log(\lossPower\log^{1-\frac{1}{\lossPower}}(t))\log^{-1}(t)\ge0  $$ and Bernoulli's inequality:
	\begin{equation} \label{Bernoulli's inequality,r>=1, x>=-1}
	\forall r \ge 1,x \ge -1: (1+x)^r \ge 1+rx
	\end{equation}
	Therefore, $ \forall a>0 $
	\begin{flalign} \label{e^g^beta upper bound}
	\e(-a g^\lossPower(t)) \le \e\left(-a \left(\log(t) + \log(\lossPower\log^{1-\frac{1}{\lossPower}}(t))\right)\right)=t^{-a}\left(\frac{1}{\lossPower}\log^{\frac{1}{\lossPower}-1}(t)\right)^{a}.
	\end{flalign}
	\subsubsection{Proof Of Theorem~\ref{theorem: main} for $\lossPower>1$}
	Our goal is to show that $\Vert \rvec(t) \Vert$ is bounded, and therefore $\rhoVec = \rvec(t)+g^{1-\lossPower}(t)\wtilde$ is bounded. To show this, we will upper bound the following equation
	\begin{equation} \label{norm r(t+1)}
		\Vert\rvec(t+1)\Vert^2 = \Vert\rvec(t+1)-\rvec(t)\Vert^2 + 2\left(\rvec(t+1)-\rvec(t)\right)^\top\rvec(t)+\Vert\rvec(t)\Vert^2
	\end{equation}
	First, we note that the first term in this equation can be upper bounded by
	\begin{flalign} \label{norm(r(t+1)-r(t)}
	&|| \vect{r}(t+1) - \vect{r}(t) ||^2&\nonumber\\
	&\overset{(1)}{=} ||\wvec(t+1) - g(t+1)\what - g^{1-\lossPower}(t+1) \wtilde - \wvec(t) + g(t)\what + g^{1-\lossPower}(t) \wtilde||^2 \nonumber\\
	& \overset{(2)}{=} ||-\eta\nabla L(\wvec(t)) - \what(g(t+1)-g(t)) - \wtilde(g^{1-\lossPower}(t+1) - g^{1-\lossPower}(t))||^2\nonumber\\
	& = \eta^2||\nabla L(\wvec(t))||^2 + ||\what||^2 (g(t+1)-g(t))^2 + ||\wtilde||^2 (g^{1-\lossPower}(t+1) - g^{1-\lossPower}(t))^2\nonumber\\
	&+2\eta(g(t+1)-g(t)) \what^\top  \nabla L(\wvec(t))
	+2\eta(g^{1-\lossPower}(t+1)-g^{1-\lossPower}(t)) \wtilde^\top  \nabla L(\wvec(t))\nonumber\\
	& + 2(g(t+1)-g(t))(g^{1-\lossPower}(t+1) - g^{1-\lossPower}(t)) \what^\top  \wtilde\nonumber\\
	&\overset{(3)}{\le} \eta^2||\nabla L(\wvec(t))||^2 +o(t^{-1}\log^{\frac{1}{\lossPower}-2}(t)),
	\end{flalign}
	where in (1) we used eq.~\eqref{define r(t), beta>1}, in (2) we used eq.~\eqref{GD} and in (3) we used eq.~\eqref{diff g o(t^-1), beta>1}, eq.~\eqref{diff g-beta Theta(...), beta>1}, Lemma \ref{Lemma: L converge}, and also that
	\begin{equation}
		\what^\top \derL(\wvec(t)) = \sumn \ell'(\wvec(t)^\top\xn)\what^\top\xn\le0
	\end{equation}
	since $\what^\top\xn\ge1$ from the definition of $\what$ and $\ell^{\prime}(u) \leq 0$.\\
	Also, from Lemma \ref{Lemma: L converge}  we know that:
	\begin{equation} \label{eq: derL converge}
		\Vert\derL(\wvec(u))\Vert^2=o(1) \text{ and } \sum_{u=0}^{\infty}\Vert\derL(\wvec(u))\Vert^2<\infty
	\end{equation}
	Substituting eq.~\eqref{eq: derL converge} into eq.~\eqref{norm(r(t+1)-r(t)}, and recalling that the power series $t^{-1}\log^{-v}(t)$ converges for any $\nu>1$, we can find $C_0$ such that
	\begin{equation} \label{sum norm diff r bound}
		|| \vect{r}(t+1) - \vect{r}(t) ||^2 = o(1) \text{ and } \sum_{t=0}^{\infty}|| \vect{r}(t+1) - \vect{r}(t) ||^2 = C_0<\infty
	\end{equation}
	This equation also implies that
	\begin{equation}
		\forall\epsilon_0,\ \exists t_0\ :\ \forall t>t_0\ :\ ||| \vect{r}(t+1) - \vect{r}(t) |||<\epsilon_0
	\end{equation}
	Next, we would like to bound the second term in eq.~\eqref{norm r(t+1)}.
    To do so, will use the following Lemma, which we will prove in appendix \ref{proof of lemma 4}
	\begin{restatable}{lemma}{rBoundLemmaBetaGE} \label{Lemma: (r(t+1)-r(t)r(t) bound for beta>1}
		We have
		\begin{equation} \label{eq: (r(t+1)-r(t)r(t) bound for beta>1}
			\exists C_1,t_1\ :\ \forall t>t_1 \ :\ \left(\rvec(t+1)-\rvec(t)\right)^\top\rvec(t)\le C_1 t^{-1}\left(\log(t)\right)^{-1-\frac{1}{2}\left(1-\frac{1}{\lossPower}\right)}
		\end{equation}
		Additionally, $\forall \epsilon_1>0,\ \exists C_2,t_2$, such that $\forall t>t_2$, if
		\begin{equation} \label{P1r(t)>epsilon, beta>1}
			\Vert \vect{P}_1\rvec(t)\Vert\ge\epsilon_1
		\end{equation}
		then the following improved bounds holds
		\begin{equation} \label{eq: (r(t+1)-r(t)r(t) improved bound for beta>1}
		 \left(\rvec(t+1)-\rvec(t)\right)^\top\rvec(t)\le -C_2 t^{-1}\log^{\frac{1}{\lossPower}-1}(t)<0
		\end{equation}
		
	\end{restatable}
   
    From eq.~\eqref{eq: (r(t+1)-r(t)r(t) bound for beta>1} in Lemma \ref{Lemma: (r(t+1)-r(t)r(t) bound for beta>1}, we can find $C_1,t_1$ such that $\forall t>t_1$:
	\begin{equation} \label{eq2: (r(t+1)-r(t)r(t) bound for beta>1}
	 \left(\rvec(t+1)-\rvec(t)\right)^\top\rvec(t)\le C_1 t^{-1}\left(\log(t)\right)^{-1-\frac{1}{2}\left(1-\frac{1}{\lossPower}\right)}
	\end{equation}
	Thus, by combining eqs. \ref{eq2: (r(t+1)-r(t)r(t) bound for beta>1} and \ref{sum norm diff r bound} into eq.~\eqref{norm r(t+1)}, we find
	\begin{align*}
		&\Vert \rvec (t) \Vert^2 - \Vert \rvec (t_1) \Vert^2\\
		& = \sum_{u=t_1}^{t-1}\left[\Vert \rvec (u+1) \Vert^2 - \Vert \rvec (u) \Vert^2\right]\\
		& \le C_0+2 \sum_{u=t_1}^{t-1} C_1 t^{-1}\left(\log(t)\right)^{-1-\frac{1}{2}\left(1-\frac{1}{\lossPower}\right)}
	\end{align*}
	which is bounded, since $\lossPower>1$. Therefore, $\Vert \rvec (t) \Vert$ is bounded.
	
	\subsubsection{Proof Of Lemma \ref{Lemma: (r(t+1)-r(t)r(t) bound for beta>1}} \label{proof of lemma 4}
	Recall that we defined $\rvec(t) = \wvec(t)-g(t)\what  - g^{1-\lossPower}(t)\wtilde$. $\what$ and $\wtilde$ were defined in section \ref{sec:main-proof}.
	\rBoundLemmaBetaGE*
	We examine the expression we wish to bound, recalling that $\rvec(t) = \wvec(t)-g(t)\what  - g^{1-\lossPower}(t)\wtilde$:
	\begin{flalign} \label{(r(t+1)-r(t))r(t) bound, beta>1}
	&(\rvec(t+1) - \rvec(t))^\top  \rvec(t) &\nonumber\\
	&=(-\eta\nabla L(\wvec(t)) - \what(g(t+1)-g(t)) - \wtilde(g^{1-\lossPower}(t+1) - g^{1-\lossPower}(t)))^\top \rvec(t)\nonumber\\
	& = -\eta\sumn \ell'(\wvec(t)^\top \xn)\xnT\rvec(t) - \what^\top \rvec(t)(g(t+1)-g(t)) - \wtilde^\top  \rvec(t) (g^{1-\lossPower}(t+1) - g^{1-\lossPower}(t))\nonumber\\
	& \overset{(1)}{=} \what^\top \rvec(t)\left[\frac{1}{\lossPower}t^{-1}\log^{\frac{1}{\lossPower}-1}(t)\left(1+h(t)\right) - (g(t+1)-g(t))\right]  -\eta\sumnnsv \ell'(\wvec(t)^\top \xn)\xnT\rvec(t)\nonumber\nonumber\\
	&-\eta\sumnsv \left[\frac{1}{\lossPower}t^{-1}\log^{\frac{1}{\lossPower}-1}(t)\left(1+h(t)\right) \e(-\lossPower\xnT \wtilde) + \ell'(\wvec^\top (t)\xn)\right]\xnT\rvec(t)\nonumber\\
	& -\wtilde^\top  \rvec(t)(g^{1-\lossPower}(t+1) - g^{1-\lossPower}(t)-\left( g^{1-\lossPower}(t)\right)')-\wtilde^\top  \rvec(t) \left( g^{1-\lossPower}(t)\right)' \nonumber\\
	& \overset{(2)}{\le} \what^\top \rvec(t)\left[\frac{1}{\lossPower}t^{-1}\log^{\frac{1}{\lossPower}-1}(t)\left(1+h(t)\right) - (g(t+1)-g(t))\right]  -\eta\sumnnsv \ell'(\wvec(t)^\top \xn)\xnT\rvec(t)\nonumber\nonumber\\
	&-\eta\sumnsv \left[\gamma_n  t^{-1}\log^{\frac{1}{\lossPower}-2}h_3(t) +\frac{1}{\lossPower}t^{-1}\log^{\frac{1}{\lossPower}-1}(t)\left(1+h(t)\right) \e(-\lossPower\xnT \wtilde)\right.\nonumber\\
	& \left. + \ell'(\wvec^\top (t)\xn)\right]\xnT\rvec(t) -\wtilde^\top  \rvec(t)(g^{1-\lossPower}(t+1) - g^{1-\lossPower}(t)-\left( g^{1-\lossPower}(t)\right)')\nonumber\\
	& \overset{(3)}{\le} \what^\top \rvec(t)\left[\frac{1}{\lossPower}t^{-1}\log^{\frac{1}{\lossPower}-1}(t)\left(1+h(t)\right) - (g(t+1)-g(t))\right]  -\eta\sumnnsv \ell'(\wvec(t)^\top \xn)\xnT\rvec(t)\nonumber\nonumber\\
	&-\eta\sumnsv \left[\gamma_n  t^{-1}\log^{\frac{1}{\lossPower}-2}h_3(t) +\frac{1}{\lossPower}t^{-1}\log^{\frac{1}{\lossPower}-1}(t)\left(1+h(t)\right) \e(-\lossPower\xnT \wtilde)\right.\nonumber\\
	& \left. + \ell'(\wvec^\top (t)\xn)\right]\xnT\rvec(t) + o(t^{-1}\log^{\frac{1}{\lossPower}-2}(t)),
	\end{flalign}
	where in (1) we used eqs. \ref{eq: what def with alpha}, \ref{eq: wtilde def} to obtain $\what = \sumnsv \alpha_n \xn = \sumnsv \exp(-\wtilde^\top \xn) \xn$. In (2) we defined $h_3(t) = o(\log^{\epsilon}(t)), \forall \epsilon>0$ and used eq.~\eqref{eq: g(1-beta) first derivative, beta>1} and the fact that from $\wtilde$ definition (eq.~\eqref{eq: wtilde def}) we can find $ \{\gamma_n\}_{n=1}^d $ such that:
	\begin{equation}
	    \wtilde = -\eta \sumnsv \gamma_n \xn
	\end{equation}
	In (3) we used eq.~\eqref{eq: g(1-beta) second derivative, beta>1} and the fact that $\wtilde^\top \rvec(t)=o(t)$ since
	\begin{align*}
    	&\wtilde^\top\rvec(t) = \wtilde^\top\left(\wvec(0) - \eta \sum_{u=0}^{t-1}\derL(\wvec(u)) - g(t)\what - g^{1-\lossPower}(t) \wtilde \right)\\
    	&\le -\eta t \min_{0\le u \le t} \wtilde^\top\derL(\wvec(u)) + O(g(t)) = o(t),
	\end{align*}
	where in the last line we used that $\derL(\wvec(t))=o(1)$, from Lemma \ref{Lemma: L converge}.\\
	We examine the three terms in eq.~\eqref{(r(t+1)-r(t))r(t) bound}.
	The first term can be upper bounded $ \forall t>t_1'$ by
	\begin{align} \label{eq: first term bound, beta>1}
	&\what^\top \rvec(t)\left[\frac{1}{\lossPower}t^{-1}\log^{\frac{1}{\lossPower}-1}(t)\left(1+h(t)\right) - (g(t+1)-g(t))\right]\nonumber\\
	&\overset{(1)}{\le} \max\left[\what^\top  \mathbf{P}_1 \rvec(t),0\right] C_1 t^{-2}\log^{\frac{1}{\lossPower}-1}(t)\nonumber\\
	& \overset{(2)}{\le} \left\{
	\begin{array}{ll}
	||\what||\epsilon_1 C_3 t^{-2}\log^{\frac{1}{\lossPower}-1}(t)&  ,\text{if $||\mathbf{P}_1\rvec(t)||\le \epsilon_1$}\\
	o(t^{-1}\log^{\frac{1}{\lossPower}-1}(t))&,  \text{if $||\mathbf{P}_1\rvec(t)||> \epsilon_1$}.
	\end{array}
	\right.
	\end{align}
	where in (1) we used the fact that $$\frac{1}{\lossPower}t^{-1}\log^{\frac{1}{\lossPower}-1}(t)\left(1+h(t)\right) - (g(t+1)-g(t))=\Theta(t^{-2}\log^{\frac{1}{\lossPower}-1}(t))$$
	and therefore $\exists t_1',C_3$ such that 
	\[
	\forall t>t_1': \frac{1}{\lossPower}t^{-1}\log^{\frac{1}{\lossPower}-1}(t)\left(1+h(t)\right) - (g(t+1)-g(t))\le C_3 t^{-2}\log^{\frac{1}{\lossPower}-1}(t).
	\]
	In (2) we used $\what^\top \rvec(t)=o(t)$.\\
	Next, we want to upper bound the second term in eq.~\eqref{(r(t+1)-r(t))r(t) bound, beta>1}. If $\xn^\top\rvec(t)\ge0$ then we can show that $\exists t_c$ so that $\forall t>t_c$
    \begin{equation} \label{eq: (w(t)x)^v expansion for beta>1}
        (\wvec(t)^\top  \xn)^\lossPower \ge (\what^\top  \xn)^{\lossPower}g^\lossPower(t)+\lossPower (\what^\top  \xn)^{\lossPower-1}\wtilde^\top  \xn + \xnT\rvec(t).
    \end{equation}
	\begin{proof}
	    \begin{align*}
	    &(\wvec(t)^\top  \xn)^\lossPower \overset{(1)}{=} -\left(g(t)\what^\top  \xn + g^{1-\lossPower}(t)\wtilde^\top  \xn + \xnT\rvec(t) \right)^\lossPower\nonumber\\
	    &=(g(t)\what^\top  \xn)^{\lossPower} 
    	\left(1 + \frac{g^{1-\lossPower}(t)\wtilde^\top  \xn + \xnT\rvec(t)}{g(t)\what^\top  \xn} \right)^\lossPower \nonumber\\
    	&\overset{(2)}{\ge} (g(t)\what^\top  \xn)^{\lossPower} 
    	\left(1 + \lossPower\frac{g^{1-\lossPower}(t)\wtilde^\top  \xn + \xnT\rvec(t)}{g(t)\what^\top  \xn} \right)\nonumber\\
    	&= (\what^\top  \xn)^{\lossPower}g^\lossPower(t)+\lossPower (\what^\top  \xn)^{\lossPower-1}\wtilde^\top  \xn + \lossPower g^{\lossPower-1}(t) (\what^\top  \xn)^{\lossPower-1} \xnT\rvec(t)\\
    	&\overset{(3)}{\ge} (\what^\top  \xn)^{\lossPower}g^\lossPower(t)+\lossPower (\what^\top  \xn)^{\lossPower-1}\wtilde^\top  \xn + \xnT\rvec(t), \forall t>t_c
	\end{align*}
	where in (1) we used $\rvec(t)$ definition (eq.~\eqref{define r(t), beta>1}),
	in (2) we used $\lossPower \ge 1$, Bernoulli's inequality (eq.~\eqref{Bernoulli's inequality,r>=1, x>=-1}) and the fact that $ \exists t'$ such that $\forall t>t'$: 
	$$ g^{-\lossPower}(t) \frac{\wtilde^\top  \xn }{\what^\top  \xn} \ge -1$$
	and therefore:
	$$\frac{g^{-(\lossPower-1)}(t)\wtilde^\top  \xn + \rvec(t)^\top  \xn}{g(t)\what^\top  \xn}
	= \frac{g^{-\lossPower}(t)\wtilde^\top  \xn }{\what^\top  \xn} 
	+ \frac{\rvec(t)^\top  \xn}{g(t)\what^\top  \xn}
	\ge-1.$$
	In (3) we used the fact that 
	$ \exists t'' $ such that $ \forall t>t'': \lossPower g^{\lossPower-1}(t)(\what^\top  \xn)^{\lossPower-1} \ge 1$ (since $\lossPower \ge 1$).
	We define $t_c=\max\left( t',t_B,t''\right)$.
	\end{proof}
	Using this result, we upper bound the second term in eq.~\eqref{(r(t+1)-r(t))r(t) bound, beta>1}, $ \forall t>t_3 $:
	\begin{flalign} \label{2nd term bound, beta>1}
	&-\eta\sumnnsv \ell'(\wvec(t)^\top \xn)\xnT\rvec(t)&\nonumber\\
	&\overset{(1)}{\le} \eta\sumnnsvp (1+\e(-\mu_+(\wvec(t)^\top  \xn)^\lossPower))\e(-(\wvec(t)^\top  \xn)^\lossPower ) \xnT\rvec(t) \nonumber\\
	&\overset{(2)}{\le} \eta\sumnnsvp 2\e(-(\wvec(t)^\top  \xn)^\lossPower ) \xnT\rvec(t) \nonumber\\
	&\overset{(3)}{\le} 2\eta\sumnnsvp \e(-(\what^\top  \xn)^{\lossPower}g^\lossPower(t))
	\e\left(-\lossPower (\what^\top  \xn)^{\lossPower-1}\wtilde^\top  \xn\right) \e(- \xnT\rvec(t)) \xnT\rvec(t) \nonumber\\
	&\overset{(4)}{\le} 2\eta\sumnnsvp \lossPower^{-\theta^\lossPower} t^{-\theta^\lossPower} (\log(t))^{(\frac{1}{\lossPower}-1)\theta^\lossPower} \e\left(-\lossPower (\what^\top  \xn)^{\lossPower-1}\wtilde^\top  \xn\right) \e(- \xnT\rvec(t)) \xnT\rvec(t) \nonumber\\
	&\overset{(5)}{\le} 2\eta N \lossPower^{-\theta^\lossPower} t^{-\theta^\lossPower} (\log(t))^{(\frac{1}{\lossPower}-1)\theta^\lossPower} \e\left(-\lossPower\min_{n\notin \mathcal{S}}\left( (\what^\top  \xn)^{\lossPower-1}\wtilde^\top  \xn\right) \right)\nonumber\\
	&\le \tilde{C}_2 t^{-\theta^\lossPower} (\log(t))^{(\frac{1}{\lossPower}-1)\theta^\lossPower},\ \forall t>t_3,
	\end{flalign}
	where in (1) we used eq.~\eqref{-l' upper bound}, in (2) we used the fact that $\wvec^\top\xn\to\infty$ (from Lemma \ref{Lemma: w(t)->infty}) and thus $\exists t_L$ so that $\forall t>t_L: (1+\e(-\mu_+(\wvec(t)^\top  \xn)^\lossPower))\le2$, 
	in (3) we used eq.~\eqref{eq: (w(t)x)^v expansion for beta>1}, in (4) we used $\theta = \min_{n\notin \mathcal{S}} \what^\top  \xn >1 $ and eq.~\eqref{e^g^beta upper bound}
	and in (5) we used $ \xnT \rvec\ge0 $ and also that $ x \e(-x) \le 1$.
	We define $t_3=\max(\bar{t},t_L,t_c)$ and $$ \tilde{C}_2= 2\eta N \lossPower^{-\theta^\lossPower} \e\left(-\lossPower\min_{n\notin \mathcal{S}}\left( (\what^\top  \xn)^{\lossPower-1}\wtilde^\top  \xn\right) \right)$$
	Lastly, we will aim to bound the third term in eq.~\eqref{(r(t+1)-r(t))r(t) bound, beta>1}
	\begin{equation} \label{3rd term bound, beta>1}
	\eta \sumnsv \left[-\gamma_n t^{-1}\log^{\frac{1}{\lossPower}-2}h_3(t) -\frac{1}{\lossPower}t^{-1}\log^{\frac{1}{\lossPower}-1}(t)\left(1+h(t)\right) \e(-\lossPower\xnT \wtilde) - \ell'(\wvec^\top (t)\xn)\right]\xnT\rvec(t)
	\end{equation}
	We examine each term $k$ in this sum, and divide into two cases, depending on the sign of $\xkT \rvec (t)$.\\
	First, if $\xkT \rvec(t) \ge 0$ then term $k$ in eq.~\eqref{3rd term bound, beta>1} can be upper bounded $\forall t>\bar{t}$ using eq.~\eqref{-l' upper bound} by
	\begin{flalign} \label{3rd term: xkr>0}
	&\Big[\left(1+\left[\lossPower^{-1} t^{-1} (\log(t))^{\left(\frac{1}{\lossPower}-1\right)}\right]^{\mu_+} \e\left(-\mu_+ \lossPower \wtilde^\top\xk \right)\right) \e\left( - \lossPower g^{\lossPower-1}(t) \xkT\rvec(t) \right)  &\nonumber\\
	&-(1+h(t)+f(t))\Big]\left(\eta \frac{1}{\lossPower}t^{-1}\log^{\frac{1}{\lossPower}-1}(t) \e(-\lossPower\xkT \wtilde)\right)\xkT r(t),
	\end{flalign}
	where we defined $ f(t)=\gamma_k \lossPower \e(\lossPower\xkT \wtilde) \log^{-1}(t)h_3(t)$.\\
	We further divide into cases:\\
	1. If $|\xkT r|\le C_0 \log^{-\frac{1}{2}-\frac{1}{2\lossPower}}(t)$ then we can upper bound eq.~\eqref{3rd term: xkr>0} with
	\begin{equation}
	2 \eta C_0  \frac{1}{\lossPower}t^{-1}\log^{-1-\frac{1}{2}\left(1-\frac{1}{\lossPower}\right)}(t) \e(-\lossPower\xkT \wtilde)
	\end{equation}
	2. If $|\xkT r|> C_0 \log^{-\frac{1}{2}-\frac{1}{2\lossPower}}(t)$ then we can upper bound eq.~\eqref{3rd term: xkr>0} with zero since:
	\begin{flalign}
	&\Big[\left(1+\left[\lossPower^{-1} t^{-1} (\log(t))^{\left(\frac{1}{\lossPower}-1\right)}\right]^{\mu_+} \e\left(-\mu_+ \lossPower \wtilde^\top\xk \right)\right) \e\left( - C_0 \log^{-\frac{1}{2}-\frac{1}{2\lossPower}}(t) \right)-(1+h(t)+f(t))\Big]\nonumber\\
	& \overset{(1)}{\le} \Big[\left(1+\left[\lossPower^{-1} t^{-1} (\log(t))^{\left(\frac{1}{\lossPower}-1\right)}\right]^{\mu_+} \e\left(-\mu_+ \lossPower \wtilde^\top\xk \right)\right) \left( 1- C_0 \log^{-\frac{1}{2}-\frac{1}{2\lossPower}}(t) +  C_0^2 \log^{-1-\frac{1}{\lossPower}}(t) \right)\nonumber\\
    &-(1+h(t)+f(t))\Big]\nonumber\\
	& \le \left[\left( 1- C_0 \log^{-\frac{1}{2}-\frac{1}{2\lossPower}}(t) +  C_0^2 \log^{-1-\frac{1}{\lossPower}}(t) \right)\left[\lossPower^{-1} t^{-1} (\log(t))^{\left(\frac{1}{\lossPower}-1\right)}\right]^{\mu_+} \e\left(-\mu_+ \lossPower \wtilde^\top\xk\right)\right.\nonumber\\
	& \left. - C_0 \log^{-\frac{1}{2}-\frac{1}{2\lossPower}}(t) +  C_0^2 \log^{-1-\frac{1}{\lossPower}}(t)-h(t)-f(t)\right] \overset{(2)}{\le} 0,\ \forall t>t_4,
	\end{flalign}
	where in (1) we used the fact that $e^{-x}\le1-x+x^2$ for $x>0$ and in (2) we defined $t_4$ so that $\forall t>t_4>\bar{t}$ the previous expression is negative - this is possible because $\log^{-\frac{1}{2}-\frac{1}{2\lossPower}}(t)$ decreases slower than $h(t)$ and $f(t)$ ($\forall \lossPower>1:\ -\frac{1}{2}-\frac{1}{2\lossPower}>-1$).\\
	3. If, in addition, $ |\xkT\rvec| \ge \epsilon_2 $, then we can find $C_4$ so that we can upper bound eq.~\eqref{3rd term: xkr>0} with
	\begin{equation} \label{eq: positive case c}
	-C_4 \eta \frac{1}{\lossPower}t^{-1}\log^{\frac{1}{\lossPower}-1}(t) \e(-\lossPower\max_{n}(\xnT \wtilde))\epsilon_2,\ \forall t>\bar{t}
	\end{equation}
	Second, if $\xkT \rvec(t) < 0$, we again further divide into cases:\\
	1. If $ |\xkT\rvec| \le C_0\log^{-\frac{1}{2}-\frac{1}{2\lossPower}} $, then, since $ -\ell'(\wvec(t)^\top \xk)>0 $, we can upper bound term $k$ in equation \ref{3rd term bound, beta>1} by
	\begin{equation*}
	C_0\eta \frac{1}{\lossPower}t^{-1}\log^{-1-\frac{1}{2}\left(1-\frac{1}{\lossPower}\right)}(t) \e(-\lossPower\min_n(\xnT \wtilde))
	\end{equation*}
	2. If $ |\xkT\rvec| \ge C_0\log^{-\frac{1}{2}-\frac{1}{2\lossPower}} $, then, using eq.~\eqref{-l' lower bound} we can upper bound term $k$ in equation \ref{3rd term bound, beta>1}, $\forall t>t_5$ by
	\begin{flalign} \label{3rd term, xr<0 not bounded, beta>1}
	& \eta \left[\frac{1}{\lossPower}t^{-1}\log^{\frac{1}{\lossPower}-1}(t) \e(-\lossPower\xkT \wtilde) -\left(1-\exp\left(-\mu_-\left(\wvec(t)^\top  \xk \right)^\lossPower\right)\right)\e\left(-\left(\wvec(t)^\top  \xk \right)^\lossPower \right) \right]|\xkT\rvec(t)|
	\end{flalign}
	We used the fact that $f(t)=o(h(t))$ and therefore $\exists t_5>\bar{t}$ such that $\forall t>t_5:\ h(t)+f(t)\le0$.\\
	We can use Taylor's theorem to show that:
	\begin{flalign} \label{binomial expansion for ()^beta, beta>1}
	&\left(g(t) + g^{1-\lossPower}(t)\wtilde^\top  \xn \right)^\lossPower= \log(t) + \log(\lossPower \log^{1-\frac{1}{\lossPower}}(t)) + \lossPower\wtilde^\top  \xk+ f_2(t),
	\end{flalign}
	where $|f_2(t)|=o(\log^{-\frac{1}{\lossPower}}(t))$.\\
	From Lemma \ref{Lemma: w(t)->infty} we know that $ \wvec(t)^\top \xn \to \infty $ and therefore $\exists t_6>t_5$ so that $\forall t>t_6$:
	$$\left(1-\exp\left(-\mu_-\left(\wvec(t)^\top  \xk \right)^\lossPower\right)\right)\ge0$$
	$$ \wvec(t)^\top \xn = g(t)\what^\top  \xn + g^{1-\lossPower}(t)\wtilde^\top  \xn + \rvec(t)^\top  \xn >0  $$
	Using the last equations and the fact that  $\forall a>1,\ x\in[-1,0]:\ (1+x)^a\le(1+x) $ and\\ $ -1\le\frac{\rvec(t)^\top  \xk}{g(t) + g^{1-\lossPower}(t)\wtilde^\top  \xk}\le0 $, we can show that $\forall t>t_6$:
	\begin{flalign} \label{g+...+xr bound, xr<0, beta>1}
		&\left(g(t) + g^{1-\lossPower}(t)\wtilde^\top  \xk +\rvec(t)^\top  \xk \right)^\lossPower &\nonumber\\
		& \le \left(g(t) + g^{1-\lossPower}(t)\wtilde^\top  \xk \right)^\lossPower +\left(g(t) + g^{1-\lossPower}(t)\wtilde^\top  \xk \right)^{\lossPower-1}\rvec(t)^\top  \xk\nonumber\\
		& \le \log(t) + \log(\lossPower \log^{1-\frac{1}{\lossPower}}(t)) + \lossPower\wtilde^\top  \xk+ f_2(t) +\left(g(t) + g^{1-\lossPower}(t)\wtilde^\top  \xk \right)^{\lossPower-1}\rvec(t)^\top  \xk\nonumber\\
		& \le \log(t) + \log(\lossPower \log^{1-\frac{1}{\lossPower}}(t)) + \lossPower\wtilde^\top  \xk+ \log^{-\frac{1}{\lossPower}}(t) +\left(g(t) + g^{1-\lossPower}(t)\wtilde^\top  \xk \right)^{\lossPower-1}\rvec(t)^\top  \xk\nonumber\\
	\end{flalign}
	 Using eq.~\eqref{g+...+xr bound, xr<0, beta>1}, eq.~\eqref{3rd term, xr<0 not bounded, beta>1} can be upper bounded by
	 \begin{flalign*}
	 & \eta \frac{1}{\lossPower}t^{-1}\log^{\frac{1}{\lossPower}-1}(t) \e(-\lossPower\xnT \wtilde) \Big[1 -\left(1-e^{-\mu_-\left(\wvec(t)^\top  \xk \right)^\lossPower}\right) \\
     & \e\left(-\left(g(t) + g^{1-\lossPower}(t)\wtilde^\top  \xk \right)^{\lossPower-1}\rvec(t)^\top  \xk - \log^{-\frac{1}{\lossPower}}(t)\right) \Big]|\xkT\rvec(t)|
	 \end{flalign*}
	Next, we will show that $\exists t''>t_6$ such that the last expression is strictly negative $\forall t>t''$. Let $M>1$ be some arbitrary constant. Then, since $\exp\left(-\mu_-\left(\wvec(t)^\top  \xk \right)^\lossPower\right)\to0$ from Lemma \ref{Lemma: w(t)->infty}, $\exists t_M>\bar{t}$ such that $\forall t>t_M$ and if $\e(-\rvec(t)^\top\xk)\ge M>1$ then
	\begin{equation} \label{3rd term []>1, first case beta>1}
	\left(1-\exp\left(-\mu_-\left(\wvec(t)^\top  \xk \right)^\lossPower\right)\right)\e\left(\left(g(t) + g^{1-\lossPower}(t)\wtilde^\top  \xk \right)^{\lossPower-1}|\rvec(t)^\top  \xk| - \log^{-\frac{1}{\lossPower}}(t)\right)\ge M'>1
	\end{equation}
	Furthermore, if $\exists t>t_M$ such that $\e(-\rvec(t)^\top\xk)< M$, then
	\begin{flalign} \label{3rd term []>1, sec case beta>1}
	&\left(1-\exp\left(-\mu_-\left(\wvec(t)^\top  \xk \right)^\lossPower\right)\right)\e\left(\left(g(t) + g^{1-\lossPower}(t)\wtilde^\top  \xk \right)^{\lossPower-1}|\rvec(t)^\top  \xk| - \log^{-\frac{1}{\lossPower}}(t)\right)&\nonumber\\
	&\overset{(1)}{\ge} \left(1-\left[t^{-1}\lossPower^{-1}\log^{\frac{1}{\lossPower}-1}(t)e^{-\lossPower\wtilde^\top\xk}M\right]^{\mu_-}\right)\e\left(\left(g(t) + g^{1-\lossPower}(t)\wtilde^\top  \xk \right)^{\lossPower-1}|\rvec(t)^\top  \xk| - \log^{-\frac{1}{\lossPower}}(t)\right)\nonumber\\
	&\overset{(2)}{\ge} \left(1-\left[t^{-1}\lossPower^{-1}\log^{\frac{1}{\lossPower}-1}(t)e^{-\lossPower\wtilde^\top\xk}M\right]^{\mu_-}\right)\e\left(\log^{-1+\frac{5}{4}\left(1-\frac{1}{\lossPower}\right)}(t)- \log^{-\frac{1}{\lossPower}}(t)\right)\nonumber\\
	&\overset{(3)}{\ge} \left(1-\left[t^{-1}\lossPower^{-1}\log^{\frac{1}{\lossPower}-1}(t)e^{-\lossPower\wtilde^\top\xk}M\right]^{\mu_-}\right)\left(1+\log^{-1+\frac{5}{4}\left(1-\frac{1}{\lossPower}\right)}(t)- \log^{-\frac{1}{\lossPower}}(t)\right)\nonumber\\
	&\ge 1+\log^{-1+\frac{5}{4}\left(1-\frac{1}{\lossPower}\right)}(t)- \log^{-\frac{1}{\lossPower}}(t)\nonumber\\
	&- \left[t^{-1}\lossPower^{-1}\log^{\frac{1}{\lossPower}-1}(t)e^{-\lossPower\wtilde^\top\xk}M\right]^{\mu_-}\left(1+\log^{-1+\frac{5}{4}\left(1-\frac{1}{\lossPower}\right)}(t)- \log^{-\frac{1}{\lossPower}}(t)\right)\nonumber\\
	&>1,
	\end{flalign}
	where in transition (1) we used eq.~\eqref{Bernoulli's inequality,r>=1, x>=-1} and in transition (2) we used:
	\begin{flalign*}
	&\left(g(t) + g^{1-\lossPower}(t)\wtilde^\top  \xk \right)^{\lossPower-1}|\rvec(t)^\top  \xk| - \log^{-\frac{1}{\lossPower}}(t)\ge \left(g(t) + g^{1-\lossPower}(t)\wtilde^\top  \xk \right)^{\lossPower-1}C_0\log^{-\frac{1}{2}-\frac{1}{2\lossPower}} - \log^{-\frac{1}{\lossPower}}(t)\\
	&\ge \log^{-1+\frac{5}{4}\left(1-\frac{1}{\lossPower}\right)}(t)- \log^{-\frac{1}{\lossPower}}(t)
	\end{flalign*}
	and in (3) we used $e^x\ge1+x$.
	eq.~\eqref{3rd term []>1, sec case beta>1} is greater than 1 since $\log^{-1+\frac{5}{4}\left(1-\frac{1}{\lossPower}\right)}(t)$ decrease slower than the other terms.
	Therefore, after we substitute eqs. \ref{3rd term []>1, first case beta>1} and \ref{3rd term []>1, sec case beta>1} into eq.~\eqref{3rd term, xr<0 not bounded, beta>1}, we find that $\exists t'_->t''$ such that $\forall t>t'_-$ term  $k$ in equation \ref{3rd term bound} is strictly negative.\\
	3. If $ |\xkT\rvec| \ge \epsilon_2 $, then $\exists t_7,\epsilon_3$ such that $\forall t>t_7$
	\begin{flalign}
	&\left(g(t) + g^{1-\lossPower}(t)\wtilde^\top  \xk \right)^{\lossPower-1}|\rvec(t)^\top  \xk| - \log^{-\frac{1}{\lossPower}}(t)>\epsilon_3
	\end{flalign}
	and we can find $ C_5 $ such that we can upper bound term $k$ in equation \ref{3rd term bound, beta>1} $\forall t>t_7$ by
	\begin{equation} \label{eq: negative case c}
	-\eta C_6 \frac{1}{\lossPower}t^{-1}\log^{\frac{1}{\lossPower}-1}(t) \e(-\lossPower\max_n(\xnT \wtilde)) \epsilon_2
	\end{equation}
	To conclude, we choose $t_{0}=\max\left[t_1',t_3,t_4,t_6,t_{-}',t_7\right]$: 
	\begin{enumerate}
		\item If $\left\Vert \mathbf{P}_{1}\mathbf{r}\left(t\right)\right\Vert \geq\epsilon_{1}$
		(as in eq.~\eqref{eq: first term bound, beta>1}), we have that 
		\begin{equation}
		\max_{n\in\set}\left|\mathbf{x}_{n}^{\top}\mathbf{r}\left(t\right)\right|^{2}\overset{\left(1\right)}{\geq}\frac{1}{\left|\set\right|}\sum_{n\in\set}\left|\mathbf{x}_{n}^{\top}\mathbf{P}_{1}\mathbf{r}\left(t\right)\right|^{2}=\frac{1}{\left|\set\right|}\left\Vert \mathbf{X}_{\set}^{\top}\mathbf{P}_{1}\mathbf{r}\left(t\right)\right\Vert ^{2}\overset{\left(2\right)}{\geq}\frac{1}{\left|\set\right|}\sigma_{\min}^{2}\left(\mathbf{X}_{\set}\right)\epsilon_{1}^{2}\label{eq: epsilon 2},
		\end{equation}
		where in $\left(1\right)$ we used $\mathbf{P}_{1}^{\top}\mathbf{x}_{n}=\mathbf{x}_{n}$
		$\forall n\in\set$, in $\left(2\right)$ we denoted $\mathbf{X}_{\set}\in \mathbb{R}^{d\times |\mathcal{S}|}$ as the matrix whose columns are the support vectors and by $\sigma_{\min}\left(\mathbf{X}_{\set}\right)$,
		the minimal non-zero singular value of $\mathbf{X}_{\set}$ and used
		eq.  \ref{P1r(t)>epsilon, beta>1}. Therefore, for some $k$, $\left|\mathbf{x}_{k}^{\top}\mathbf{r}\right|\geq\epsilon_{2}\triangleq\left|\set\right|^{-1}\sigma_{\min}^{2}\left(\mathbf{X}_{\set}\right)\epsilon_{1}^{2}$.
		In this case, we denote $C_{0}^{\prime\prime}$ as the minimum between
		$C_{6}\eta\lossPower^{-1}\exp\left(-\lossPower\max_{n}\tilde{\mathbf{w}}^{\top}\mathbf{x}_{n}\right)\epsilon_{2}$ (eq.~\eqref{eq: negative case c}) and $C_{4}\eta\lossPower^{-1}\exp\left(-\lossPower\max_{n}\tilde{\mathbf{w}}^{\top}\mathbf{x}_{n}\right)\epsilon_{2}$
		(eq.~\eqref{eq: positive case c}). Then we find that eq.~\eqref{3rd term bound, beta>1}
		can be upper bounded by $-C_{0}^{\prime\prime}t^{-1}\log^{\frac{1}{\lossPower}-1}(t)+o\left(t^{-1}\log^{\frac{1}{\lossPower}-1}(t)\right)$,
		$\forall t>t_{0}$, given eq.~\eqref{P1r(t)>epsilon, beta>1}. Substituting
		this result, together with eqs. \ref{eq: first term bound, beta>1} and \ref{2nd term bound, beta>1}
		into eq.~\eqref{(r(t+1)-r(t))r(t) bound, beta>1}, we obtain $\forall t>t_{0}$ 
		\[
		\left(\mathbf{r}\left(t+1\right)-\mathbf{r}\left(t\right)\right)^{\top}\mathbf{r}\left(t\right)\leq-C_{0}^{\prime\prime}t^{-1}\log^{\frac{1}{\lossPower}-1}(t)+o\left(t^{-1}\log^{\frac{1}{\lossPower}-1}(t)\right)\,.
		\]
		This implies that $\exists C_{2}<C_{0}^{\prime\prime}$ and $\exists t_{2}>t_{0}$
		such that eq.~\eqref{eq: (r(t+1)-r(t)r(t) improved bound for beta>1} holds. This implies also that
		eq.~\eqref{eq: (r(t+1)-r(t)r(t) bound for beta>1} holds for $\left\Vert \mathbf{P}_{1}\mathbf{r}\left(t\right)\right\Vert \geq\epsilon_{1}$. 
		\item Otherwise, if $\left\Vert \mathbf{P}_{1}\mathbf{r}\left(t\right)\right\Vert <\epsilon_{1}$,
		we find that $\forall t>t_{0}$ , each term in eq.~\eqref{3rd term bound, beta>1}
		can be upper bounded by either zero, or terms proportional to $t^{-1}\log^{-1-\frac{1}{2}\left(1-\frac{1}{\lossPower}\right)}$. Combining this together with eqs. \ref{eq: first term bound, beta>1},
		\ref{2nd term bound, beta>1} into eq.~\eqref{(r(t+1)-r(t))r(t) bound, beta>1}
		we obtain (for some positive constants $C_{7}$, $C_{8}$, $C_{9}$,
		and $C_{6}$) 
		\[
		\left(\rvec(t+1)-\rvec(t)\right)^\top\rvec(t)\le C_7 t^{-2}\log^{\frac{1}{\lossPower}-1}(t) + C_8  t^{-\theta^\lossPower} (\log(t))^{(\frac{1}{\lossPower}-1)\theta^\lossPower} + C_9 t^{-1}\log^{-1-\frac{1}{2}\left(1-\frac{1}{\lossPower}\right)}(t)
		\]
		Therefore, $\exists t_{1}>t_{0}$ and $C_{1}$ such that eq.~\eqref{eq: (r(t+1)-r(t)r(t) bound for beta>1}
		holds.
	\end{enumerate}
	\newpage
	\subsection{$\frac{1}{4}<\lossPower<1$} \label{sec: defenitions for beta<1}
	In the following proofs, for any solution $\wvec(t)$, we define 
	\begin{equation} \label{define r(t), beta<1}
	\vect{r}(t) = \vect{w}(t) - g(t)\what - g^{1-\lossPower}(t) \wtilde + \frac{1}{\lossPower}g^{1-\lossPower}(t)\htilde(t)\what + \frac{1}{\lossPower}g^{1-\lossPower}(t)\log^{-1}(t)\wcheck+ \frac{1}{\lossPower}g^{1-2\lossPower}(t)\wcheckSec,
	\end{equation}
	where $\htilde(t)$ is defined below (subsection \ref{sec: auxRes beta smaller than 1}), $\what$ follows the conditions of Theorem \ref{theorem: main}, that is $\what$ is the $L_2$ max margin vector, which satisfies eq.~\eqref{w_hat equation}:
	\begin{equation*}
	\what = \argmin_{\wvec(t)\in \mathbb{R}^d} \wNorm^2 \text{ s.t. } \wvec(t)^\top\xn\ge1
	\end{equation*}
	and $\wtilde$ is a vector which satisfies the equations:
	\begin{flalign}
	    &\forall n\in \mathcal{S}\ :\ \eta\e(-\lossPower\xnT\wtilde)=\alpha_n,\ \bar{\mathbf{P}}_1 \wtilde = 0
	\end{flalign}
	From $\wtilde$ definition, we know that $\exists \{\lambda_n\}_{n=1}^d$ such that
	\begin{equation} \label{eq: wtilde lc of sv, beta smaller than 1}
	    \forall n\in \set\ : \ \wtilde=\sumnsv \lambda_n \xn
	\end{equation}
	Using the last equation, we define $\wcheck$, a vector which satisfies the equations:
	\begin{flalign} \label{eq: wcheck def}
	    &\forall n\in \set\ : \ \wcheck^\top \xn=\frac{\lossPower\bar{C_1}\lambda_n}{\eta}\exp(\lossPower\xnT\wtilde) \triangleq \tilde{\gamma}_n,\ \bar{\mathbf{P}}_1 \wcheck = 0,
	\end{flalign}
	where $\bar{C_1} = \frac{1-\lossPower}{\lossPower}$.\\
	$\wcheckSec$ is a vector which satisfies the equations:
	\begin{flalign} \label{eq: wcheckSec def}
	    &\forall n\in \set\ : \ \wcheckSec^\top\xn=\frac{\lossPower(\lossPower-1)}{2}(\wtilde^\top\xn)^2,\ \bar{\mathbf{P}}_1 \wcheckSec = 0
	\end{flalign}
	These equations have a unique solution for almost every dataset from Lemma 8 in \cite{soudry2017implicit}. 
	\subsubsection{Auxiliary results} \label{sec: auxRes beta smaller than 1}
	We will denote the functions:
	\begin{flalign} \label{eq: beta smaller than 1 functions def}
		& g(t) = \log^{\frac{1}{\lossPower}}(t) + \frac{1}{\lossPower} \log(\lossPower\log^{1-\frac{1}{\lossPower}}(t))\log^{\frac{1}{\lossPower}-1}(t)\nonumber\\
		& h(t) = \left(1-\frac{1}{\lossPower}\right)\log^{-1}(t)\left(1-\log\left(\lossPower \log^{1-\frac{1}{\lossPower}}(t)\right)\right)\nonumber\\
		& h_2(t) =\frac{1}{t} \left[\left(-1+\left(\frac{1}{\lossPower}-1\right)\log^{-1}(t)\right)\left(1+h(t)\right)
		+\left(1-\frac{1}{\lossPower}\right)\log^{-2}(t)\left(\log\left(\lossPower\log^{1-\frac{1}{\lossPower}}(t)\right)-\frac{1}{\lossPower}\right)\right]\nonumber\\
		& \htilde(t) = h(t) + \frac{(\lossPower-1)}{2\lossPower} \log^{-1}(t)\log^2(\lossPower\log^{1-\frac{1}{\lossPower}}(t))
	\end{flalign}
	Using basic analysis, it is straightforward to show that these functions have the following properties:\\
	\begin{equation} \label{diff g Theta(t^-1), beta<1}
		g(t+1) - g(t) = \Theta(t^{-1}\log^{\frac{1}{\lossPower}-1}(t))
	\end{equation}
	\begin{equation}\label{diff g-beta Theta(t^-1), beta<1}
		g^{1-\lossPower}(t+1) - g^{1-\lossPower}(t) = \log^{\frac{1}{\lossPower}-1}(t+1) - \log^{\frac{1}{\lossPower}-1}(t) = \Theta(t^{-1}\log^{\frac{1}{\lossPower}-2}(t))
	\end{equation}
	\begin{equation} 
	\left\vert \frac{1}{\lossPower t} \log^{\frac{1}{\lossPower}-1}(t)\left[1+h(t)+\frac{1}{2}h_2(t)\right] - (g(t+1)-g(t)) \right\vert = \Theta(t^{-3}\log^{\frac{1}{\lossPower}-1}(t))
	\end{equation}
	\begin{equation} \label{gh=o(...)}
		\left\vert g^{1-\lossPower}(t)h(t) - g^{1-\lossPower}(t+1)h(t+1)\right\vert = o(t^{-1}\log^{\frac{1}{\lossPower}-2})
	\end{equation}
	We denote $\bar{C_1} = \frac{1-\lossPower}{\lossPower}$. $\exists m_1(t)=o(\log^{\epsilon}(t)),m_2(t)=o(\log^{\epsilon}(t))$ such that $\forall \epsilon>0$:
	\begin{equation} \label{eq: g(1-beta) first derivative}
	    \left( g^{1-\lossPower}(t)\right)' = (1-\lossPower)g^{-\lossPower}(t)g'(t) = \bar{C}_1\frac{1}{t} \log^{\frac{1}{\lossPower}-2}(t) + \frac{1}{t} \log^{\frac{1}{\lossPower}-3}(t)m_1(t)
	\end{equation}
	\begin{equation}
	    \left( g^{1-\lossPower}(t)\right)'' = \frac{1}{t^2}\log^{\frac{1}{\lossPower}-2}m_2(t)
	\end{equation}
	\begin{equation}
	    \left( g^{1-\lossPower}(t)\right)''' = o(t^{-2-\epsilon})
	\end{equation}
	There exists $ m_3(t)=o(\log^{\epsilon}(t)),m_4(t)=o(\log^{\epsilon}(t))$ such that $\forall \epsilon>0$:
	\begin{align} 
	    (g^{1-\lossPower}(t)\htilde(t))'&= \frac{1}{t} \log^{\frac{1}{\lossPower}-3}(t)m_3(t) \label{eq: (g^(1-lossPower)htilde)'} \\
	    (g^{1-\lossPower}(t)\htilde(t))''&= \frac{1}{t^2} \log^{\frac{1}{\lossPower}-3}(t)m_4(t)\\
	    (g^{1-\lossPower}(t)\htilde(t))'''&= o(t^{-2-\epsilon})
	\end{align}
	There exists $m_5(t)=o(\log^{\epsilon}(t)),m_6(t)=o(\log^{\epsilon}(t))$ such that $\forall \epsilon>0$:
	\begin{align} 
	    (g^{1-\lossPower}(t)\log^{-1}(t))'&= \frac{1}{t} \log^{\frac{1}{\lossPower}-3}(t)m_5(t) \label{eq: g^(1-v)/log = o(...)}\\
	    (g^{1-\lossPower}(t)\log^{-1}(t))''&= \frac{1}{t^2} \log^{\frac{1}{\lossPower}-3}(t)m_6(t)\\
	    (g^{1-\lossPower}(t)\log^{-1}(t))'''&= o(t^{-2-\epsilon})
	\end{align}
	There exists $\tilde{C}_1$ and $m_7(t)=o(\log^{\epsilon}(t)),m_8(t)=o(\log^{\epsilon}(t))$ such that $\forall \epsilon>0$:
	\begin{align} 
	    \left( g^{1-2\lossPower}(t)\right)' &= \frac{1}{t} \log^{\frac{1}{\lossPower}-3}(t)m_7(t) \label{eq: ( g^(1-2lossPower))'}\\
	    \left( g^{1-2\lossPower}(t)\right)'' &= \frac{1}{t^2}\log^{\frac{1}{\lossPower}-3}m_8(t)\\
	    \left( g^{1-2\lossPower}(t)\right)''' &= o(t^{-2-\epsilon})  \label{eq: g(1-2beta) third derivative}
	\end{align}
	Combining these properties, for arbitrary constants $\alpha_1,...,\alpha_4$ we get:
	\begin{flalign}
	    &\alpha_1\left(\left( g^{1-\lossPower}(t)\right)'+\frac{1}{2}\left( g^{1-\lossPower}(t)\right)'' \right)
	    +\alpha_2\left((g^{1-\lossPower}(t)\htilde(t))'+\frac{1}{2}(g^{1-\lossPower}(t)\htilde(t))'' \right)&\nonumber\\
	    &+\alpha_3\left((g^{1-\lossPower}(t)\log^{-1}(t))'+\frac{1}{2}g^{1-\lossPower}(t)\log^{-1}(t))'' \right)
	     +\alpha_4\left(\left( g^{1-2\lossPower}(t)\right)'+\frac{1}{2}\left( g^{1-2\lossPower}(t)\right)'' \right)\nonumber\\
	     & = \alpha_1\bar{C}_1 \frac{1}{t} \log^{\frac{1}{\lossPower}-2}(t) + \frac{1}{t}\log^{\frac{1}{\lossPower}-3}\tilde{m}_1(t) + \frac{1}{t^2}\log^{\frac{1}{\lossPower}-2}\tilde{m}_2(t),
	\end{flalign}
	where $\forall \epsilon>0:$ $ \tilde{m}_1(t)=o(\log^{\epsilon}(t)),\tilde{m}_2(t)=o(\log^{\epsilon}(t))$.\\
	We can use Taylor's theorem to show that:
	\begin{flalign*}
	&g^\lossPower(t) = \log(t) + \log(\lossPower \log^{1-\frac{1}{\lossPower}}(t)) +
	\frac{(\lossPower-1)}{2\lossPower} \log^{-1}(t)\log^2(\lossPower\log^{1-\frac{1}{\lossPower}}(t))\\
	&+\frac{\left(\lossPower-1\right)\left(\lossPower-2\right)}{6\lossPower^2}\log^{-2}(t)\log^3\left(\lossPower\log^{1-\frac{1}{\lossPower}}(t)\right) + o\left(\log^{-2}(t)\log^3\left(\lossPower\log^{1-\frac{1}{\lossPower}}(t)\right)\right)
	\end{flalign*}
	and also that

	\begin{flalign} \label{binomial expansion for (g+g(1-beta))^beta, beta<1}
	 &\Bigg(g(t)\what^\top  \xn + g^{1-\lossPower}(t)\underbrace{\left(\wtilde^\top  \xn -\frac{1}{\lossPower}\htilde(t)\what^\top\xn - \frac{1}{\lossPower}\log^{-1}(t)\wcheck^\top\xn
	- \frac{1}{\lossPower}g^{-\lossPower}(t)\wcheckSec^\top\xn
	\right)}_{\displaystyle \triangleq \ftilde(t)}\Bigg)^\lossPower \nonumber\\
	& = g^\lossPower(t)(\what^\top  \xn)^\lossPower + \lossPower(\what^\top  \xn)^{\lossPower-1} \ftilde(t)+ \frac{\lossPower(\lossPower-1)}{2} g^{-\lossPower}(t)\left(\ftilde(t)\right)^2 (\what^\top  \xn)^{\lossPower-2} \nonumber \\
	& + \frac{\lossPower\left(\lossPower-1\right)\left(\lossPower-2\right)}{6}g^{-2\lossPower}(t)\left(\ftilde(t)\right)^3 (\what^\top  \xn)^{\lossPower-3}+ o(g^{-2\lossPower}(t))\nonumber\\	
	& = (\what^\top  \xn)^\lossPower\log(t) + (\what^\top  \xn)^\lossPower \log(\lossPower \log^{1-\frac{1}{\lossPower}}(t))+ \lossPower(\what^\top  \xn)^{\lossPower-1} \ftilde(t)+ \ftilde_2(t)+ f_2(t),
	\end{flalign}
	where we defined
	\begin{align*}
	&\ftilde(t) \triangleq \wtilde^\top  \xn -\frac{1}{\lossPower}\htilde(t)\what^\top\xn - \frac{1}{\lossPower}\log^{-1}(t)\wcheck^\top\xn
	- \frac{1}{\lossPower}g^{-\lossPower}(t)\wcheckSec^\top\xn,\\
	&\ftilde_2(t) \triangleq \frac{\lossPower(\lossPower-1)}{2} g^{-\lossPower}(t)\left(\wtilde^\top  \xn \right)^2 (\what^\top  \xn)^{\lossPower-2}+\frac{(\lossPower-1)}{2\lossPower} (\what^\top  \xn)^\lossPower \log^{-1}(t)\log^2(\lossPower\log^{1-\frac{1}{\lossPower}}(t)).
	\end{align*}
	We note that $|f_2(t)|=o(\log^{-2+\epsilon}(t)),\ \forall \epsilon>0$ .
	\newpage
	
	\subsubsection{Proof Of main Theorem ($\frac{1}{4}<\lossPower<1$)}
	Our goal is to show that $\Vert \rvec(t) \Vert=o(g^{1-\lossPower}(t))$ , and therefore\\ $\rhoVec = \rvec(t)+g^{1-\lossPower}(t)\wtilde+o(g^{1-\lossPower}(t))=O(g^{1-\lossPower}(t))$. To show this, we will upper bound the following equation
	\begin{equation} \label{norm r(t+1), beta<1}
	\Vert\rvec(t+1)\Vert^2 = \Vert\rvec(t+1)-\rvec(t)\Vert^2 + 2\left(\rvec(t+1)-\rvec(t)\right)^\top\rvec(t)+\Vert\rvec(t)\Vert^2
	\end{equation}
	First, we note that the first term in eq.~\eqref{norm r(t+1), beta<1} can be upper bounded by
	\begin{flalign} \label{norm(r(t+1)-r(t), beta<1}
	&\left\Vert \vect{r}(t+1) - \vect{r}(t) \right\Vert^2&\nonumber\\
	&\overset{(1)}{=} \left\Vert\wvec(t+1) - \left(g(t+1)-\frac{1}{\lossPower}g^{1-\lossPower}(t+1)\htilde(t+1)\right)\what - g^{1-\lossPower}(t+1) \wtilde +\frac{1}{\lossPower}g^{1-\lossPower}(t+1)\log^{-1}(t+1)\wcheck \right.\nonumber\\
	& \left. + \frac{1}{\lossPower}g^{1-2\lossPower}(t+1)\wcheckSec - \wvec(t) + \left(g(t)-\frac{1}{\lossPower}g^{1-\lossPower}(t)\htilde(t)\right)\what + g^{1-\lossPower}(t) \wtilde-\frac{1}{\lossPower}g^{1-\lossPower}(t)\log^{-1}(t)\wcheck- \frac{1}{\lossPower}g^{1-2\lossPower}(t)\wcheckSec\right\Vert^2 \nonumber\\
	& \overset{(2)}{=} \left\Vert-\eta\nabla L(\wvec(t)) - \what\left(g(t+1)-g(t)\right) +\bm{\psi}_1(t) \right\Vert^2\nonumber\\
	& \overset{(3)}{=} \eta^2\left\Vert \nabla L(\wvec(t))\right\Vert^2 +2\eta\left(g(t+1)-g(t)\right) \what^\top  \nabla L(\wvec(t)) +o\left(t^{-1}\log^{\frac{1}{\lossPower}-2}(t)\right) \nonumber\\
	&\overset{(4)}{\le} \eta^2\left\Vert\nabla L(\wvec(t))\right\Vert^2 +C t^{-1}\log^{\frac{1}{\lossPower}-2}(t),
	\end{flalign}
	where in (1) we used  $$\vect{r}(t) = \vect{w}(t) - g(t)\what - g^{1-\lossPower}(t) \wtilde + \frac{1}{\lossPower}g^{1-\lossPower}(t)\htilde(t)\what + \frac{1}{\lossPower}g^{1-\lossPower}(t)\log^{-1}(t)\wcheck+ \frac{1}{\lossPower}g^{1-2\lossPower}(t)\wcheckSec$$ (eq.~\eqref{define r(t), beta<1}), in (2) we used $\wvec(t+1) = \wvec(t) - \eta \nabla \mathcal{L}(\wvec(t))$ (eq.~\eqref{GD}) and denoted 
	\begin{multline*}
	    \bm{\psi}_1(t)=-\what\left( \frac{1}{\lossPower}g^{1-\lossPower}(t)\htilde(t)-\frac{1}{\lossPower}g^{1-\lossPower}(t+1)\htilde(t+1)\right)- \wtilde\left(g^{1-\lossPower}(t+1) - g^{1-\lossPower}(t)\right)\\
	    -\frac{1}{\lossPower}\wcheck\left(g^{1-\lossPower}(t)\log^{-1}(t)-g^{1-\lossPower}(t+1)\log^{-1}(t+1)\right)
	-\frac{1}{\lossPower}\wcheckSec\left(g^{1-2\lossPower}(t)-g^{1-2\lossPower}(t+1)\right).    
	\end{multline*}
	In (3) we used $g(t+1)-g(t)=\Theta\left( t^{-1}\log^{\frac{1}{\lossPower}-1}(t)\right)$ (from eq.~\eqref{diff g Theta(t^-1), beta<1}) and $\norm{\bm{\psi}_1}=O\left(t^{-1}\log^{\frac{1}{\lossPower}-2}(t)\right)$ (from eqs. \ref{diff g-beta Theta(t^-1), beta<1}, \ref{eq: g(1-beta) first derivative}, \ref{eq: (g^(1-lossPower)htilde)'}, \ref{eq: g^(1-v)/log = o(...)}, \ref{eq: ( g^(1-2lossPower))'}) and also $ \nabla \mathcal{L}(\wvec(t))=o(1)$ from lemma 5. In 4 we used that
	\begin{equation*}
	\what^\top \derL(\wvec(t)) = \sumn \ell'(\wvec(t)^\top\xn)\what^\top\xn\le0
	\end{equation*}
	since $\what^\top\xn\ge1$ from the definition of $\what$ and $\ell^{\prime}(u) \leq 0$ and defined $C>0$.\\
        Next, we will use the following the following Lemma \ref{Lemma: (r(t+1)-r(t)r(t) bound for beta<1} which we prove in appendix \ref{proof of lemma (r(t+1)-r(t)r(t) upper bound for beta<1}:
    	\begin{restatable}{lemma}{rBoundLemmaBetaLE} \label{Lemma: (r(t+1)-r(t)r(t) bound for beta<1}
		We have
		\begin{equation} \label{eq: (r(t+1)-r(t)r(t) bound for beta<1}
		\exists C_1,t_1\ :\ \forall t>t_1 \ :\ \left(\rvec(t+1)-\rvec(t)\right)^\top\rvec(t)\le C_1 t^{-1}\log^{\frac{1}{\lossPower}-2}(t)
		\end{equation}
	\end{restatable}    
\noindent From eq.~\eqref{eq: (r(t+1)-r(t)r(t) bound for beta<1} in Lemma \ref{Lemma: (r(t+1)-r(t)r(t) bound for beta<1}, we can find $C_1,t_1$ such that $\forall t>t_1$:
	\begin{equation} \label{eq2: (r(t+1)-r(t)r(t) bound for beta<1}
	\left(\rvec(t+1)-\rvec(t)\right)^\top\rvec(t)\le C_1 t^{-1}\log^{\frac{1}{\lossPower}-2}(t)
	\end{equation}
	Also, from Lemma \ref{Lemma: L converge}  we know that:
	\begin{equation} \label{eq: derL converge, beta<1}
	\Vert\derL(\wvec(u))\Vert^2=o(1) \text{ and } \sum_{u=0}^{\infty}\Vert\derL(\wvec(u))\Vert^2<\infty
	\end{equation}
    
	Substituting eqs. \ref{norm(r(t+1)-r(t), beta<1}, \ref{eq: derL converge, beta<1} and \ref{eq2: (r(t+1)-r(t)r(t) bound for beta<1} into eq.~\eqref{norm r(t+1), beta<1}, we find
	\begin{align*}
	&\Vert \rvec (t) \Vert^2 - \Vert \rvec (t_1) \Vert^2\\
	& = \sum_{u=t_1}^{t-1}\left[\Vert \rvec (u+1) \Vert^2 - \Vert \rvec (u) \Vert^2\right]\\
	& \le \sum_{u=t_1}^{t-1} \tilde{C} t^{-1}\log^{\frac{1}{\lossPower}-2}(t) = O(\log^{\frac{1}{\lossPower}-1}(t))
	\end{align*}
	Therefore, $\Vert \rvec (t) \Vert^2 = O(g^{1-\lossPower}(t))$ and $\Vert \rvec (t) \Vert=o(g^{1-\lossPower}(t))$.
	
	\subsubsection{Proof Of Lemma \ref{Lemma: (r(t+1)-r(t)r(t) bound for beta<1}} \label{proof of lemma (r(t+1)-r(t)r(t) upper bound for beta<1}
	Recall that we defined $$\vect{r}(t) = \vect{w}(t) - g(t)\what - g^{1-\lossPower}(t) \wtilde + \frac{1}{\lossPower}g^{1-\lossPower}(t)\htilde(t)\what + \frac{1}{\lossPower}g^{1-\lossPower}(t)\log^{-1}(t)\wcheck+ \frac{1}{\lossPower}g^{1-\lossPower}(t)g^{-\lossPower}(t)\wcheckSec , $$
	where $\what,\ \wtilde,\ \wcheck$ and $\wcheckSec$ were defined in section \ref{sec: defenitions for beta<1}.
	\rBoundLemmaBetaLE*
	\noindent We examine the expression we wish to bound:
	\begin{flalign} \label{(r(t+1)-r(t))r(t) bound, pre}
	&(\rvec(t+1) - \rvec(t))^\top  \rvec(t) &\nonumber\\
	&\overset{(1)}{=}\left(\vphantom{\frac{1}{a}}-\eta\nabla L(\wvec(t)) - \what\left(g(t+1)-g(t)+\lossPower^{-1}g^{1-\lossPower}(t)\htilde(t)-\lossPower^{-1}g^{1-\lossPower}(t+1)\htilde(t+1)\right)  \right.\nonumber\\
	& \left. - \wtilde(g^{1-\lossPower}(t+1) - g^{1-\lossPower}(t))
	-\frac{1}{\lossPower}\wcheck\left(g^{1-\lossPower}(t)\log^{-1}(t)-g^{1-\lossPower}(t+1)\log^{-1}(t+1)\right) \right.\nonumber\\
	& \left.
	-\frac{1}{\lossPower}\wcheckSec\left(g^{1-2\lossPower}(t)-g^{1-2\lossPower}(t+1)\right)
	 \right)^\top \rvec(t)\nonumber\\
	& = -\eta\sumn \ell'(\wvec(t)^\top \xn)\xnT\rvec(t) - \what^\top \rvec(t)\left(g(t+1)-g(t)+\frac{1}{\lossPower}g^{1-\lossPower}(t)\htilde(t)-\frac{1}{\lossPower}g^{1-\lossPower}(t+1)\htilde(t+1)\right) \nonumber\\
	&- \wtilde^\top  \rvec(t) (g^{1-\lossPower}(t+1) - g^{1-\lossPower}(t))
	 - \wcheck^\top  \rvec(t)\frac{1}{\lossPower}\left(g^{1-\lossPower}(t)\log^{-1}(t)-g^{1-\lossPower}(t+1)\log^{-1}(t+1)\right) \nonumber\\
     &  - \wcheckSec^\top  \rvec(t)\frac{1}{\lossPower}\left(g^{1-2\lossPower}(t)-g^{1-2\lossPower}(t+1)\right) \nonumber \\
	& \overset{(2)}{=} \what^\top \rvec(t)\left[\frac{1}{\lossPower}t^{-1}\log^{\frac{1}{\lossPower}-1}(t)\left(1+h(t)+h_2(t)\right) - (g(t+1)-g(t))\right]  -\eta\sumnnsv \ell'(\wvec(t)^\top \xn)\xnT\rvec(t)\nonumber\nonumber\\
	&-\eta\sumnsv \left[\frac{1}{\lossPower}t^{-1}\log^{\frac{1}{\lossPower}-1}(t)\left(1+h(t)+h_2(t)\right) \e(-\lossPower\xnT \wtilde) + \ell'(\wvec^\top (t)\xn)\right]\xnT\rvec(t)\nonumber\\
	& -\wtilde^\top  \rvec(t)(g^{1-\lossPower}(t+1) - g^{1-\lossPower}(t)) -\what^\top \rvec(t) \left(\frac{1}{\lossPower}g^{1-\lossPower}(t)\htilde(t)-\frac{1}{\lossPower}g^{1-\lossPower}(t+1)\htilde(t+1)\right)\nonumber\\
	& - \wcheck^\top  \rvec(t)\frac{1}{\lossPower}\left(g^{1-\lossPower}(t)\log^{-1}(t)-g^{1-\lossPower}(t+1)\log^{-1}(t+1)\right)\nonumber\\
	& - \wcheckSec^\top  \rvec(t)\frac{1}{\lossPower}\left(g^{1-2\lossPower}(t)-g^{1-2\lossPower}(t+1)\right),
	\end{flalign}
	where in (1) we used eq.~\eqref{define r(t), beta<1} ($\rvec(t)$ definition) and in (2) we used $$\what = \sumnsv \alpha_n \xn = \sumnsv \exp(-\wtilde^\top \xn) \xn.$$
	From $\wcheck$ and $\wcheckSec$ definitions (eqs. \ref{eq: wcheck def} and \ref{eq: wcheckSec def}) we can find $ \{\delta_n\}_{n=1}^d,\{\zeta_n\}_{n=1}^d $ such that:
	\begin{align}
	    \wcheck = \sumnsv \delta_n \xn\\
	    \wcheckSec = \sumnsv \zeta_n \xn
	\end{align}
	Using Taylor's theorem, the last two equations, eq \ref{eq: wtilde lc of sv, beta smaller than 1} and eqs. \ref{eq: g(1-beta) first derivative}-\ref{eq: g(1-2beta) third derivative} we can find $\left\{\tilde{\lambda}_n\right\}_{n=1}^d$ such that:
	\begin{flalign} \label{eq: aux calc for (r(t+1)-r(t))r(t) bound}
	    & -\wtilde^\top  \rvec(t)(g^{1-\lossPower}(t+1) - g^{1-\lossPower}(t)) -\what^\top \rvec(t) \left(\frac{1}{\lossPower}g^{1-\lossPower}(t)\htilde(t)-\frac{1}{\lossPower}g^{1-\lossPower}(t+1)\htilde(t+1)\right) &\nonumber\\
    	& - \wcheck^\top  \rvec(t)\frac{1}{\lossPower}\left(g^{1-\lossPower}(t)\log^{-1}(t)-g^{1-\lossPower}(t+1)\log^{-1}(t+1)\right)\nonumber\\
    	& - \wcheckSec^\top  \rvec(t)\frac{1}{\lossPower}\left(g^{1-2\lossPower}(t)-g^{1-2\lossPower}(t+1)\right)\nonumber\\
    	& = -\wtilde^\top  \rvec(t)\left(g^{1-\lossPower}(t+1)  g^{1-\lossPower}(t)-\left((g^{1-\lossPower}(t))'+\frac{1}{2}(g^{1-\lossPower}(t))''\right)\right)\nonumber\\
    	&- \sumnsv \lambda_n \xnT \rvec(t)  \left((g^{1-\lossPower}(t))'+\frac{1}{2}(g^{1-\lossPower}(t))''\right) \nonumber\\
    	& +\what^\top \rvec(t)\frac{1}{\lossPower} \left(g^{1-\lossPower}(t+1)\htilde(t+1)-g^{1-\lossPower}(t)\htilde(t)-\left((g^{1-\lossPower}(t)\htilde(t))'+\frac{1}{2}(g^{1-\lossPower}(t)\htilde(t))'' \right)\right) \nonumber\\
    	&+ \frac{1}{\lossPower}\sumnsv \alpha_n \xnT\rvec(t)\left((g^{1-\lossPower}(t)\htilde(t))'+\frac{1}{2}(g^{1-\lossPower}(t)\htilde(t))'' \right) &\nonumber\\
    	& + \wcheck^\top  \rvec(t)\frac{1}{\lossPower}\left(g^{1-\lossPower}(t+1)\log^{-1}(t+1)-g^{1-\lossPower}(t)\log^{-1}(t)
    	- \left( (g^{1-\lossPower}(t)\log^{-1}(t))'+\frac{1}{2}(g^{1-\lossPower}(t)\log^{-1}(t))''\right)\right)\nonumber\\
    	& + \frac{1}{\lossPower}\sumnsv \delta_n \xnT\rvec(t) \left( (g^{1-\lossPower}(t)\log^{-1}(t))'+\frac{1}{2}(g^{1-\lossPower}(t)\log^{-1}(t))''\right) \nonumber\\
    	& + \wcheckSec^\top  \rvec(t)\frac{1}{\lossPower}\left(g^{1-2\lossPower}(t+1)-g^{1-2\lossPower}(t) - \left((g^{1-2\lossPower}(t))'+\frac{1}{2}(g^{1-2\lossPower}(t))'' \right)\right)\nonumber\\
    	& + \frac{1}{\lossPower} \sumnsv \zeta_n \xnT \rvec(t) \left((g^{1-2\lossPower}(t))'+\frac{1}{2}(g^{1-2\lossPower}(t))'' \right) \nonumber\\
    	& \le -\sumnsv \lambda_n \bar{C}_1 \xnT \rvec(t) \frac{1}{t}\log^{\frac{1}{\lossPower}-2}(t) - \frac{\eta}{\lossPower}\sumnsv \tilde{\lambda}_n \e(-\lossPower\xnT \wtilde)  \frac{1}{t}\log^{\frac{1}{\lossPower}-3}(t) f_3(t)\xnT \rvec(t) + o(t^{-1-\epsilon}),
	\end{flalign}
	where $\bar{C_1} = \frac{1-\lossPower}{\lossPower}$ was defined in section \ref{sec: auxRes beta smaller than 1}, $\epsilon>0$ and $f_3(t)=o(\log^\epsilon(t))$. In the last transition we used $\what^\top \rvec(t)=o(t),\ \wtilde^\top \rvec(t)=o(t),\ \wcheck^\top \rvec(t)=o(t),\ \wcheckSec^\top \rvec(t)=o(t)$ since:
	\begin{flalign*}
    	& \Vert\rvec(t)\Vert = &\\
    	&\left\Vert\wvec(0) - \eta \sum_{u=0}^{t-1}\derL(\wvec(u)) - g(t)\what - g^{1-\lossPower}(t) \wtilde + \frac{1}{\lossPower}g^{1-\lossPower}(t)\htilde(t)\what + \frac{1}{\lossPower}g^{1-\lossPower}(t)\log^{-1}(t)\wcheck+ \frac{1}{\lossPower}g^{1-\lossPower}(t)g^{-\lossPower}(t)\wcheckSec\right\Vert&\\
    	& \le \eta t \min_{0\le u \le t} \Vert\derL(\wvec(u))\Vert + O(g(t)) = o(t),
	\end{flalign*}
	where in the last line we used that $\derL(\wvec(t))=o(1)$, from Lemma \ref{Lemma: L converge}.\\
	Using eq.~\eqref{eq: aux calc for (r(t+1)-r(t))r(t) bound}, eq.~\eqref{(r(t+1)-r(t))r(t) bound, pre} can be upper bounded by
	\begin{flalign} \label{(r(t+1)-r(t))r(t) bound}
	& \what^\top \rvec(t)\left[\frac{1}{\lossPower}t^{-1}\log^{\frac{1}{\lossPower}-1}(t)\left(1+h(t)+h_2(t)\right) - (g(t+1)-g(t))\right]  -\eta\sumnnsv \ell'(\wvec(t)^\top \xn)\xnT\rvec(t)\nonumber\nonumber\\
	&-\eta\sumnsv \left[ \frac{1}{\lossPower}t^{-1}\log^{\frac{1}{\lossPower}-1}(t)\tilde{f}_4(t) \e(-\lossPower\xnT \wtilde) + \ell'(\wvec(t)^\top (t)\xn)\right]\xnT\rvec(t),
	\end{flalign}
	where we recall we defined $\tilde{\gamma}_n \triangleq \frac{\lossPower\bar{C_1}\lambda_n}{\eta}\exp(\lossPower\xnT\wtilde)$ and we define $\tilde{f}_4(t) = 1+h(t)+\tilde{\gamma}_n \log^{-1}(t)+\tilde{\lambda}_n\log^{-2}(t) f_3(t)+h_2(t)$.\\
	We examine the three terms in eq.~\eqref{(r(t+1)-r(t))r(t) bound}.
	The first term can be upper bounded $ \forall t>t_2$ by
	\begin{align} \label{eq: first term bound, beta<1}
	&\what^\top \rvec(t)\left[\frac{1}{\lossPower}t^{-1}\log^{\frac{1}{\lossPower}-1}(t)\left(1+h(t)+h_2(t)\right) - (g(t+1)-g(t))\right]\nonumber\\
	&\overset{(1)}{\le} \left\vert \what^\top \rvec(t) \right\vert C_2t^{-3}\log^{\frac{1}{\lossPower}-1}(t)\nonumber\\
	& \overset{(2)}{\le} C_2 t^{-2}\log^{\frac{1}{\lossPower}-1}(t),
	\end{align}
	where in (1) we used $$\left \vert \frac{1}{\lossPower}t^{-1}\log^{\frac{1}{\lossPower}-1}(t)\left(1+h(t)+h_2(t)\right) - (g(t+1)-g(t)) \right \vert =\Theta(t^{-3}\log^{\frac{1}{\lossPower}-1}(t))),$$and therefore $\exists t_2,C_2$ such that $ \forall t>t_2$: $$\left \vert \frac{1}{\lossPower}t^{-1}\log^{\frac{1}{\lossPower}-1}(t)\left(1+h(t)+h_2(t)\right) - (g(t+1)-g(t)) \right \vert \le C_2 t^{-3}\log^{\frac{1}{\lossPower}-1}(t) $$. In (2) we used $\what^\top \rvec(t)=o(t)$.\\
	Next, we wish to upper bound the second term in eq.~\eqref{(r(t+1)-r(t))r(t) bound}.
	If $\xnT\rvec(t)\ge 0$ then using eq.~\eqref{define r(t), beta<1} ($\rvec(t)$ definition) we can show that
	\begin{equation}
		\left(\wvec(t)^\top\xn\right)\ge \left(g(t)\what^\top  \xn + g^{1-\lossPower}(t)\left(\wtilde^\top  \xn -\frac{1}{\lossPower}\htilde(t)\what^\top\xn - \frac{1}{\lossPower}\log^{-1}(t)\wcheck^\top\xn
		- \frac{1}{\lossPower}g^{-\lossPower}(t)\wcheckSec^\top\xn
		\right) \right)^\lossPower
	\end{equation}
	Using the last equation, and eq.~\eqref{binomial expansion for (g+g(1-beta))^beta, beta<1} we can also show that
	\begin{flalign} \label{e^{-wx} upper bound for xr>0}
		&\e(-(\wvec(t)^\top \xn)^\lossPower)&\nonumber\\
		& \le t^{-(\what^\top  \xn)^\lossPower}\left(\lossPower\log^{1-\frac{1}{\lossPower}}(t)\right)^{-(\what^\top  \xn)^\lossPower} \e\left(-\lossPower(\what^\top  \xn)^{\lossPower-1} \ftilde(t)-\ftilde_2(t)-f_2(t)\right) \nonumber\\
		& \overset{(1)}{\le} t^{-(\what^\top  \xn)^\lossPower}\left(\lossPower\log^{1-\frac{1}{\lossPower}}(t)\right)^{-(\what^\top  \xn)^\lossPower} \e\left(-\lossPower(\what^\top  \xn)^{\lossPower-1} \left(\wtilde^\top  \xn -1 \right)\right),\ \forall t>t_3,
	\end{flalign}
	where in (1) we used the fact that $\exists t_3>\bar{t}$ such that $\forall t>t_3$:
	$$ \tilde{f}(t)+\frac{(\what^\top\xn)^{1-\lossPower}}{\lossPower}(f_2(t)+\ftilde_2(t))\ge-1$$\\ 
	We divide into two cases.\\
	First if $ \xnT \rvec (t) \le C_3\log^{\frac{1}{\lossPower}+1}(t)$:
	\begin{flalign} \label{eq: 2nd term bound, xr bounded, beta<1}
	&-\eta\sumnnsv \ell'(\wvec(t)^\top \xn)\xnT\rvec(t)&\nonumber\\
	&\overset{(1)}{\le} \eta\sumnnsvp \left(1+\e(-\mu_+(\wvec(t)^\top \xn)^\lossPower)\right)\e(-(\wvec(t)^\top  \xn)^\lossPower ) \xnT\rvec(t) \nonumber\\
	&\overset{(2)}{\le} \eta\sumnnsvp \left(1+\left[t^{-\theta^\lossPower}\left(\lossPower\log^{1-\frac{1}{\lossPower}}(t)\right)^{-\theta^\lossPower} \e\left(-\lossPower(\what^\top  \xn)^{\lossPower-1} \left(\wtilde^\top  \xn -1\right)\right)\right]^{\mu_+}\right) \nonumber\\
	&t^{-\theta^\lossPower}\left(\lossPower\log^{1-\frac{1}{\lossPower}}(t)\right)^{-\theta^\lossPower} \e\left(-\lossPower\theta^{\lossPower-1} \left(\wtilde^\top  \xn -1\right)\right) C_1\log^{\frac{1}{\lossPower}+1}(t) \nonumber\\
	&\overset{(3)}{\le} 2\eta C_3 N t^{-\theta^\lossPower} \lossPower^{-\theta^{\lossPower}}\log^{\theta^{\lossPower}\left(\frac{1}{\lossPower}-1\right)}\e\left(- \lossPower \min_{n\notin \mathcal{S}}(\what^T \xn)^{\lossPower-1}(\wtilde^\top \xn-1) \right)\log^{\frac{1}{\lossPower}+1}(t) \nonumber\\
	&\le \tilde{C}_2 t^{-\theta^\lossPower} \log^{\theta^{\lossPower}\left(\frac{1}{\lossPower}-1\right)+\frac{1}{\lossPower}+1}(t),\ \forall t>t_2^{\prime},
	\end{flalign}
	where in (1) we used eq.~\eqref{-l' upper bound}, in (2) we used \ref{e^{-wx} upper bound for xr>0} and $\theta=\min_{n\notin \mathcal{S}}\what^T \xn>1$, in (3) we used the fact that $\exists t'$ such that $\forall t>t':\ \left[t^{-\theta^\lossPower}\left(\lossPower\log^{1-\frac{1}{\lossPower}}(t)\right)^{-\theta^\lossPower} \e\left(-\lossPower(\what^\top  \xn)^{\lossPower-1} \left(\wtilde^\top  \xn -1\right)\right)\right]^{\mu_+}\le1$. We defined $t_2^{\prime}=\max(t_3,t')$ and $\tilde{C}_2 = 2\eta C_3 N \lossPower^{-\theta^{\lossPower}} \e\left(- \lossPower \min_{n\notin \mathcal{S}}(\what^T \xn)^{\lossPower-1}(\wtilde^\top \xn-1) \right)$.\\
	Second, if $ \xnT \rvec (t) > C_3\log^{\frac{1}{\lossPower}+1}(t)$:
	\begin{flalign} \label{wx lower bound for 2nd term when xr is unbounded}
		&(\wvec(t)^\top  \xn)^\lossPower&\nonumber\\
		& \overset{(1)}{=}\left(g(t)\what^\top  \xn + g^{1-\lossPower}(t)\ftilde(t)+ \xnT\rvec(t) \right)^\lossPower&\nonumber\\
		& = (\xnT \rvec(t) )^{\lossPower} 
		\left(1 + \frac{g(t)\what^\top\xn + g^{1-\lossPower}(t)\ftilde(t)}{\xnT\rvec(t) } \right)^\lossPower \nonumber\\
		& \overset{(2)}{\ge} (\xnT \rvec(t) )^{\lossPower} 
		\left(\lossPower +  \left( \frac{g(t)\what^\top\xn + g^{1-\lossPower}(t)\ftilde(t)}{\xnT\rvec(t) } \right)^\lossPower \right)\nonumber\\
		& = \lossPower (\xnT \rvec(t) )^{\lossPower} +  \left(g(t)\what^\top\xn + g^{1-\lossPower}(t)\ftilde(t)\right)^\lossPower \nonumber\\
		& \overset{(3)}{=} \lossPower (\xnT \rvec(t) )^{\lossPower}+ (\what^\top  \xn)^\lossPower\log(t) + (\what^\top  \xn)^\lossPower \log(\lossPower \log^{1-\frac{1}{\lossPower}}(t)) + \lossPower(\what^\top  \xn)^{\lossPower-1}\ftilde(t)\nonumber\\
		& +\ftilde_2(t)+ f_2(t),
	\end{flalign}
	where in (1) we used eq.~\eqref{define r(t), beta<1}, in (2) we used that $\forall x<0.5,a<1:\ (1+x)^a \ge a+x^a $ and that $ \exists t_4 $ such that $$\forall t>t_4:\ \frac{g(t)\what^\top\xn + g^{1-\lossPower}(t)\ftilde(t)}{\xnT\rvec(t) }<0.5 $$ and in (3) we used eq.~\eqref{binomial expansion for (g+g(1-beta))^beta, beta<1}.\\
	Using \ref{wx lower bound for 2nd term when xr is unbounded} we can upper bound $\e(-(\wvec(t)^\top \xn)^\lossPower),\ \forall t>t_4$:
	\begin{flalign} \label{e^{-wx} upper bound for xr>0 and unbounded}
	&\e(-(\wvec(t)^\top \xn)^\lossPower)&\nonumber\\
	& \overset{(1)}{\le} \e\left(-\lossPower \left(\xnT \rvec(t) \right)^{\lossPower}\right)t^{-(\what^\top  \xn)^\lossPower}\left(\lossPower\log^{1-\frac{1}{\lossPower}}(t)\right)^{-(\what^\top  \xn)^\lossPower}\e\left(-\lossPower(\what^\top  \xn)^{\lossPower-1} \left(\wtilde^\top  \xn -1\right)\right)
	\end{flalign}
	Using the last equation we can upper bound the second term in eq.~\eqref{(r(t+1)-r(t))r(t) bound} $\forall t>t_5$:
	\begin{flalign}  \label{eq: 2nd term bound, xr not bounded, beta<1}
	&-\eta\sumnnsv \ell'(\wvec(t)^\top \xn)\xnT\rvec(t)&\nonumber\\
	&\overset{(1)}{\le} \eta\sumnnsvp \left(1+\e(-\mu_+(\wvec(t)^\top \xn)^\lossPower)\right)\e(-(\wvec(t)^\top  \xn)^\lossPower ) \xnT\rvec(t) \nonumber\\
	&\overset{(2)}{\le} \eta\sumnnsvp \left(1+\left[t^{-\theta^\lossPower}\left(\lossPower\log^{1-\frac{1}{\lossPower}}(t)\right)^{-\theta^\lossPower} \e\left(-\lossPower(\what^\top  \xn)^{\lossPower-1} \left(\wtilde^\top  \xn -1\right)\right)\right]^{\mu_+}\right)\nonumber\\
	&t^{-\theta^\lossPower}\left(\lossPower\log^{1-\frac{1}{\lossPower}}(t)\right)^{-\theta^\lossPower}\e\left(-\lossPower\theta^{\lossPower-1} \left(\wtilde^\top  \xn -1\right)\right)\e\left(-\lossPower \left(\xnT \rvec \right)^{\lossPower}\right) \xnT\rvec(t) \nonumber\\
	&\overset{(3)}{\le} 2\eta N t^{-\theta^\lossPower} \lossPower^{-\theta^{\lossPower}}\log^{\theta^{\lossPower}\left(\frac{1}{\lossPower}-1\right)}\e\left(- \lossPower \min_{n\notin \mathcal{S}}(\what^T \xn)^{\lossPower-1}(\wtilde^\top \xk-1) \right)\nonumber\\
	&\le \tilde{C}_2^\prime t^{-\theta^\lossPower} \log^{\theta^{\lossPower}\left(\frac{1}{\lossPower}-1\right)},
	\end{flalign}
	where in (1) we used \ref{-l' upper bound}, in (2) we used eqs. \ref{e^{-wx} upper bound for xr>0} and \ref{e^{-wx} upper bound for xr>0 and unbounded} and $\theta=\min_{n\notin \mathcal{S}}\what^T \xn>1$ (eq.~\eqref{theta definition}). In (3) we used the fact that $\lim_{t\to\infty} \e\left(-\lossPower(\xnT \rvec(t) )^{\lossPower}\right) \xnT \rvec(t) =0 $, therefore exists $t^{\prime\prime}$ such that $\forall t>t':\ \e\left(-\lossPower(\xnT \rvec(t) )^{\lossPower}\right) \xnT \rvec(t)\le1 $. We define $t_5=\max(t_4,t^{\prime \prime})$ and $$\tilde{C}_2^\prime = 2\eta N \lossPower^{-\theta^{\lossPower}} \e\left(- \lossPower \min_{n\notin \mathcal{S}}(\what^T \xn)^{\lossPower-1}(\wtilde^\top \xk-1) \right).$$
	\pagebreak\\
	Next, we wish to upper bound the third term in eq.~\eqref{(r(t+1)-r(t))r(t) bound}:
	\begin{flalign} \label{3rd term bound}
	&-\eta\sumnsv \left[ \frac{1}{\lossPower}t^{-1}\log^{\frac{1}{\lossPower}-1}(t)\tilde{f}_4(t) \e(-\lossPower\xnT \wtilde) + \ell'(\wvec(t)^\top (t)\xn)\right]\xnT\rvec(t) ,
	\end{flalign}
    where we recall we denoted $\tilde{f}_4(t) = 1+h(t)+\tilde{\gamma}_n \log^{-1}(t)+\tilde{\lambda}_n\log^{-2}(t) f_3(t)+h_2(t)$.\\
	We can use eq.~\eqref{binomial expansion for (g+g(1-beta))^beta, beta<1} to show that for $n\in\mathcal{S}$:
	\begin{flalign} \label{binomial expansion for ()^beta for SV}
		 &\left(g(t) + g^{1-\lossPower}(t)\left(\wtilde^\top  \xn -\frac{1}{\lossPower}\htilde(t) - \frac{1}{\lossPower}\log^{-1}(t)\wcheck^\top\xn
		- \frac{1}{\lossPower}g^{-\lossPower}(t)\wcheckSec^\top\xn
		\right) \right)^\lossPower&\nonumber\\
		& = \log(t) + \log(\lossPower \log^{1-\frac{1}{\lossPower}}(t)) + \lossPower\wtilde^\top  \xn-\htilde(t)-\log^{-1}(t)(\wcheck^\top\xn)-g^{-\lossPower}(t)(\wcheckSec^\top\xn)\nonumber\\
		& + \frac{\lossPower(\lossPower-1)}{2} g^{-\lossPower}(t)\left(\wtilde^\top  \xn \right)^2+\frac{(\lossPower-1)}{2\lossPower} \log^{-1}(t)\log^2(\lossPower\log^{1-\frac{1}{\lossPower}}(t))+ f_2(t)\nonumber\\
		& = \log(t) + \log(\lossPower \log^{1-\frac{1}{\lossPower}}(t)) + \lossPower\wtilde^\top  \xn-h(t)-\log^{-1}(t)(\wcheck^\top\xn)+f_2(t),
	\end{flalign}
	where in the last transition we used that $\wcheckSec^\top\xn=\frac{\lossPower(\lossPower-1)}{2}(\wtilde^\top\xn)^2$ (eq.~\eqref{eq: wcheckSec def}) and $\htilde(t)$ definition (eq.~\eqref{eq: beta smaller than 1 functions def}).\\
	We examine each term $k$ in the sum, and divide into two cases, depending on the sign of $\xkT \rvec(t)$.\\
	From this point we assume $\lossPower>\frac{1}{4}$.\\
	First, if $\xkT \rvec(t) \ge 0$, then using eqs. \ref{-l' upper bound} and \ref{binomial expansion for ()^beta for SV} term $k$ in equation \ref{3rd term bound} can be upper bounded by:
	\begin{flalign}
	&-\eta \frac{1}{\lossPower}t^{-1}\log^{\frac{1}{\lossPower}-1}(t)\e(-\lossPower\xnT \wtilde)\left[ \left(1+h(t)+\tilde{\gamma}_n \log^{-1}(t)+\tilde{\lambda}_n\log^{-2}(t) f_3(t)+h_2(t)\right)  \right.&\nonumber\\
	& \left. - \left(1+\e(-\mu_+(\wvec(t)^\top \xn)^\lossPower)\right) \e\left(h(t)+\log^{-1}(t)\wcheck^\top\xn-f_2(t)\right)\right]\xkT\rvec(t) \nonumber\\
	& \overset{(1)}{\le} -\eta \frac{1}{\lossPower}t^{-1}\log^{\frac{1}{\lossPower}-1}(t)\e(-\lossPower\xnT \wtilde)\Bigg[ \tilde{\lambda}_n\log^{-2}(t) f_3(t)+h_2(t)+f_2(t)- \tilde{f}_5^2(t) &\nonumber\\
	&  - \left[\frac{1}{t}\log^{\frac{1}{\lossPower}-1}(t)\frac{1}{\lossPower}\e\left(-\lossPower\left(\wtilde^\top  \xn -1\right)\right)\right]^{\mu_+} \left(1+\tilde{f}_5(t)+\tilde{f}_5^2(t)\right)\Bigg]\xkT\rvec(t) \nonumber\\
	& \overset{(2)}{\le} \eta \frac{1}{\lossPower}t^{-1}\log^{\frac{1}{\lossPower}-1}(t)\e(-\lossPower\xnT \wtilde)\log^{-\frac{1}{2\lossPower}}(t)\xkT\rvec(t),
	\end{flalign}
	where in (1) we defined $\tilde{f}_5(t) = h(t)+\tilde{\gamma}_n\log^{-1}(t)-f_2(t)$ and used that:
	\begin{flalign}
		&\left(1+\e(-\mu_+(\wvec(t)^\top \xn)^\lossPower)\right) \e\left(h(t)+\log^{-1}(t)\wcheck^\top\xn-f_2(t)\right)&\nonumber\\
		& \overset{(a)}{\le} \left(1+\left[\frac{1}{t}\log^{\frac{1}{\lossPower}-1}(t)\frac{1}{\lossPower}\e\left(-\lossPower\left(\wtilde^\top  \xn -1\right)\right)\right]^{\mu_+}\right) \left(1+\tilde{f}_5(t)+\tilde{f}_5^2(t)\right), \nonumber
	\end{flalign}
	where in (a) we used  $ \forall x\ge0:\ e^{-x}\le1-x+x^2$, $h(t)\le0$ and $\wcheck^\top \xn=\tilde{\gamma}_n$ (eq.~\eqref{eq: wcheck def}). In transition (2) we used the fact that all the terms in the square brackets are $o(\log^{-2+\epsilon}(t)),\ \forall \epsilon>0$ and $-\frac{1}{2\lossPower}>-2$.\\
	1. If $ |\xkT \rvec(t)|\le C_4 \log^{-1+\frac{1}{2\lossPower}}(t) $ then term $k$ in equation \ref{3rd term bound} can be upper bounded by:
	\begin{equation}
		\eta C_4 \frac{1}{\lossPower}t^{-1}\log^{\frac{1}{\lossPower}-2}(t)\e(-\lossPower\xnT \wtilde)
	\end{equation}
	2. If $ |\xkT \rvec(t)|> C_4 \log^{-1+\frac{1}{2\lossPower}}(t) $ term $k$ in equation \ref{3rd term bound} can be upper bounded by:
	\begin{flalign} \label{term k in 3rd term when xr>0 and not bounded}
		&-\eta \frac{1}{\lossPower}t^{-1}\log^{\frac{1}{\lossPower}-1}(t)\e(-\lossPower\xnT \wtilde)\Big[\tilde{f}_4(t) -\left(1+\left[\frac{1}{t}\log^{\frac{1}{\lossPower}-1}(t)\frac{1}{\lossPower}\e\left(-\lossPower\left(\wtilde^\top  \xn -1\right)\right)\right]^{\mu_+}\right) &\nonumber\\
		&   \e\left(\tilde{f}_5(t)-C_5\log^{-\frac{1}{2\lossPower}}(t)-f_4(t)\right) \Big]\xnT\rvec(t),
	\end{flalign}
	where $f_4(t) = o(\log^{-\frac{1}{2\lossPower}}(t))$, since:
	\begin{flalign}
	& (\wvec(t)^\top  \xn)^\lossPower = \left(g(t)\what^\top  \xn + g^{1-\lossPower}(t)\left(\wtilde^\top  \xn -\frac{1}{\lossPower}h(t)\what^\top\xn - \frac{1}{\lossPower}\log^{-1}(t)\wcheck^\top\xn\right)+ \xnT\rvec(t) \right)^\lossPower&\nonumber\\
	&\ge \left(g(t) + g^{1-\lossPower}(t)\tilde{f}(t) +  C_4 \log^{-1+\frac{1}{2\lossPower}}(t) \right)^\lossPower&\nonumber\\
	& = \left(g(t) + g^{1-\lossPower}(t)\tilde{f}(t)\right)^\lossPower
	\left(1 + \frac{ C_4 \log^{-1+\frac{1}{2\lossPower}}(t) }{g(t) + g^{1-\lossPower}(t)\tilde{f}(t)}\right)^\lossPower\nonumber\\
	& = \left(g(t) + g^{1-\lossPower}(t)\tilde{f}(t)\right)^\lossPower + \lossPower \left(g(t) + g^{1-\lossPower}(t)\tilde{f}(t)\right)^{\lossPower-1}C_4 \log^{-1+\frac{1}{2\lossPower}}(t)+f_4(t)\nonumber\\
	& \ge \log(t) + \log(\lossPower \log^{1-\frac{1}{\lossPower}}(t)) + \lossPower \wtilde^\top  \xn -h(t) - \tilde{\gamma}_n \log^{-1}(t) +f_2(t) + C_5\log^{-\frac{1}{2\lossPower}}(t)+f_4(t), \nonumber
	\end{flalign}
	in the last two transitions we used Taylor's theorem and eq.~\eqref{binomial expansion for ()^beta for SV}.\\
	eq.~\eqref{term k in 3rd term when xr>0 and not bounded} can be upper bounded by zero since:
	\begin{flalign*} 
	&\left[\tilde{f}_4(t) -\left(1+\left[\frac{1}{t}\log^{\frac{1}{\lossPower}-1}(t)\frac{1}{\lossPower}\e\left(-\lossPower\left(\wtilde^\top  \xn -1\right)\right)\right]^{\mu_+}\right) \e\left(\tilde{f}_5(t)-C_5\log^{-\frac{1}{2\lossPower}}(t)-f_4(t)\right) \right]\nonumber\\
	&\overset{(1)}{\ge}\left[\tilde{f}_4(t)-\left(1+\left[\frac{1}{t}\log^{\frac{1}{\lossPower}-1}(t)\frac{1}{\lossPower}\e\left(-\lossPower\left(\wtilde^\top  \xn -1\right)\right)\right]^{\mu_+}\right) \right.&\\
	& \left.  \left(1 +\tilde{f}_5(t)-C_5\log^{-\frac{1}{2\lossPower}}(t)-f_4(t)+\left(\tilde{f}_5(t)-C_5\log^{-\frac{1}{2\lossPower}}(t)-f_4(t)\right)^2\right) \right]\nonumber\\
	& \overset{(2)}{\ge} \left[\tilde{\lambda}_n\log^{-2}(t) f_3(t)+h_2(t)+f_2(t)+C_5\log^{-\frac{1}{2\lossPower}}(t)+f_4(t)-\left(\tilde{f}_5(t)-C_5\log^{-\frac{1}{2\lossPower}}(t)-f_4(t)\right)^2 \right.\\
	& \left. -\left[\frac{1}{t}\log^{\frac{1}{\lossPower}-1}(t)\frac{1}{\lossPower}\e\left(-\lossPower\left(\wtilde^\top  \xn -1\right)\right)\right]^{\mu_+} \left(1 + \tilde{f}_5(t)-C_5\log^{-\frac{1}{2\lossPower}}(t)-f_4(t) \right. \right.\\
	&\left. \left. +\left(\tilde{f}_5(t)-C_5\log^{-\frac{1}{2\lossPower}}(t)-f_4(t)\right)^2\right) \right] \overset{(3)}{\ge} 0,
	\end{flalign*}
	where in (1) we used  $ e^{-x}\le1-x+x^2, \forall x\ge0$ and $h(t)<0$, in (2) we used $\tilde{f}_4(t) = 1+h(t)+\tilde{\gamma}_n \log^{-1}(t)+\tilde{\lambda}_n\log^{-2}(t) f_3(t)+h_2(t)$, $\tilde{f}_5(t) = h(t)+\tilde{\gamma}_n\log^{-1}(t)-f_2(t)$ and in (3) we used the fact that $\log^{-\frac{1}{2\lossPower}}(t)$ decreases slower than the other terms.
	Second, if $\xkT \rvec(t) < 0$ , then using eq.~\eqref{-l' lower bound}, term $k$ in equation \ref{3rd term bound} can be upper bounded  by:
	\begin{align} \label{3rd term upper bound, xr<0}
		&\eta \left[ \lossPower^{-1} t^{-1}\log^{\frac{1}{\lossPower}-1}(t)\tilde{f}_4(t) \e(-\lossPower\xnT \wtilde) -(1-\e(-\mu_- (\wvec(t)^\top \xn)^\lossPower))\e(-(\wvec(t)^\top \xn)^\lossPower) \right]|\xkT\rvec(t)|\nonumber\\
		& \le \eta \lossPower^{-1}t^{-1}\log^{\frac{1}{\lossPower}-1}(t) \e(-\lossPower\xkT \wtilde) \left[\tilde{f}_4(t)-(1-\e(-\mu_- (\wvec(t)^\top \xn)^\lossPower))\right.\nonumber\\
		& \left. \e\left(\tilde{f}_5(t)-\lossPower\left(g(t) + g^{1-\lossPower}(t)\ftilde(t) \right)^{\lossPower-1}\xnT\rvec(t)\right) \right]|\xkT\rvec(t)|,
	\end{align}
	where in the last transition we used:
	\begin{flalign*}
	&(\wvec(t)^\top \xn)^\lossPower&\\
	&\overset{(1)}{=}\left(g(t) + g^{1-\lossPower}(t)\ftilde(t)+ \xnT\rvec(t) \right)^\lossPower&\nonumber\\
	& \overset{(2)}{\le} \left(g(t) + g^{1-\lossPower}(t)\ftilde(t) \right)^\lossPower +\lossPower\left(g(t) + g^{1-\lossPower}(t)\ftilde(t) \right)^{\lossPower-1}\xnT\rvec(t)\\
	& \overset{(3)}{\le} \log(t) + \log(\lossPower \log^{1-\frac{1}{\lossPower}}(t)) + \lossPower\wtilde^\top  \xn-h(t)-\log^{-1}(t)(\wcheck^\top\xn)+f_2(t)+\lossPower\left(g(t) + g^{1-\lossPower}(t)\ftilde(t) \right)^{\lossPower-1}\xnT\rvec(t)\\
	& \le \log(t) + \log(\lossPower \log^{1-\frac{1}{\lossPower}}(t)) + \lossPower\wtilde^\top  \xn-h(t)-\log^{-1}(t)(\wcheck^\top\xn)+f_2(t),
	\end{flalign*}
	where in (1) we used eq.~\eqref{define r(t), beta<1} ($\rvec(t)$ definition), in (2) we used Bernoulli's inequality:
	\begin{equation} \label{Bernoulli's inequality,r<1, x>=-1}
		\forall 0< r < 1,x \ge -1: (1+x)^r \le 1+rx
	\end{equation}
	 and the fact that from Lemma \ref{Lemma: w(t)->infty} $\lim_{t\to\infty} \wvec(t)^\top \xn=\infty$ and therefore for sufficiently large t
	\begin{equation}
	\frac{\xnT\rvec(t)}{g(t) + g^{1-\lossPower}(t)\tilde{f}(t)} \ge -1
	\end{equation}
	In (3) we used eq.~\eqref{binomial expansion for ()^beta for SV}.\\
	We further divide into cases. \\
	1. If $ |\xkT \rvec(t)|\le C_4 \log^{-1+\frac{1}{2\lossPower}}(t) $:\\
	We can lower bound $(\wvec(t)^\top \xn)^\lossPower$ as follows
	\begin{flalign*}
	&(\wvec(t)^\top \xn)^\lossPower&\\
	&\overset{(1)}{=}\left(g(t) + g^{1-\lossPower}(t)\ftilde(t)+ \xnT\rvec(t) \right)^\lossPower&\nonumber\\
	& \ge \left(g(t) + g^{1-\lossPower}(t)\ftilde(t) -C_4 \log^{-1+\frac{1}{2\lossPower}}(t) \right)^\lossPower\\
	& = \left(g(t) + g^{1-\lossPower}(t)\tilde{f}(t)\right)^\lossPower
	\left(1 + \frac{ -C_4 \log^{-1+\frac{1}{2\lossPower}}(t) }{g(t) + g^{1-\lossPower}(t)\tilde{f}(t)}\right)^\lossPower\nonumber\\
	& = \left(g(t) + g^{1-\lossPower}(t)\tilde{f}(t)\right)^\lossPower - \lossPower \left(g(t) + g^{1-\lossPower}(t)\tilde{f}(t)\right)^{\lossPower-1}C_4 \log^{-1+\frac{1}{2\lossPower}}(t)+f_5(t)\nonumber\\
	& \ge \log(t) + \log(\lossPower \log^{1-\frac{1}{\lossPower}}(t)) + \lossPower\wtilde^\top  \xn-h(t)-\log^{-1}(t)(\wcheck^\top\xn)+f_2(t) - C_6\log^{-\frac{1}{2\lossPower}}(t)+f_5(t) \nonumber\\
	& \ge \log(t) + \log(\lossPower \log^{1-\frac{1}{\lossPower}}(t)) + \lossPower \wtilde^\top  \xn -1, \nonumber
	\end{flalign*}
	where $f_5(t) = o(\log^{-\frac{1}{2\lossPower}}(t))$. In the last transition we used Taylor's theorem and eq.~\eqref{binomial expansion for ()^beta for SV}.\\
	Using this bound and the fact that $e^x>1+x$ we can find $C_7$ such that eq.~\eqref{3rd term upper bound, xr<0} can be upper bounded by
	\begin{flalign}
	& \eta C_7 \frac{1}{\lossPower}t^{-1}\log^{\frac{1}{\lossPower}-2}(t)\e(-\lossPower\xnT \wtilde)
	\end{flalign}
	2. If $ |\xkT \rvec(t)|> C_4 \log^{-1+\frac{1}{2\lossPower}}(t) $ then we will show that $\exists t'_-$ such that eq.~\eqref{3rd term upper bound, xr<0} is strictly negative $\forall t>t'_-$.\\
	Let $M>1$ be some arbitrary constant. Then, since $\exp\left(-\mu_-\left(\wvec(t)^\top  \xk \right)^\lossPower\right)\to0$ from Lemma \ref{Lemma: w(t)->infty}, $\exists t_M>\bar{t}$ such that $\forall t>t_M$ and if
	\begin{equation}
		\e\left(-\lossPower\left(g(t) + g^{1-\lossPower}(t)\ftilde(t) \right)^{\lossPower-1}\xnT\rvec(t)\right)\ge M>1
	\end{equation} 
	then
	\begin{flalign*}
		&\left(1-\e(-\mu_- (\wvec(t)^\top \xn)^\lossPower)\right)\e\left(\tilde{f}_5(t)\right)\e\left(-\lossPower\left(g(t) + g^{1-\lossPower}(t)\ftilde(t) \right)^{\lossPower-1}\xnT\rvec(t)\right)\ge M'>1
	\end{flalign*}
	Furthermore, if $\exists t>t_M$ such that $$\e\left(-\lossPower\left(g(t) + g^{1-\lossPower}(t)\ftilde(t) \right)^{\lossPower-1}\xnT\rvec(t)\right)< M$$
	then $\exists M'',t_6$ such that $\forall t>t_6$: 
	$ \left\vert \left(g(t) + g^{1-\lossPower}(t)\ftilde(t) \right)^{\lossPower-1}\xkT \rvec(t) \right \vert \le M''$.
	We can use this to show that
	\begin{flalign*}
	&\left(1-\e(-\mu_- (\wvec(t)^\top \xn)^\lossPower)\right)\e\left(\tilde{f}_5(t)-\lossPower\left(g(t) + g^{1-\lossPower}(t)\ftilde(t) \right)^{\lossPower-1}\xnT\rvec(t)\right)&\\
	&\ge \left(1-\left[\frac{1}{t}\log^{\frac{1}{\lossPower}-1}(t)\frac{1}{\lossPower}\e\left(-\lossPower\left(\wtilde^\top  \xn -1-M''\right)\right)\right]^{\mu_-}\right)\e\left(\tilde{f}_5(t)+C_6 \log^{-\frac{1}{2\lossPower}}(t)\right)\\
	&\overset{(1)}{\ge} \left(1-\left[\frac{1}{t}\log^{\frac{1}{\lossPower}-1}(t)\frac{1}{\lossPower}\e\left(-\lossPower\left(\wtilde^\top  \xn -1-M''\right)\right)\right]^{\mu_-}\right)\left(1+\tilde{f}_5(t)+C_6 \log^{-\frac{1}{2\lossPower}}(t)\right),
	\end{flalign*}
	where in (1) we used $e^x\ge1+x$.\\
	Using the last equation we can show that eq.~\eqref{3rd term upper bound, xr<0} is negative since:
	\begin{flalign}
		&\left[\left(1+h(t)+\tilde{\gamma}_n \log^{-1}(t)+\tilde{\lambda}_n\log^{-2}(t) f_3(t)+h_2(t)\right)-(1-\e(-\mu_- (\wvec(t)^\top \xn)^\lossPower))\right.&\nonumber\\
		& \left. \e\left(h(t)+\tilde{\gamma}_n\log^{-1}(t)-f_2(t)-\lossPower\left(g(t) + g^{1-\lossPower}(t)\ftilde(t) \right)^{\lossPower-1}\xnT\rvec(t)\right) \right]\nonumber\\
		& \le \left(\tilde{\lambda}_n\log^{-2}(t) f_3(t)+h_2(t)+f_2(t)-C_6 \log^{-\frac{1}{2\lossPower}}(t)\right.\nonumber\\
		&\left.+ \left[\frac{1}{t}\log^{\frac{1}{\lossPower}-1}(t)\frac{1}{\lossPower}\e\left(-\lossPower\left(\wtilde^\top  \xn -\tilde{M}\right)\right)\right]^{\mu_-} \left(1+h(t)+\tilde{\gamma}_n\log^{-1}(t)-f_2(t)+C_6 \log^{-\frac{1}{2\lossPower}}(t)\right)
		\right)\nonumber\\
		& \le 0 
	\end{flalign}
	In the last transition we used the fact that $\log^{-\frac{1}{2\lossPower}}(t)$ decreases slower than the other terms.\\
    
	To conclude, we choose $t_{0}=\max\left[t_2,t_2',t_5,t_{-}'\right]$. 
	We find that $\forall t>t_{0}$ , each term in eq.~\eqref{3rd term bound}
	can be upper bounded by either zero, or terms proportional to $t^{-1}\log^{-\frac{1}{\lossPower}-2}$. Combining this together with eqs. \ref{eq: first term bound, beta<1}, \ref{eq: 2nd term bound, xr bounded, beta<1}, \ref{eq: 2nd term bound, xr not bounded, beta<1} into eq.~\eqref{(r(t+1)-r(t))r(t) bound}
	we obtain (for some positive constants $C_{8}$, $C_{9}$, $C_{10}$) 
	\[
	\left(\rvec(t+1)-\rvec(t)\right)^\top\rvec(t)\le C_8 t^{-2}\log^{\frac{1}{\lossPower}-1}(t) + C_9  t^{-\theta^\lossPower} (\log(t))^{(\frac{1}{\lossPower}-1)\theta^\lossPower} + C_10 t^{-1}\log^{\frac{1}{\lossPower}-2}(t)
	\]
	Therefore, $\exists t_{1}>t_{0}$ and $C_{1}$ such that eq.~\eqref{eq: (r(t+1)-r(t)r(t) bound for beta>1}
	holds.

\end{document}